%% file: main.tex
\documentclass[11pt, a4paper, logo, copyright]{pluralisresearchiclr}
\pdftrailerid{redacted}

\makeatletter
\renewcommand\bibentry[1]{\nocite{#1}{\frenchspacing\@nameuse{BR@r@#1\@extra@b@citeb}}}
\makeatother

\usepackage{kantlipsum, lipsum}
\usepackage{dsfont}
\usepackage{algorithm}
\usepackage{algorithmic}
\usepackage{adjustbox}
\usepackage{multirow}
\usepackage{nicefrac}
\usepackage{float}
\usepackage{microtype}
\usepackage{graphicx}
\usepackage{subcaption}
\usepackage{booktabs} 
\usepackage{xcolor}
\usepackage{enumitem}

\definecolor{ForestGreen}{RGB}{34,139,34}

\usepackage{hyperref}

\newcommand{\INLINECOMMENT}[1]{\hfill$\triangleright$ #1}

\usepackage{enumitem}

\newcolumntype{R}[2]{%
    >{\adjustbox{angle=#1,lap=\width-(#2)}\bgroup}%
    l%
    <{\egroup}%
}

\usepackage{amsmath}
\usepackage{amssymb}
\usepackage{mathtools}
\usepackage{amsthm}
\usepackage{thm-restate}
\usepackage{multirow}

\DeclareMathOperator{\Ortho}{Ortho}
\DeclareMathOperator*{\argmax}{arg\,max}
\DeclareMathOperator*{\argmin}{arg\,min}
\DeclareMathOperator{\tr}{tr}
\newcommand{\diag}{\operatorname{diag}}

\newcommand{\R}{\mathbb{R}}
\newcommand{\E}{\mathbb{E}}
\newcommand{\F}{\mathrm{F}}

\usepackage{rotating}
\usepackage{textcomp}
\usepackage{multibib}
\renewcommand{\textuparrow}{$\uparrow$}
\renewcommand{\textdownarrow}{$\downarrow$}

\usepackage[capitalize,noabbrev]{cleveref}
\usepackage{etoolbox}

\crefname{appendix}{appendix}{appendices}
\Crefname{appendix}{Appendix}{Appendices}

\pretocmd{\appendix}{%
  \crefalias{section}{appendix}%
  \crefalias{subsection}{appendix}%
  \crefalias{subsubsection}{appendix}%
}{}{}

\crefname{lemma}{lemma}{lemmas}
\crefname{corollary}{corollary}{corollaries}
\crefname{proposition}{proposition}{propositions}
\crefname{definition}{definition}{definitions}
\crefname{assumption}{assumption}{assumptions}

\Crefname{lemma}{Lemma}{Lemmas}
\Crefname{corollary}{Corollary}{Corollaries}
\Crefname{proposition}{Proposition}{Propositions}
\Crefname{definition}{Definition}{Definitions}
\Crefname{assumption}{Assumption}{Assumptions}

\usepackage[most]{tcolorbox}

\definecolor{magmaPurple}{HTML}{3B0F70}
\definecolor{magmaMagenta}{HTML}{B5367A}
\definecolor{magmaOrange}{HTML}{F1605D}
\definecolor{magmaGold}{HTML}{FECF92}

\newtcolorbox{blockquote}{
  enhanced,
  breakable,
  colback=magmaOrange!12,      
  leftrule=1pt,
  rightrule=0pt,
  arc=2pt,
  boxrule=0pt,
  left=1pt, right=1pt, top=1pt, bottom=1pt,
}

\newtcolorbox{assumptionblockquote}{
  enhanced,
  breakable,
  colback=magmaMagenta!12,      
  leftrule=1pt,
  rightrule=0pt,
  arc=2pt,
  boxrule=0pt,
  left=3pt, right=3pt, top=3pt, bottom=3pt,
}

\newtcolorbox{theoremblockquote}{
  enhanced,
  breakable,
  colback=magmaPurple!12,      
  leftrule=1pt,
  rightrule=0pt,
  arc=2pt,
  boxrule=0pt,
  left=3pt, right=3pt, top=3pt, bottom=3pt,
}

\newtcolorbox{lemmablockquote}{
  enhanced,
  breakable,
  colback=magmaGold!12,      
  leftrule=1pt,
  rightrule=0pt,
  arc=2pt,
  boxrule=0pt,
  left=3pt, right=3pt, top=3pt, bottom=3pt,
}

\theoremstyle{plain}
\newtheorem{theorem}{Theorem}[section]

\newtheorem{lemma}[theorem]{Lemma}

\theoremstyle{definition}
\newtheorem{definition}[theorem]{Definition}

\theoremstyle{remark}

\usepackage{textcomp}
\renewcommand{\textuparrow}{$\uparrow$}
\renewcommand{\textdownarrow}{$\downarrow$}

\usepackage[textsize=tiny]{todonotes}

\usepackage{pifont}

\usepackage{subcaption}
\usepackage{url}

\input{figures.tex}
\input{tables.tex}

\title{NuMuon: Nuclear-Norm-Constrained Muon for Compressible LLM Training}

\correspondingauthor{hadi@pluralis.ai}


\reportnumber{} 

\renewcommand{\today}{} 

\author{Hadi Mohaghegh Dolatabadi, Thalaiyasingam Ajanthan, Sameera Ramasinghe, Chamin P Hewa Koneputugodage, Shamane Siriwardhana, Violetta Shevchenko, Karol Pajak, James Snewin, Gil Avraham, and Alexander Long \\
Pluralis Research
}

\begin{abstract}
The rapid progress of large language models (LLMs) is increasingly constrained by memory and deployment costs, motivating compression methods for practical deployment. Many state-of-the-art compression pipelines leverage the low-rank structure of trained weight matrices, a phenomenon often associated with the properties of popular optimizers such as Adam. In this context, Muon is a recently proposed optimizer that improves LLM pretraining via full-rank update steps, but its induced weight-space structure has not been characterized yet. In this work, we report a surprising empirical finding: despite imposing full-rank updates, Muon-trained models exhibit pronounced low-rank structure in their weight matrices and are readily compressible under standard pipelines. Motivated by this insight, we propose \emph{NuMuon}, which augments Muon with a nuclear-norm constraint on the update direction, further constraining the learned weights toward low-rank structure. Across billion-parameter-scale models, we show that NuMuon increases weight compressibility and improves post-compression model quality under state-of-the-art LLM compression pipelines while retaining Muon's favorable convergence behavior.
\end{abstract}

\begin{document}
\maketitle

\section{Introduction}
\label{sec:introduction}
Large language models~(LLMs)~\citep{radford2018improving, openai2023gpt4, yang2024qwen2, guo2025deepseek} have driven rapid progress in natural language processing and enabled a wide range of applications, including coding~\citep{bairi2024codeplan}, mathematics~\citep{romera2024math}, and agentic systems~\citep{wang2024agents}. Much of this progress has been fueled by scaling model and data size~\citep{kaplan2020chinchilla, hoffmann2022compute}, which often leads to qualitative shifts in capability as model size grows~\citep{wei2022emergent, schaeffer2023mirage}. However, deploying LLMs with billions of parameters incurs substantial memory, storage, and accelerator costs~\citep{zhuo2024survey}. This has motivated a large body of work on LLM compression~\citep{zhu2024compression} to reduce deployment-time footprint. Such compression methods often exploit structure in the weight matrices (e.g., low-rank structure~\citep{wang2025svdllm} or sparsity~\citep{frantar2023sparse}). Consequently, training choices directly impact how well a model can be compressed for deployment.

Recently, Muon~\citep{jordan2024muon,bernstein2025modular} has been proposed as an alternative optimizer for pretraining LLMs. Unlike elementwise adaptive optimizers such as Adam~\citep{kingma2015adam} and AdamW~\citep{loshchilov2019adamw}, Muon orthogonalizes the matrix-valued momentum update (e.g., via Newton–Schulz), effectively adapting updates to spectral geometry. Empirically, Muon has shown strong performance in large-scale language model training~\citep{moonlight2025,kimik2025}. In contrast to common optimizers such as AdamW that are known to exhibit an implicit low-rank bias~\citep{zhao2022implicit,huh2023simplicity}, Muon's training dynamics and the induced structure of the learned weight space remain underexplored. Understanding such implicit biases are crucial for the downstream deployability of Muon-trained models through existing LLM compression pipelines.

In this paper, we shed light on these dynamics and report a surprising phenomenon: despite using full-rank, orthogonalized update directions and imposing no explicit rank control, \emph{Muon-trained models already exhibit pronounced low-rank structure in their weight matrices} (\Cref{fig:stable_rank_comparison_qwen3}). As a result, we empirically observe that Muon-trained models are readily compressible under common low-rank compression pipelines. However, Muon's emergent low-rank structure is not sufficiently robust against aggressive compression, deteriorating performance rapidly as the compression rate increases. Motivated by this limitation, we seek a principled mechanism to strengthen this low-rank structure during training to improve high-rate compressibility while preserving Muon's favorable optimization behavior.

\figAdamwVsMuonStableRank

To this end, we view Muon's orthogonalization step through a projection-free lens where it implements a linear minimization oracle~(LMO)~\citep{pethick2025training,riabinin2025gluon} and returns the steepest descent direction by minimizing a linearized objective over a spectral-norm–bounded set. Building on this interpretation, we propose \emph{NuMuon}: a variant of Muon that augments its spectral-norm LMO with an additional nuclear-norm budget on the update direction, which we use as a convex proxy to control update rank. Using an LMO formulation, we show that the NuMuon step reduces to a linear program over singular values with a closed-form solution using the top-$k$ singular vectors. NuMuon yields a rank-controlled alternative to Muon's full orthogonalization and directly shapes the singular-value profile of the update. By constraining the update direction, we demonstrate that NuMuon consistently achieves lower stable rank in the learned weights than Muon.

We test NuMuon on models in the 0.6–1.8B parameter range, and show that it converges similarly to Muon while producing weight matrices with more concentrated spectra (lower stable rank), which in turn improves downstream compressibility. Under state-of-the-art~(SoTA) LLM compression pipelines, NuMuon-trained models achieve up to 55.9\% better compression–quality tradeoffs (i.e., lower perplexity at a fixed compression setting), making NuMuon attractive for large-scale or cost-sensitive deployment scenarios with strict memory requirements.

In summary, we make the following contributions:
\begin{itemize}
    \item We show that Muon-trained models exhibit a surprisingly pronounced low-rank structure in their weight matrices despite the explicit full-rank updates in the absence of rank constraints.
    \item Inspired by this observation, we propose NuMuon which augments Muon with a per-layer nuclear-norm budget on the LMO update direction. We show that the resulting LMO reduces to a linear program over singular values and admits a closed-form top-$k$ singular-vector solution, yielding rank-controlled updates.
    \item We provide convergence guarantees for NuMuon under non-convex assumptions and describe practical techniques for efficient top-$k$ computation and rank scheduling strategies for large-scale training using our approach.
    \item We empirically show that NuMuon achieves comparable performance to Muon while yielding substantially improved compressibility under SoTA LLM compression pipelines, improving deployability.
\end{itemize}

Our results clarify aspects of Muon's implicit bias and provide a practical optimizer variant for settings where controlling low-rank structure is important.

\section{Background}
\label{sec:background}
In this section, we introduce notation and briefly review background related to our work.

\paragraph{Muon Optimizer.}
Muon~\citep{jordan2024muon} is an optimizer that produces \textit{geometry-aware} updates for matrix-valued parameters by orthogonalizing momentum updates.
In contrast to elementwise adaptive methods such as AdamW~\citep{loshchilov2019adamw}, which rescale coordinates independently, Muon post-processes the momentum update so that its singular directions are treated uniformly.

Concretely, consider a single matrix parameter ${\boldsymbol{W} \in \mathbb{R}^{d_{\rm out} \times d_{\rm in}}}$ (e.g., a linear layer weight) from a neural network.\footnote{For clarity, we present the update for one matrix block; in practice, Muon is applied block-wise/layer-wise.} Through training, we aim to minimize an objective with respect to the weights $\min_{\boldsymbol{W}} \; f(\boldsymbol{W})$. Let $\boldsymbol{G}_t = \nabla_{\boldsymbol{W}} f(\boldsymbol{W}_t)$ be the gradient, and let $\boldsymbol{M}_t$ be a momentum buffer: 
\begin{equation}\label{eq:momentum}
    \boldsymbol{M}_t = \beta \boldsymbol{M}_{t-1} + (1 - \beta)\,\boldsymbol{G}_t,
\end{equation}
where $\beta\in[0,1)$.
Assuming a truncated SVD ${\boldsymbol{M}_t=\boldsymbol{U}_t \boldsymbol{S}_t \boldsymbol{V}_t^\top}$, Muon replaces $\boldsymbol{M}_t$ by its \textit{orthogonal (polar) factor}:
\begin{equation}\label{eq:ortho_def}
    \Ortho(\boldsymbol{M}_t) := \boldsymbol{U}_t \boldsymbol{V}_t^\top,
\end{equation}
which has singular values equal to $1$. The parameter update is then written as:
\begin{equation}\label{eq:muon_update_rule}
    \boldsymbol{W}_{t+1} = \boldsymbol{W}_t - \gamma \, \Ortho(\boldsymbol{M}_t),
\end{equation}
where $\gamma$ is the learning rate.
Intuitively, dropping $\boldsymbol{S}_t$ removes anisotropic scaling and yields an update that treats singular directions uniformly~\citep{jordan2024muon,bernstein2025modular}. In practice, the $\Ortho(\cdot)$ operation is typically approximated efficiently, e.g., via a few rounds of Newton--Schulz approximation (see~\Cref{alg:newton_schulz}).

\paragraph{Linear Minimization Oracles.}
Recent work by \citet{pethick2025training,riabinin2025gluon} has highlighted that Muon can be interpreted within a broader family of methods based on \textit{linear minimization oracles} (LMOs) over norm balls. In particular, consider the norm-constrained problem:
\begin{equation}\label{eq:nn_constrained}
    \min_{\boldsymbol{W}} \; f(\boldsymbol{W})
    \quad \text{s.t.}\quad
    \boldsymbol{W}\in\mathcal{W}
    \;:=\;
    \bigl\{\boldsymbol{W} : \|\boldsymbol{W}\|\le \rho\bigr\},
\end{equation}
for some radius $\rho>0$ and a chosen norm $\|\cdot\|$. As shown in Scion~\citep{pethick2025training}, LMO-based training rules can be written in an \textit{unconstrained} (additive) form:
\begin{equation}\label{eq:lmo_unconstrained}
    \boldsymbol{W}_{t+1} = \boldsymbol{W}_t + \gamma_t \, \mathrm{lmo}_{\mathcal{W}}(\boldsymbol{M}_t),
\end{equation}
or in a \textit{constrained} Frank--Wolfe~(FW)~\citep{frank1956algorithm} / conditional-gradient~(CG)~\citep{jaggi2013frank} form:
\begin{equation}\label{eq:lmo_fw}
    \boldsymbol{W}_{t+1} = (1-\gamma_t)\boldsymbol{W}_{t} + \gamma_t\, \mathrm{lmo}_{\mathcal{W}}(\boldsymbol{M}_t),
\end{equation}
where $\gamma_t\in(0,1]$ is a stepsize. Here, an LMO for $\mathcal{W}$ is the mapping:
\begin{equation}\label{eq:lmo_def}
    \mathrm{lmo}_{\mathcal{W}}(\boldsymbol{M})
    \in
    \argmin_{\boldsymbol{\Delta W}\in\mathcal{W}}
    \langle \boldsymbol{M}, \boldsymbol{\Delta W}\rangle.
\end{equation}
When $\mathcal{W}$ is a \textit{spectral-norm ball}, i.e., ${\mathcal{W}=\{\boldsymbol{\Delta W}:\|\boldsymbol{\Delta W}\|_2\le \rho\}}$, the LMO has a closed form. Specifically, if $\boldsymbol{M}=\boldsymbol{U}\boldsymbol{S}\boldsymbol{V}^\top$ is a thin SVD, then:
\begin{equation}\label{eq:lmo_spectral}
    \mathrm{lmo}_{\mathcal{W}}(\boldsymbol{M})
    \in
    \argmin_{\|\boldsymbol{\Delta W}\|_2\le \rho}\langle \boldsymbol{M}, \boldsymbol{\Delta W}\rangle
    =
    -\rho\, \boldsymbol{U}\boldsymbol{V}^\top.
\end{equation}
Thus, up to stepsize conventions (absorbing $\rho$ into the learning rate), the additive LMO update in \Cref{eq:lmo_unconstrained} recovers Muon's orthogonalized step in \Cref{eq:ortho_def}.

\paragraph{LLM Compression via Low-Rank Structure.}
LLM compression seeks to lower storage and deployment-time memory costs and speed without sacrificing accuracy~\citep{zhuo2024survey}. To this end, prior work typically exploits low-rank structure~\citep{wang2025svdllm}, reduced-precision representations via quantization~\citep{dettmers2022int8,huang2024billm,yuan2024pbllm}, and sparsification or pruning~\citep{frantar2023sparse,kim2024squeezellm}.

A prominent line of work leverages approximate low-rank structure in weight matrices by factorizing a full-rank weight matrix into low-rank factors~\citep{hsu2022fswd,yuan2023asvd,wang2025svdllm,wang2025dobisvd}. In particular, for a weight matrix $\boldsymbol{W}\in\mathbb{R}^{d_{\rm out}\times d_{\rm in}}$, a rank-$k$ factorization in such methods takes the form $\boldsymbol{W} \approx \boldsymbol{W}_u \boldsymbol{W}_v^\top$, where $\boldsymbol{W}_u \in \mathbb{R}^{d_{\rm out} \times k}$ and $\boldsymbol{W}_v \in \mathbb{R}^{d_{\rm in} \times k}$ are low-rank factors. The factors can be obtained in several ways, most commonly via truncated SVD or SVD variants that account for activation statistics~\citep{yuan2023asvd} and layer sensitivity~\citep{wang2025svdllm}. These methods are often paired with lightweight post-processing and/or extensive fine-tuning to recover accuracy~\citep{wang2025dobisvd,wang2025svdllm}. Such approaches are typically evaluated by measuring perplexity and downstream task performance at a fixed compression budget, aiming to minimize the gap to the original full-rank model~\citep{wang2025svdllm}.

This perspective also clarifies why the optimizer matters for compressibility: training influences the singular-value profile and effective rank of learned weights~\citep{le2022lowrank,timor2023implicit}, which directly affects how well low-rank factorization-based compression performs. In the next section, we use this lens to study Muon's learned weight structure and to motivate rank-controlled updates that better align training with downstream low-rank compression.

\section{Our Method}
\label{sec:method}
In this section, we present NuMuon. We begin by probing how Muon shapes the rank structure of transformer weight matrices during pretraining. Although Muon applies full-rank, orthogonalized update directions, we surprisingly find that the resulting weight matrices nonetheless exhibit pronounced low-rank structure throughout training. This observation suggests that low-rank structure can emerge naturally under Muon's training dynamics. As such, Muon-trained models are readily compressible with SVD-based LLM compression methods. However, as we see in our experiments, the performance of these models deteriorate rapidly as the compression rate increases. To address this limitation, we propose to \textit{explicitly control the rank of Muon's update directions} by augmenting its spectral-norm LMO with an additional nuclear-norm budget on the \textit{LMO step}. We present theoretical guarantees for NuMuon's convergence and examine the validity of our assumptions. Finally, we describe practical considerations for using NuMuon.

\subsection{Motivation}
\label{sec:method:motivation}
As discussed in~\Cref{sec:background}, many LLM compression pipelines exploit structure in weight matrices, and low-rank factorization is a particularly common mechanism for reducing memory footprint while preserving model quality~\citep{yuan2023asvd,wang2025svdllm}. A natural question, then, is how the rank structure of transformer weights evolves during training with Muon, and whether Muon produces weights that are amenable to low-rank compression.

To study this, we track the \textit{stable rank} of transformer weight matrices during pretraining. Stable rank is a robust proxy for effective dimensionality that is less sensitive to small singular values than the exact rank, and is defined as $\mathrm{sr}(\boldsymbol{W}) = {\|\boldsymbol{W}\|_F^2}/{\|\boldsymbol{W}\|_2^2}$. To compare across layers and matrix shapes, we normalize by the maximum attainable rank (i.e., $\min(d_{\rm out}, d_{\rm in})$), yielding a \textit{normalized stable rank} between $1/\min(d_{\rm in}, d_{\rm out})$ and 1.

We measure this quantity throughout training for all the weight matrices of a Qwen-3-0.6B decoder-only model with 28 layers~\citep{qwen3} trained for 16.8B tokens, equivalent to 1.4$\times$ the Chinchilla compute-optimal token budget~\citep{kaplan2020chinchilla}. \Cref{fig:stable_rank_comparison_qwen3} summarizes the mean normalized stable rank across layers over training for the feedforward projection weights, with shaded regions denoting layer-to-layer variability.

As seen, \textit{Muon-trained weights remain strongly low-rank throughout training} despite Muon using full-rank orthogonalized update directions. In other words, the learned weight matrices occupy a relatively low-dimensional subspace as training progresses, and this is reflected by the consistently low normalized stable rank. Compared to AdamW, Muon tends to maintain a somewhat higher (but still clearly sub-maximal) stable rank in these feed-forward projections. This indicates that Muon does not eliminate low-rank structure; instead, low-rank structure appears to be an emergent property of training that persists under both optimizers~\citep{feng2022rank,huh2023simplicity}. Also note that the early training dynamics differ: under Muon, stable rank stays elevated during the first few hundred iterations before gradually settling, whereas under AdamW it drops sharply at the start of training and then stabilizes.

We observe similar qualitative behavior across other weight matrices (e.g., attention and output projections); additional results are provided in~\Cref{fig:stable_rank_comparison_qwen3_all}. To validate that this phenomenon extends beyond our experiments, \Cref{fig:stable_rank_large_scale} shows the normalized stable rank for Moonlight-16B-A3B~\citep{moonlight2025} and Kimi-K2~\citep{kimik2025}, two large-scale models that pioneered the use of Muon for pretraining. Across both models, the stable rank of various weight matrices remains a small fraction of the maximum attainable rank, corroborating our finding that Muon-trained models exhibit pronounced low-rank structure. Taken together, these measurements indicate that Muon-trained weights admit accurate low-rank approximations, and thus should be amenable to standard low-rank compression pipelines. We confirm this empirically using SoTA LLM compression methods in~\Cref{fig:efficiency,tab:llama_comp_sum}. However, this compressibility is brittle: while moderate compression rates preserve performance, pushing to higher compression ratios leads to a rapid degradation relative to the base model. This motivates shaping the training dynamics so that the learned weights remain reliably compressible at high rates, while retaining Muon's favorable optimization behavior. 

\figEffLlama

\subsection{NuMuon: Nuclear-norm-constrained Muon}
\label{sec:method:numuon}

Motivated by our observations in the previous section, we propose to \textit{control the rank of Muon's update directions} so that training dynamics better align with downstream low-rank compression. Concretely, we adopt a conditional gradient view in which Muon's orthogonalization step can be interpreted as implementing an LMO over a spectral-norm–bounded set of update directions. We then modify this update selection by adding a nuclear-norm budget. The nuclear norm is defined as the sum of singular values ${\|\boldsymbol{A}\|_*=\sum_i \sigma_i(\boldsymbol{A})}$ and is a standard convex proxy to encourage low-rank structure~\citep{recht2010guaranteed}.

In particular, let $\boldsymbol{M}_t$ denote Muon's matrix-valued momentum. NuMuon replaces Muon's spectral-norm LMO with $\mathcal{W}_*$ over the set of admissible update directions:
\begin{equation}\label{eq:numuon_domain}
    \mathcal{W}_* \;:=\; \bigl\{\, \boldsymbol{\Delta W} \ \big| \ \|\boldsymbol{\Delta W}\|_2 \le \rho,\ \|\boldsymbol{\Delta W}\|_* \le \tau \,\bigr\}.
\end{equation}
Note that $\mathcal{W}_*$ is the intersection of a spectral-norm ball and a nuclear-norm ball, hence closed and convex. We define
\begin{equation}\label{eq:lmo_def_nn}
    \boldsymbol{\Delta W}_t \in \mathrm{lmo}_{\mathcal{W}_*}(\boldsymbol{M}_t)
    \;:=\;
    \argmin_{\boldsymbol{\Delta W} \in \mathcal{W}_*} \langle \boldsymbol{M}_t, \boldsymbol{\Delta W} \rangle,
\end{equation}
and use $\boldsymbol{\Delta W}_t$ in either an additive or a convex-combination (CG/FW-style) update following~\citet{pethick2025training}. Note that by this definition, NuMuon constrains the \textit{update direction} and places no constraint on the weight matrices. However, under CG/FW-style updates, the iterates are convex combinations of low-rank directions from $\mathcal{W}_*$ and are thus progressively attracted toward the feasible set (see \Cref{lem:feasibility_attraction}), which provides intuition for why low-rank structure reliably emerges in practice~(see~\Cref{fig:stable_rank_comparison_qwen3_all}).

To apply NuMuon in practice, we need to reliably and efficiently solve the LMO outlined in~\Cref{eq:lmo_def_nn}. A key appeal of Muon is that its spectral norm LMO has a closed-form solution. In the following propositions, we show this property is also present in NuMuon; we first demonstrate that NuMuon's LMO reduces to a linear program~(LP) and then find a closed-form solution for this LP.

\begin{blockquote}
\begin{restatable}{proposition}{numuonlp}
\label{prop:numuon_lp}
Let $\boldsymbol{M}\in\mathbb{R}^{d_{\rm out}\times d_{\rm in}}$ with thin SVD ${\boldsymbol{M}=\boldsymbol{U}\,\mathrm{diag}(\boldsymbol{\sigma})\,\boldsymbol{V}^\top}$, where $\sigma_1\ge \sigma_2\ge \cdots \ge 0$ and $q=\min(d_{\rm out},d_{\rm in})$. Consider the LMO
\begin{align}\label{eq:numuon_lmo_primal}\nonumber
    \min_{\boldsymbol{\Delta W}}\ \langle \boldsymbol{M},\boldsymbol{\Delta W}\rangle~\text{s.t.}~
    \|\boldsymbol{\Delta W}\|_2\le \rho,~\|\boldsymbol{\Delta W}\|_* \le \tau.
\end{align}
There exists an optimal solution of the form ${\boldsymbol{\Delta W}^\star=-\boldsymbol{U}\,\mathrm{diag}(\boldsymbol{s}^\star)\,\boldsymbol{V}^\top}$ for some $\boldsymbol{s}^\star\in\mathbb{R}^q_{\ge 0}$ where
\begin{equation}\label{eq:numuon_lp}
    \max_{\boldsymbol{s}\in\mathbb{R}^q}\ \sum_{i=1}^q \sigma_i\, s_i
    \quad \text{s.t.}\quad
    \boldsymbol{s}\in\mathcal{P}(\rho,\tau)
\end{equation}
is a linear program over the polytope ${\mathcal{P}(\rho,\tau):=\Bigl\{\boldsymbol{s}:\ 0\le s_i\le \rho\ \ \forall i,\ \ \sum_{i=1}^q s_i\le \tau\Bigr\}}$.
\end{restatable}
\end{blockquote}

\begin{proof}[Proof sketch]
By unitary invariance, we may rotate $\boldsymbol{\Delta W}$ into the singular basis of $\boldsymbol{M}$ without changing $\|\boldsymbol{\Delta W}\|_2$ or $\|\boldsymbol{\Delta W}\|_*$. Von Neumann's trace inequality~\citep{vonNeumann1937} implies the optimum aligns singular vectors with $\boldsymbol{M}$, so the objective depends only on the singular values of $\boldsymbol{\Delta W}$, reducing to the LP in \Cref{eq:numuon_lp}. A full proof is provided in~\Cref{app:proofs}.
\end{proof}

\begin{blockquote}
\begin{restatable}{proposition}{numuoncf}
\label{prop:numuon_closed_form}
Under the assumptions of \Cref{prop:numuon_lp}, the optimal solution of the LP in~\Cref{eq:numuon_lp} is $\boldsymbol{s}^\star$ where, for $i=1,\dots,q$ and $k = \lfloor \tau/\rho\rfloor$, we have
\begin{equation}\label{eq:numuon_s_star}
s_i^\star \;=\;
\begin{cases}
    \rho, & 1 \le i \le \min(k,q),\\[2pt]
    r, & i = k+1 \ \text{and}\ k+1 \le q,\\[2pt]
    0, & k+2 \le i \le q,
\end{cases}
\end{equation}
and $r = \tau-k\rho$ is the residual singular value. Consequently, $\mathrm{rank}(\boldsymbol{\Delta W}^\star)\le \min\big(q,\lceil \tau/\rho\rceil\big)$.
\end{restatable}
\end{blockquote}

\begin{proof}[Proof sketch]
The feasible set in~\Cref{eq:numuon_lp} is a capped simplex polytope. Since $\sigma_1 \ge \sigma_2 \ge \cdots$, a standard greedy argument shows the LP is maximized by allocating as much budget as possible to the largest coefficients first, i.e., $s_1=\rho, s_2=\rho,\ldots$ until the budget $\tau$ is exhausted yielding~\Cref{eq:numuon_s_star}. Proof can be found in~\Cref{app:proofs}.
\end{proof}

Now let us consider the common case where the nuclear-norm budget is an integer multiple of the spectral cap, i.e., $\tau = k\rho$ so that $r=0$ in~\Cref{prop:numuon_closed_form}.\footnote{When $\tau$ is not an exact multiple of $\rho$, the solution additionally allocates a fractional amount $r\in(0,\rho)$ to the $(k\!+\!1)$-st singular direction; see \Cref{prop:numuon_closed_form}. In practice, however, we ignore this term and set the rank $k$ directly.}
Combining \Cref{prop:numuon_lp,prop:numuon_closed_form}, the NuMuon update is:

\begin{equation}\label{eq:numuon_topk_update}
    \boldsymbol{\Delta W}^\star
    = -\rho \sum_{i=1}^{k} \boldsymbol{u}_i \boldsymbol{v}_i^\top := -\rho \boldsymbol{U}_{k} \boldsymbol{V}_k^\top,
\end{equation}
where $\{\boldsymbol{u}_i,\boldsymbol{v}_i\}_{i=1}^{k}$ are the singular vector pairs of $\boldsymbol{M}$ associated with the top-$k$ singular values $\sigma_1\ge\cdots\ge\sigma_k$.
Thus, NuMuon takes a top-$k$ singular-direction update in which all nonzero update singular values are equal to $\rho$. Setting $k=q$ (or equivalently, choosing $\tau \ge q\rho$) recovers Muon's full-rank spectral-norm LMO update ${\boldsymbol{\Delta W}^\star = -\rho\,\boldsymbol{U}\boldsymbol{V}^\top}$ (up to the same stepsize and rescaling conventions as in \Cref{eq:lmo_spectral,eq:rms_scale}). In this sense, NuMuon interpolates between a rank-$1$ nuclear-norm update and Muon's full-rank orthogonalized update by adjusting $k$.

\subsection{Convergence Analysis}
\label{sec:method:analysis}
In this section, we analyze the convergence behavior of NuMuon and show that the bounds provided by~\citet{shen2025convergence} for Muon can be extended to NuMuon. To this end, we provide a generalization of Theorem~4.3 from~\citet{shen2025convergence} to NuMuon. This generalization establishes a stationarity bound for non-convex functions under the nuclear norm. Apart from the standard smoothness and bounded gradient variance assumptions, our generalization requires that the gradient's tail energy outside its top-$k$ components remain bounded. More formally, we make the following assumptions:

\figDeltaQwen

\begin{assumptionblockquote}
\begin{restatable}[Smoothness]{assumption}{smoothness}\label{ass:smoothness}
The function $f:\R^{d_{\rm out}\times d_{\rm in}}\to\R$ is $L$-smooth with respect to the Frobenius norm:
\[
    \|\nabla f(\boldsymbol{W})-\nabla f(\boldsymbol{W}')\|_\F \le L\|\boldsymbol{W}-\boldsymbol{W}'\|_\F.
\]
\end{restatable}

\begin{restatable}[Unbiased gradients with bounded variance]{assumption}{variance}\label{ass:variance}
The stochastic gradient estimator ${\boldsymbol{G}_t=\frac{1}{b}\sum_{i=1}^b \nabla f(\boldsymbol{W}_t;\boldsymbol{\xi}_{t,i})}$ satisfies
\[
    \E[\boldsymbol{G}_t]=\nabla f(\boldsymbol{W}_t),\qquad \E\|\boldsymbol{G}_t-\nabla f(\boldsymbol{W}_t)\|_\F^2 \le \frac{\nu^2}{b},
\]
where $b$ is the batch-size.
\end{restatable}

\begin{restatable}[Bounded Tail Energy]{assumption}{tailcontrol}\label{ass:tail_control}
Along the iterates $\{\boldsymbol{W}_t\}_{t=0}^{T-1}$, the gradient tail energy is bounded. In other words, there exists $\delta_k\ge 0$ such that
$\E\big[{\|\nabla f(\boldsymbol{W}_t)-(\nabla f(\boldsymbol{W}_t))_k\|_*\big] \le \delta_k}$ for all $t$, where $(\nabla f(\boldsymbol{W}_t))_k$ denotes the best rank-$k$ approximation.
\end{restatable}
\end{assumptionblockquote}

To support the validity of this assumption in practice, we measure ${\delta_k^{(\F)}(\boldsymbol{W})\;:=\;\|\nabla f(\boldsymbol{W}_t)-(\nabla f(\boldsymbol{W}_t))_k\|_\F^2}$ for our Qwen3-0.6B model. In particular, we show in the Appendix that $\delta_k^{(\F)}$ is monotonically decreasing in $k$, and hence, it suffices to observe the behavior of $\delta_1^{(\F)}$ (\Cref{lem:residual_monotone}). As seen in~\Cref{fig:delta_1_ffn}, Muon and NuMuon's residual energy is typically small and close to zero, mirroring the validity of~\Cref{ass:tail_control} in practice. Such low-rank gradient structure is consistent with prior empirical findings that neural network gradients concentrate in a small subspace during training~\citep{vogels2019psgd,zhao2024galore}. Under this assumption, we present our convergence guarantee.

\begin{theoremblockquote}
\begin{restatable}[Nonconvex Nuclear-Norm Stationarity]{theorem}{nonconvex}\label{thm:nonconvex_topk_main}
Let ${f:\R^{d_{\rm out}\times d_{\rm in}}\to\R}$ be $L$-smooth with respect to the Frobenius norm. Assume the stochastic gradient estimator $\boldsymbol{G}_t$ is unbiased with variance bounded by $\nu^2/b$, where $b$ is the batch-size. Additionally, assume that the gradient spectra has bound energy outside its top-$k$ components~(\Cref{ass:tail_control}). Let $\{\boldsymbol{W}_t,\boldsymbol{M}_t\}$ be generated by NuMuon with momentum parameter $\beta$ and constant stepsize $\gamma$. Then
\begin{align}\label{eq:thm_topk_main}
    &\frac{1}{T}\sum_{t=0}^{T-1}\E\|\nabla f(\boldsymbol{W}_t)\|_{*} \le
        \frac{\E[f(\boldsymbol{W}_0)-f(\boldsymbol{W}_T)]}{T\gamma}
        +\frac{Lk\gamma}{2}
        +\frac{2\nu\sqrt{k(1-\beta)}}{\sqrt{(1+\beta)b}} +\frac{2\beta\nu\sqrt{k}}{(1-\beta)T\sqrt{b}}
        +\frac{2k\gamma\beta L}{1-\beta} + \delta_k.
\end{align}
\end{restatable}
\end{theoremblockquote}

\begin{proof}[Proof Sketch]
We prove this in two steps. First, following the momentum analysis framework of~\citet{shen2025convergence} with the key modification that the update direction $\boldsymbol{U}_{t,k}\boldsymbol{V}_{t,k}^\top$\footnote{We assume that $\rho$ has been absorbed in the stepsize $\gamma$.} has Frobenius norm $\sqrt{k}$ (rather than $\sqrt{\min(d_{\rm out},d_{\rm in})}$), we propagate a factor of $k$ through the smoothness and momentum-lag terms. This establishes convergence under the Ky Fan $k$-norm~(\Cref{def:ky_fan}). Then, we employ~\Cref{ass:tail_control} to complete the proof. The full proof is provided in \Cref{app:proofs:convergence_analysis}.
\end{proof}

As we can see, setting $k=\min(d_{\rm out},d_{\rm in})$ recovers~\citet{shen2025convergence}'s bound for Muon. Furthermore, we see a trade-off in the upper bound on nuclear-norm stationarity. Choosing a smaller $k$ reduces almost all the terms in \Cref{eq:thm_topk_main} but increases the tail residual $\delta_k$, and vice versa. An interesting theoretical direction for future work is to characterize conditions under which an optimal $k < \min(d_{\rm out},d_{\rm in})$ exists that yields tighter convergence guarantees than full-rank Muon. Please see our discussion on the validity of the tail control condition in~\Cref{app:convergence:tail}.

\subsection{Practical Considerations}
\label{sec:method:practical}
To make NuMuon practical at scale, we address two implementation questions: (i) how to compute the top-$k$ singular vectors efficiently, and (ii) how to choose a per-layer rank.

\paragraph{Top-$k$ SVD via Randomized Block Krylov Method.}
NuMuon requires the leading $k$ singular vector pairs of the matrix driving the LMO for each parameter block. Computing a full SVD is prohibitively expensive at scale, so we approximate the top-$k$ subspace using a randomized block Krylov method proposed by~\citet{musco2015bksvd} (see~\Cref{alg:block_krylov}). Intuitively, the method alternates multiplying by $\boldsymbol{A}$ and $\boldsymbol{A}^\top$ to build a low-dimensional Krylov subspace that concentrates on $\boldsymbol{A}$'s dominant singular directions; we then compute a small SVD in this subspace to recover approximate top-$k$ singular vectors. In practice, a small number of Krylov iterations and a modest block size (we chose 2 and 8, respectively) are sufficient.

\paragraph{Rank Scheduler.}
As we observed in \Cref{fig:stable_rank_comparison_qwen3}, the stable rank of transformer weight matrices evolve over training, and enforcing very low-rank updates too early can be suboptimal. This is due to the fact that the early phase of training with Muon is typically characterized by high stable rank across layers. Motivated by this, we expose NuMuon rank as a fractional schedule rather than a fixed integer. At training step $t$, the scheduler outputs $r(t)\in(0,1]$ and we set $k_t=\left\lceil r(t)\,\min(d_{\rm in}, d_{\rm out})\right\rceil$ for all the layers. This makes rank control shape-agnostic across parameter blocks and enables a gradual transition from higher-rank to lower-rank updates. We consider the fixed, piece-wise, and cosine rank schedulers in our experiments~(see~\Cref{app:algs} for a formal definition). As we show in \Cref{sec:experimental_results}, scheduled rank control can improve the final model performance compared to using a fixed $k$ throughout training.

\section{Related Work}
\label{app:related_work}

Our work sits at the intersection of three lines of research: the emergence of low-rank structure in neural networks, compression methods that exploit this structure, and optimizers that shape weight geometry during training. Below, we provide related work to our approach.

\paragraph{Low-Rank Structure in Neural Networks.}
Learned neural network weights often exhibit significant redundancy and lie near low-dimensional subspaces. \citet{denil2013predicting} provided early evidence that a large fraction of parameters can be predicted from a small subset, and subsequent work exploited low-rank decompositions to compress convolutional networks~\citep{jaderberg2014speeding,denton2014exploiting}. Beyond explicit compression, gradient-based training itself tends to produce low-rank weights: theoretical results for linear models show an implicit nuclear-norm bias under gradient descent~\citep{gunasekar2017implicit}, and this rank-minimizing tendency has been observed and formalized in deeper settings~\citep{timor2023implicit,le2022lowrank,ramasinghe2025ssn}. These findings suggest that low-rank structure is an emergent property of optimization, motivating our investigation of how different optimizers shape this structure.

\paragraph{LLM Compression via Low-Rank Factorization.}
LLM compression encompasses quantization~\citep{dettmers2022int8,huang2024billm,yuan2024pbllm}, pruning~\citep{frankle2019lottery,kim2024squeezellm}, and distillation~\citep{hinton2015distilling,xu2024distillation}, among other techniques~\citep{zhuo2024survey}. Low-rank factorization has emerged as a particularly effective approach, approximating each weight matrix $\boldsymbol{W}$ with a product of smaller factors~\citep{xue2013restructuring,yu2017compressing,hsu2022fswd}. Recent methods improve upon naive SVD truncation by incorporating activation statistics (ASVD;~\citealp{yuan2023asvd}), layer sensitivity and lightweight adaptation (SVD-LLM;~\citealp{wang2025svdllm}), or data-driven optimization objectives (Dobi-SVD;~\citealp{wang2025dobisvd}). The progression from basic truncation to these refined approaches underscores that compression quality depends critically on the singular-value distribution learned during training which is a key motivation for our work.

\paragraph{Optimizers and Weight Geometry.}
Optimizer choice influences the geometry of learned weights and thus downstream compressibility. While elementwise adaptive methods such as Adam and AdamW rescale coordinates independently~\citep{kingma2015adam,loshchilov2019adamw}, recent work interprets modern optimizers through constrained-optimization and projection-free lenses~\citep{pethick2025training,riabinin2025gluon}. Muon~\citep{jordan2024muon,bernstein2025modular} orthogonalizes momentum updates via the polar factor, which can be viewed as an LMO over a spectral-norm ball. Unlike standard training dynamics that exhibit implicit low-rank bias~\citep{zhao2022implicit,huh2023simplicity}, Muon applies full-rank, spectrally-uniform updates. Nevertheless, we show that Muon-trained models still develop pronounced low-rank structure (\Cref{sec:method:motivation}). Building on this observation, NuMuon augments the spectral-norm LMO with a nuclear-norm constraint, a standard convex surrogate for rank~\citep{recht2010guaranteed}, to explicitly control update rank and improve alignment with downstream low-rank compression. Close to our work, \citet{zimmer2022compression} uses an LMO view to encourage compression robustness by constraining the weights directly, and is tested primarily on CNN models; \citet{miao2022sfw} similarly studies Stochastic Frank--Wolfe for pruning-friendly training in the unstructured setting for CNNs. In contrast, NuMuon constrains the \emph{update direction} via a nuclear-norm budget, preserving Muon's spectral dynamics while explicitly controlling update rank for low-rank LLM compression.

\section{Experimental Results}
\label{sec:experimental_results}

\paragraph{Settings.}
We evaluate on three model architectures: Qwen3-0.6B~\citep{qwen3}, Olmo2-1.4B~\citep{olmo2}, and Llama3-1.8B~\citep{dubey2024llama3}, training each past Chinchilla optimality~\citep{kaplan2020chinchilla} using AdamW, Muon, and NuMuon. Following~\citet{wen2025fantastic,semenov2025benchmarking}, we use a cosine learning rate scheduler for AdamW and a WSD~\citep{hu2024minicpm} scheduler for Muon and NuMuon. We train on FineWeb-EDU~\citep{penedo2024fineweb}, which has been shown to yield rapid downstream performance gains~\citep{mayilvahanan2025llms}. For NuMuon, we use a cosine rank scheduler that anneals the relative rank from $k=1$ to $k=0.25$ across all weight matrices. Training details are provided in~\Cref{app:extended_results}.

For downstream compression, we evaluate three SoTA LLM compression methods: ASVD~\citep{yuan2023asvd}, SVD-LLM~\citep{wang2025svdllm}, and Dobi-SVD~\citep{wang2025dobisvd}. For SVD-LLM, we test both the whitening variant and whitening plus LoRA. All methods use their official default settings, with compression rates ranging from 20\% to 80\%. Note that even though these methods rely on a low-rank structure, they still do their post-processing which could involve extensive fine-tuning on a target dataset. Following standard practice, we report validation perplexity on WikiText2 and accuracy on standard benchmarks: ARC-Easy and ARC-Challenge~\citep{bhakthavatsalam2021arc}, HellaSwag~\citep{zellers2019hellaswag}, LAMBADA~\citep{paperno2016lambada}, OpenbookQA~\citep{mihaylov2018openbookqa}, PIQA~\citep{bisk2020piqa}, and Winogrande~\citep{sakaguchi2020winogrande}. For more information on each method, please visit their official implementation.

\tabLlamaConvergence
\figLossCurvesLarge

\paragraph{Convergence.}
First, we benchmark the training behavior of NuMuon compared to AdamW and Muon. As shown in~\Cref{fig:loss_curves}, NuMuon closely tracks Muon throughout most of training. Toward the end of training and once the ranks across weight matrices stabilize under the cosine rank scheduler, NuMuon exhibits a small deviation from Muon; nevertheless, its final training/validation perplexity remains comparable to Muon and improves over AdamW as reported in~\Cref{tab:llama_convergence}. In terms of training resources, each NuMuon step is slightly slower due to the overhead of managing large SVD approximations (via Block Krylov SVD) when ranks are changing dynamically. As we show in the ablations, using alternative rank schedulers reduces this overhead and narrows the runtime gap. Importantly, GPU memory consumption remains close to Muon and lower than AdamW.

We also report the normalized stable rank of the final trained models across all layers and weight matrices for Qwen3-0.6B, Olmo2-1.4B, and Llama3-1.8B in~\Cref{fig:stable_rank_vs_layer_comparison_qwen3_all,fig:stable_rank_vs_layer_comparison_olmo2_all,fig:stable_rank_vs_layer_comparison_llama3_all}. These figures show that NuMuon induces a strictly lower stable rank across all weight matrices and layers in our transformer blocks. In the next set of experiments, we demonstrate that this reduced stable rank translates into substantially improved robustness to SVD-based compression, allowing NuMuon-trained models to retain performance at higher compression rates than Muon.

\paragraph{Compressibility.}
Next, we study SVD-based LLM compression and show how NuMuon's explicit rank-control updates translate into weights that are more amenable to low-rank approximation than those obtained with Muon. We apply ASVD, SVD-LLM, and Dobi-SVD to all pretrained models across compression rates from 20\% to 80\%, and evaluate the compressed checkpoints on WikiText2 and standard downstream benchmarks. We show a scatter plot of NuMuon's relative performance improvement against Muon in~\Cref{fig:compression_all} as well as raw results for Llama3-1.8B at 40\% compression in~\Cref{tab:llama_comp_sum}, and defer the complete set of results for all models and compression levels (20-80\%) to~\Cref{app:extended_results:results} and \Cref{tab:qwen3_asvd,tab:qwen3_whiten,tab:qwen3_svdllm,tab:qwen3_dobisvd,tab:olmo2_asvd,tab:olmo2_whiten,tab:olmo2_svdllm,tab:olmo2_dobisvd,tab:llama3_asvd,tab:llama3_whiten,tab:llama3_svdllm,tab:llama3_dobisvd}.

\tabLlamaLLMComp
\figScatterAllMethods

Across all three compression methods at 40-80\% compression, NuMuon consistently achieves the strongest downstream average (\textbf{4.2-55.8\%}) and substantially lower validation perplexity (\textbf{up to 99.8\%}) than AdamW and Muon trained baselines. This suggests that NuMuon produces weight matrices whose information is concentrated in fewer effective directions, enabling SVD-based methods to discard parameters with less degradation. Moreover, as the compression method improves (e.g., moving from ASVD to SVD-LLM to Dobi-SVD), NuMuon retains a larger fraction of the underlying model's performance, indicating that better approximators can more effectively exploit the low-rank structure induced by NuMuon.

To show how NuMuon's superiority could lead to better deployment, we report the validation perplexity against the generation throughput (at batch-size 256) in~\Cref{fig:efficiency}. We see that for a fixed perplexity, NuMuon unlocks a higher compression rate and in turn, faster generation throughput compared to both AdamW and Muon. This mirrors the importance of NuMuon in deployability of trained LLMs.

\begin{figure*}[t!]
\centering
\begin{minipage}[t]{0.4\textwidth}
    \figGrassmannQwenSmall
\end{minipage}%
\hfill
\begin{minipage}[t]{0.5\textwidth}
    \figLossVsSchedAblation
\end{minipage}%
\end{figure*}

\tabAblationRankSchedCompact

\paragraph{Update--Weight Subspace Alignment.}
To gain further insight into the training dynamics of NuMuon and its superior compressibility, we measure how well the optimizer updates aligns with the dominant spectral subspace of the weights. Concretely, for a given weight matrix $\boldsymbol{W}$ and its update $\boldsymbol{\Delta W}$, we form the top-$k$ left-singular subspaces (here $k=64$) and compute their Grassmann distance. Let $\boldsymbol{U_W},\boldsymbol{U_{\Delta W}}\in\mathbb{R}^{d\times k}$ be orthonormal bases spanning these two $k$-dimensional subspaces. The principal angles $\{\theta_i\}_{i=1}^k$ are defined via the singular values of $\boldsymbol{U_W}^\top\boldsymbol{U_{\Delta W}}$ as ${\cos(\theta_i)=\sigma_i(\boldsymbol{U_W}^\top\boldsymbol{U_{\Delta W}})}$, and the Grassmann distance

\[
    d_G(\boldsymbol{U_W},\boldsymbol{U_{\Delta W}}) \;=\; \Big(\sum_{i=1}^{k}\theta_i^2\Big)^{1/2}.
    \\[-0.5ex]
\]
\Cref{fig:grassman_distance_wk} plots this quantity over training for the attention key projection matrix $\boldsymbol{W}_k$ in Qwen3-0.6B. We observe that NuMuon maintains a consistently smaller Grassmann distance than Muon, indicating that its updates remain more aligned with the top spectral subspace of $\boldsymbol{W}$. In contrast, Muon applies orthonormalized updates that are less coupled to the evolving spectral geometry of the weights, resulting in larger subspace misalignment that remains almost fixed throughout training. This sustained alignment provides a mechanistic explanation for why NuMuon tends to induce lower stable rank and improved robustness under SVD-based compression.

\paragraph{Ablation Study.}
Next, we study the impact of our design choices on the training and compressibility of LLMs via NuMuon. The most important aspects of our optimizer are (i) the choice of the final update rank and (ii) the rank scheduler, both of which can affect training dynamics and downstream compression. To this end, we ablate the rank scheduler type (cosine (C) vs.\ piecewise (P) vs.\ fixed (F)) and the rank budget. For the rank budget, we use the fixed scheduler and vary the rank fraction in \{0.05, 0.25, 0.50, 0.80\} of the maximum rank, i.e., a fraction of $\max(d_{\rm in}, d_{\rm out})$. We use Qwen3-0.6B for these experiments to enable a broader ablation sweep, and we use Dobi-SVD for LLM compression as it is the strongest method in our comparisons. The rank fractions are summarized in~\Cref{fig:rank_schedulers} in the Appendix.

\begin{figure*}[t!]
\centering
\begin{minipage}[t]{0.475\textwidth}
    \figLossVsRankAblation
\end{minipage}%
\hfill
\begin{minipage}[t]{0.425\textwidth}
    \figStableRankvsLayerAblationQwen
\end{minipage}%
\end{figure*}

Our training curves for various ranks are shown in~\Cref{fig:loss_vs_rank_budget}. As can be seen, decreasing the rank and making it too restrictive harms convergence and results in worse final loss. However, beyond 0.25 the gains from increasing the rank exhibit diminishing returns. Meanwhile, in~\Cref{fig:stable_rank_vs_layer_comparison_qwen3_budget_ablation_small} we observe that a more restrictive nuclear-norm/rank budget induces a lower stable rank across weight matrices and layers, whereas increasing the budget increases the stable rank. This directly translates to compressibility: as shown in~\Cref{tab:ablation_compact}, lower stable-rank (more restrictive) settings typically incur less degradation after 80\% compression. However, as discussed above, there is a balance to strike so that the base model remains strong and we do not over-restrict training.

Finally, we compare different scheduler types in~\Cref{fig:loss_vs_rank_sched,fig:stable_rank_vs_layer_comparison_qwen3_scheduler_ablation}. As we see, fixed-rank schedules tend to degrade the base performance more than cosine and piecewise schedules, where NuMuon benefits from a higher-rank period early in training and then transitions to a lower-rank regime (see ~\Cref{sec:method:practical} for more discussion on this point). In terms of efficiency, low, fixed-rank schedules can be faster per step since they avoid the high-rank Block Krylov SVD regime that arises early with cosine and piecewise schedules (see the Time/Step in \Cref{tab:ablation_compact}). Nevertheless, training speed is comparable with Muon in all cases. For more details on our ablations, see~\Cref{app:extended_results:ablation}.

\section{Conclusion and Future Work}
\label{sec:conclusion}
In this paper, we studied Muon's weight space structure and revealed that its spectral geometry, while designed for optimization efficiency, also shapes the compressibility of the learned weights. The emergence of low-rank structure under full-rank updates suggests that effective dimensionality in trained models arises not just from explicit constraints but from the interplay between optimizer dynamics and loss landscape geometry. NuMuon leverages this insight by directly controlling update rank through a nuclear-norm budget, which we showed reduces to a simple top-$k$ truncation of the momentum's singular directions. The result is an optimizer that retains Muon's training efficiency while producing weights better suited for aggressive post-hoc compression, achieving substantial gains in the high-compression regime where standard Muon-trained models underperform.

In addition to being beneficial for LLM compression workflows, the factored form of NuMuon's updates is a natural fit for distributed training over bandwidth-constrained settings~\citep{douillard2023diloco} and gossip-based protocols~\citep{blot2016gossip}, complementing recent efforts such as DION~\citep{ahn2025dion} that approximate Muon for distributed training. We leave exploration of this interesting direction to future work.

\bibliographystyle{plainnat}
\nobibliography*
\bibliography{references}

\clearpage
\appendix
\section{Proofs}
\label{app:proofs}

\subsection{NuMuon LMO Proofs}
\label{app:proofs:numuon_lmo}

\begin{blockquote}
    \numuonlp*
\end{blockquote}

\begin{proof}
Let us assume that $\bar{\boldsymbol{M}}:=-\boldsymbol{M}$. Then, the LMO is equivalent to:
\begin{equation}\label{eq:app:max_lmo}
    -\max_{\boldsymbol{\Delta W}\in\mathcal{W}} \langle \bar{\boldsymbol{M}},\boldsymbol{\Delta W}\rangle,
\end{equation}
over $\mathcal{W}=\{\boldsymbol{\Delta W}:\|\boldsymbol{\Delta W}\|_2\le\rho,\ \|\boldsymbol{\Delta W}\|_*\le\tau\}$.
The optimal solution of this LMO can be written as:
\begin{equation}
    \argmax_{\boldsymbol{\Delta W}\in \mathcal{W}} \langle \bar{\boldsymbol{M}}, \boldsymbol{\Delta W}\rangle = \argmax_{\Delta \boldsymbol{W}\in \mathcal{W}} \mathrm{Tr}(\bar{\boldsymbol{M}}^{\top}\boldsymbol{\Delta W}).
\end{equation}
Let $\bar{\boldsymbol{M}}=\bar{\boldsymbol{U}}\,\mathrm{diag}(\boldsymbol{\sigma})\,\boldsymbol{V}^\top$ where $\sigma_1\ge \sigma_2\ge \cdots \ge 0$ and $\bar{\boldsymbol{U}} = - \boldsymbol{U}$. Let us assume a new variable $\boldsymbol{Y} := \bar{\boldsymbol{U}}^{\top}~\boldsymbol{\Delta W}~\boldsymbol{V}$, then we have:
\begin{align}
    \mathrm{Tr}(\bar{\boldsymbol{M}}^{\top}\boldsymbol{\Delta W}) 
        &= \mathrm{Tr}(\boldsymbol{V}\mathrm{diag}(\boldsymbol{\sigma})^{\top}\boldsymbol{U}^{\top}\boldsymbol{\Delta W}) \nonumber \\
        &= \mathrm{Tr}(\mathrm{diag}(\boldsymbol{\sigma})^{\top}\boldsymbol{U}^{\top}\boldsymbol{\Delta W}~\boldsymbol{V}) \nonumber \\
        &= \mathrm{Tr}(\mathrm{diag}(\boldsymbol{\sigma})^{\top}\boldsymbol{Y})
\end{align}
and $\|\boldsymbol{\Delta W}\|_2 = \|\boldsymbol{U}^{\top}\boldsymbol{\Delta W}~\boldsymbol{V}\|_2$ by unitary invariance.
Hence, our LMO is equivalent to:
\begin{equation}
    \max_{\boldsymbol{Y}} \mathrm{Tr}(\mathrm{diag}(\boldsymbol{\sigma})^{\top}\boldsymbol{Y}) \quad \text{s.t.}\quad \|\boldsymbol{Y}\|_* \le \tau,\ \ \|\boldsymbol{Y}\|_2 \le \rho.
\end{equation}

From von Neumann's trace inequality~\citep{vonNeumann1937}, we have:
\begin{equation}
    \big|\mathrm{Tr}(\boldsymbol{A}^{\top}\boldsymbol{B})\big|
    \le
    \sum_i \sigma_i(\boldsymbol{A})\sigma_i(\boldsymbol{B}),
\end{equation}
with equality if and only iff the eigenvectors are the same. Thus, the maximum of $\mathrm{Tr}(\mathrm{diag}(\boldsymbol{\sigma})\,\boldsymbol{Y})$
under these constraints is achieved when $\boldsymbol{Y}$ shares singular vectors with
$\mathrm{diag}(\boldsymbol{\sigma})$, i.e., when $\boldsymbol{Y}=\mathrm{diag}(\boldsymbol{s})$ with $\boldsymbol{s}\ge 0$.
Under this choice, the constraints become $0\le s_i\le \rho$ and $\sum_i s_i\le \tau$,
and the objective becomes $\sum_i \sigma_i s_i$, yielding \eqref{eq:numuon_lp}.
Negating the maximizer returns $\boldsymbol{\Delta W}^\star=-\boldsymbol{U}\,\mathrm{diag}(\boldsymbol{s}^\star)\,\boldsymbol{V}^\top$.
\end{proof}

\begin{blockquote}
    \numuoncf*
\end{blockquote}

\begin{proof}
The feasible region of \eqref{eq:numuon_lp} is the capped simplex $\{\boldsymbol{s}\in\mathbb{R}^q: 0\le s_i\le \rho,\ \sum_i s_i\le\tau\}$. Since $\sigma_1\ge \sigma_2\ge\cdots\ge 0$, the objective $\sum_i \sigma_i s_i$ is maximized by allocating as much mass as possible to the smallest indices $i$ (largest weights $\sigma_i$), subject to the cap $\rho$ and the total budget $\tau$. Formally, if there exist indices $i<j$ with $s_i<\rho$ and $s_j>0$, then moving $\varepsilon:=\min\{\rho-s_i,s_j\}$ mass from $j$ to $i$ preserves feasibility and increases the objective by $\varepsilon(\sigma_i-\sigma_j)\ge 0$. Repeating this exchange argument yields the greedy fill-in solution in \Cref{eq:numuon_s_star}, which sets $s_1=\rho$, then $s_2=\rho$, etc., until the budget $\tau$ is exhausted. The rank bound follows from the number of strictly positive entries of $\boldsymbol{s}^\star$.
\end{proof}

\subsection{Convergence Analysis}
\label{app:proofs:convergence_analysis}
In this section, we provide the convergence analysis of NuMuon. For this analysis, we assume that at iteration $t$, given momentum buffer $\boldsymbol{M}_t\in \R^{d_{\rm out}\times d_{\rm in}}$, let $\boldsymbol{U}_{t,k}\in\R^{d_{\rm out}\times k}$ and $\boldsymbol{V}_{t,k}\in\R^{d_{\rm in}\times k}$ contain the top-$k$ left and right singular vectors of $\boldsymbol{M}_t$, respectively. NuMuon performs the following update:
\begin{equation}\label{eq:numuon_update}
\begin{aligned}
    & \boldsymbol{G}_t \leftarrow \frac{1}{b}\sum_{i=1}^{b}\nabla f\left(\boldsymbol{W}_{t},\boldsymbol{\xi}_{t, i}\right), \\
    & \boldsymbol{M}_t \leftarrow \beta \boldsymbol{M}_{t-1} + (1-\beta) \boldsymbol{G}_t, \\
    & \boldsymbol{U}_{t,k}, \boldsymbol{V}_{t,k} \leftarrow \text{top-}k\text{ singular vectors of } \boldsymbol{M}_t,\\
    & \boldsymbol{W}_{t+1} \leftarrow \boldsymbol{W}_{t}-\gamma\, \boldsymbol{U}_{t,k}\boldsymbol{V}_{t,k}^\top,
\end{aligned}
\end{equation}
for a fixed learning rate $\gamma$.\footnote{Note that we omit $\rho$ from~\Cref{eq:numuon_topk_update} and assume it is absorbed into $\gamma$.} Following standard practice in Muon convergence analysis~\citep{pethick2025training,riabinin2025gluon,shen2025convergence}, we assume the following:
\begin{assumptionblockquote}
\smoothness*
\variance*
\end{assumptionblockquote}

Additionally, we also assume that the gradient's energy is concentrated around its top-$k$ singular directions.

\begin{assumptionblockquote}
\tailcontrol*
\end{assumptionblockquote}

Next, we will review some definitions and identities that would be used in our proofs.

\begin{assumptionblockquote}
\begin{definition}
\label{def:ky_fan}
For a matrix $\boldsymbol{A}\in\R^{d_{\rm out}\times d_{\rm in}}$ with singular values $\sigma_1(\boldsymbol{A})\ge \cdots \ge \sigma_q(\boldsymbol{A})$, $q=\min(d_{\rm out},d_{\rm in})$, the Ky Fan $k$-norm is defined as:
\[
    \|\boldsymbol{A}\|_{(k)} := \sum_{i=1}^k \sigma_i(\boldsymbol{A}).
\]
\end{definition}
\end{assumptionblockquote}

\begin{lemmablockquote}
\begin{lemma}[Ky Fan norm bounded by Frobenius norm]
For any matrix $\boldsymbol{A}$ and any positive integer $k$,
\begin{equation}\label{eq:kyfan_frob}
    \|\boldsymbol{A}\|_{(k)} \le \sqrt{k}\,\|\boldsymbol{A}\|_{\mathrm{F}}.
\end{equation}
\end{lemma}
\end{lemmablockquote}

\begin{proof}
Let $\sigma_1 \ge \sigma_2 \ge \cdots \ge \sigma_q \ge 0$ denote the singular values of $\boldsymbol{A}$. By \Cref{def:ky_fan}, the Ky Fan $k$-norm is $\|\boldsymbol{A}\|_{(k)} = \sum_{i=1}^{k} \sigma_i$. Applying the Cauchy--Schwarz inequality yields
\[
    \|\boldsymbol{A}\|_{(k)} = \sum_{i=1}^{k} 1 \cdot \sigma_i \le \sqrt{\sum_{i=1}^{k} 1^2} \cdot \sqrt{\sum_{i=1}^{k} \sigma_i^2} = \sqrt{k} \cdot \sqrt{\sum_{i=1}^{k} \sigma_i^2}.
\]
Since $\sum_{i=1}^{k} \sigma_i^2 \le \sum_{i=1}^{q} \sigma_i^2 = \|\boldsymbol{A}\|_{\mathrm{F}}^2$, the result follows.
\end{proof}

\begin{lemmablockquote}
\begin{lemma}[Top-$k$ SVD identities]\label{lemma:topk_identities}
Let $\boldsymbol{M} \in \mathbb{R}^{d_{\rm out} \times d_{\rm in}}$ have singular value decomposition $\boldsymbol{M} = \boldsymbol{U} \boldsymbol{S} \boldsymbol{V}^\top$ with singular values $\sigma_1 \ge \sigma_2 \ge \cdots \ge \sigma_q \ge 0$, where $q = \min(d_{\rm out},d_{\rm in})$. Let $\boldsymbol{U}_k \in \mathbb{R}^{d_{\rm out} \times k}$ and $\boldsymbol{V}_k \in \mathbb{R}^{d_{\rm in} \times k}$ denote the matrices containing the top-$k$ left and right singular vectors, respectively. Then:
\begin{enumerate}
    \item[(i)] $\langle \boldsymbol{M}, \boldsymbol{U}_k \boldsymbol{V}_k^\top \rangle = \sum_{i=1}^k \sigma_i(\boldsymbol{M}) = \|\boldsymbol{M}\|_{(k)}$.
    \item[(ii)] $\|\boldsymbol{U}_k \boldsymbol{V}_k^\top\|_\F^2 = k$.
\end{enumerate}
\end{lemma}
\end{lemmablockquote}

\begin{proof}
\textbf{(i)} Using the cyclic property of trace and the SVD $\boldsymbol{M} = \boldsymbol{U} \boldsymbol{S} \boldsymbol{V}^\top$:
\begin{align*}
\langle \boldsymbol{M}, \boldsymbol{U}_k \boldsymbol{V}_k^\top \rangle 
    &= \tr(\boldsymbol{M}^\top \boldsymbol{U}_k \boldsymbol{V}_k^\top) \\
    &= \tr(\boldsymbol{V} \boldsymbol{S} \boldsymbol{U}^\top \boldsymbol{U}_k \boldsymbol{V}_k^\top).
\end{align*}
Since $\boldsymbol{U}$ has orthonormal columns and $\boldsymbol{U}_k$ consists of its first $k$ columns, we have $\boldsymbol{U}^\top \boldsymbol{U}_k = \begin{bmatrix} \boldsymbol{I}_k \\ \boldsymbol{0} \end{bmatrix}$. Thus $\boldsymbol{S} \boldsymbol{U}^\top \boldsymbol{U}_k = \begin{bmatrix} \boldsymbol{S}_k \\ \boldsymbol{0} \end{bmatrix}$, where $\boldsymbol{S}_k = \diag(\sigma_1, \ldots, \sigma_k)$. Similarly, $\boldsymbol{V}_k$ consists of the first $k$ columns of $\boldsymbol{V}$, so:
\begin{align*}
\tr(\boldsymbol{V} \boldsymbol{S} \boldsymbol{U}^\top \boldsymbol{U}_k \boldsymbol{V}_k^\top) 
    &= \tr\left(\boldsymbol{V} \begin{bmatrix} \boldsymbol{S}_k \\ \boldsymbol{0} \end{bmatrix} \boldsymbol{V}_k^\top\right) \\
    &= \tr(\boldsymbol{V}_k \boldsymbol{S}_k \boldsymbol{V}_k^\top) \\
    &= \tr(\boldsymbol{S}_k \boldsymbol{V}_k^\top \boldsymbol{V}_k) \\
    &= \tr(\boldsymbol{S}_k) = \sum_{i=1}^k \sigma_i.
\end{align*}

\textbf{(ii)} Since $\boldsymbol{U}_k$ and $\boldsymbol{V}_k$ have orthonormal columns:
\begin{align*}
\|\boldsymbol{U}_k \boldsymbol{V}_k^\top\|_\F^2 
    &= \tr(\boldsymbol{V}_k \boldsymbol{U}_k^\top \boldsymbol{U}_k \boldsymbol{V}_k^\top) \\
    &= \tr(\boldsymbol{V}_k \boldsymbol{I}_k \boldsymbol{V}_k^\top) \\
    &= \tr(\boldsymbol{V}_k^\top \boldsymbol{V}_k) \\
    &= \tr(\boldsymbol{I}_k) = k.
\end{align*}
Equivalently, $\boldsymbol{U}_k \boldsymbol{V}_k^\top$ has exactly $k$ singular values equal to $1$ and the rest equal to $0$, so $\|\boldsymbol{U}_k \boldsymbol{V}_k^\top\|_\F^2 = \sum_{i=1}^k 1^2 = k$.
\end{proof}

For completeness, we also provide the following Lemma from~\citet{shen2025convergence}.
\begin{lemmablockquote}
\begin{lemma}[Momentum Error Bound~(Lemma~A.3~\citep{shen2025convergence})]\label{lemma:momentum_error}
Under~\Cref{ass:variance}, let 
\[
    \boldsymbol{C}_t=\beta \boldsymbol{C}_{t-1}+(1-\beta)\nabla f(\boldsymbol{W}_t)
\]
and $\boldsymbol{M}_t=\beta \boldsymbol{M}_{t-1}+(1-\beta)\boldsymbol{G}_t$ with $\boldsymbol{C}_0=\nabla f(\boldsymbol{W}_0)$ and $\boldsymbol{M}_0=\boldsymbol{G}_0$. Then
\[
    \E\big[\|\boldsymbol{C}_t-\boldsymbol{M}_t\|_\F\big]
    \le \sqrt{\frac{1-\beta}{1+\beta}}\frac{\nu}{\sqrt{b}}+\beta^t\frac{\nu}{\sqrt{b}}.
\]
\end{lemma}
\end{lemmablockquote}

Now, we are ready to prove our main result.

\begin{theoremblockquote}
    \nonconvex*
\end{theoremblockquote}

\begin{proof}
Our proof follows the proof of Theorem~4.3 in~\citet{shen2025convergence} with modifying the fully orthogonalized update with the top-$k$ NuMuon updates. In particular, by $L$-smoothness of $f$ we have
\begin{align*}
    &\E[f(\boldsymbol{W}_t)-f(\boldsymbol{W}_{t+1})]\\
    &\qquad\ge
    \E\Big[\gamma\langle \nabla f(\boldsymbol{W}_t),\boldsymbol{U}_{t,k}\boldsymbol{V}_{t,k}^\top\rangle
    -\frac{L}{2}\gamma^2\|\boldsymbol{U}_{t,k}\boldsymbol{V}_{t,k}^\top\|_\F^2\Big]\\
    &\qquad=
    \E\Big[\gamma\langle \boldsymbol{M}_t,\boldsymbol{U}_{t,k}\boldsymbol{V}_{t,k}^\top\rangle
    -\frac{L}{2}\gamma^2\|\boldsymbol{U}_{t,k}\boldsymbol{V}_{t,k}^\top\|_\F^2 + \gamma\langle \nabla f(\boldsymbol{W}_t)-\boldsymbol{M}_t,\boldsymbol{U}_{t,k}\boldsymbol{V}_{t,k}^\top\rangle\Big]\\
    &\qquad\ge
    \E\Big[\gamma\langle \boldsymbol{M}_t,\boldsymbol{U}_{t,k}\boldsymbol{V}_{t,k}^\top\rangle
    -\frac{L}{2}\gamma^2\|\boldsymbol{U}_{t,k}\boldsymbol{V}_{t,k}^\top\|_\F^2 -\gamma\|\nabla f(\boldsymbol{W}_t)-\boldsymbol{M}_t\|_\F\|\boldsymbol{U}_{t,k}\boldsymbol{V}_{t,k}^\top\|_\F\Big].
\end{align*}
Using \Cref{lemma:topk_identities}, the above becomes
\begin{align*}
    &\E[f(\boldsymbol{W}_t)-f(\boldsymbol{W}_{t+1})]\\
    &\qquad\ge\E\Big[\gamma\|\boldsymbol{M}_t\|_{(k)}-\frac{L}{2}\gamma^2 k-\gamma\sqrt{k}\,\|\nabla f(\boldsymbol{W}_t)-\boldsymbol{M}_t\|_\F\Big].
\end{align*}
By the reverse triangle inequality for $\|\cdot\|_{(k)}$ and \Cref{eq:kyfan_frob},
\[
\|\boldsymbol{M}_t\|_{(k)}\ge \|\nabla f(\boldsymbol{W}_t)\|_{(k)}-\sqrt{k}\|\boldsymbol{M}_t-\nabla f(\boldsymbol{W}_t)\|_\F.
\]
Substituting yields
\begin{align}\label{eq:descent_topk}
    &\E[f(\boldsymbol{W}_t)-f(\boldsymbol{W}_{t+1})]\notag\\
    &\qquad\ge\E\Big[\gamma\|\nabla f(\boldsymbol{W}_t)\|_{(k)}-\frac{L}{2}\gamma^2 k -2\gamma\sqrt{k}\,\|\nabla f(\boldsymbol{W}_t)-\boldsymbol{M}_t\|_\F\Big].
    \end{align}

\textbf{Bounding $\E\Big[\|\nabla f(\boldsymbol{W}_t)-\boldsymbol{M}_t\|_\F\Big]$.}
Define $\boldsymbol{C}_0=\nabla f(\boldsymbol{W}_0)$ and for $t>0$ let
\[
    \boldsymbol{C}_t=\beta \boldsymbol{C}_{t-1}+(1-\beta)\nabla f(\boldsymbol{W}_t).
\]
Then, we can write
\[
    \E \Big[\|\nabla f(\boldsymbol{W}_t)-\boldsymbol{M}_t\|_\F \Big] \le \E \Big[\|\boldsymbol{C}_t-\boldsymbol{M}_t\|_\F\Big] + \E\Big[\|\nabla f(\boldsymbol{W}_t)-\boldsymbol{C}_t\|_\F\Big].
\]

By~\Cref{lemma:momentum_error} from~\citet{shen2025convergence}, for the first term we have
\begin{equation}\label{eq:ct_mt_bound_topk}
    \E\Big[\|\boldsymbol{C}_t-\boldsymbol{M}_t\|_\F]
    \le \sqrt{\frac{1-\beta}{1+\beta}}\frac{\nu}{\sqrt{b}}+\beta^t\frac{\nu}{\sqrt{b}}.
\end{equation}

For the second term at $t>0$, we can write
\begin{align*}
    &\E\big[\|\nabla f(\boldsymbol{W}_t)-\boldsymbol{C}_t\|_\F\big]\\
    &\qquad= \E\big[\|\nabla f(\boldsymbol{W}_t)-\big(\beta \boldsymbol{C}_{t-1}+(1-\beta)\nabla f(\boldsymbol{W}_t)\big)\|_\F\big]\\
    &\qquad= \E\big[\beta\|\nabla f(\boldsymbol{W}_t)-\boldsymbol{C}_{t-1}\|_\F\big]\\
    &\qquad\overset{(a)}{\le} \E\big[\beta\|\nabla f(\boldsymbol{W}_{t-1})-\boldsymbol{C}_{t-1}\|_\F+\beta\|\nabla f(\boldsymbol{W}_{t-1})-\nabla f(\boldsymbol{W}_t)\|_\F\big]\\
    &\qquad\overset{(b)}{\le} \E\big[\beta\|\nabla f(\boldsymbol{W}_{t-1})-\boldsymbol{C}_{t-1}\|_\F+\beta L\|\boldsymbol{W}_{t-1}-\boldsymbol{W}_t\|_\F\big]\\
    &\qquad= \E\big[\beta\|\nabla f(\boldsymbol{W}_{t-1})-\boldsymbol{C}_{t-1}\|_\F+\beta L\gamma\|\boldsymbol{U}_{t-1,k}\boldsymbol{V}_{t-1,k}^\top\|_\F\big]\\
    &\qquad= \E\big[\beta\|\nabla f(\boldsymbol{W}_{t-1})-\boldsymbol{C}_{t-1}\|_\F+\beta L\gamma\sqrt{k}\big]\\
    &\qquad\le \sum_{i=1}^t \beta^i L\gamma\sqrt{k}\le \frac{\sqrt{k}\beta L\gamma}{1-\beta}.
\end{align*}
where $(a)$ is by triangle inequality and $(b)$ is using~\Cref{ass:smoothness}. Combining with~\Cref{eq:ct_mt_bound_topk} gives
\begin{align}\label{eq:grad_momentum_gap_topk}
\E\big[\|\nabla f(\boldsymbol{W}_t)-\boldsymbol{M}_t\|_\F\big] \le
    \sqrt{\frac{1-\beta}{1+\beta}}\frac{\nu}{\sqrt{b}}
    +\beta^t\frac{\nu}{\sqrt{b}}
    +\frac{\sqrt{k}\beta L\gamma}{1-\beta}.
\end{align}

\textbf{Deriving the Ky Fan Stationarity.}
Plugging~\Cref{eq:grad_momentum_gap_topk} into \Cref{eq:descent_topk}:
\begin{align*}
    &\E[f(\boldsymbol{W}_t)-f(\boldsymbol{W}_{t+1})]\\
    &\qquad\ge
        \E\Big[\gamma\|\nabla f(\boldsymbol{W}_t)\|_{(k)}-\frac{L}{2}\gamma^2 k\Big]-2\gamma\sqrt{k}\Big(
        \sqrt{\frac{1-\beta}{1+\beta}}\frac{\nu}{\sqrt{b}}
        +\beta^t\frac{\nu}{\sqrt{b}}
        +\frac{\sqrt{k}\beta L\gamma}{1-\beta}
        \Big)\\
    &\qquad=
        \E\Big[\gamma\|\nabla f(\boldsymbol{W}_t)\|_{(k)}-\frac{L}{2}\gamma^2 k
        -2\gamma\sqrt{\frac{1-\beta}{1+\beta}}\frac{\nu\sqrt{k}}{\sqrt{b}}-2\gamma\beta^t\frac{\nu\sqrt{k}}{\sqrt{b}}
        -\frac{2k\gamma^2\beta L}{1-\beta}
        \Big].
\end{align*}
Summing over $t=0,1,\dots,T-1$ and dividing by $T\gamma$ yields:
\begin{align}\label{eq:ky_fan_stationarity}
    &\frac{1}{T}\sum_{t=0}^{T-1}\E\|\nabla f(\boldsymbol{W}_t)\|_{(k)} \le
        \frac{\E[f(\boldsymbol{W}_0)-f(\boldsymbol{W}_T)]}{T\gamma}
        +\frac{Lk\gamma}{2}
        +\frac{2\nu\sqrt{k(1-\beta)}}{\sqrt{(1+\beta)b}} +\frac{2\beta\nu\sqrt{k}}{(1-\beta)T\sqrt{b}}
        +\frac{2k\gamma\beta L}{1-\beta}.
\end{align}

\paragraph{Completing the Proof.} Now, we need to apply~\Cref{ass:tail_control} to the above bound. Recall that for any matrix $\boldsymbol{G}$, by the singular value decomposition we have:
\[
    \|\boldsymbol{G}\|_* = \sum_{i=1}^k \sigma_i(\boldsymbol{G})+\sum_{i>k}\sigma_i(\boldsymbol{G})=\|\boldsymbol{G}\|_{(k)}+\|\boldsymbol{G}-\boldsymbol{G}_k\|_*.
\]
Applying this to $\boldsymbol{G}=\nabla f(\boldsymbol{W}_t)$ gives
\[
    \|\nabla f(\boldsymbol{W}_t)\|_* = \|\nabla f(\boldsymbol{W}_t)\|_{(k)}+\|\nabla f(\boldsymbol{W}_t)-(\nabla f(\boldsymbol{W}_t))_k\|_*.
\]
Taking expectation from both sides, using~\Cref{ass:tail_control}, and averaging over $t$, we would have:
\begin{equation}\label{eq:cor_nuclear_main}
    \frac{1}{T}\sum_{t=0}^{T-1}\E\|\nabla f(\boldsymbol{W}_t)\|_*
    \le
    \frac{1}{T}\sum_{t=0}^{T-1}\E\|\nabla f(\boldsymbol{W}_t)\|_{(k)}+\delta_k.
\end{equation}
Putting \Cref{eq:cor_nuclear_main} in \Cref{eq:ky_fan_stationarity} completes the proof.
\end{proof}

\subsubsection{A Note on the Validity of Tail \Cref{ass:tail_control}}
\label{app:convergence:tail}

Recall that \Cref{ass:tail_control} requires a \textit{tail control} condition of the form
\begin{equation}\label{eq:tail_assumption_repeat}
    \E\|\nabla f(\boldsymbol{W}_t)-(\nabla f(\boldsymbol{W}_t))_k\|_* \le \delta_k \qquad\text{for all iterates }t,
\end{equation}
i.e., the nuclear-norm mass outside the top-$k$ singular directions of the gradient is uniformly small along the optimization trajectory.
A convenient way to interpret and validate~\Cref{eq:tail_assumption_repeat} is via the \textit{residual spectral energy} of the gradient after top-$k$ truncation.
Let $\boldsymbol{G} := \nabla f(\boldsymbol{W})$ and write its thin SVD as
\[
    \boldsymbol{G} = \boldsymbol{U} \boldsymbol{S} \boldsymbol{V}^\top, \qquad \boldsymbol{S}=\diag(\sigma_1,\dots,\sigma_q),\quad q=\min(d_{\rm out},d_{\rm in}),
\]
where $\sigma_1\ge \cdots \ge \sigma_q\ge 0$, and let $\boldsymbol{G}_k = \boldsymbol{U}_k \boldsymbol{S}_k \boldsymbol{V}_k^\top$ be the best rank-$k$ truncation (top-$k$ SVD reconstruction).
Define the (squared) Frobenius residual energy
\begin{equation}\label{eq:delta_k_F_def}
    \delta_k^{(\F)}(\boldsymbol{W})\;:=\;\|\boldsymbol{G}-\boldsymbol{G}_k\|_\F^2.
\end{equation}
Then
\begin{align}
    \delta_k^{(\F)}(\boldsymbol{W})
    &= \|\boldsymbol{U} \boldsymbol{S} \boldsymbol{V}^\top - \boldsymbol{U}_k \boldsymbol{S}_k \boldsymbol{V}_k^\top\|_\F^2 \nonumber\\
    &= \|\boldsymbol{U} \big(\boldsymbol{S} - \begin{bmatrix}\boldsymbol{S}_k\\\boldsymbol{0}\end{bmatrix}\big) \boldsymbol{V}^\top\|_\F^2 \nonumber\\
    &= \|\begin{bmatrix}\boldsymbol{S}_k\\\boldsymbol{S}_{\bar{k}}\end{bmatrix} - \begin{bmatrix}\boldsymbol{S}_k\\\boldsymbol{0}\end{bmatrix}\|_\F^2 \nonumber\\
    &= \|\boldsymbol{S}_{\bar{k}}\|_\F^2 \nonumber \\
    &= \sum_{i>k} \sigma_i^2.\label{eq:delta_k_F_tail}
\end{align}
Thus, $\delta_k^{(\F)}$ is exactly the residual spectral energy beyond rank $k$. An immediate consequence is that increasing $k$ can only reduce the residual energy.
\begin{lemmablockquote}
\begin{lemma}[Monotonicity of Residual Energy]\label{lem:residual_monotone}
Let $\delta_k^{(\F)}(\boldsymbol{W})$ be defined as in~\Cref{eq:delta_k_F_def}. Then $\delta_k^{(\F)}(\boldsymbol{W})$ is monotonically decreasing in $k$:
\[
    \delta_{k+1}^{(\F)}(\boldsymbol{W}) \le \delta_k^{(\F)}(\boldsymbol{W}) \qquad \text{for all } k \ge 1.
\]
\end{lemma}
\end{lemmablockquote}

\begin{proof}
From \Cref{eq:delta_k_F_tail},
\begin{equation}\label{eq:delta_monotone}
    \delta_{k+1}^{(\F)}(\boldsymbol{W}) 
    = \sum_{i>k+1}\sigma_i^2
    = \sum_{i>k}\sigma_i^2 - \sigma_{k+1}^2
    = \delta_k^{(\F)}(\boldsymbol{W})-\sigma_{k+1}^2
    \;\le\; \delta_k^{(\F)}(\boldsymbol{W}),
\end{equation}
where the inequality follows from $\sigma_{k+1}^2 \ge 0$.
\end{proof}

By \Cref{lem:residual_monotone}, $\delta_1^{(\F)}(\boldsymbol{W}) \ge \delta_2^{(\F)}(\boldsymbol{W}) \ge \cdots \ge \delta_q^{(\F)}(\boldsymbol{W}) = 0$. Hence, as long as we show that $\delta_{1}^{(\F)}(\boldsymbol{W})$ is bounded and close to zero, the tail assumption of~\Cref{ass:tail_control} holds for all $k > 1$.

\figDeltaQwenLarge

\paragraph{Empirical validation.} To empirically validate this assumption, we measure $\delta_{1}^{(\F)}(\boldsymbol{W})$ via
\begin{equation}\label{eq:delta1F_sr}
    \delta_1^{(\F)}(\boldsymbol{W})=\sigma_1^2\big(\mathrm{sr}(\boldsymbol{G})-1\big),
\end{equation}
since the stable rank is defined as $\mathrm{sr}(\boldsymbol{G}):=\|\boldsymbol{G}\|_\F^2/\sigma_1^2\ge 1$. \Cref{fig:delta_1} displays this quantity across all transformer block weight matrices for the Qwen3-0.6B models. As shown, this residual typically decreases over the course of training for NuMuon, indicating that the gradient spectrum becomes increasingly concentrated in its leading directions. By \Cref{eq:tail_assumption_repeat}, this concentration implies a small nuclear tail for modest values of $k$, supporting the validity of the tail assumption in~\Cref{eq:tail_assumption_repeat}. Consequently, the Ky Fan $k$-norm stationarity guarantees of~\Cref{thm:nonconvex_topk_main} can be meaningfully converted into approximate nuclear norm stationarity bounds that we derive for NuMuon.

\subsection{Feasibility Attraction}

In~\Cref{sec:method:numuon}, we discussed how NuMuon modifies the spectral-norm LMO by incorporating an additional nuclear-norm constraint. While this constraint ensures that each update $\boldsymbol{\Delta W}_t$ is low-rank, it is not immediately clear whether the iterates $\boldsymbol{W}_t$ themselves converge toward the feasible set. The following lemma establishes that under FW/CG updates, the distance to NuMuon's feasible set contracts at each iteration, guaranteeing that the iterates are progressively attracted toward feasibility.

\begin{lemmablockquote}
\begin{lemma}[Feasibility Attraction under FW/CG Updates]\label{lem:feasibility_attraction}
Let $\mathcal{W}_*$ be the NuMuon feasible set in \Cref{eq:numuon_domain}, and assume $\mathcal{W}_*$ is nonempty, closed, and convex.
Consider a FW/CG-style update of the form
\begin{equation}\label{eq:fw_update_generic}
    \boldsymbol{W}_{t+1} \;=\; (1-\gamma_t)\,\boldsymbol{W}_t \;+\; \gamma_t\,\boldsymbol{\Delta W}_t,
    \qquad \boldsymbol{\Delta W}_t \in \mathcal{W}_*,\quad \gamma_t\in(0,1].
\end{equation}
Then the distance to $\mathcal{W}_*$ contracts:
\begin{equation}\label{eq:dist_contract}
    \mathrm{dist}\!\left(\boldsymbol{W}_{t+1},\mathcal{W}_*\right)
    \;\le\;
    (1-\gamma_t)\,\mathrm{dist}\!\left(\boldsymbol{W}_{t},\mathcal{W}_*\right),
\end{equation}
where $\mathrm{dist}(\boldsymbol{X},\mathcal{W}_*):=\inf_{\boldsymbol{Y}\in\mathcal{W}_*}\|\boldsymbol{X}-\boldsymbol{Y}\|_{\F}$.
Consequently,
\begin{equation}\label{eq:dist_product}
    \mathrm{dist}\!\left(\boldsymbol{W}_{T},\mathcal{W}_*\right)
    \;\le\;
    \Bigl(\prod_{t=0}^{T-1}(1-\gamma_t)\Bigr)\,
    \mathrm{dist}\!\left(\boldsymbol{W}_{0},\mathcal{W}_*\right),
\end{equation}
and in particular $\mathrm{dist}(\boldsymbol{W}_{T},\mathcal{W}_*)\to 0$ whenever $\sum_{t=0}^{\infty}\gamma_t=\infty$
(e.g., for constant $\gamma_t=\gamma\in(0,1]$, the decay is geometric: $\mathrm{dist}(\boldsymbol{W}_T,\mathcal{W}_*)\le (1-\gamma)^T\,\mathrm{dist}(\boldsymbol{W}_0,\mathcal{W}_*)$).
\end{lemma}
\end{lemmablockquote}

\begin{proof}
Indeed, $\mathcal{W}_*=\{\Delta W:\ \|\Delta W\|_2\le \rho\}\cap\{\Delta W:\ \|\Delta W\|_*\le \tau\}$ is an intersection of two closed convex sets, and is nonempty since $\boldsymbol{0} \in \mathcal{W}_*$ so the conditions of the lemma holds. Now let $\boldsymbol{P}_t\in\argmin_{\boldsymbol{Y}\in\mathcal{W}_*}\|\boldsymbol{W}_t-\boldsymbol{Y}\|_{\F}$ be the Frobenius projection of $\boldsymbol{W}_t$ onto $\mathcal{W}_*$.
Since $\mathcal{W}_*$ is convex and $\boldsymbol{\Delta W}_t\in\mathcal{W}_*$, the convex combination
\[
    \boldsymbol{Z}_t := (1-\gamma_t)\boldsymbol{P}_t + \gamma_t \boldsymbol{\Delta W}_t
\]
also lies in $\mathcal{W}_*$.
Therefore, by the definition of distance to a set,
\[
    \mathrm{dist}(\boldsymbol{W}_{t+1},\mathcal{W}_*)
    \le \|\boldsymbol{W}_{t+1}-\boldsymbol{Z}_t\|_{\F}
    = \|(1-\gamma_t)(\boldsymbol{W}_t-\boldsymbol{P}_t)\|_{\F}
    = (1-\gamma_t)\,\mathrm{dist}(\boldsymbol{W}_{t},\mathcal{W}_*),
\]
which proves \Cref{eq:dist_contract}. Iterating this inequality yields \Cref{eq:dist_product}. Finally, if $\sum_{t=0}^\infty \gamma_t=\infty$, then $\prod_{t=0}^{T-1}(1-\gamma_t)\to 0$, implying $\mathrm{dist}(\boldsymbol{W}_{T},\mathcal{W}_*)\to 0$.
\end{proof}

\section{Algorithms}
\label{app:algs}

We summarize the differences between Muon and NuMuon in~\Cref{alg:muon_numuon_v2}. Overall, the two methods are similar; however, Muon relies on a Newton--Schulz iteration (see~\Cref{alg:newton_schulz}) to compute an orthonormalized update, whereas NuMuon uses Block Krylov methods to approximate the top-$k$ singular subspace and construct a low-rank update. NuMuon additionally employs a rank scheduler, as discussed in~\Cref{sec:method:practical}. These changes differentiate NuMuon from Muon and allow it to obtain a final low-rank solution which would greatly benefit downstream weight compression.

\paragraph{RMS-to-RMS Scaling in Practice.}
In practical large-scale implementations, it is common to apply a shape-dependent rescaling to the orthogonalized LMO direction so that the \textit{per-entry RMS} of the update is comparable across matrices with different aspect ratios and can be aligned to typical AdamW update magnitudes~\citep{bernstein2025modular,pethick2025training,moonlight2025}. Concretely, one often uses
\begin{align}\label{eq:rms_scale}
    \mathrm{lmo}_{\mathcal{W}}(\boldsymbol{M}) &= -\rho\, s(d_{\rm out}, d_{\rm in})\, \boldsymbol{U}\boldsymbol{V}^\top,
\end{align}
where $s(d_{\rm out}, d_{\rm in}) = \sqrt{{d_{\rm out}}/{d_{\rm in}}}$ (or closely related variants), and absorbs $\rho\,s(\cdot)$ into the effective stepsize.

\paragraph{NuMuon's Rank Schedulers.} NuMuon uses rank schedulers to determine the rank fraction $k_t$ to compute its truncated updates. We use three family of schedulers in our experiments, which are:
\begin{itemize}
    \item \textbf{Fixed:} $r(t)=r_0$ (constant rank fraction);
    \item \textbf{Piecewise:} $r(t)=r_i$ for $t\in[t_i,t_{i+1})$;
    \item \textbf{Cosine decay:}
\end{itemize}
\begin{small}
$$
r(t)=
\begin{cases}
    r_{\mathrm{start}}, \qquad\qquad\qquad\qquad\qquad\qquad\qquad 0\le t<T_h,\\[4pt]
    r_{\mathrm{end}} + \bigl(r_{\mathrm{start}} - r_{\mathrm{end}}\bigr)\Bigl(1+\cos\,\!\bigl(\pi\,\tfrac{t-T_h}{T_d}\bigr)\Bigr)/2, \quad \text{o.w.}
\end{cases}
$$
\end{small}
where $T$ denotes the total number of steps, $T_h$ the warm-start period, and $T_d=\max\bigl(1,\,T-T_h\bigr)$ (see~\Cref{fig:rank_schedulers}). 

\begin{algorithm}[t]
\caption{\textcolor{red}{Muon} / \textcolor{ForestGreen}{NuMuon} Optimizer}\label{alg:muon_numuon_v2}
\begin{algorithmic}[1]
\STATE \textbf{Input:} Initial parameters $\boldsymbol{W}_0$, learning rate $\gamma$, momentum $\beta$\textcolor{ForestGreen}{, rank $k$}.
\STATE \textbf{Initialize:} Momentum buffer $\boldsymbol{M}_0 = \boldsymbol{0}$.
\FOR{$t=0, 1, 2, \dots$}
    \STATE Compute gradient $\boldsymbol{G}_t = \nabla_W f(\boldsymbol{W}_t)$.
    \STATE Update momentum $\boldsymbol{M}_t = \beta \boldsymbol{M}_{t-1} + (1-\beta)\boldsymbol{G}_t$.
    \IF{\textcolor{red}{Muon}}
        \STATE Compute SVD: $\boldsymbol{M}_t = \boldsymbol{U} \boldsymbol{S} \boldsymbol{V}^\top$.
        \STATE \textcolor{red}{$\bar{\boldsymbol{M}}_t = \boldsymbol{U}\boldsymbol{V}^\top$} \INLINECOMMENT{\textcolor{red}{\textit{full rank (approx. via Newton-Schulz)}}}
    \ELSIF{\textcolor{ForestGreen}{NuMuon}}
        \STATE Compute SVD: $\boldsymbol{M}_t = \boldsymbol{U} \boldsymbol{S} \boldsymbol{V}^\top$.
        \STATE \textcolor{ForestGreen}{$\bar{\boldsymbol{M}}_t = \boldsymbol{U}_{:,1:k}\boldsymbol{V}_{:,1:k}^\top$} \INLINECOMMENT{\textit{\textcolor{ForestGreen}{top-$k$ (approx via Block Krylov SVD)}}}
    \ENDIF
    \STATE Update weights $\boldsymbol{W}_{t+1} = \boldsymbol{W}_t - \gamma \bar{\boldsymbol{M}}_t$.
\ENDFOR
\end{algorithmic}
\end{algorithm}

\begin{algorithm}[t]
\caption{Newton-Schulz Iterative Algorithm for Matrix Orthogonalization}\label{alg:newton_schulz}
\begin{algorithmic}[1]
    \STATE \textbf{Input:} Matrix $\boldsymbol{A} \in \mathbb{R}^{n \times m}$, iterations $K$, hyperparameters $a, b, c \in \mathbb{R}$.
    \STATE \textbf{Initialize:} $\boldsymbol{A}^{(0)} = \boldsymbol{A} / \|\boldsymbol{A}\|_F$.
    \FOR{$k=0, 1, \dots, K-1$}
        \STATE $\boldsymbol{A}^{(k+1)} = a\boldsymbol{A}^{(k)} + b(\boldsymbol{A}^{(k)}{\boldsymbol{A}^{(k)}}^\top)\boldsymbol{A}^{(k)} + c(\boldsymbol{A}^{(k)}{\boldsymbol{A}^{(k)}}^\top)^2 \boldsymbol{A}^{(k)}$.
    \ENDFOR
    \STATE \textbf{Return:} $\boldsymbol{A}^{(K)}$.
\end{algorithmic}
\end{algorithm}

\begin{algorithm}[t]
\caption{Randomized Block Krylov Method for Approximate Top-$k$ SVD~\citep{musco2015bksvd}}
\label{alg:block_krylov}
\begin{algorithmic}[1]
\STATE \textbf{Input:} Matrix $\boldsymbol{A}\in\mathbb{R}^{m\times n}$; target rank $k$; block size $b\ge k$; Krylov iters $L$; (optional) warm-start block $\boldsymbol{B}_0\in\mathbb{R}^{n\times b}$
\STATE \textbf{Initialize:}
\IF{$\boldsymbol{B}_0$ not provided}
    \STATE Sample $\boldsymbol{B}_0 \sim \mathcal{N}(0,1)^{n\times b}$
\ENDIF
\STATE $\boldsymbol{B}_0 \gets \mathrm{qr}(\boldsymbol{B}_0)$

\STATE $\boldsymbol{K} \gets [\ ]$ \INLINECOMMENT{Build orthonormal basis for Krylov subspace}
\FOR{$i=1$ to $L$}
    \STATE $\boldsymbol{T}_i \gets \boldsymbol{A}\boldsymbol{B}_{i-1}$
    \STATE $\boldsymbol{B}_i \gets \boldsymbol{A}^\top \boldsymbol{T}_i$
    \STATE $\boldsymbol{B}_i \gets \mathrm{qr}(\boldsymbol{B}_i)$
    \STATE $\boldsymbol{K} \gets [\boldsymbol{K}\ \ \boldsymbol{B}_i]$
\ENDFOR
\STATE $\boldsymbol{Q} \gets \mathrm{qr}(\boldsymbol{K})$

\STATE $\boldsymbol{T} \gets \boldsymbol{A}\boldsymbol{Q}$ \INLINECOMMENT{Project and compute small SVD}
\STATE Compute thin SVD: $\boldsymbol{T}=\boldsymbol{U}_T \boldsymbol{S}_T \boldsymbol{V}_T^\top$

\STATE $\boldsymbol{U}_k \gets \boldsymbol{U}_T[:,1\!:\!k]$ \INLINECOMMENT{Extract top-$k$ components and lift back}
\STATE $\boldsymbol{S}_k \gets \boldsymbol{S}_T[1\!:\!k,1\!:\!k]$
\STATE $\boldsymbol{V}_k \gets \boldsymbol{Q}\,\boldsymbol{V}_T[:,1\!:\!k]$

\STATE \textbf{Return:} $(\boldsymbol{U}_k,\boldsymbol{S}_k,\boldsymbol{V}_k)$
\end{algorithmic}
\end{algorithm}

\section{Extended Experimental Results}
\label{app:extended_results}
In this section, we present the full set of experimental results omitted from the main paper due to space constraints.

\subsection{Experiment Settings and Hyperparameters}
\label{app:extended_results:details}

\paragraph{Pretraining.} We pretrain three LLMs of different sizes on FineWeb-EDU~\citep{penedo2024fineweb}: Qwen3-0.6B~\citep{qwen3}, Olmo2-1.4B~\citep{olmo2}, and Llama3-1.8B~\citep{dubey2024llama3}. We use the \texttt{torchtitan} library for pretraining on a cluster with 8 \texttt{NVIDIA A100-SXM4-40GB} GPUs,\footnote{Note that for benchmarking the generation throughput in~\Cref{fig:svdllm_efficiency}, we use a single \texttt{NVIDIA A100-SXM4-40GB} GPU.} and train each model with AdamW~\citep{loshchilov2019adamw}, Muon~\citep{jordan2024muon}, and our approach, NuMuon. All models start with a linear warmup. For AdamW, we use a cosine learning rate decay while for Muon and NuMuon we use WSD~\citep{hu2024minicpm}, following the benchmarking study of~\citet{semenov2025benchmarking} which recommends these schedules for best performance. Please see~\Cref{fig:lr_scheduler} for the details of the rank scheduler. For Muon and NuMuon, we also apply weight decay of 0.1. Finally, for NuMuon we typically use a cosine rank scheduler with a 10\% constant period at the beginning of training. We set this scheduler to reach the final rank before the WSD cooldown stage. We visualize the rank schedule for Qwen3-0.6B in~\Cref{fig:rank_schedulers} and provide full pretraining hyperparameters in~\Cref{tab:model_config,tab:optimizer_config}. 

\paragraph{LLM Compression.}
After pretraining, we compress the resulting checkpoints using three SVD-based compression methods: ASVD~\citep{yuan2023asvd}, SVD-LLM~\citep{wang2025svdllm}, and Dobi-SVD~\citep{wang2025dobisvd}. We use the official default settings provided in each repository\footnote{ASVD: \texttt{https://github.com/hahnyuan/ASVD4LLM}.\\SVD-LLM: \texttt{https://github.com/AIoT-MLSys-Lab/SVD-LLM}.\\Dobi-SVD: \texttt{https://github.com/wangqinsi1/Dobi-SVD}.} and evaluate compressed models on WikiText2 validation perplexity and standard downstream benchmarks (ARC-Easy/Challenge, HellaSwag, LAMBADA (OpenAI), OpenbookQA, PIQA, and Winogrande). For SVD-LLM, we evaluate both the whitening and whitening + LoRA variants; in cases where the LoRA stage was unstable, we report the whitening results for both settings for completeness and note this in the corresponding tables.

\tabModelConfig
\tabOptimizerConfig

\figRankFraction

\subsection{Extended Results}
\label{app:extended_results:results}

\subsubsection{Training Convergence}

\paragraph{Loss Curves.} We report training loss trajectories for AdamW, Muon, and NuMuon on all models in~\Cref{fig:loss_curves}, illustrating that NuMuon largely tracks Muon with a small late-stage deviation.

\paragraph{Stable-rank Dynamics.} We also compare stable-rank evolution under each optimizer for the Qwen3-0.6B model in~\Cref{fig:stable_rank_comparison_qwen3_all}, highlighting how rank controled updates induce lower effective rank during training. Furthermore, we report the stable rank of each weight matrix in the transformer blocks against the layer depth in~\Cref{fig:stable_rank_vs_layer_comparison_qwen3_all,fig:stable_rank_vs_layer_comparison_olmo2_all,fig:stable_rank_vs_layer_comparison_llama3_all}. As shown, NuMuon produces weight matrices with consistently lower rank than Muon, making it compression friendly.

\paragraph{Subspace Alignment.} We report the Grassmann distance between the top-$k$ right-singular subspaces of $\mathbf{W}$ and $\mathbf{\Delta W}$ (here $k=64$) in~\Cref{fig:grassman_distance}, demonstrating improved update--weight subspace alignment for NuMuon. For more information, please see~\Cref{sec:experimental_results}.

\figLRSchedulesLarge

\figAdamwVsMuonStableRankLarge

\figStableRankvsLayerQwenLarge
\figStableRankvsLayerOlmoLarge
\figStableRankvsLayerLlamaLarge

\figStableRankWellKnown

\figGrassmannQwenLarge

\subsubsection{LLM Compression}
We summarize WikiText2 validation perplexity versus compression rate across all methods and models in~\Cref{fig:compression_ppl_comparison}. We see that NuMuon-trained weights degrade more gracefully under aggressive compression rates, illustrating that their lower rank weights are more compressible than Muon. We also report the validation perplexity and detailed downstream benchmark results in~\Cref{tab:qwen3_asvd,tab:qwen3_whiten,tab:qwen3_svdllm,tab:qwen3_dobisvd,tab:olmo2_asvd,tab:olmo2_whiten,tab:olmo2_svdllm,tab:olmo2_dobisvd,tab:llama3_asvd,tab:llama3_whiten,tab:llama3_svdllm,tab:llama3_dobisvd}.

\tabQwenASVDFull
\tabOlmoASVDFull
\tabLlamaASVDFull

\tabQwenWhitenFull
\tabOlmoWhitenFull
\tabLlamaWhitenFull

\tabQwenSVDLLMFull
\tabOlmoSVDLLMFull
\tabLlamaSVDLLMFull

\tabQwenDobiSVDFull
\tabOlmoDobiSVDFull
\tabLlamaDobiSVDFull

\figPPLCompAllMethods

\figSVDLLMEfficiencyGains

\subsection{Extended Ablation Studies}
\label{app:extended_results:ablation}

\paragraph{Rank-scheduler Ablations.}
\Cref{tab:ablation} reports the effect of rank scheduler choice (cosine, piecewise, and fixed) on training efficiency and robustness to aggressive compression (80\%) using Dobi-SVD. Consistent with our discussion in the main paper, fixed-rank schedules are typically fastest per step since they avoid an early high-rank Block Krylov SVD regime~(see~\Cref{fig:rank_schedulers}), but can trade off base performance depending on how restrictive the rank budget is. In contrast, cosine and piecewise schedules benefit from a higher-rank phase early in training and then anneal to a lower-rank regime, which tends to yield better base performance while still improving compressibility relative to full-rank Muon. Across schedulers, NuMuon consistently mitigates the severe degradation observed for Muon at 80\% compression, indicating that explicit rank control induces weight structure that is easier to approximate with SVD-based compressors.

\tabAblationRankSched

\paragraph{Layerwise Stable-rank under Rank-budget Ablations.}
\Cref{fig:stable_rank_vs_layer_comparison_qwen3_budget_ablation} visualizes how changing the rank budget affects the layerwise stable rank of the final model. As expected, tighter budgets produce uniformly lower stable rank across layers and weight matrices, while relaxing the budget increases stable rank and approaches full-rank behavior.

\figStableRankvsLayerAblationQwenLarge

\paragraph{Layerwise Stable-rank under Scheduler Ablations.}
\Cref{fig:stable_rank_vs_layer_comparison_qwen3_scheduler_ablation} compares scheduler-induced differences in layerwise stable rank. Cosine and piecewise schedules typically yield lower stable rank than full-rank Muon while maintaining stronger base performance than the fixed-rank settings, reflecting the benefit of a high-rank warm-start followed by annealing. Taking these results and the loss curves in~\Cref{fig:loss_vs_rank_sched} together, these ablations support the central trade-off in NuMuon: selecting a rank budget and schedule that is low enough to improve compressibility, yet not so restrictive that it impairs optimization.

\figStableRankvsLayerAblationSchedulerQwenLarge

\end{document}

%% file: figures.tex
\graphicspath{{./figures/}}

\newcommand{\figAdamwVsMuonStableRank}{
\begin{figure}[t!]
    \centering
    \begin{subfigure}[b]{0.33\textwidth}
        \centering
        \includegraphics[width=\textwidth]{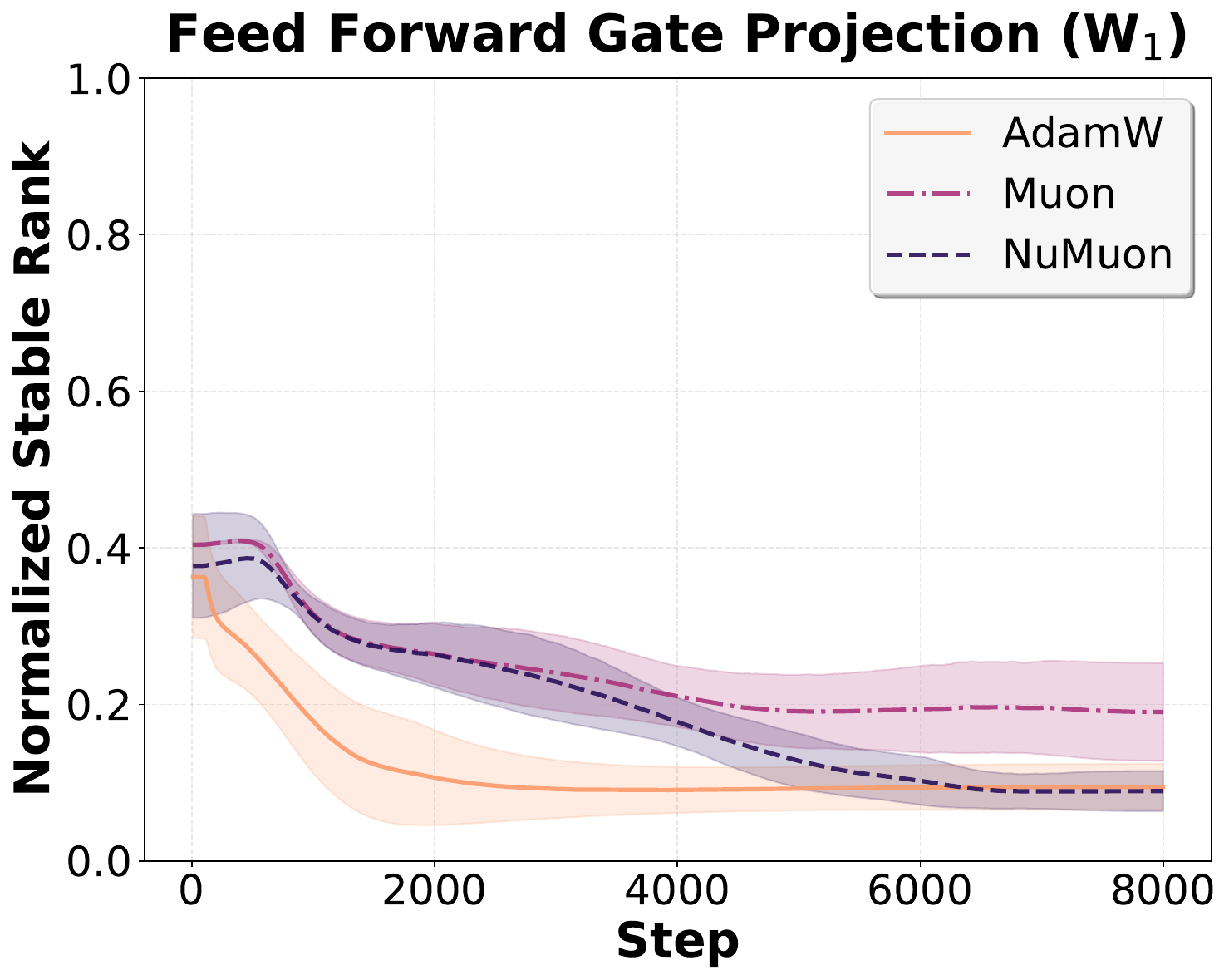}
        \caption{FFN Gate Projection}
        \label{fig:stable_rank_comparison_qwen3:stable_rank_w1}
    \end{subfigure}
    \hspace{-0.5em}
    \begin{subfigure}[b]{0.33\textwidth}
        \centering
        \includegraphics[width=\textwidth]{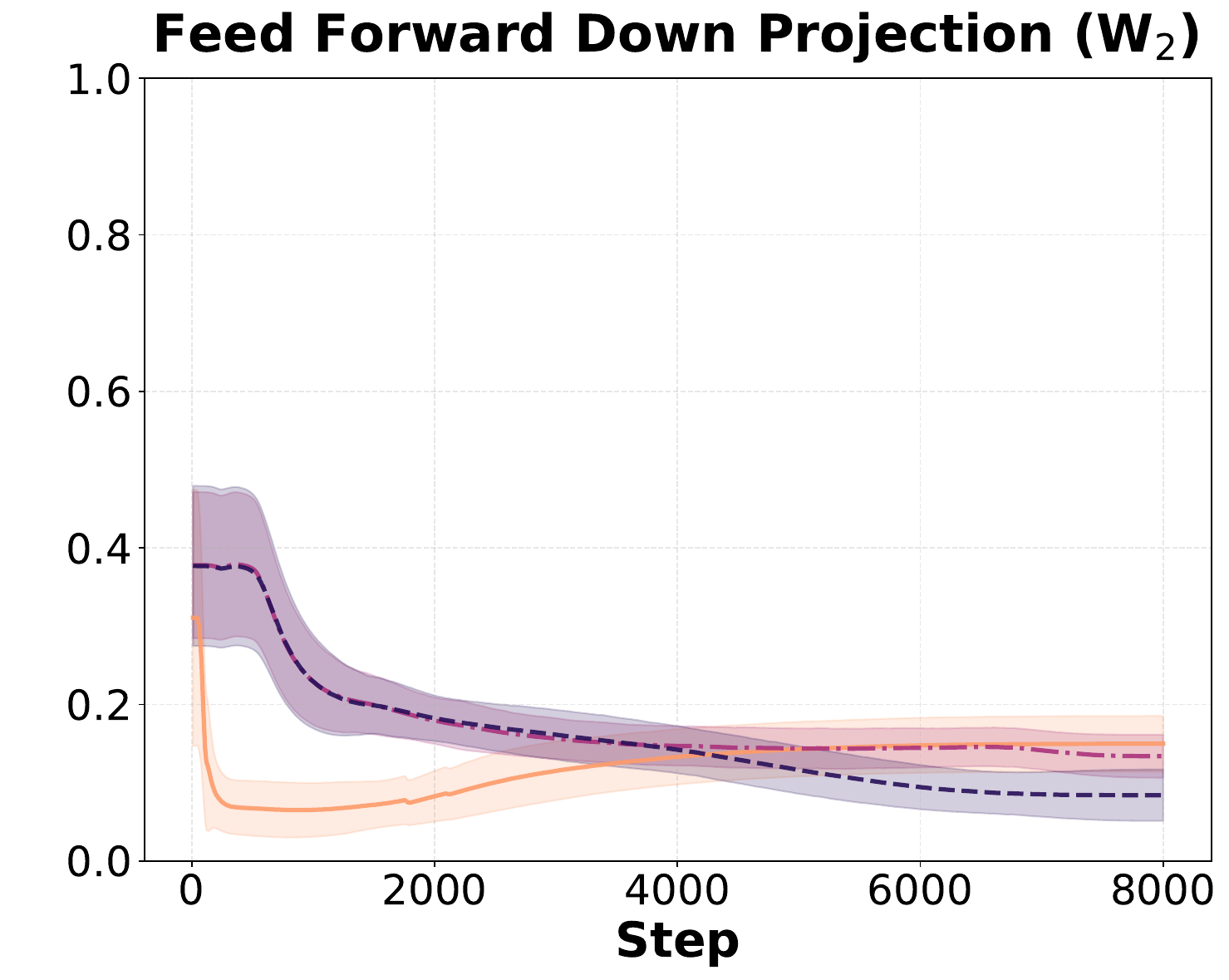}
        \caption{FFN Down Projection}
        \label{fig:stable_rank_comparison_qwen3:stable_rank_w2}
    \end{subfigure}
    \hspace{-0.5em}
    \begin{subfigure}[b]{0.33\textwidth}
        \centering
        \includegraphics[width=\textwidth]{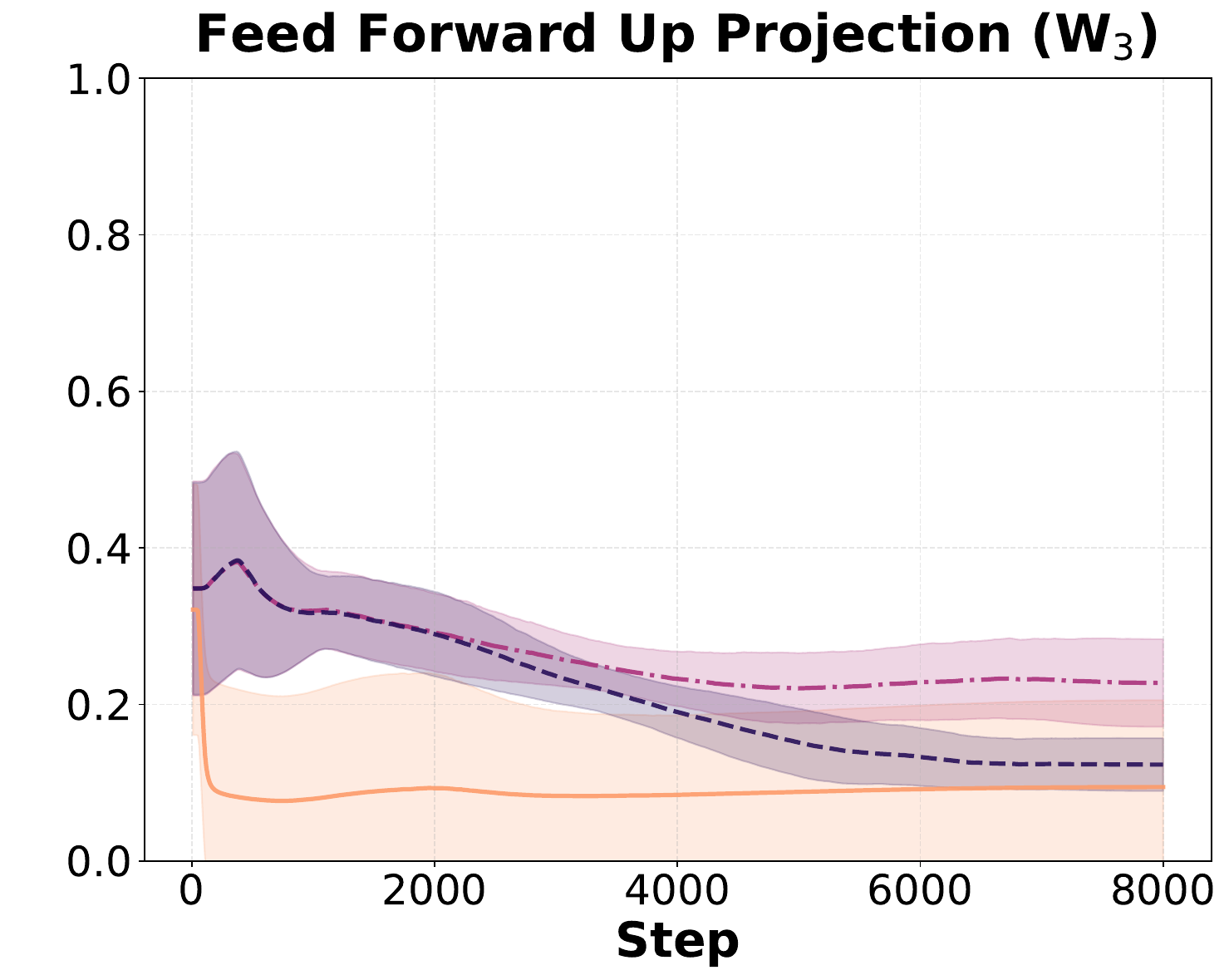}
        \caption{FFN Up Projection}
        \label{fig:stable_rank_comparison_qwen3:stable_rank_w3}
    \end{subfigure}
    \caption{Normalized stable rank evolution for Qwen3-0.6B across training steps for feedforward projection matrices. Each subplot shows the mean stable rank (normalized by the maximum rank), with shaded regions indicating standard deviation across all layers. All other weight matrices exhibit a similar low-rank behavior throughout training~(see~\Cref{fig:stable_rank_comparison_qwen3_all} in the Appendix).}
    \label{fig:stable_rank_comparison_qwen3}
    
\end{figure}
}

\newcommand{\figAdamwVsMuonStableRankLarge}{
\begin{figure*}[t!]
    \centering
    \begin{subfigure}[b]{0.32\textwidth}
        \centering
        \includegraphics[width=\textwidth]{alignment_metrics/adamw_vs_muon_vs_rmuon_qwen3_0.6B_small/stable_rank_W/feed_forward_w1.pdf}
        \caption{FFN Gate Projection}
        \label{fig:stable_rank_comparison_qwen3_all:stable_rank_w1}
    \end{subfigure}
    \hspace{-0.5em}
    \begin{subfigure}[b]{0.32\textwidth}
        \centering
        \includegraphics[width=\textwidth]{alignment_metrics/adamw_vs_muon_vs_rmuon_qwen3_0.6B_small/stable_rank_W/feed_forward_w2.pdf}
        \caption{FFN Down Projection}
        \label{fig:stable_rank_comparison_qwen3_all:stable_rank_w2}
    \end{subfigure}
    \hspace{-0.5em}
    \begin{subfigure}[b]{0.32\textwidth}
        \centering
        \includegraphics[width=\textwidth]{alignment_metrics/adamw_vs_muon_vs_rmuon_qwen3_0.6B_small/stable_rank_W/feed_forward_w3.pdf}
        \caption{FFN Up Projection}
        \label{fig:stable_rank_comparison_qwen3_all:stable_rank_w3}
    \end{subfigure} 
    \\\vspace{1em}
    \begin{subfigure}[b]{0.32\textwidth}
        \centering
        \includegraphics[width=\textwidth]{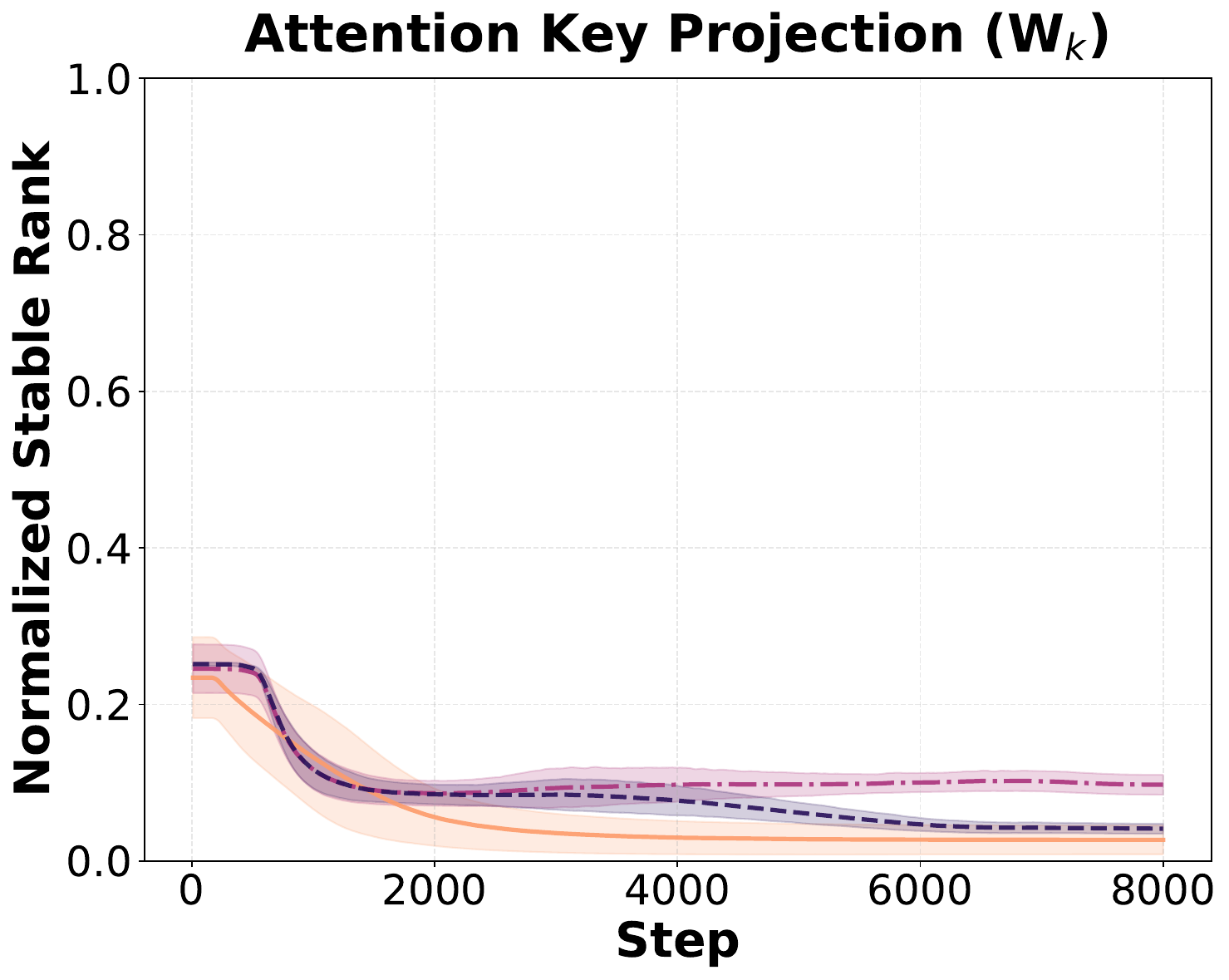}
        \caption{Attn Key Projection}
        \label{fig:stable_rank_comparison_qwen3_all:stable_rank_wk}
    \end{subfigure}
    \hspace{-0.5em}
    \begin{subfigure}[b]{0.32\textwidth}
        \centering
        \includegraphics[width=\textwidth]{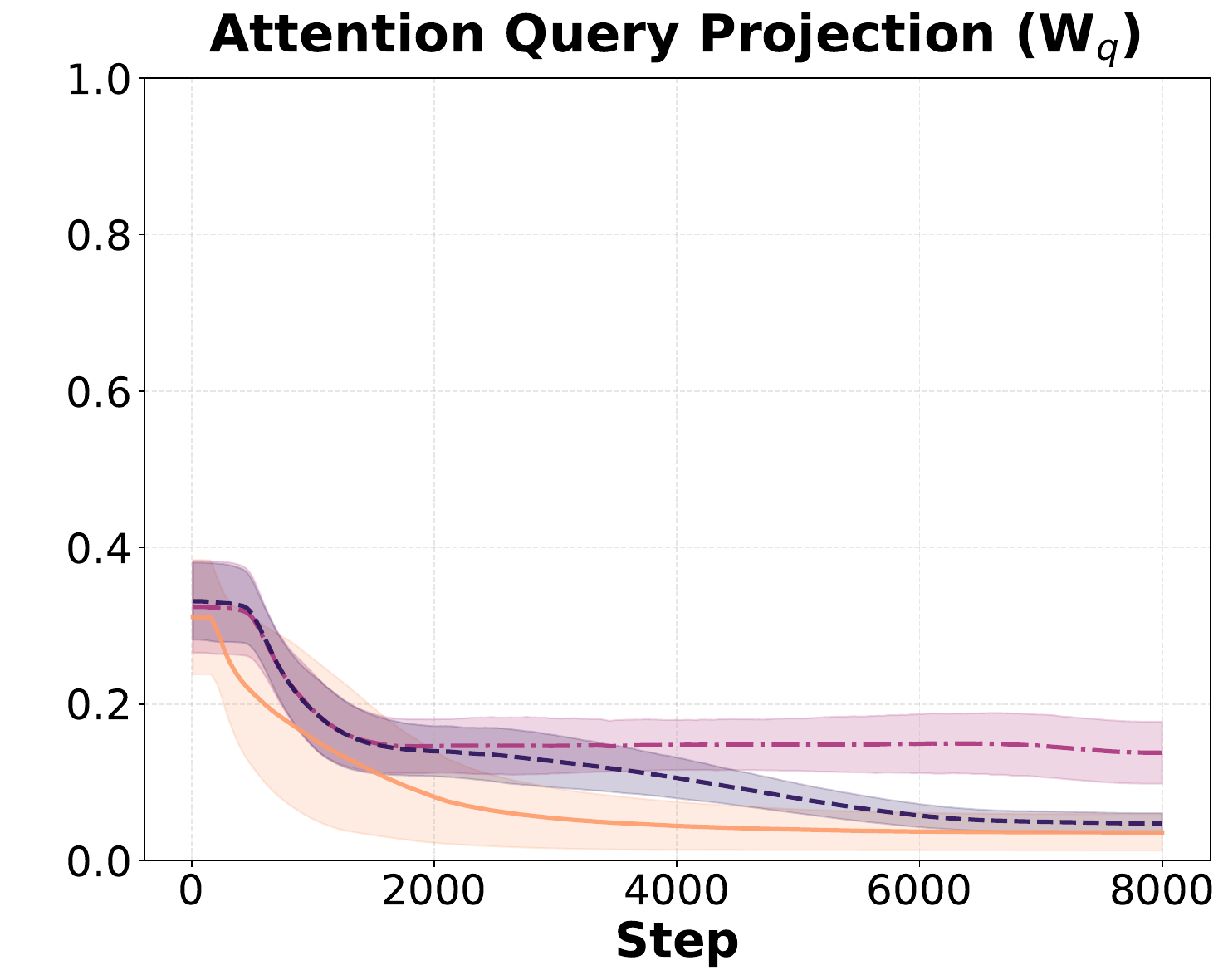}
        \caption{Attn Query Projection}
        \label{fig:stable_rank_comparison_qwen3_all:stable_rank_wq}
    \end{subfigure}
    \hspace{-0.5em}
    \begin{subfigure}[b]{0.32\textwidth}
        \centering
        \includegraphics[width=\textwidth]{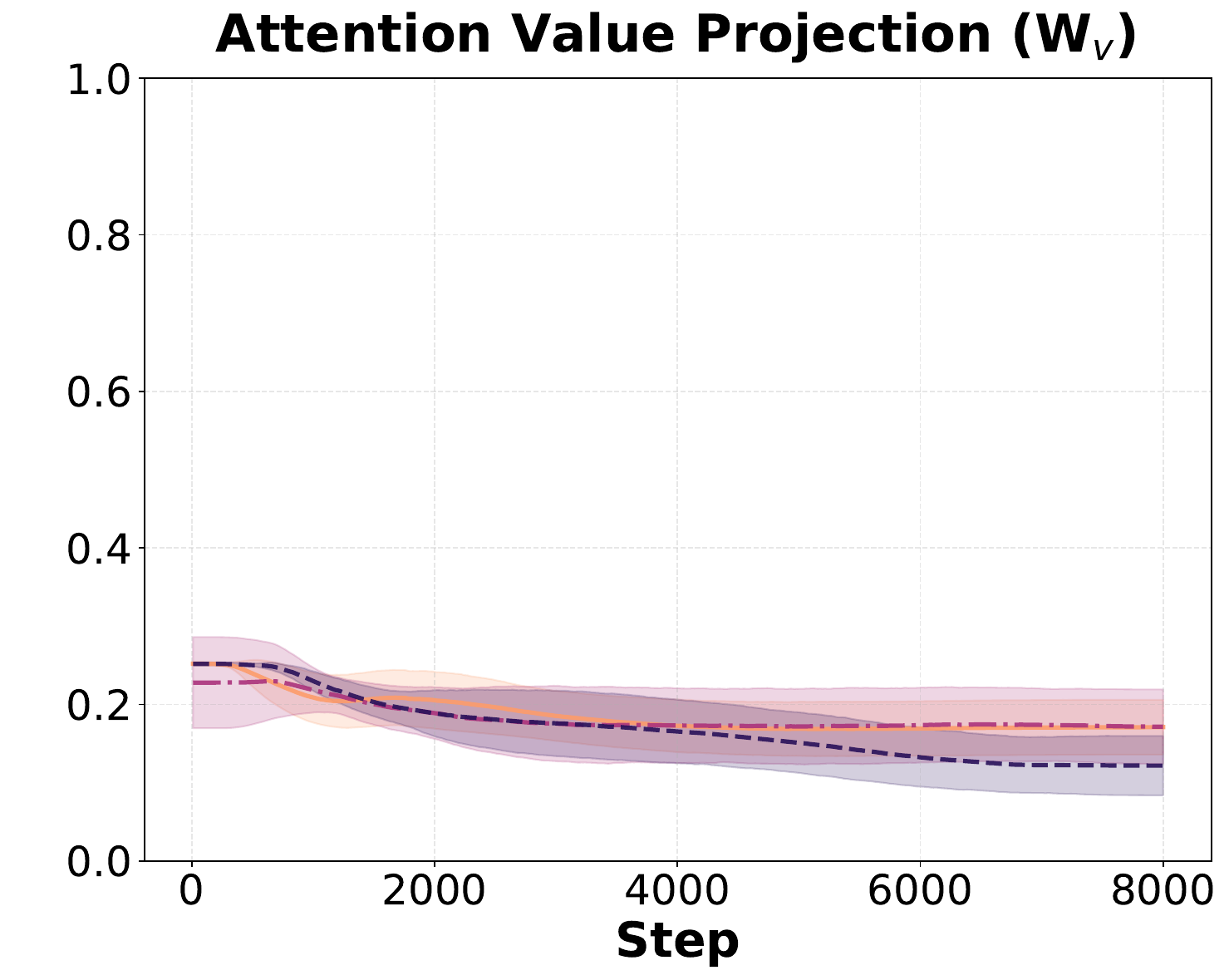}
        \caption{Attn Value Projection}
        \label{fig:stable_rank_comparison_qwen3_all:stable_rank_wv}
    \end{subfigure}
    \\\vspace{1em}
    \begin{subfigure}[b]{0.32\textwidth}
    \end{subfigure}
    \begin{subfigure}[b]{0.32\textwidth}
        \centering
        \includegraphics[width=\textwidth]{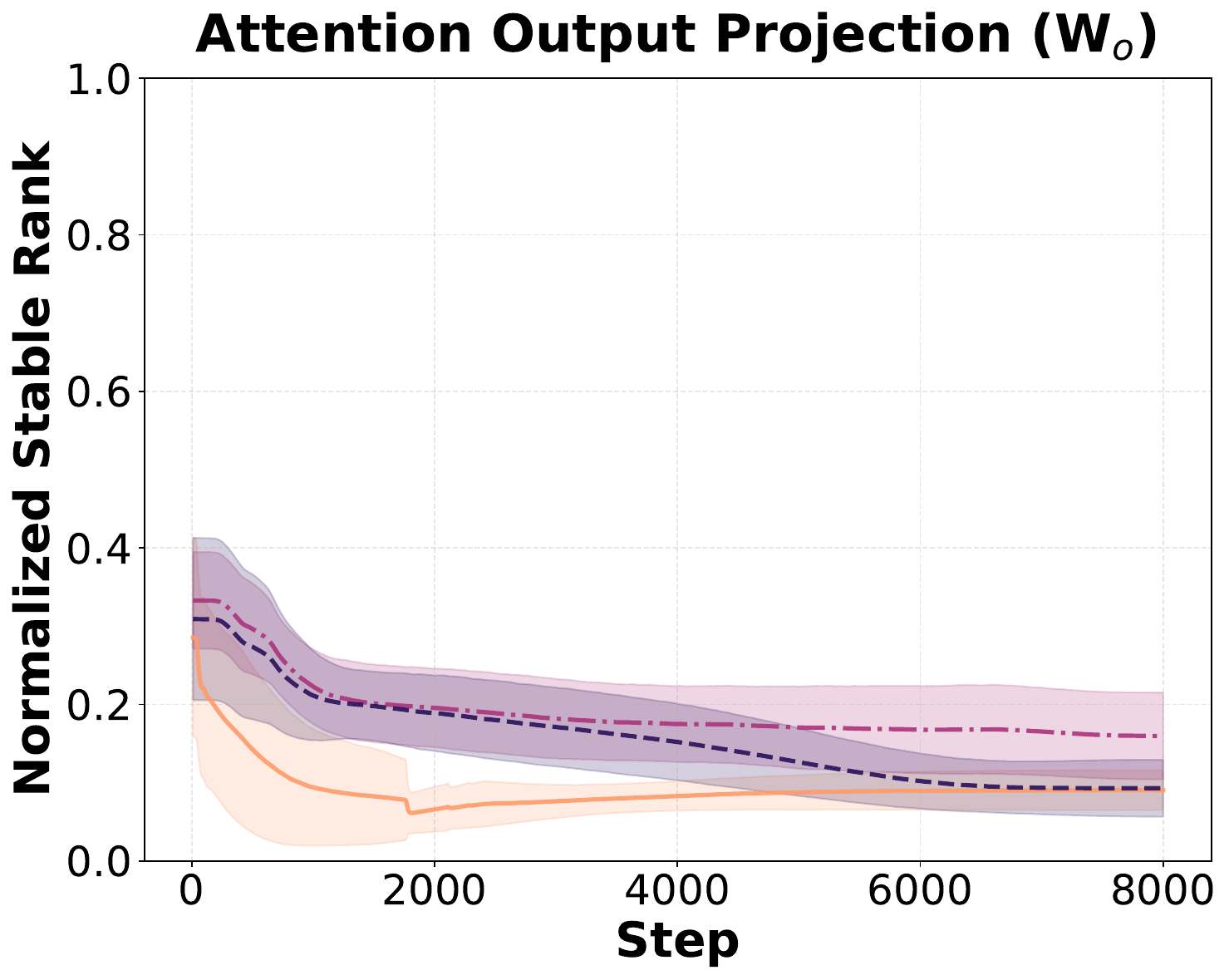}
        \caption{Attn Output Projection}
        \label{fig:stable_rank_comparison_qwen3_all:stable_rank_wo}
    \end{subfigure}
    \begin{subfigure}[b]{0.32\textwidth}
    \end{subfigure}
    \caption{Normalized stable rank evolution for Qwen3-0.6B across training steps for different weight matrices. Each subplot shows the mean stable rank (normalized by the maximum rank), with shaded regions indicating standard deviation across all layers.}
    \label{fig:stable_rank_comparison_qwen3_all}
\end{figure*}
}

\newcommand{\figStableRankvsLayerQwenLarge}{
\begin{figure*}[t!]
    \centering
    \begin{subfigure}[b]{0.32\textwidth}
        \centering
        \includegraphics[width=\textwidth]{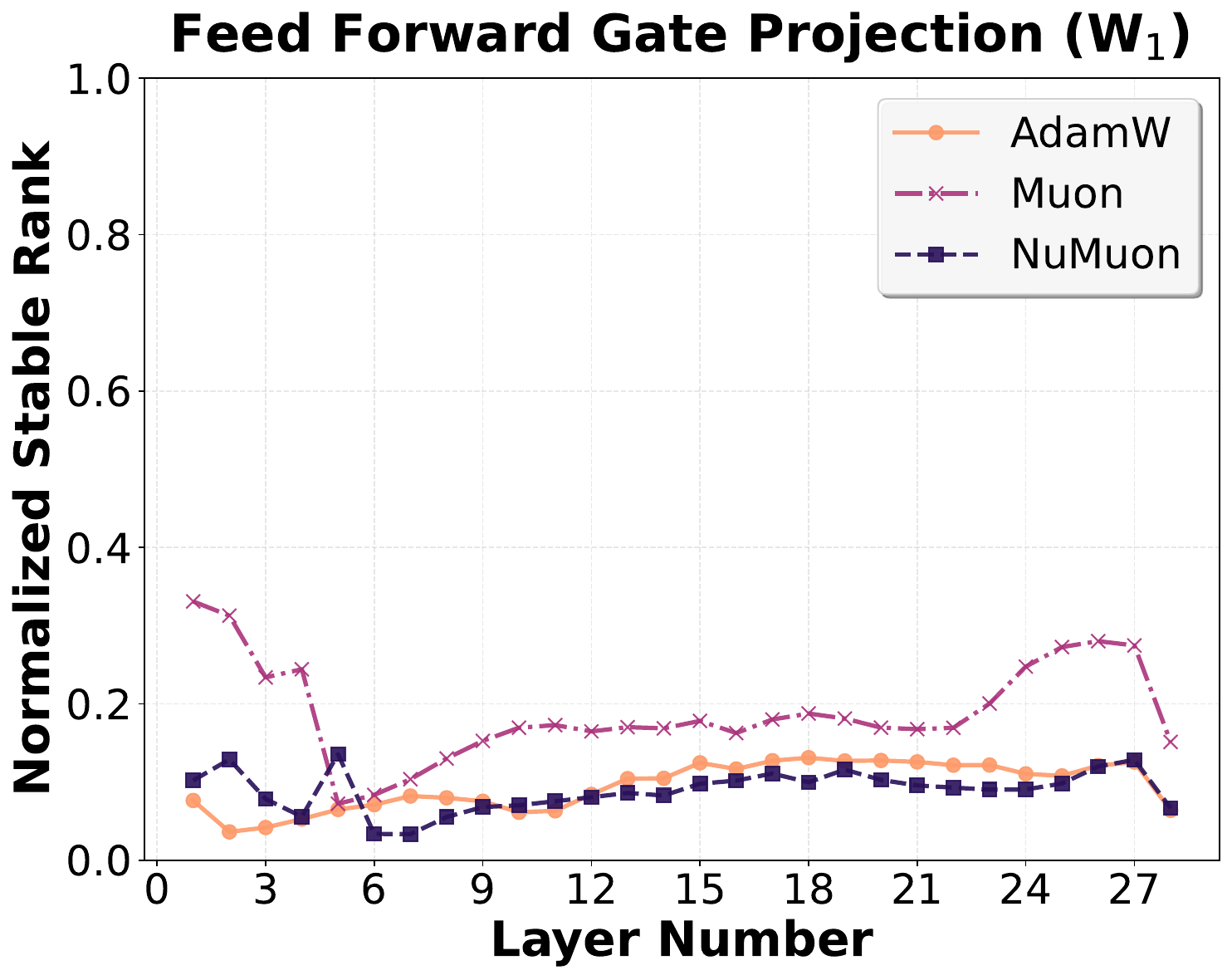}
        \caption{FFN Gate Projection}
        \label{fig:stable_rank_vs_layer_comparison_qwen3_all:stable_rank_w1}
    \end{subfigure}
    \hspace{-0.5em}
    \begin{subfigure}[b]{0.32\textwidth}
        \centering
        \includegraphics[width=\textwidth]{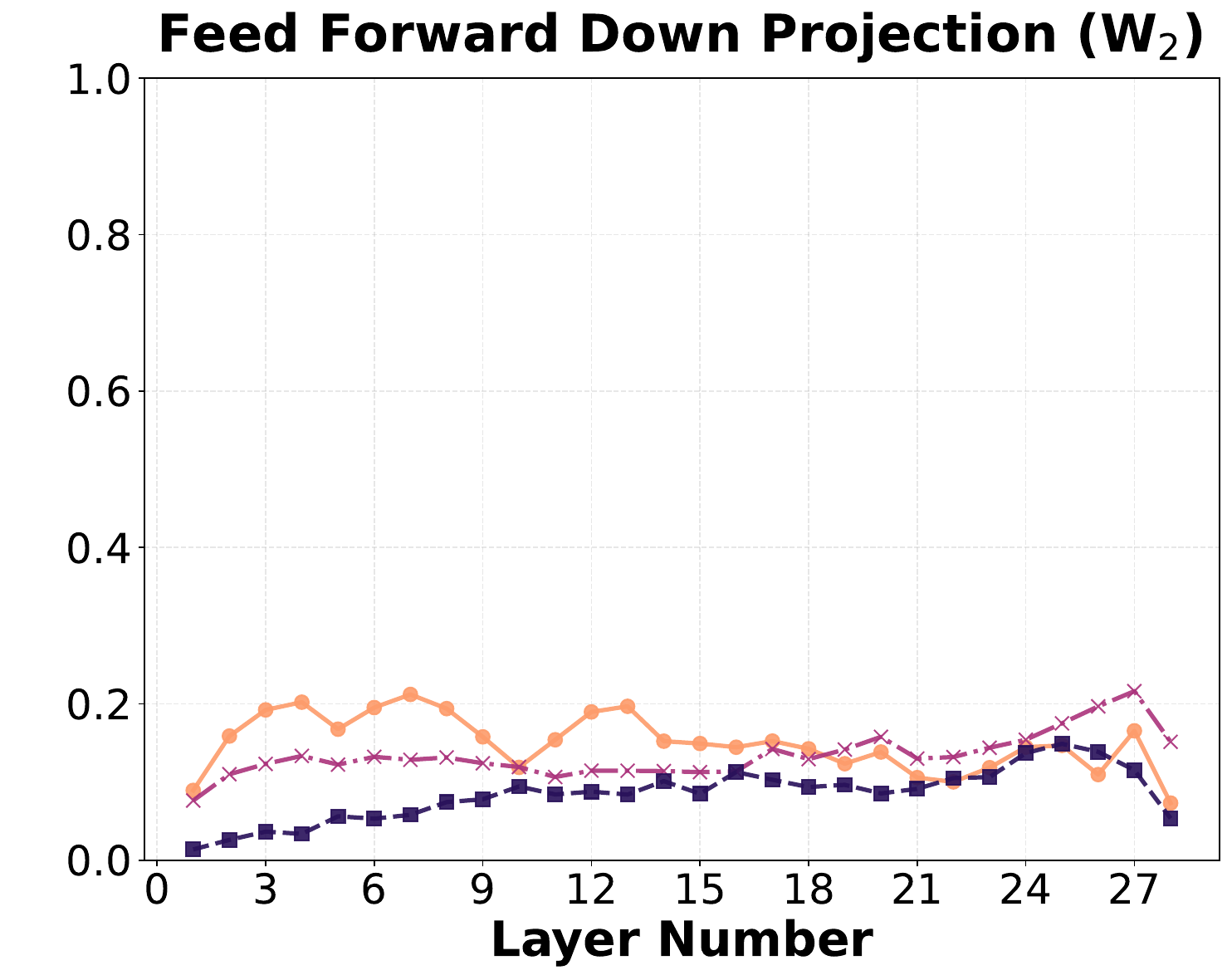}
        \caption{FFN Down Projection}
        \label{fig:stable_rank_vs_layer_comparison_qwen3_all:stable_rank_w2}
    \end{subfigure}
    \hspace{-0.5em}
    \begin{subfigure}[b]{0.32\textwidth}
        \centering
        \includegraphics[width=\textwidth]{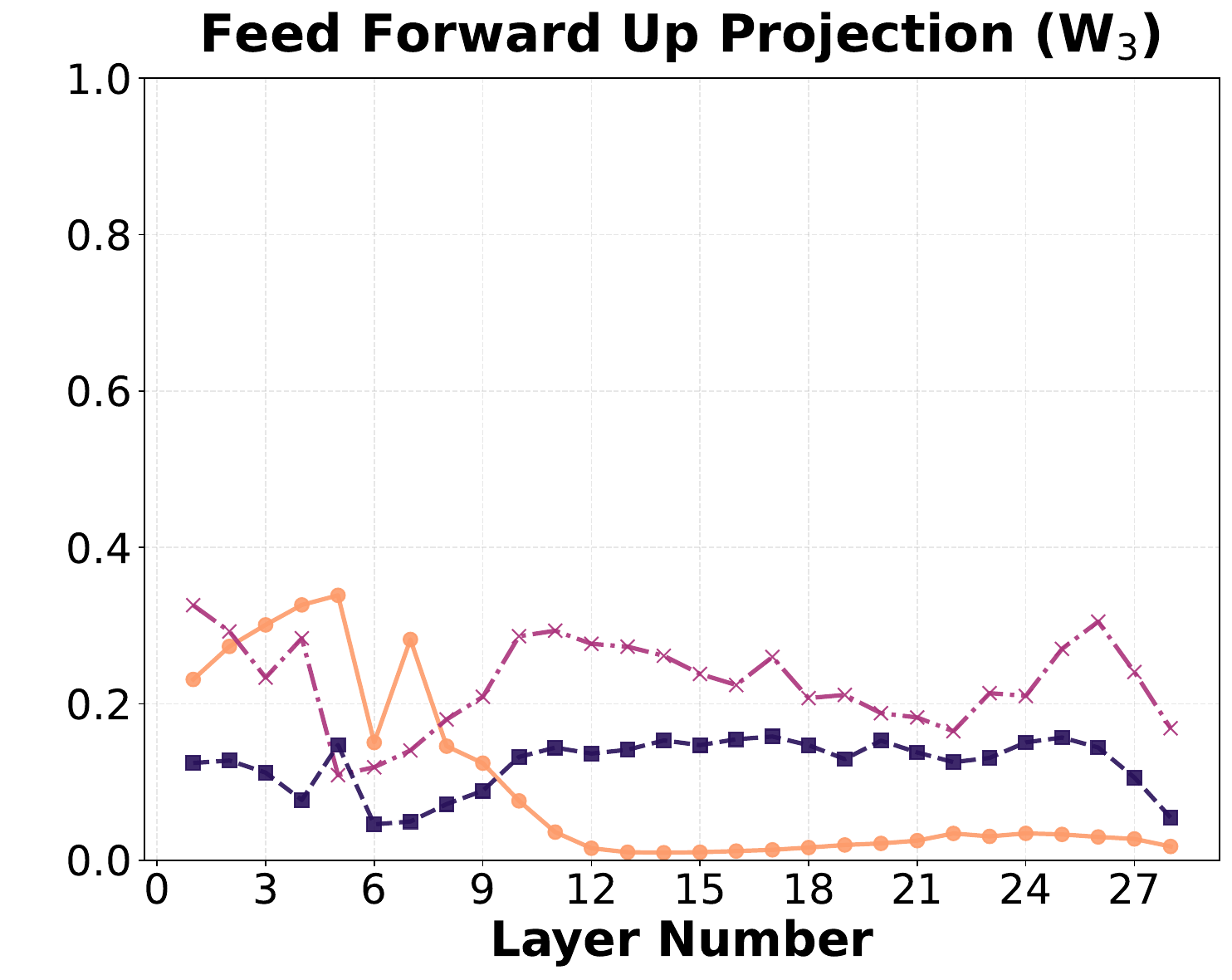}
        \caption{FFN Up Projection}
        \label{fig:stable_rank_vs_layer_comparison_qwen3_all:stable_rank_w3}
    \end{subfigure} 
    \\\vspace{1em}
    \begin{subfigure}[b]{0.32\textwidth}
        \centering
        \includegraphics[width=\textwidth]{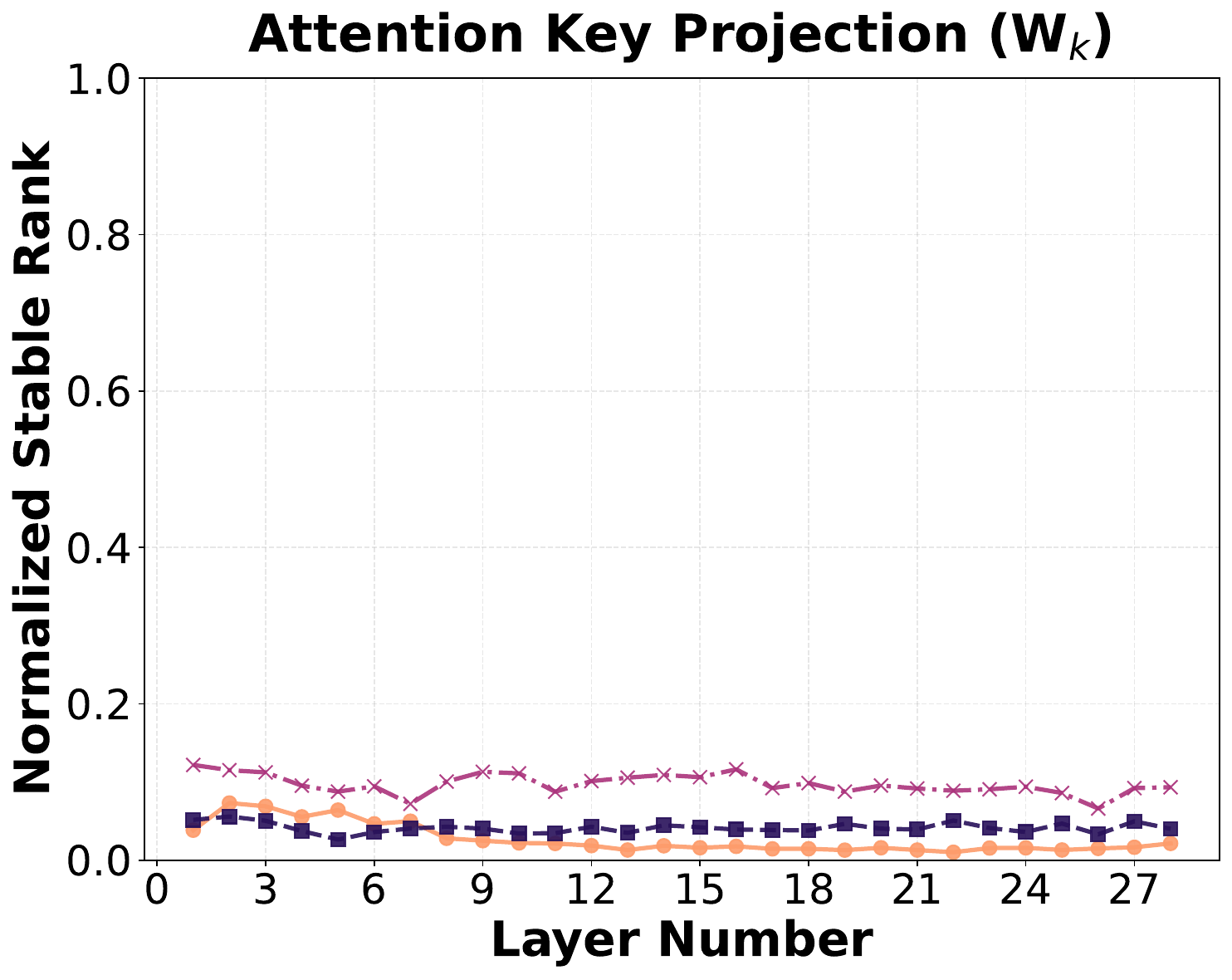}
        \caption{Attn Key Projection}
        \label{fig:stable_rank_vs_layer_comparison_qwen3_all:stable_rank_wk}
    \end{subfigure}
    \hspace{-0.5em}
    \begin{subfigure}[b]{0.32\textwidth}
        \centering
        \includegraphics[width=\textwidth]{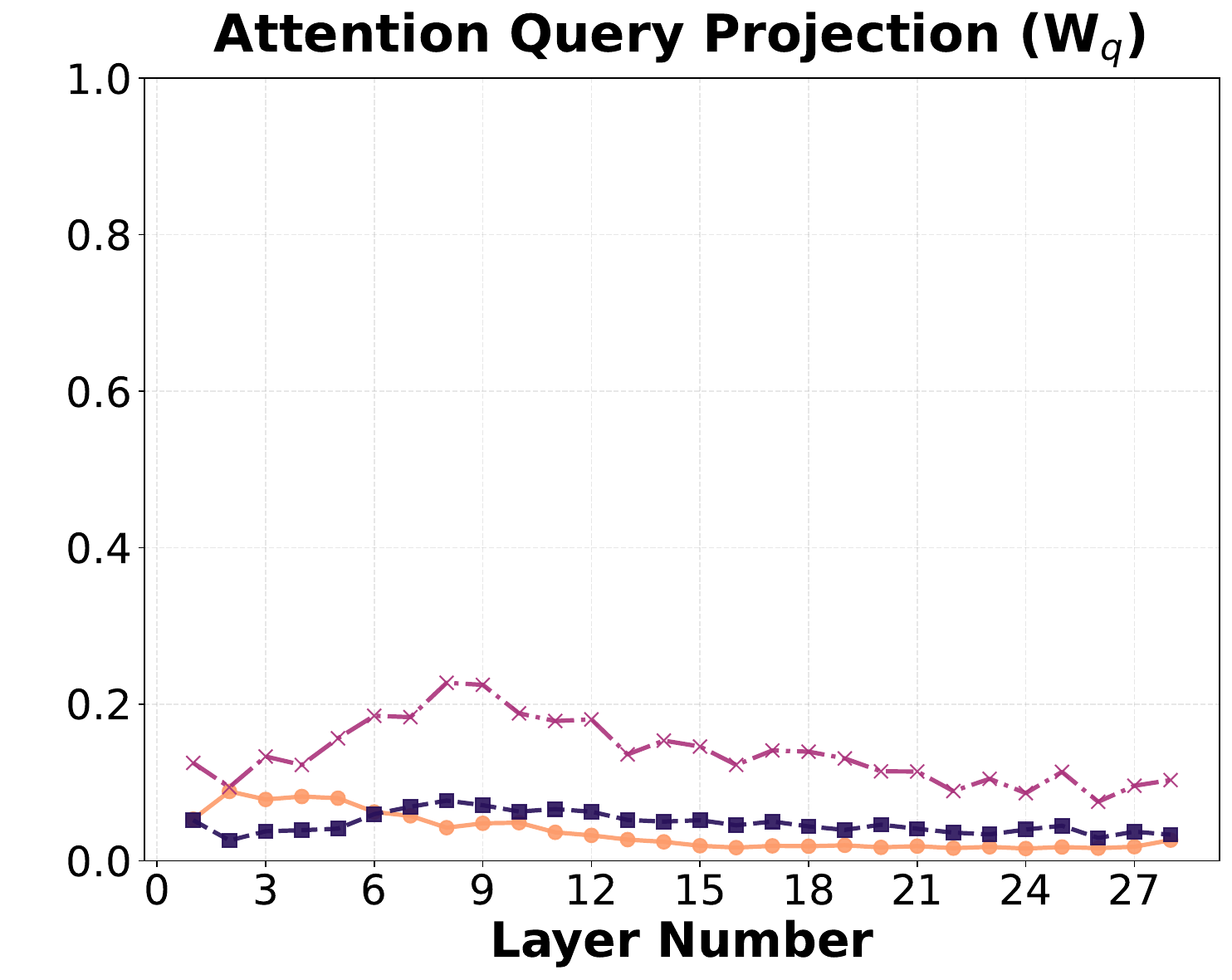}
        \caption{Attn Query Projection}
        \label{fig:stable_rank_vs_layer_comparison_qwen3_all:stable_rank_wq}
    \end{subfigure}
    \hspace{-0.5em}
    \begin{subfigure}[b]{0.32\textwidth}
        \centering
        \includegraphics[width=\textwidth]{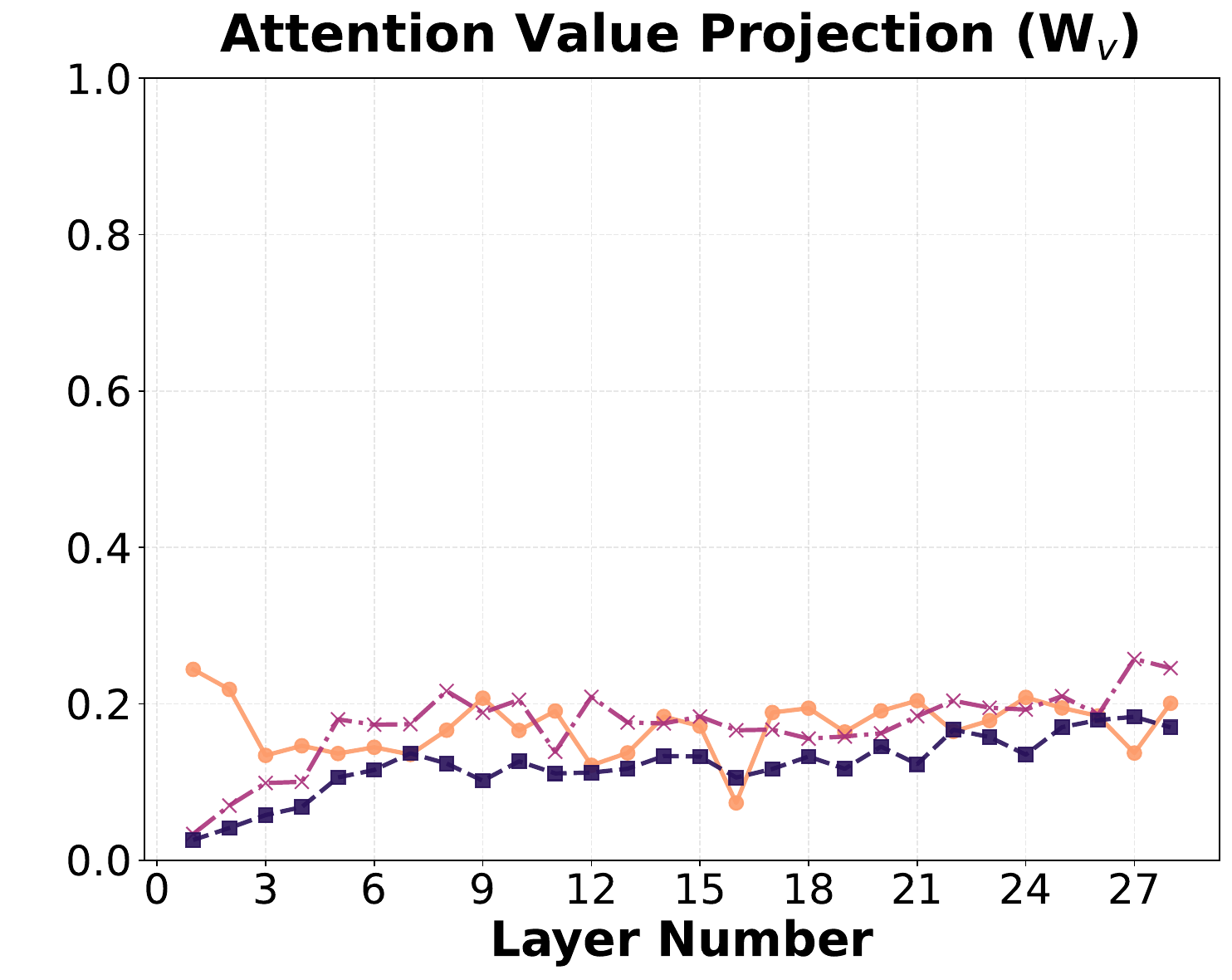}
        \caption{Attn Value Projection}
        \label{fig:stable_rank_vs_layer_comparison_qwen3_all:stable_rank_wv}
    \end{subfigure}
    \\\vspace{1em}
    \begin{subfigure}[b]{0.32\textwidth}
    \end{subfigure}
    \begin{subfigure}[b]{0.32\textwidth}
        \centering
        \includegraphics[width=\textwidth]{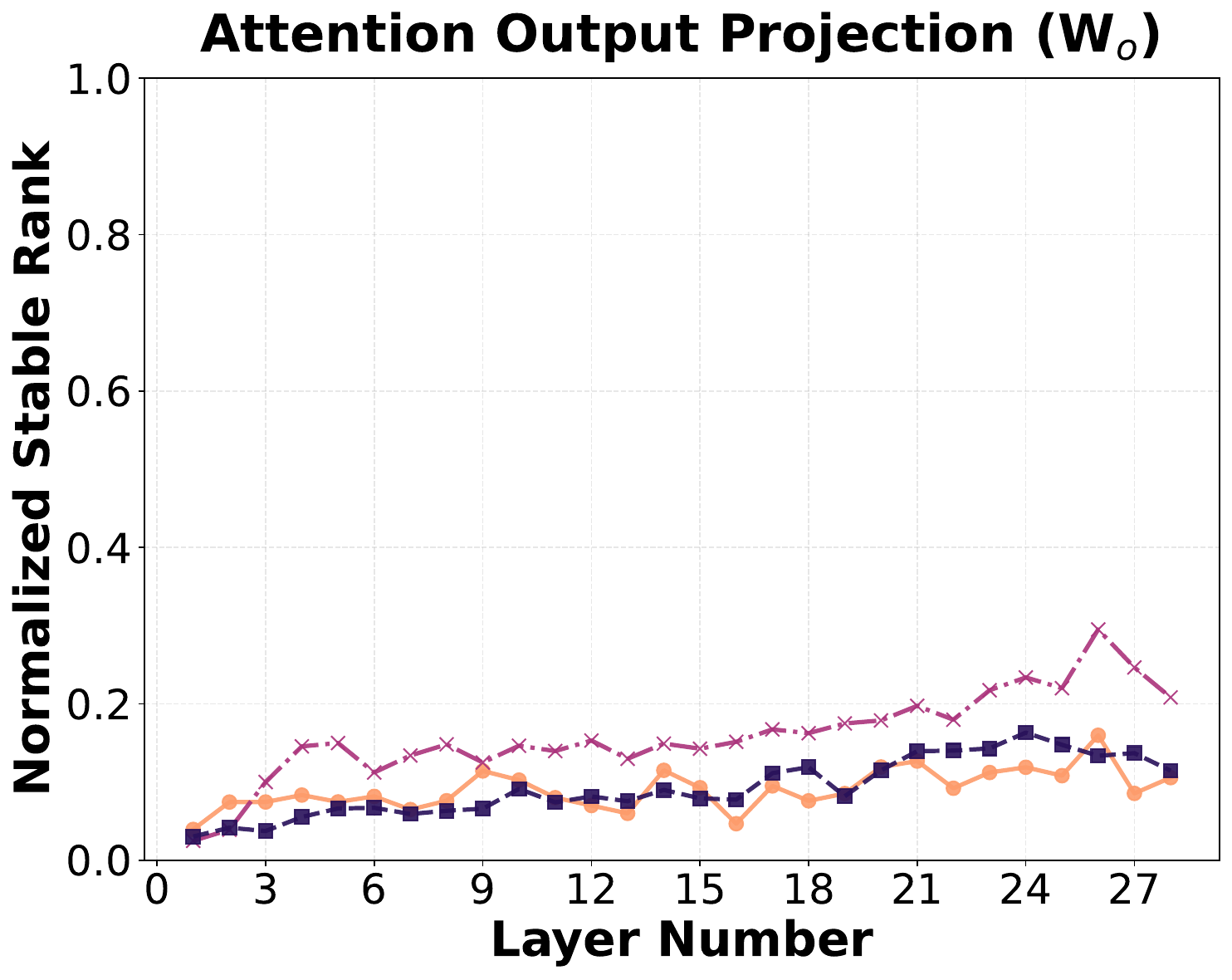}
        \caption{Attn Output Projection}
        \label{fig:stable_rank_vs_layer_comparison_qwen3_all:stable_rank_wo}
    \end{subfigure}
    \begin{subfigure}[b]{0.32\textwidth}
    \end{subfigure}
    \caption{Normalized stable rank for weight matrices of Qwen3-0.6B models. Each subplot shows the stable rank (normalized by the maximum rank) of the converged model across all layers.}
    \label{fig:stable_rank_vs_layer_comparison_qwen3_all}
    \vspace{-1em}
\end{figure*}
}

\newcommand{\figStableRankvsLayerOlmoLarge}{
\begin{figure*}[t!]
    \centering
    \begin{subfigure}[b]{0.32\textwidth}
        \centering
        \includegraphics[width=\textwidth]{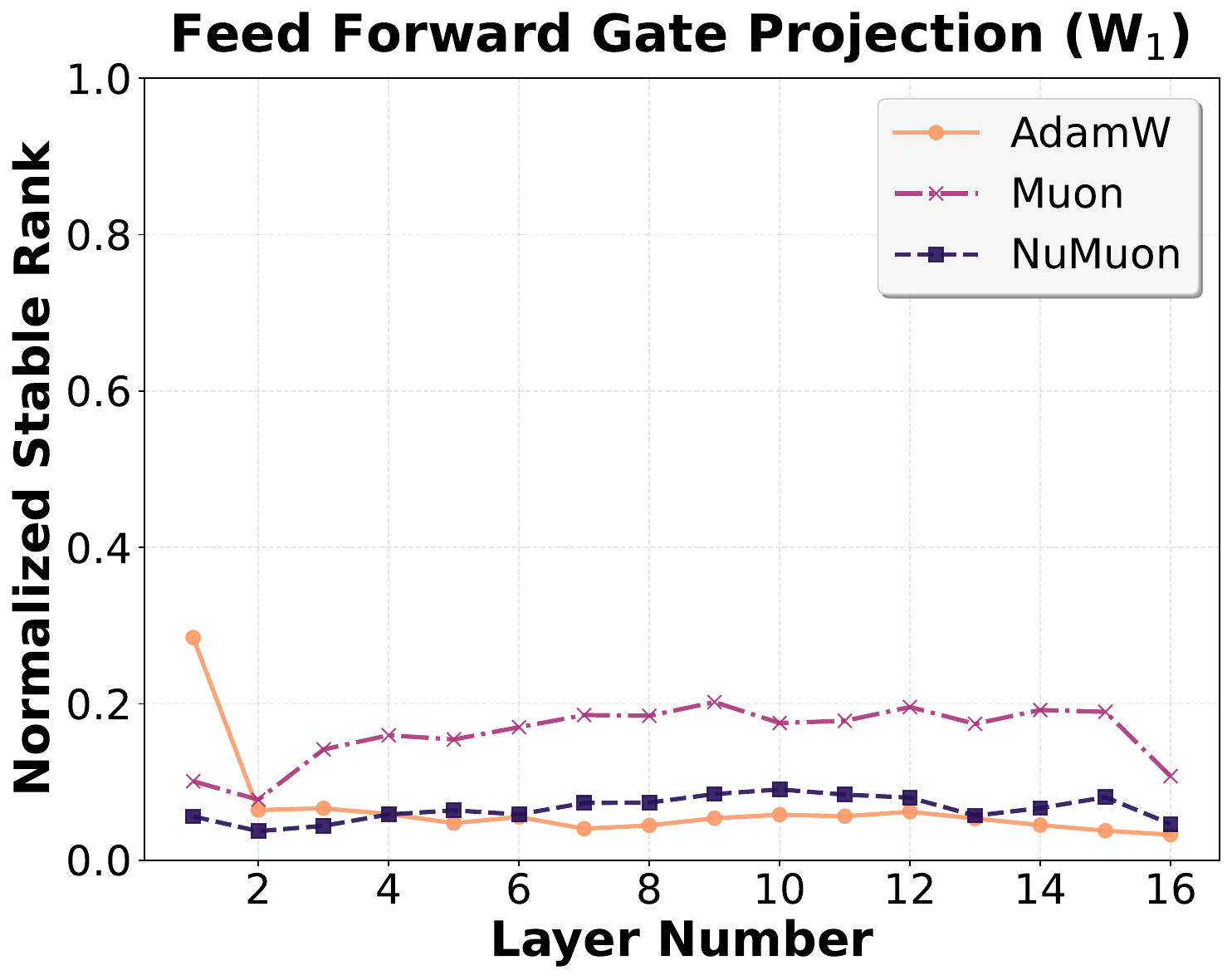}
        \caption{FFN Gate Projection}
        \label{fig:stable_rank_vs_layer_comparison_olmo2_all:stable_rank_w1}
    \end{subfigure}
    \hspace{-0.5em}
    \begin{subfigure}[b]{0.32\textwidth}
        \centering
        \includegraphics[width=\textwidth]{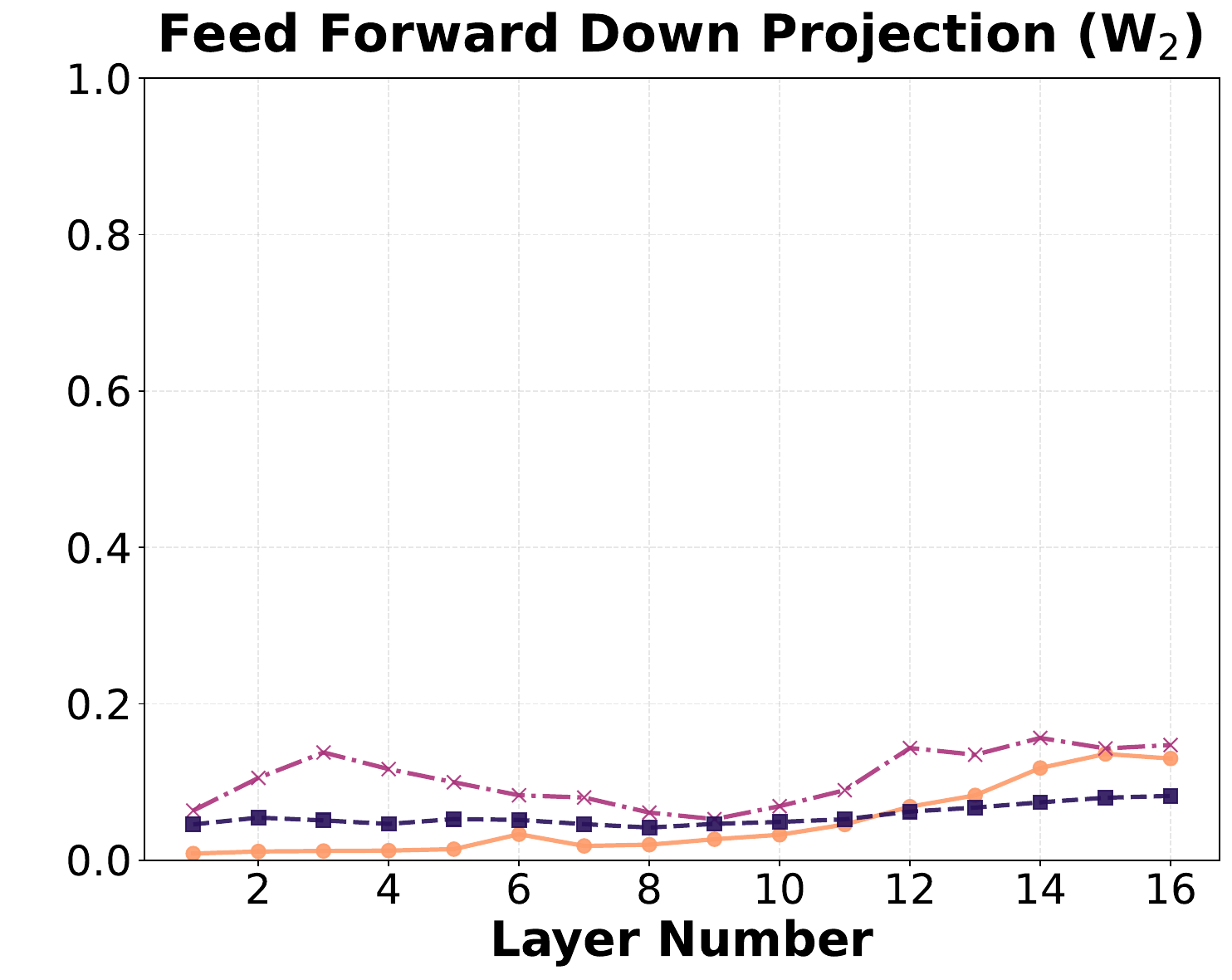}
        \caption{FFN Down Projection}
        \label{fig:stable_rank_vs_layer_comparison_olmo2_all:stable_rank_w2}
    \end{subfigure}
    \hspace{-0.5em}
    \begin{subfigure}[b]{0.32\textwidth}
        \centering
        \includegraphics[width=\textwidth]{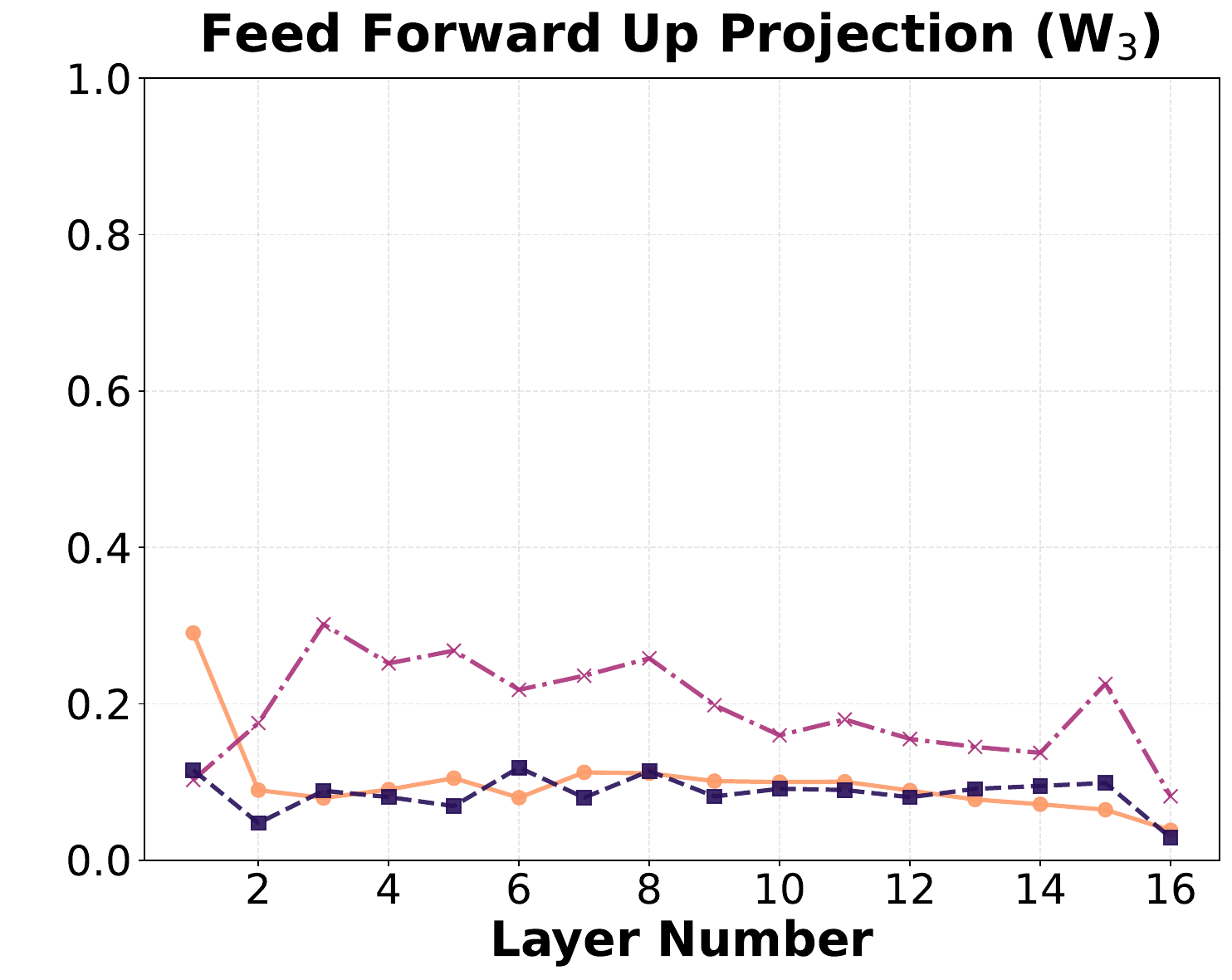}
        \caption{FFN Up Projection}
        \label{fig:stable_rank_vs_layer_comparison_olmo2_all:stable_rank_w3}
    \end{subfigure} 
    \\\vspace{1em}
    \begin{subfigure}[b]{0.32\textwidth}
        \centering
        \includegraphics[width=\textwidth]{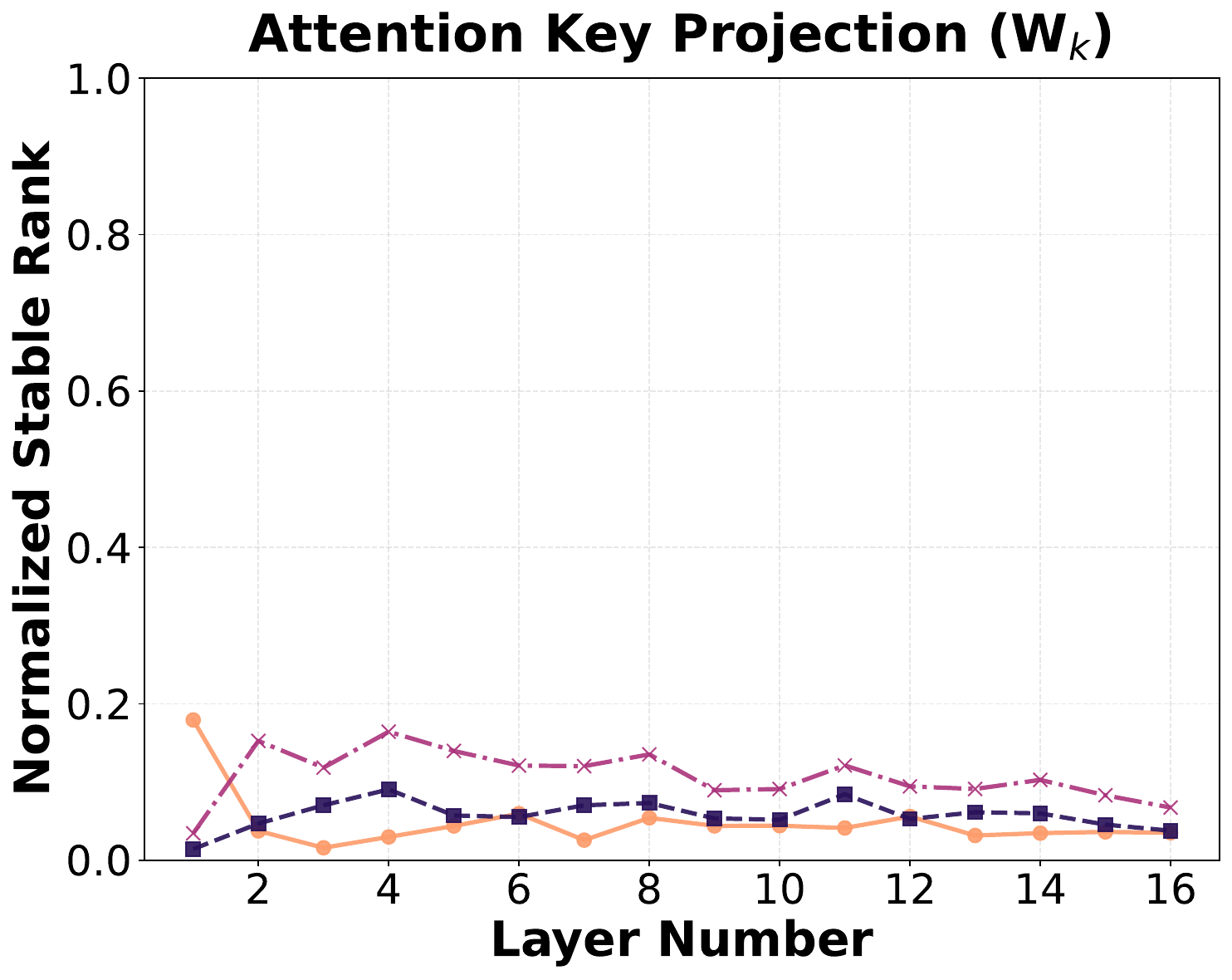}
        \caption{Attn Key Projection}
        \label{fig:stable_rank_vs_layer_comparison_olmo2_all:stable_rank_wk}
    \end{subfigure}
    \hspace{-0.5em}
    \begin{subfigure}[b]{0.32\textwidth}
        \centering
        \includegraphics[width=\textwidth]{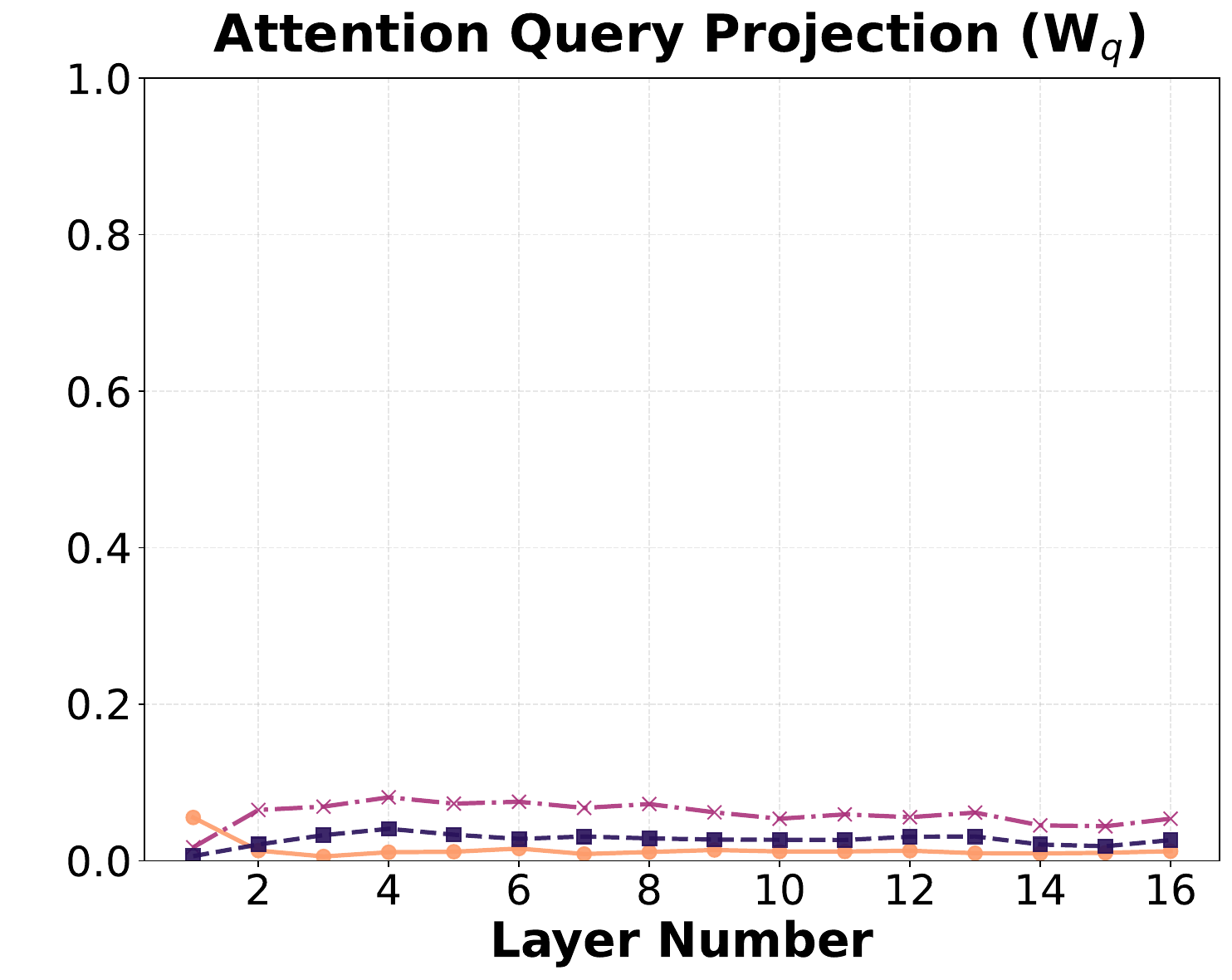}
        \caption{Attn Query Projection}
        \label{fig:stable_rank_vs_layer_comparison_olmo2_all:stable_rank_wq}
    \end{subfigure}
    \hspace{-0.5em}
    \begin{subfigure}[b]{0.32\textwidth}
        \centering
        \includegraphics[width=\textwidth]{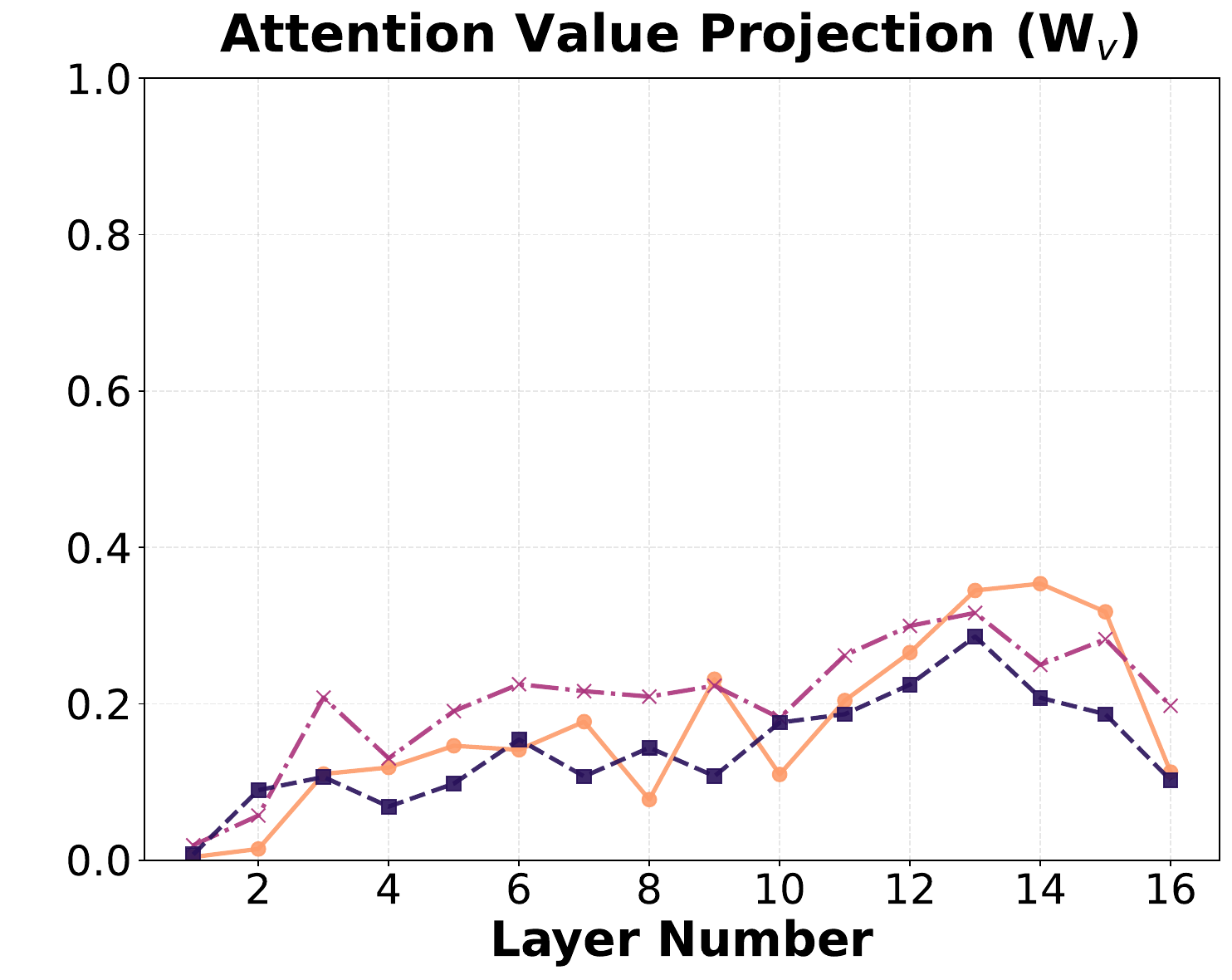}
        \caption{Attn Value Projection}
        \label{fig:stable_rank_vs_layer_comparison_olmo2_all:stable_rank_wv}
    \end{subfigure}
    \\\vspace{1em}
    \begin{subfigure}[b]{0.32\textwidth}
    \end{subfigure}
    \begin{subfigure}[b]{0.32\textwidth}
        \centering
        \includegraphics[width=\textwidth]{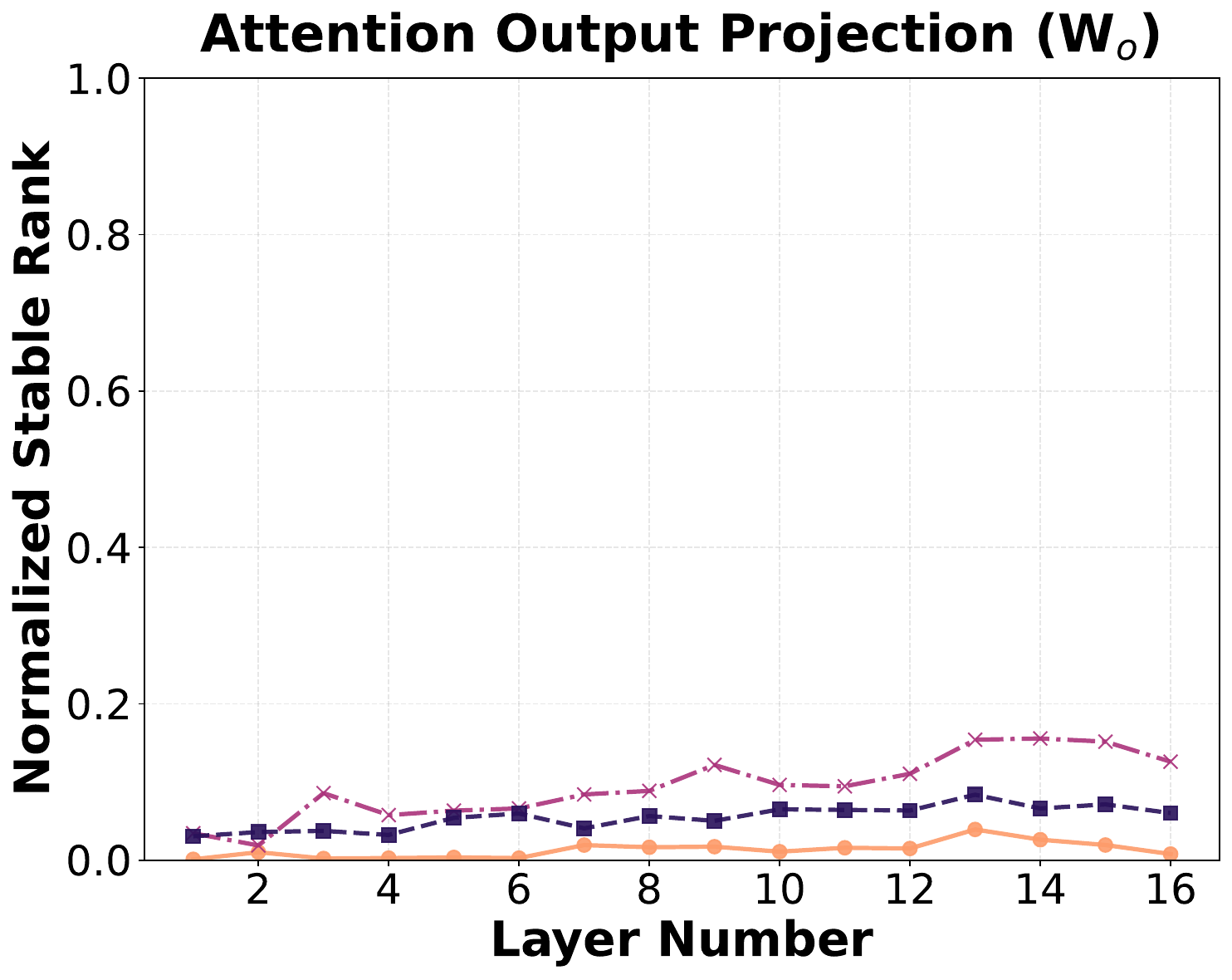}
        \caption{Attn Output Projection}
        \label{fig:stable_rank_vs_layer_comparison_olmo2_all:stable_rank_wo}
    \end{subfigure}
    \begin{subfigure}[b]{0.32\textwidth}
    \end{subfigure}
    \caption{Normalized stable rank for weight matrices of Olmo2-1.4B models. Each subplot shows the stable rank (normalized by the maximum rank) of the converged model across all layers.}
    \label{fig:stable_rank_vs_layer_comparison_olmo2_all}
\end{figure*}
}

\newcommand{\figStableRankvsLayerLlamaLarge}{
\begin{figure*}[t!]
    \centering
    \begin{subfigure}[b]{0.32\textwidth}
        \centering
        \includegraphics[width=\textwidth]{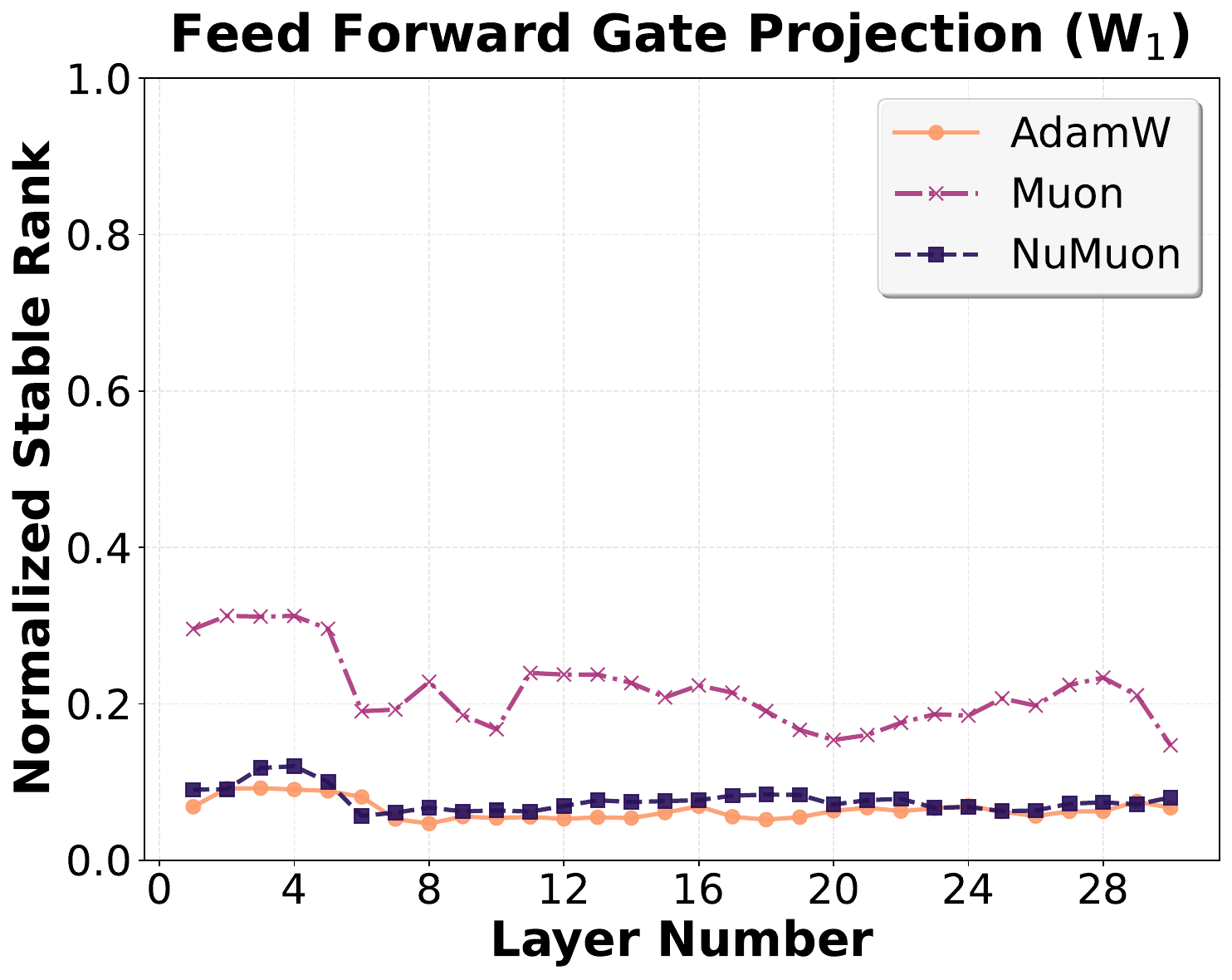}
        \caption{FFN Gate Projection}
        \label{fig:stable_rank_vs_layer_comparison_llama3_all:stable_rank_w1}
    \end{subfigure}
    \hspace{-0.5em}
    \begin{subfigure}[b]{0.32\textwidth}
        \centering
        \includegraphics[width=\textwidth]{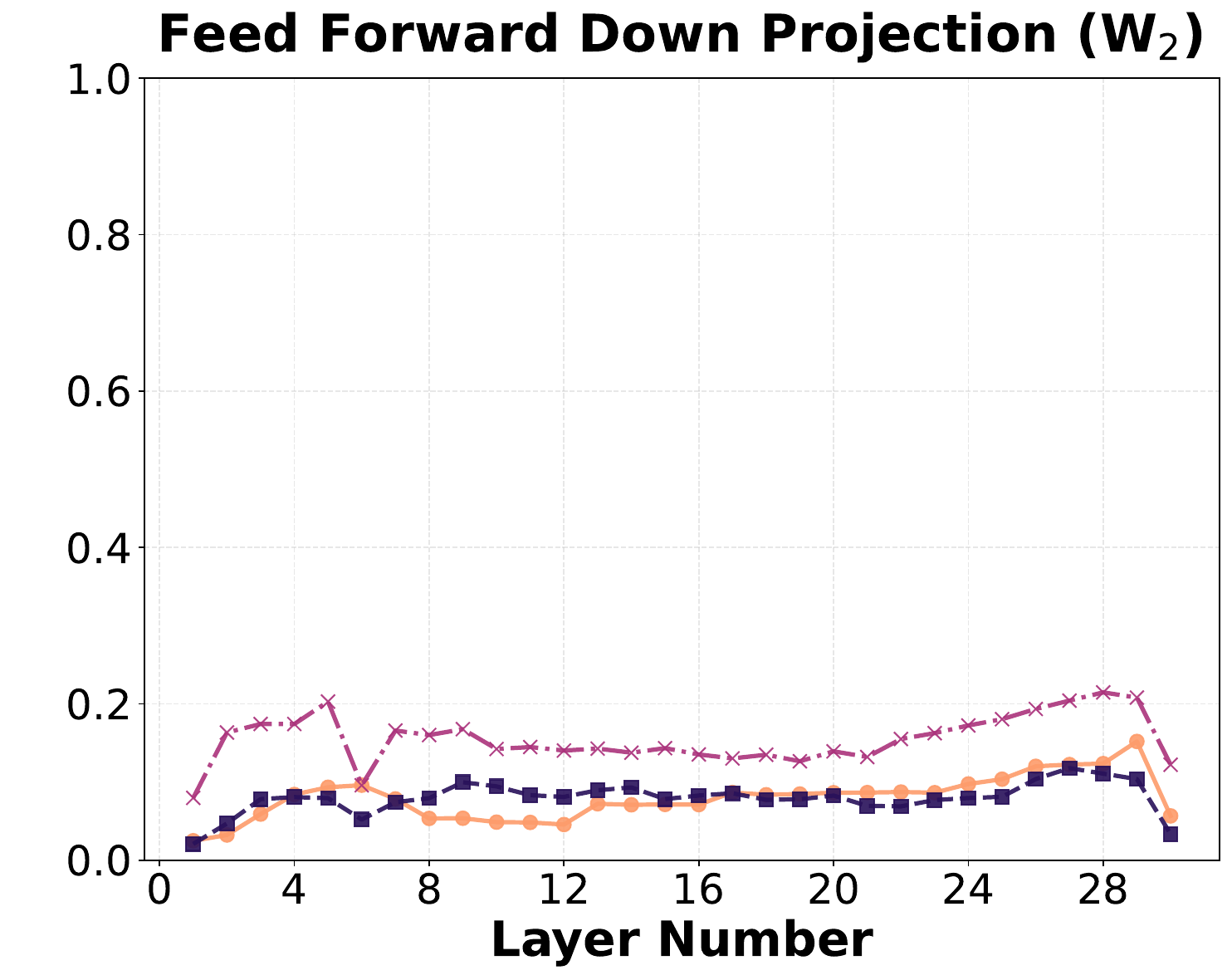}
        \caption{FFN Down Projection}
        \label{fig:stable_rank_vs_layer_comparison_llama3_all:stable_rank_w2}
    \end{subfigure}
    \hspace{-0.5em}
    \begin{subfigure}[b]{0.32\textwidth}
        \centering
        \includegraphics[width=\textwidth]{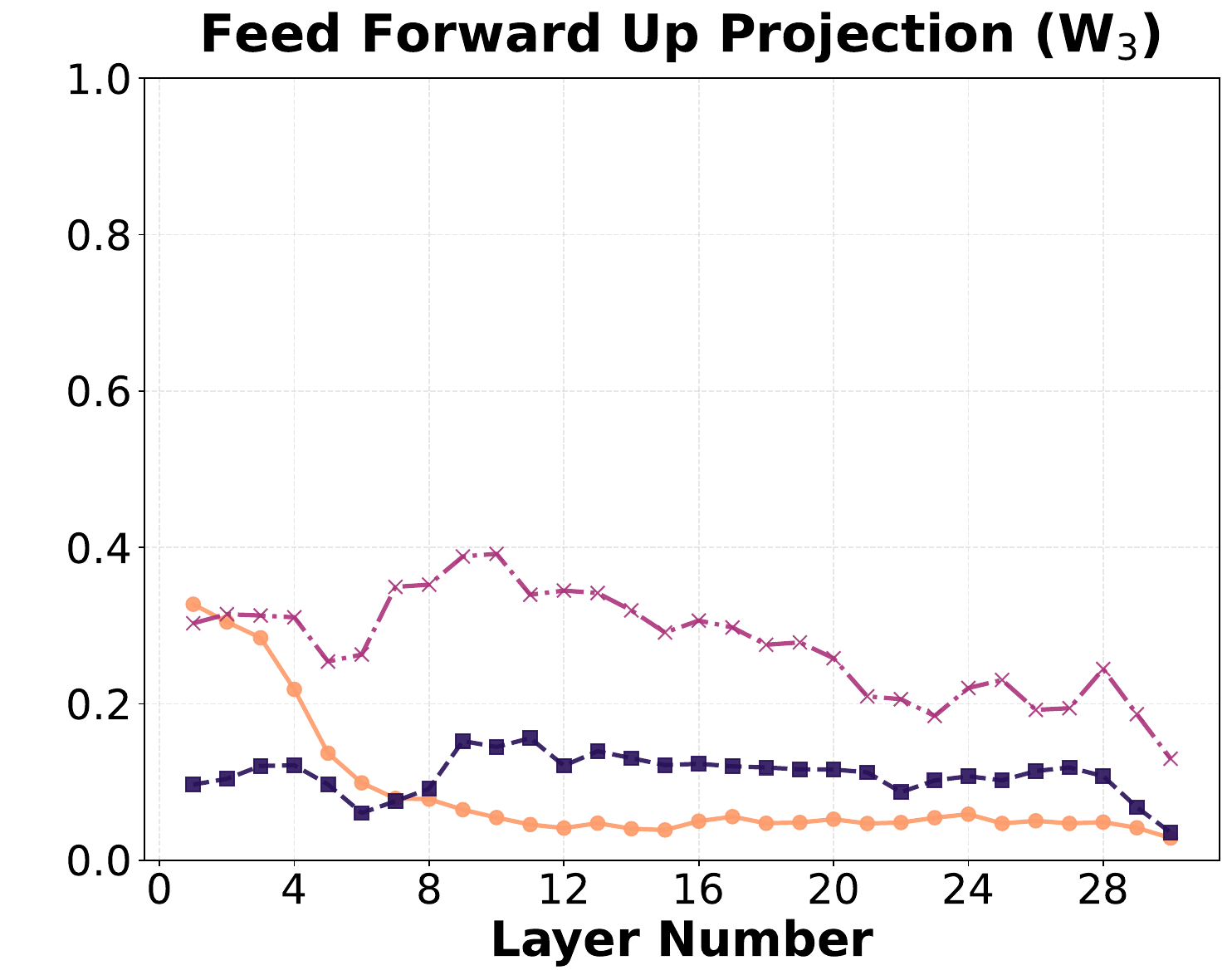}
        \caption{FFN Up Projection}
        \label{fig:stable_rank_vs_layer_comparison_llama3_all:stable_rank_w3}
    \end{subfigure} 
    \\\vspace{1em}
    \begin{subfigure}[b]{0.32\textwidth}
        \centering
        \includegraphics[width=\textwidth]{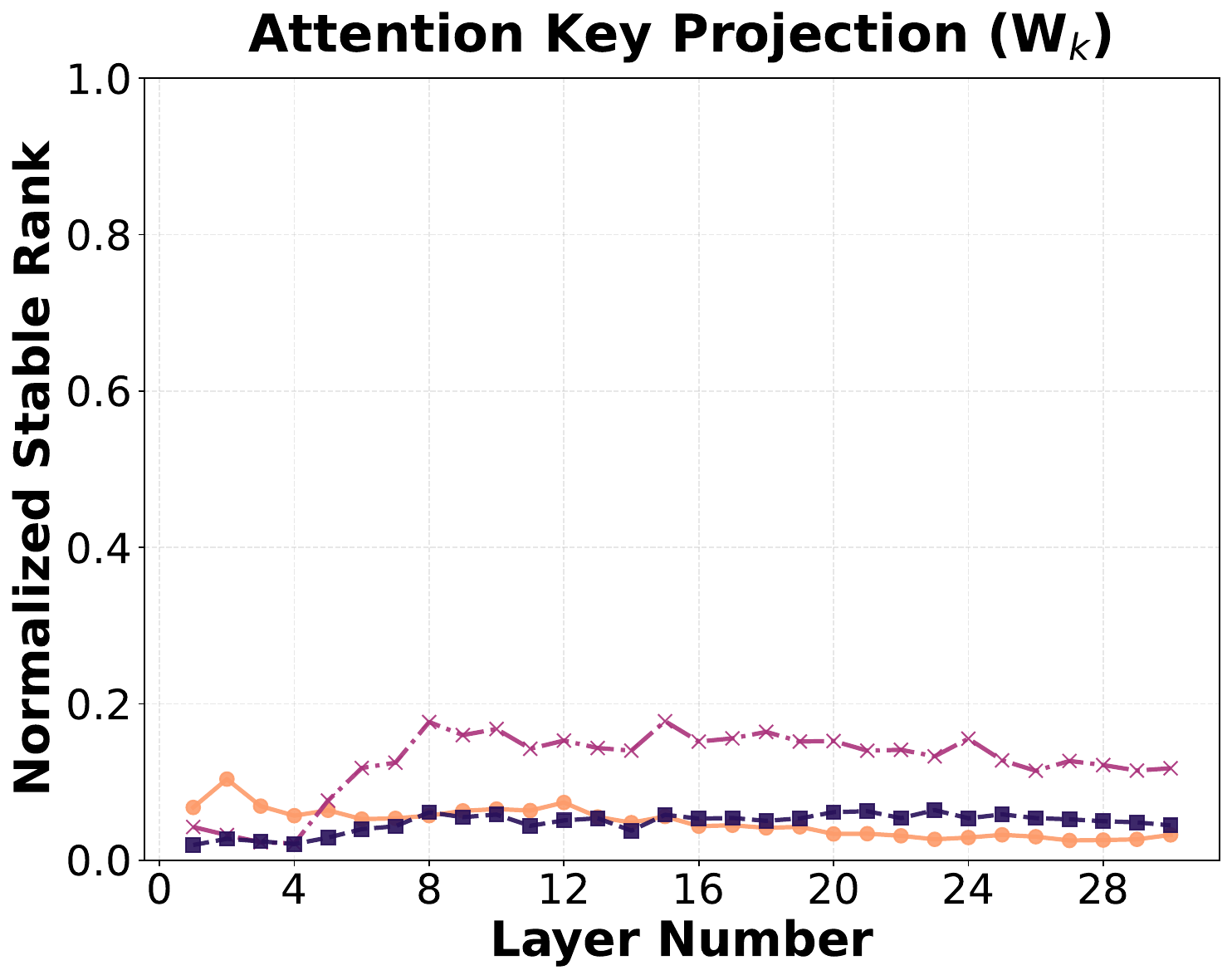}
        \caption{Attn Key Projection}
        \label{fig:stable_rank_vs_layer_comparison_llama3_all:stable_rank_wk}
    \end{subfigure}
    \hspace{-0.5em}
    \begin{subfigure}[b]{0.32\textwidth}
        \centering
        \includegraphics[width=\textwidth]{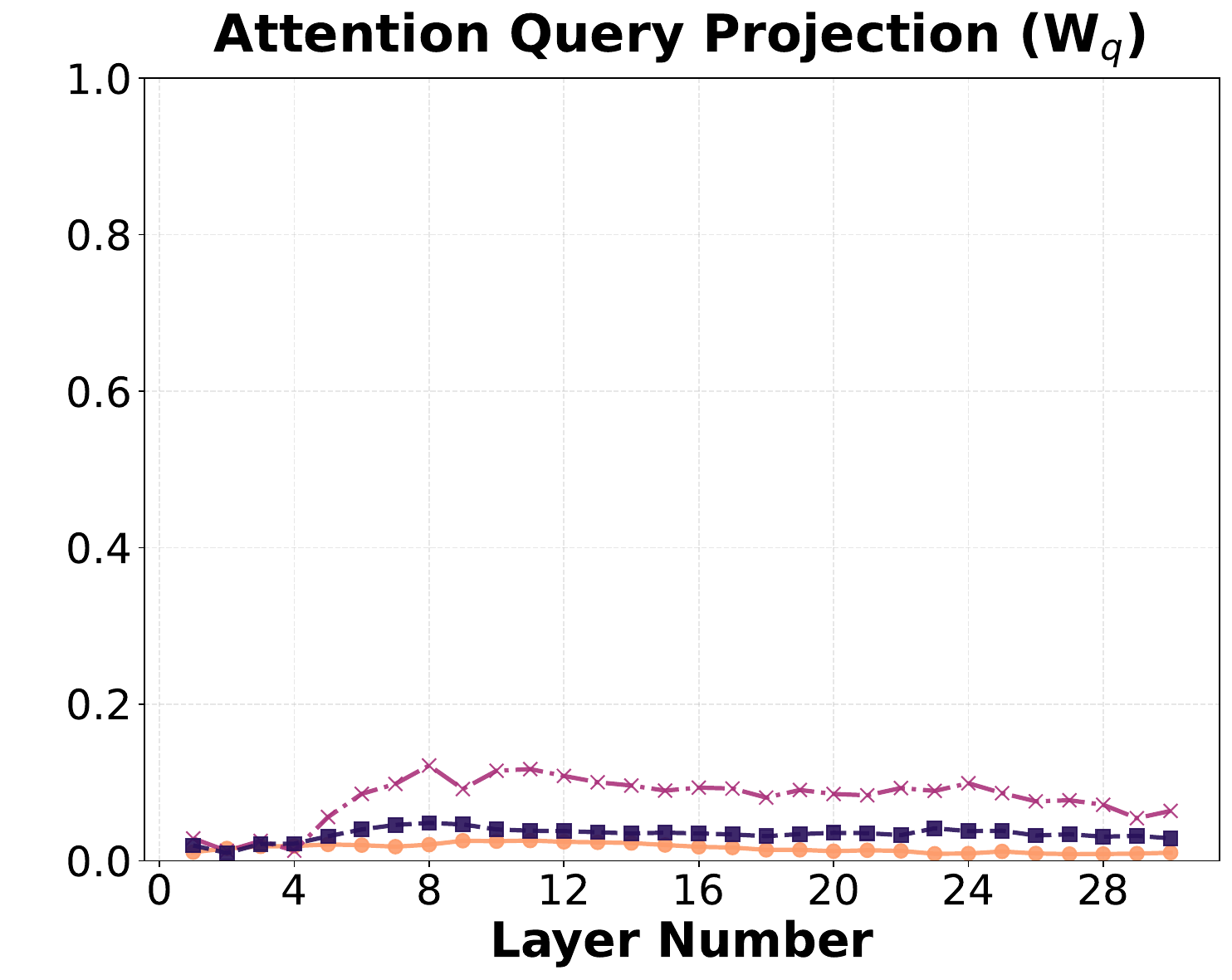}
        \caption{Attn Query Projection}
        \label{fig:stable_rank_vs_layer_comparison_llama3_all:stable_rank_wq}
    \end{subfigure}
    \hspace{-0.5em}
    \begin{subfigure}[b]{0.32\textwidth}
        \centering
        \includegraphics[width=\textwidth]{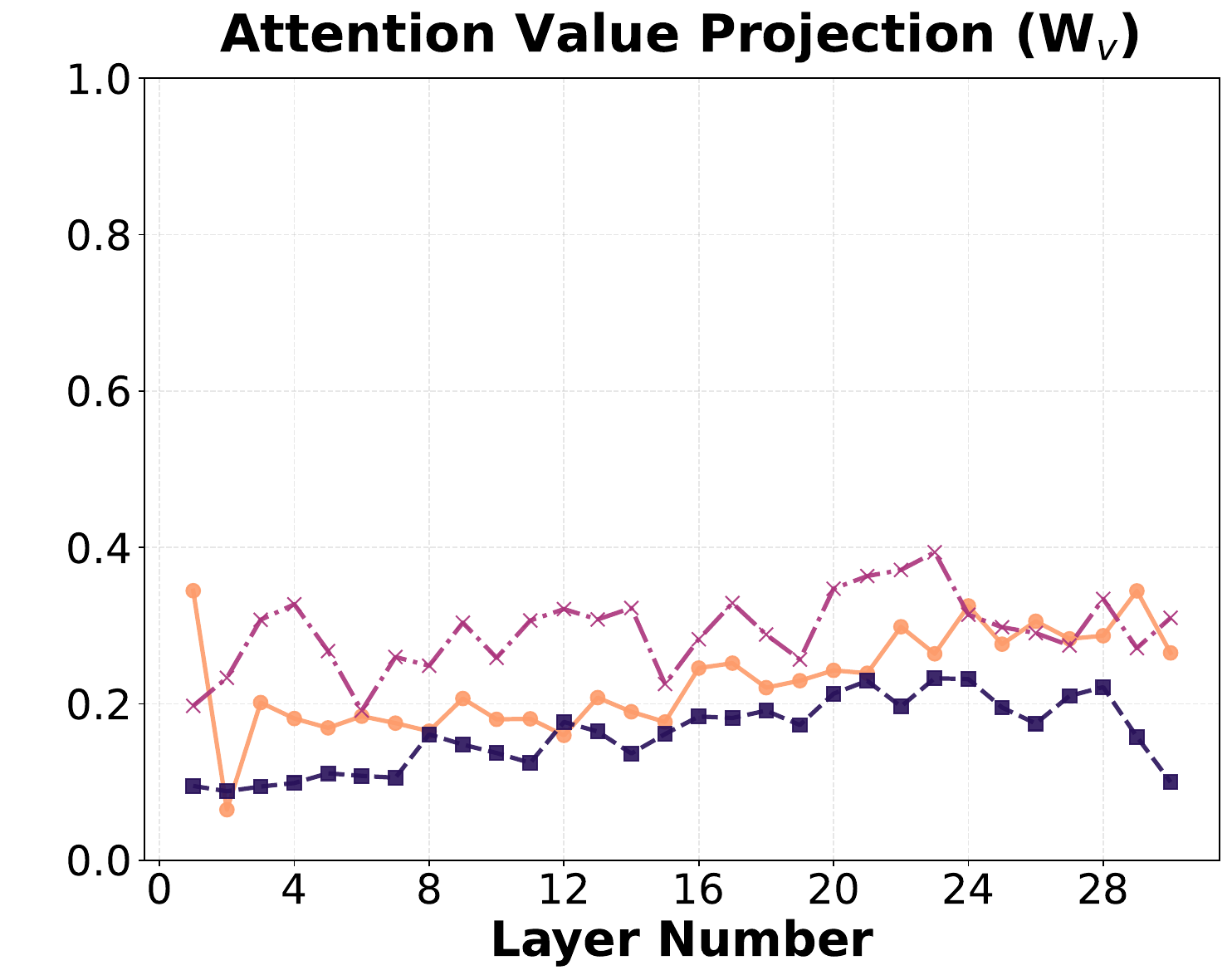}
        \caption{Attn Value Projection}
        \label{fig:stable_rank_vs_layer_comparison_llama3_all:stable_rank_wv}
    \end{subfigure}
    \\\vspace{1em}
    \begin{subfigure}[b]{0.32\textwidth}
    \end{subfigure}
    \begin{subfigure}[b]{0.32\textwidth}
        \centering
        \includegraphics[width=\textwidth]{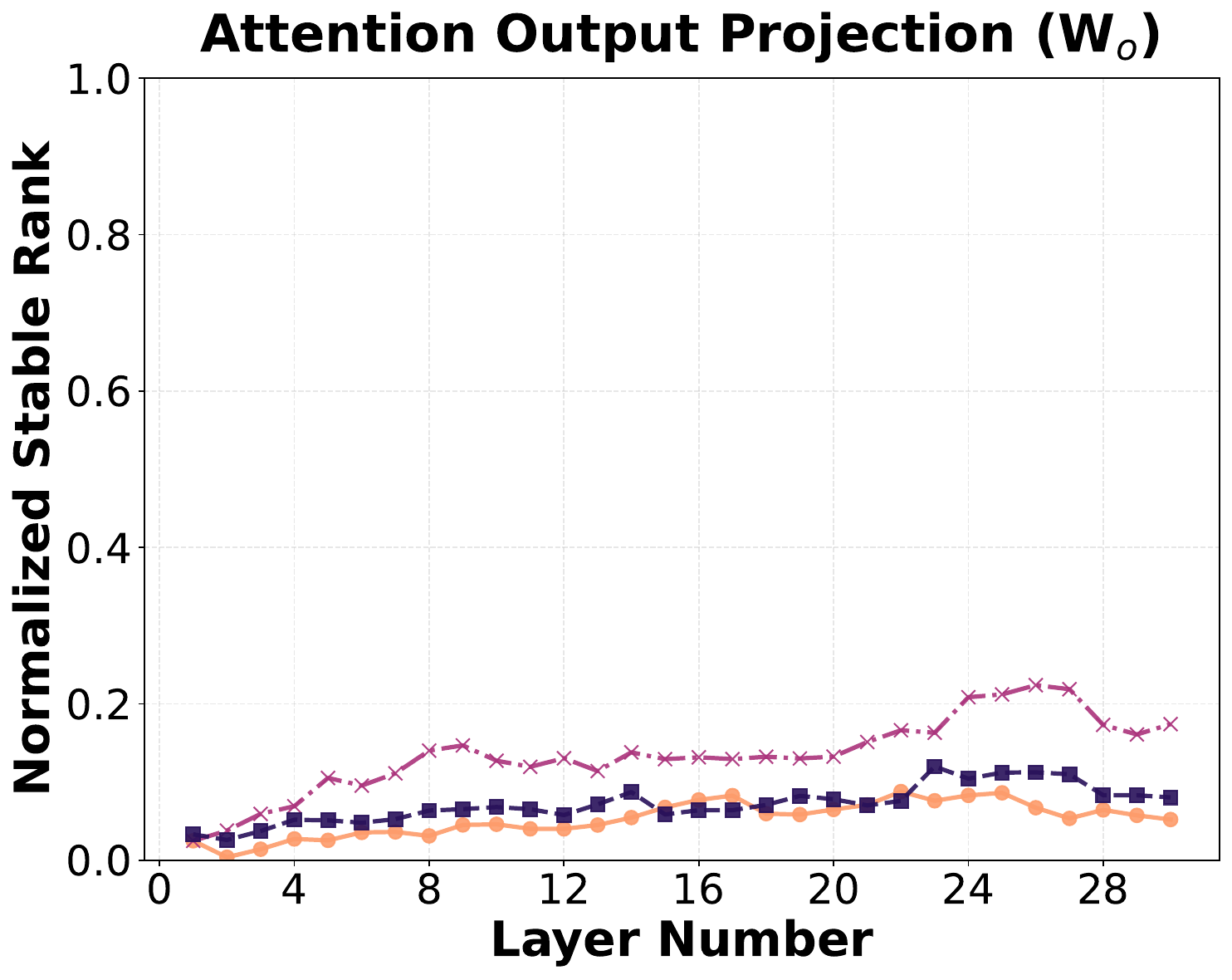}
        \caption{Attn Output Projection}
        \label{fig:stable_rank_vs_layer_comparison_llama3_all:stable_rank_wo}
    \end{subfigure}
    \begin{subfigure}[b]{0.32\textwidth}
    \end{subfigure}
    \caption{Normalized stable rank for weight matrices of Llama3-1.8B models. Each subplot shows the stable rank (normalized by the maximum rank) of the converged model across all layers.}
    \label{fig:stable_rank_vs_layer_comparison_llama3_all}
\end{figure*}
}

\newcommand{\figStableRankvsLayerAblationQwen}{
    \centering
    \includegraphics[width=\linewidth]{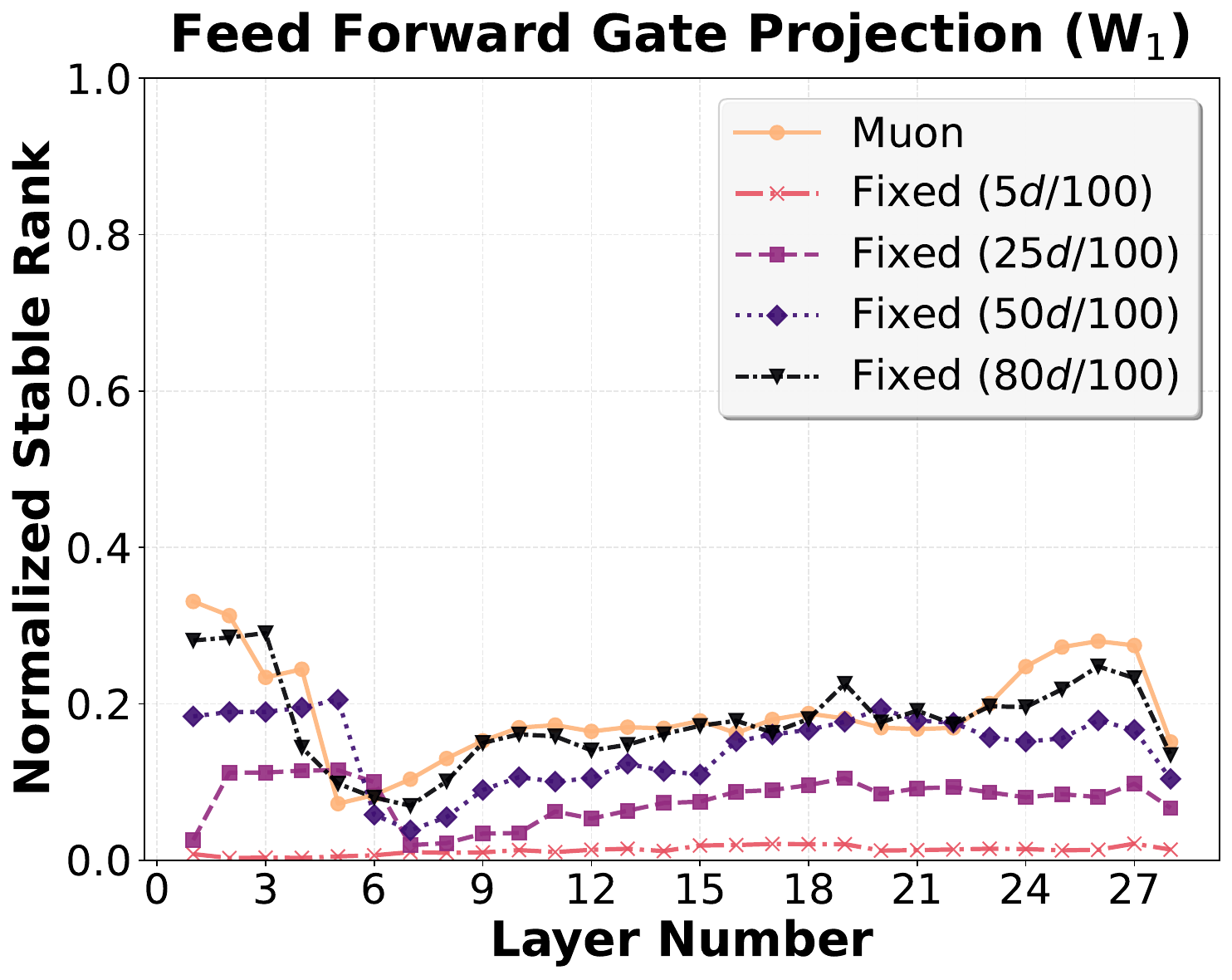}
    \captionof{figure}{Normalized stable rank for feedforward gate projection weight matrix of Qwen3-0.6B models trained with NuMuon under different ranks/nuclear norm budget. Each subplot shows the stable rank (normalized by the maximum rank) of the converged model across all layers.}
    \label{fig:stable_rank_vs_layer_comparison_qwen3_budget_ablation_small}
}

\newcommand{\figStableRankvsLayerAblationQwenLarge}{
\begin{figure*}[t!]
    \centering
    \begin{subfigure}[b]{0.32\textwidth}
        \centering
        \includegraphics[width=\textwidth]{stable_ranks/stable_ranks_budget/feed_forward_w1.pdf}
        \caption{FFN Gate Projection}
        \label{fig:stable_rank_vs_layer_comparison_qwen3_budget_ablation:stable_rank_w1}
    \end{subfigure}
    \hspace{-0.5em}
    \begin{subfigure}[b]{0.32\textwidth}
        \centering
        \includegraphics[width=\textwidth]{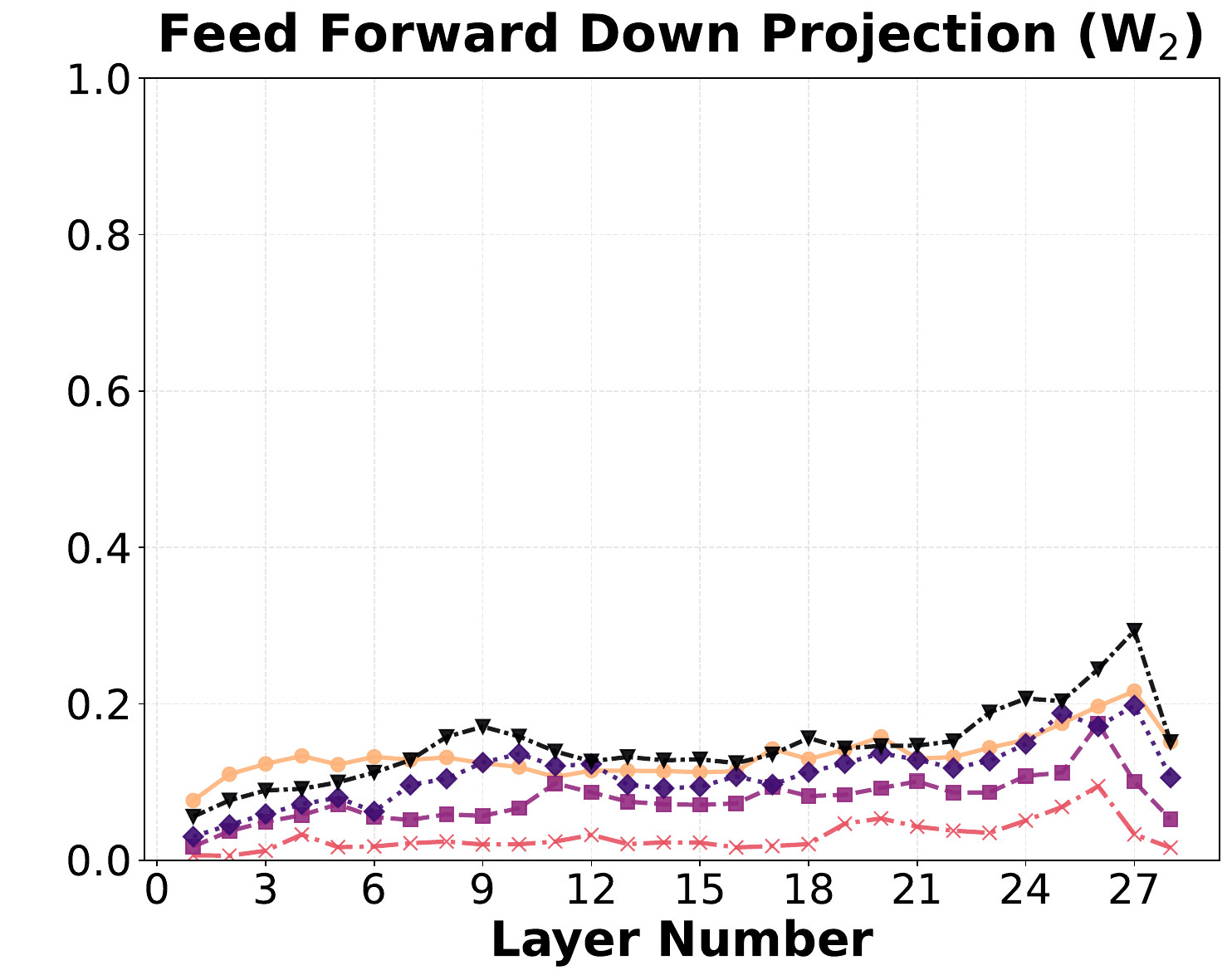}
        \caption{FFN Down Projection}
        \label{fig:stable_rank_vs_layer_comparison_qwen3_budget_ablation:stable_rank_w2}
    \end{subfigure}
    \hspace{-0.5em}
    \begin{subfigure}[b]{0.32\textwidth}
        \centering
        \includegraphics[width=\textwidth]{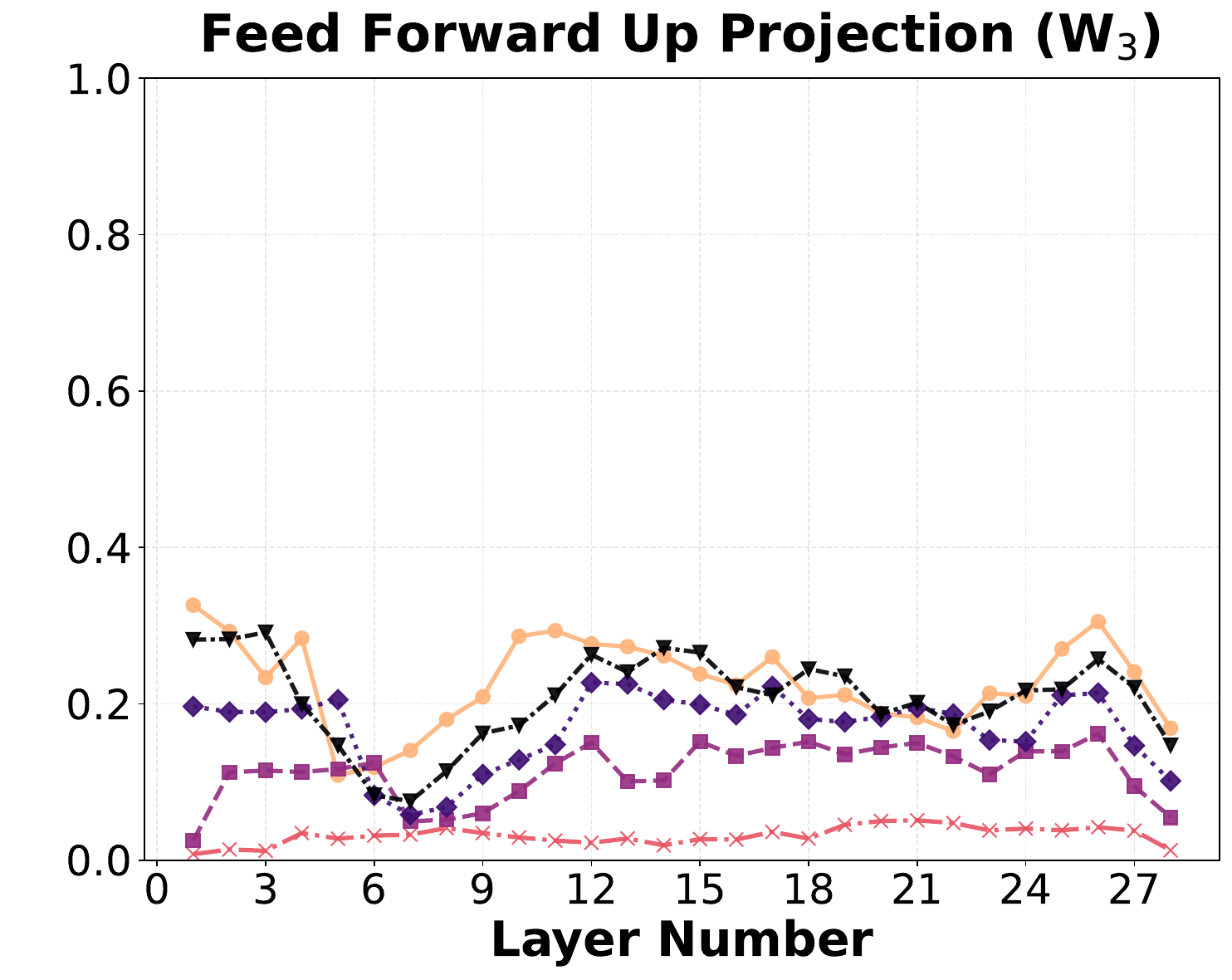}
        \caption{FFN Up Projection}
        \label{fig:stable_rank_vs_layer_comparison_qwen3_budget_ablation:stable_rank_w3}
    \end{subfigure} 
    \\\vspace{1em}
    \begin{subfigure}[b]{0.32\textwidth}
        \centering
        \includegraphics[width=\textwidth]{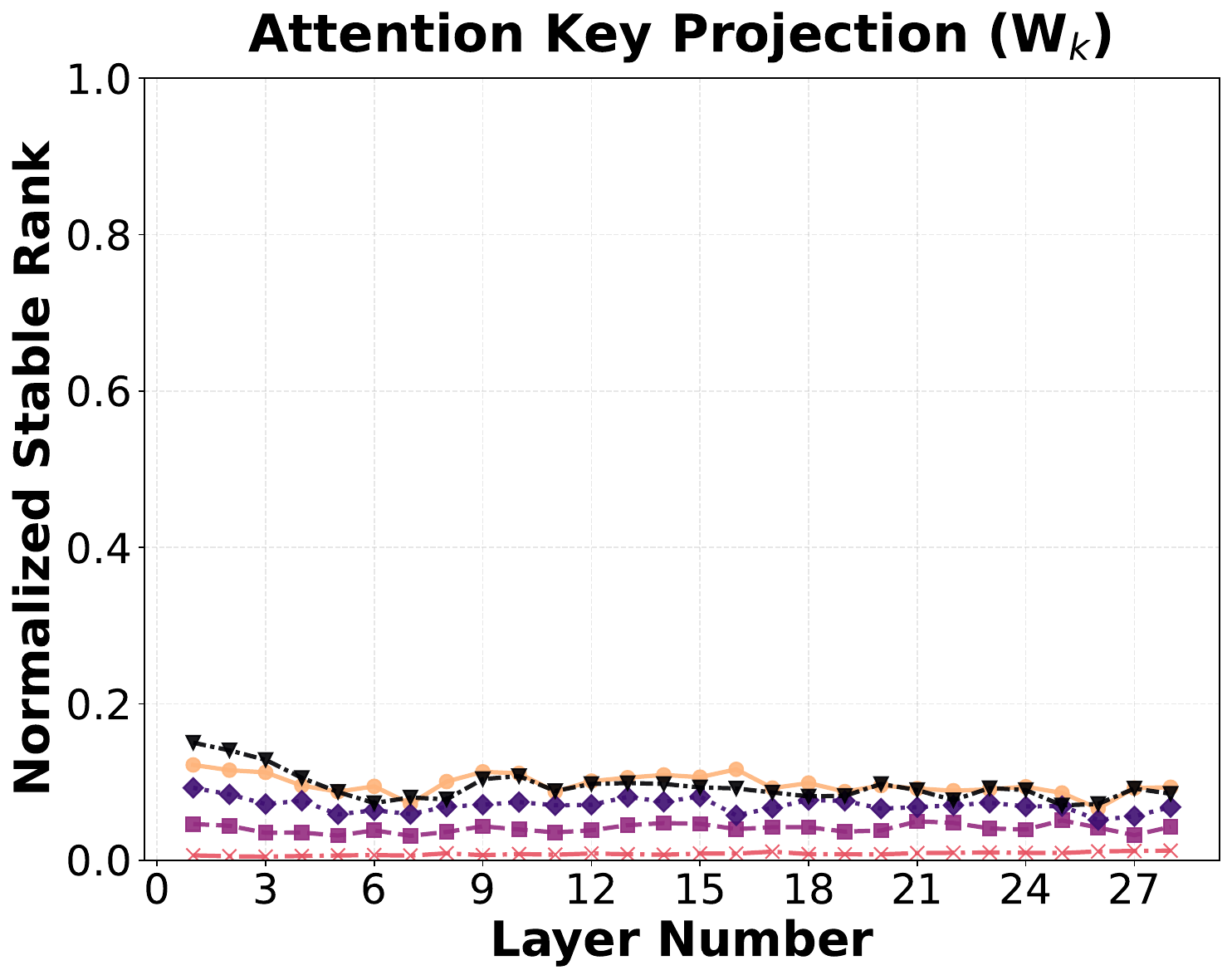}
        \caption{Attn Key Projection}
        \label{fig:stable_rank_vs_layer_comparison_qwen3_budget_ablation:stable_rank_wk}
    \end{subfigure}
    \hspace{-0.5em}
    \begin{subfigure}[b]{0.32\textwidth}
        \centering
        \includegraphics[width=\textwidth]{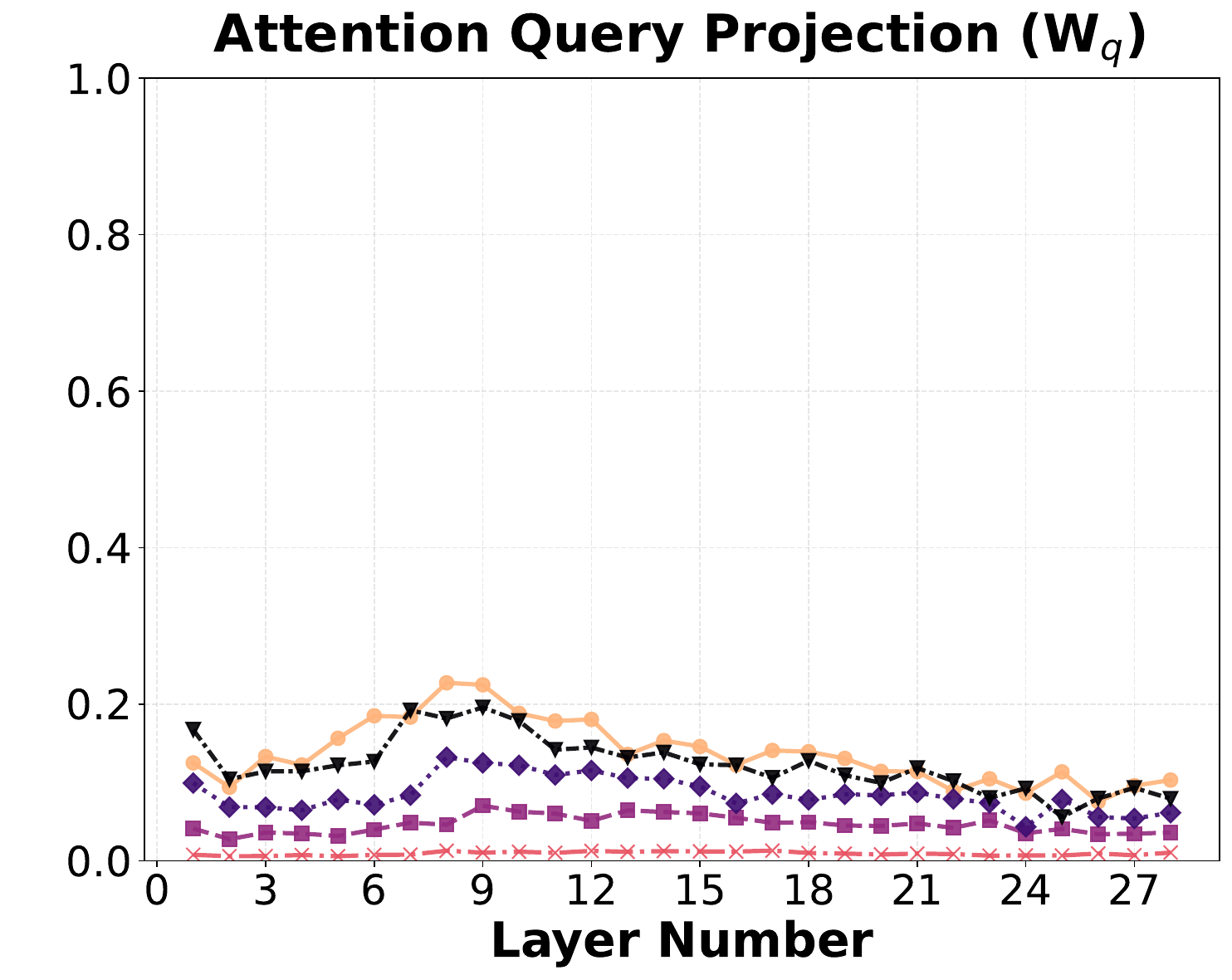}
        \caption{Attn Query Projection}
        \label{fig:stable_rank_vs_layer_comparison_qwen3_budget_ablation:stable_rank_wq}
    \end{subfigure}
    \hspace{-0.5em}
    \begin{subfigure}[b]{0.32\textwidth}
        \centering
        \includegraphics[width=\textwidth]{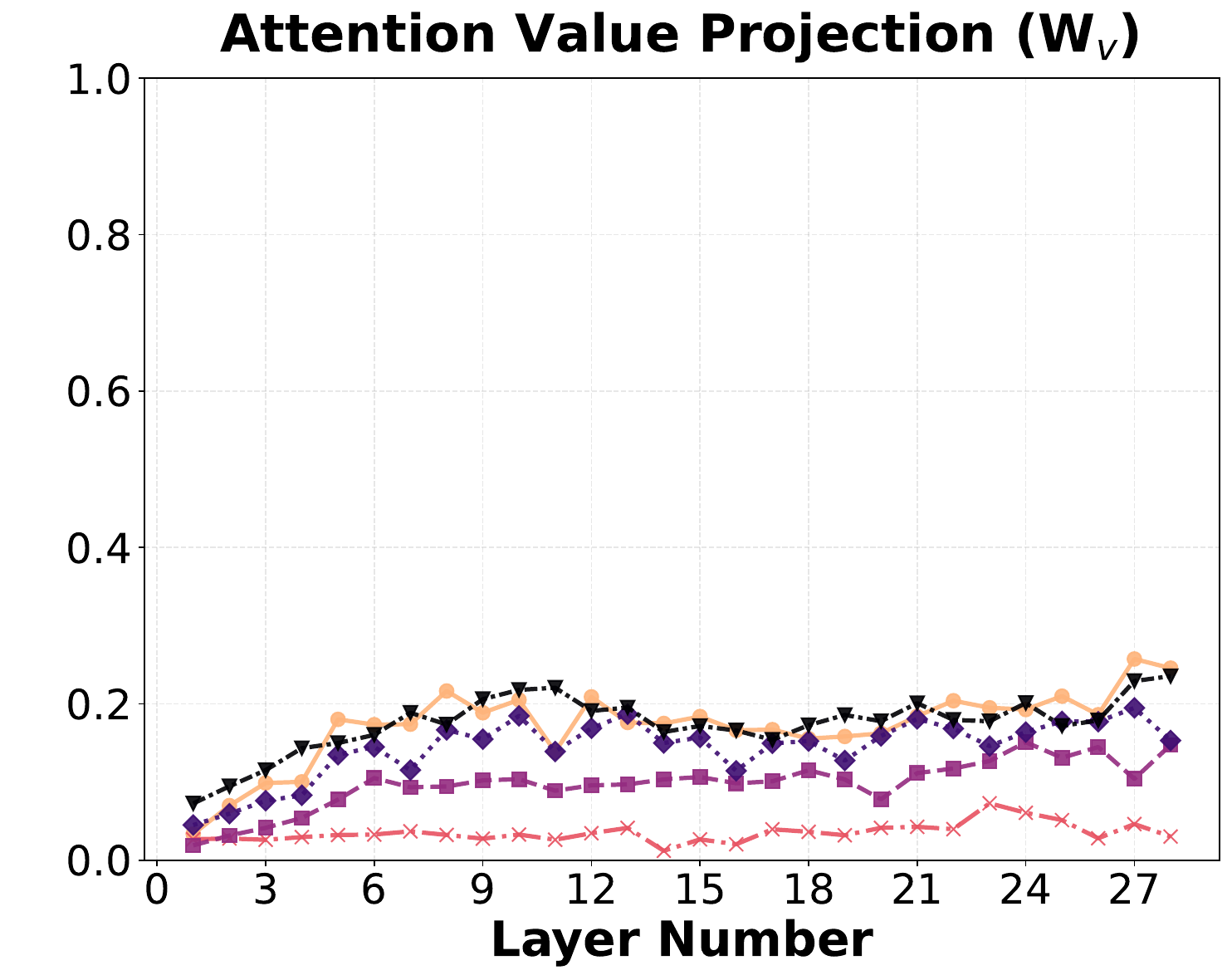}
        \caption{Attn Value Projection}
        \label{fig:stable_rank_vs_layer_comparison_qwen3_budget_ablation:stable_rank_wv}
    \end{subfigure}
    \\\vspace{1em}
    \begin{subfigure}[b]{0.32\textwidth}
    \end{subfigure}
    \begin{subfigure}[b]{0.32\textwidth}
        \centering
        \includegraphics[width=\textwidth]{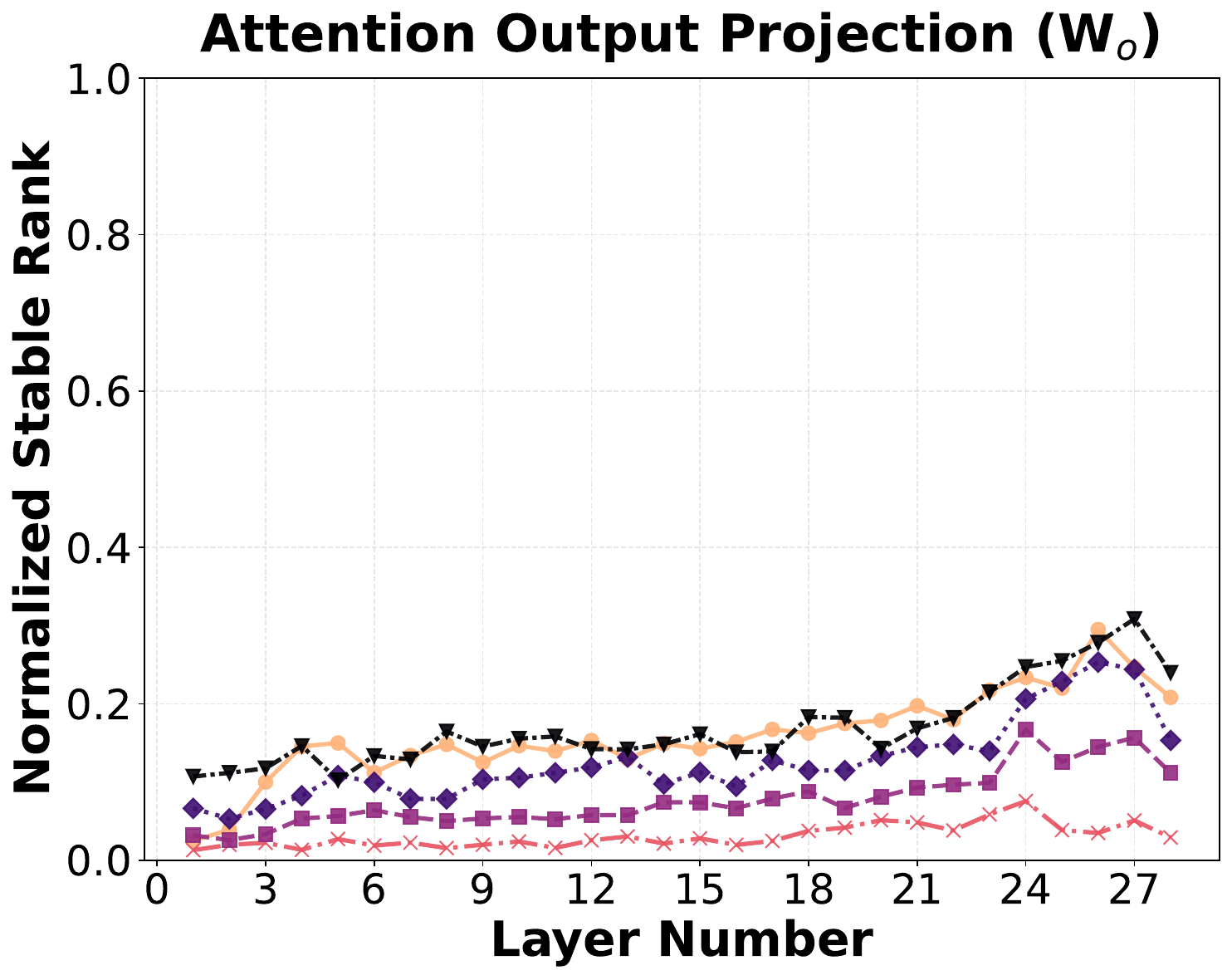}
        \caption{Attn Output Projection}
        \label{fig:stable_rank_vs_layer_comparison_qwen3_budget_ablation:stable_rank_wo}
    \end{subfigure}
    \begin{subfigure}[b]{0.32\textwidth}
    \end{subfigure}
    \caption{Normalized stable rank for weight matrices of Qwen3-0.6B models trained with NuMuon under different ranks/nuclear norm budget. Each subplot shows the stable rank (normalized by the maximum rank) of the converged model across all layers.}
    \label{fig:stable_rank_vs_layer_comparison_qwen3_budget_ablation}
\end{figure*}
}

\newcommand{\figStableRankvsLayerAblationSchedulerQwenLarge}{
\begin{figure*}[t!]
    \centering
    \begin{subfigure}[b]{0.32\textwidth}
        \centering
        \includegraphics[width=\textwidth]{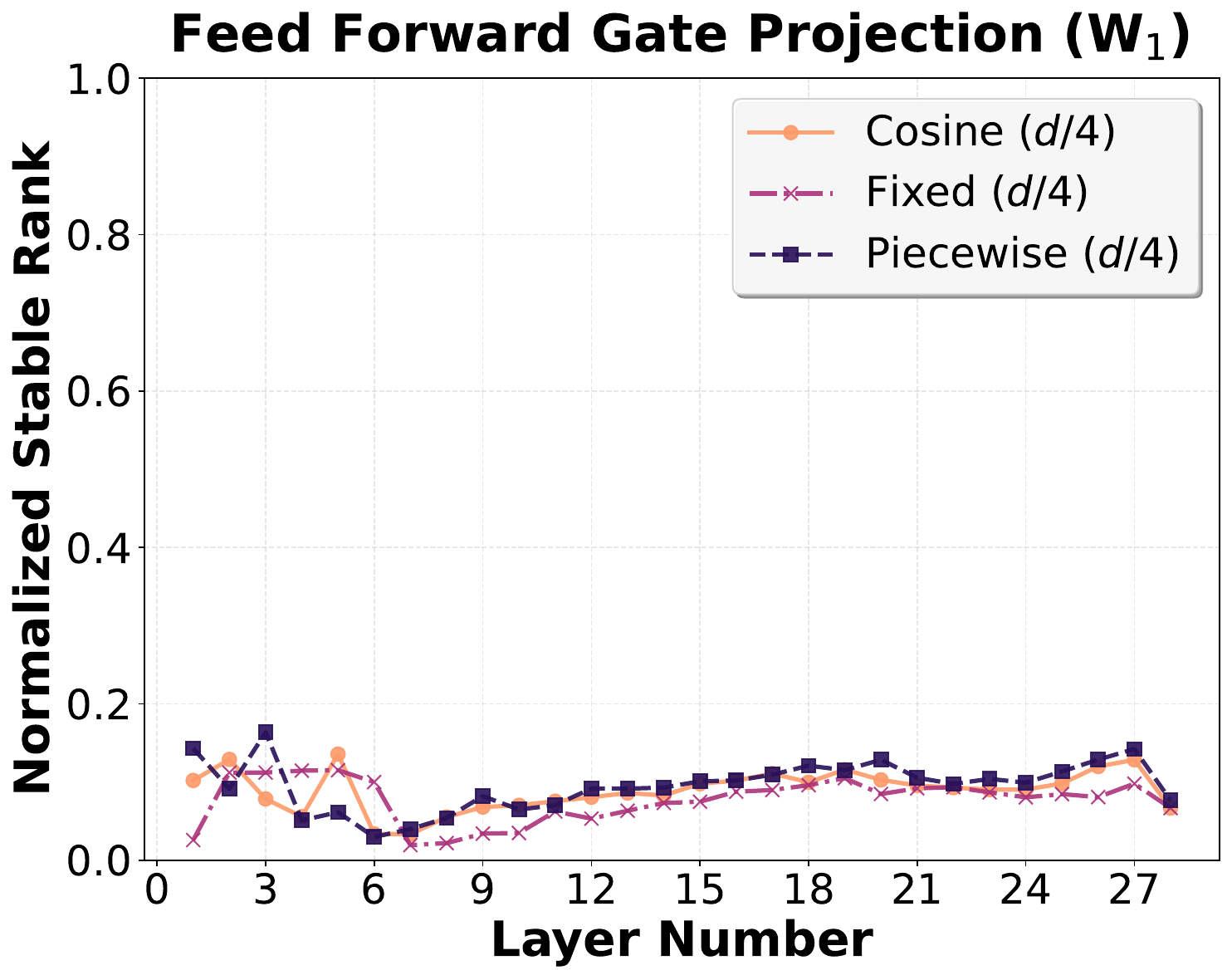}
        \caption{FFN Gate Projection}
        \label{fig:stable_rank_vs_layer_comparison_qwen3_scheduler_ablation:stable_rank_w1}
    \end{subfigure}
    \hspace{-0.5em}
    \begin{subfigure}[b]{0.32\textwidth}
        \centering
        \includegraphics[width=\textwidth]{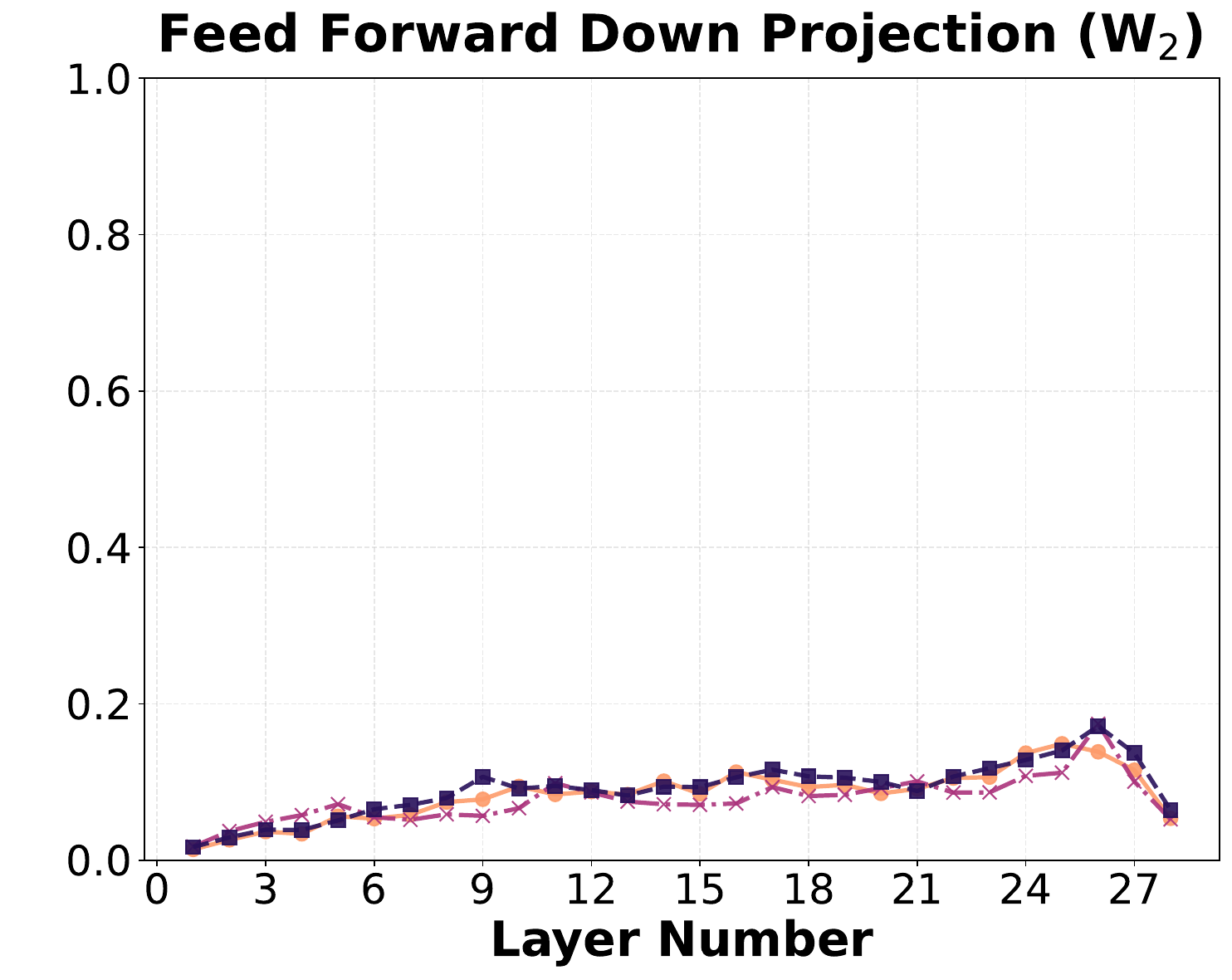}
        \caption{FFN Down Projection}
        \label{fig:stable_rank_vs_layer_comparison_qwen3_scheduler_ablation:stable_rank_w2}
    \end{subfigure}
    \hspace{-0.5em}
    \begin{subfigure}[b]{0.32\textwidth}
        \centering
        \includegraphics[width=\textwidth]{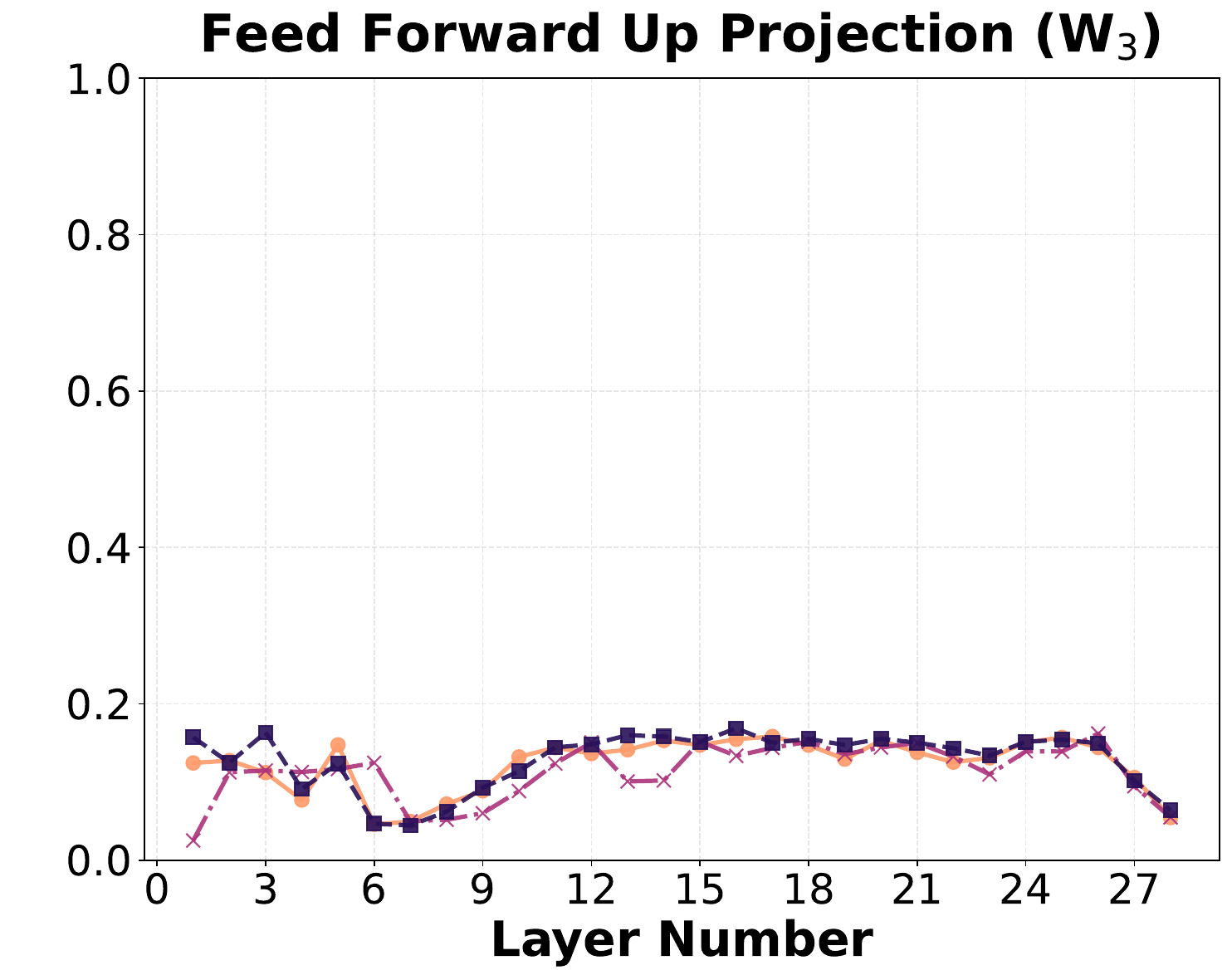}
        \caption{FFN Up Projection}
        \label{fig:stable_rank_vs_layer_comparison_qwen3_scheduler_ablation:stable_rank_w3}
    \end{subfigure} 
    \\\vspace{1em}
    \begin{subfigure}[b]{0.32\textwidth}
        \centering
        \includegraphics[width=\textwidth]{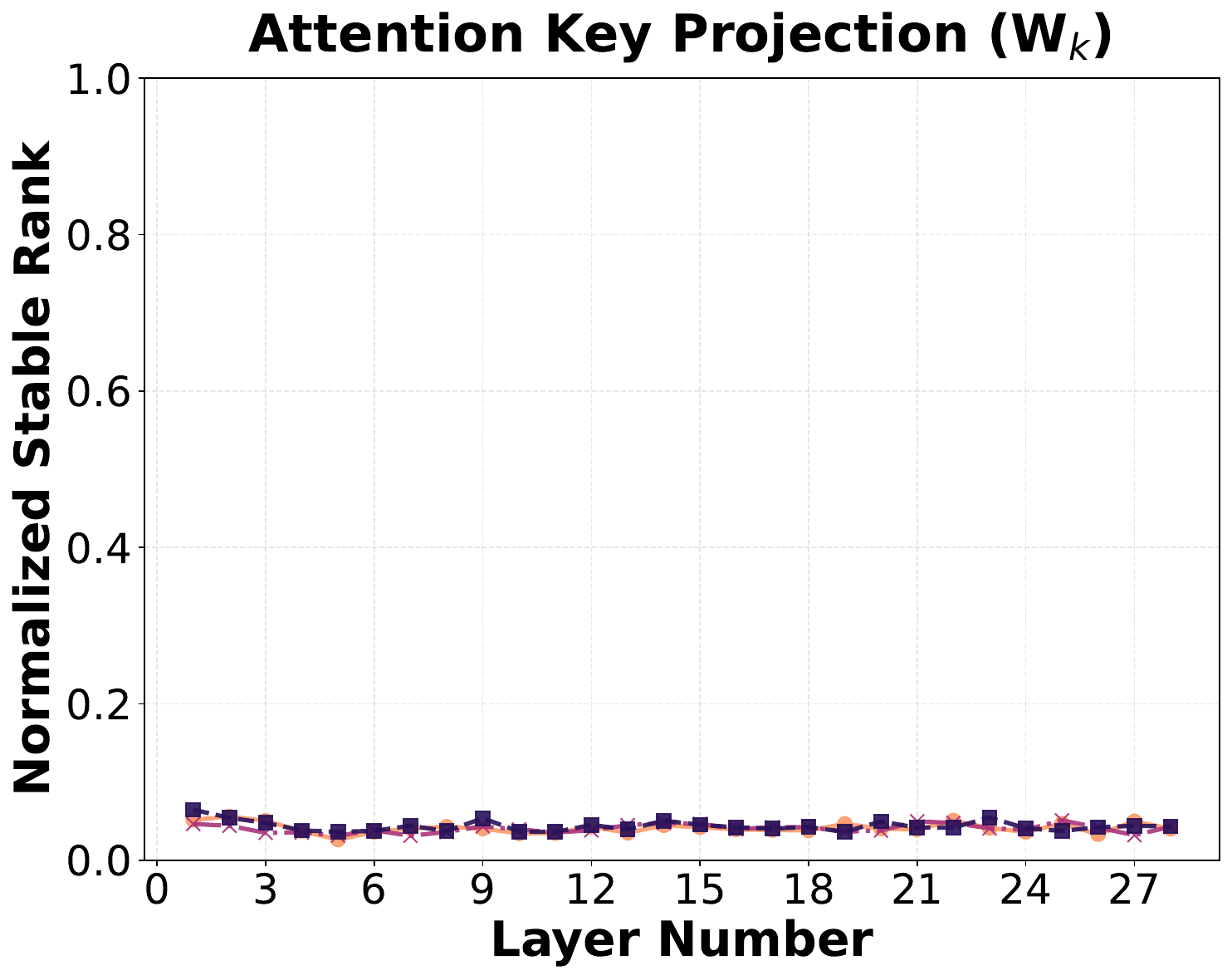}
        \caption{Attn Key Projection}
        \label{fig:stable_rank_vs_layer_comparison_qwen3_scheduler_ablation:stable_rank_wk}
    \end{subfigure}
    \hspace{-0.5em}
    \begin{subfigure}[b]{0.32\textwidth}
        \centering
        \includegraphics[width=\textwidth]{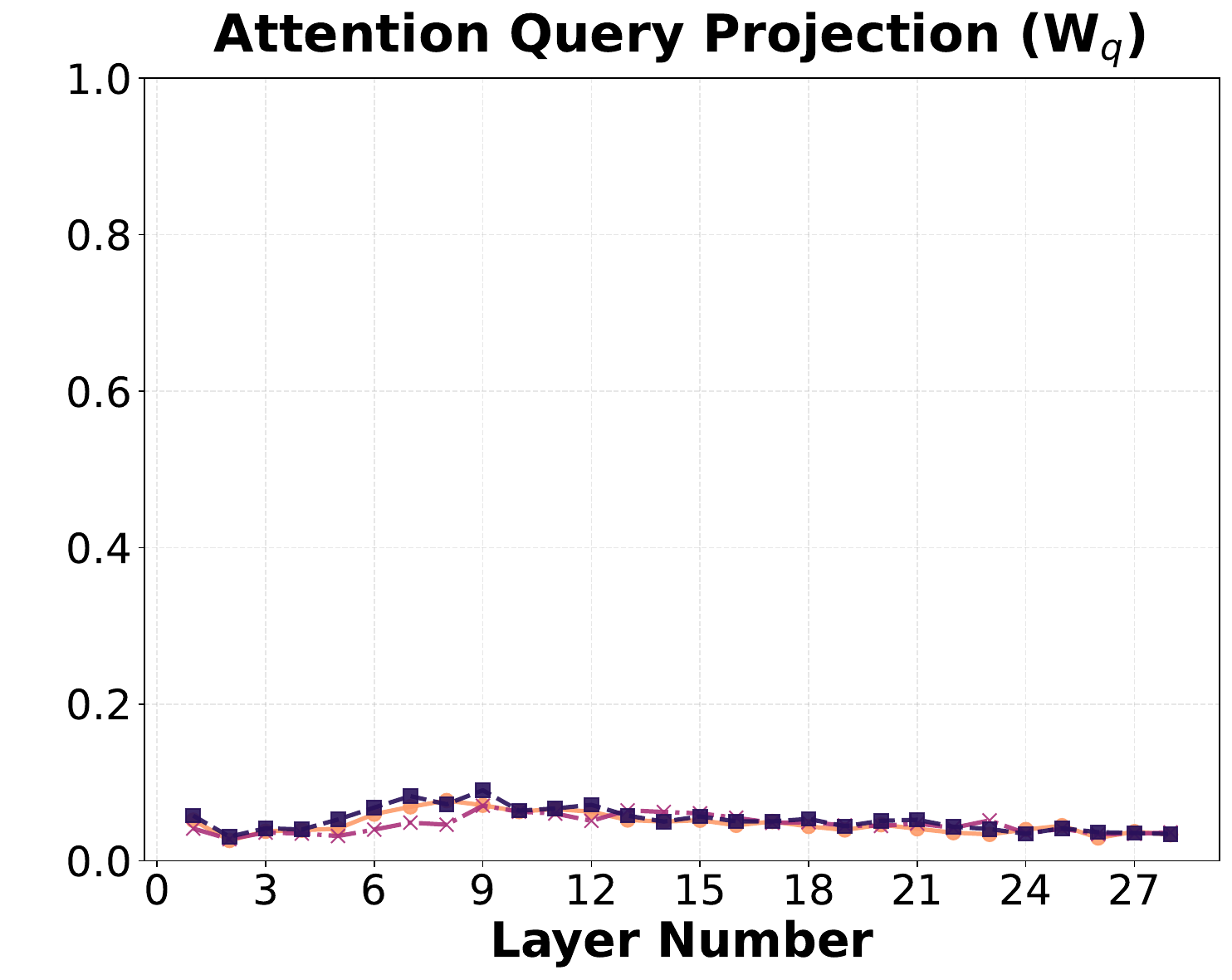}
        \caption{Attn Query Projection}
        \label{fig:stable_rank_vs_layer_comparison_qwen3_scheduler_ablation:stable_rank_wq}
    \end{subfigure}
    \hspace{-0.5em}
    \begin{subfigure}[b]{0.32\textwidth}
        \centering
        \includegraphics[width=\textwidth]{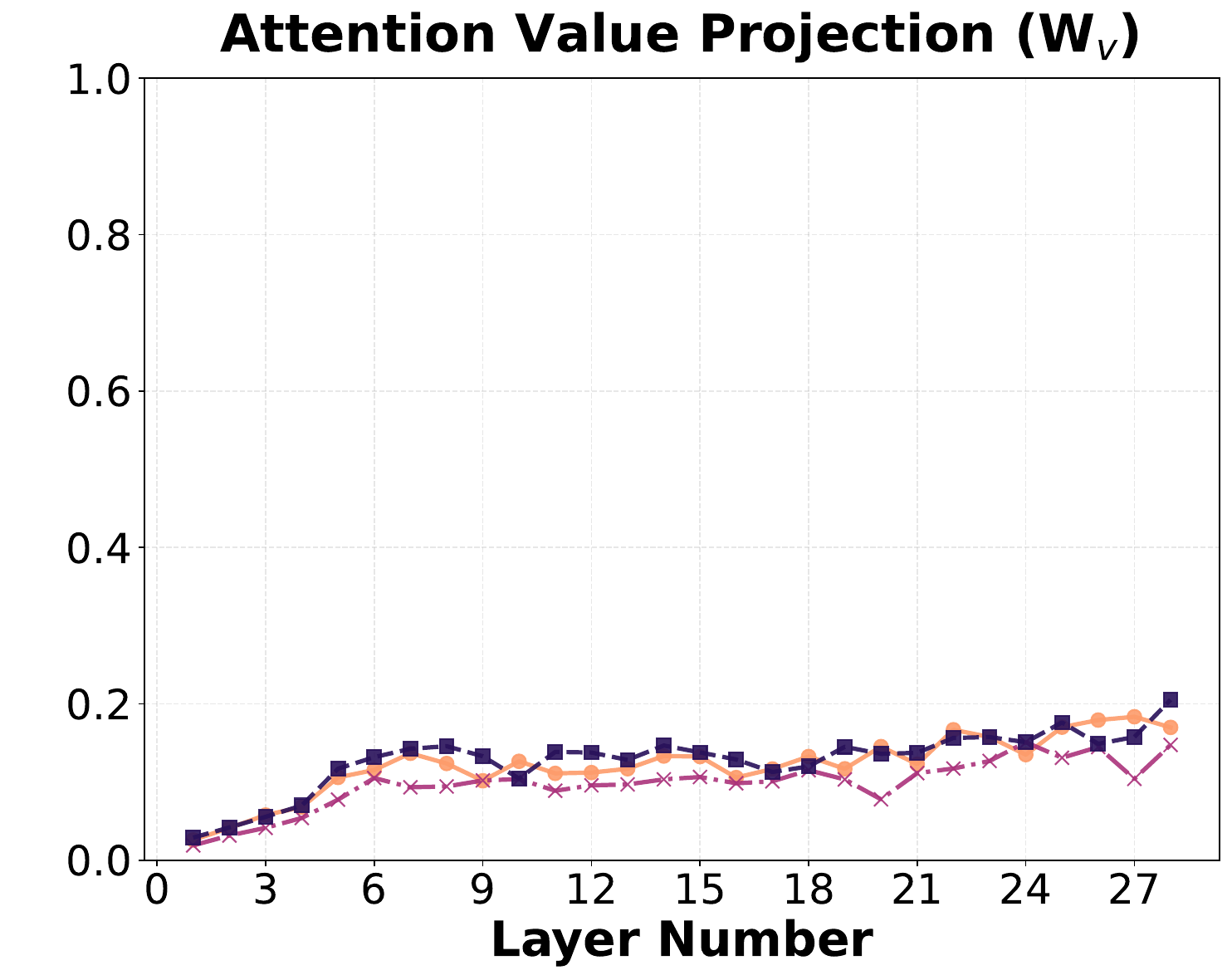}
        \caption{Attn Value Projection}
        \label{fig:stable_rank_vs_layer_comparison_qwen3_scheduler_ablation:stable_rank_wv}
    \end{subfigure}
    \\\vspace{1em}
    \begin{subfigure}[b]{0.32\textwidth}
    \end{subfigure}
    \begin{subfigure}[b]{0.32\textwidth}
        \centering
        \includegraphics[width=\textwidth]{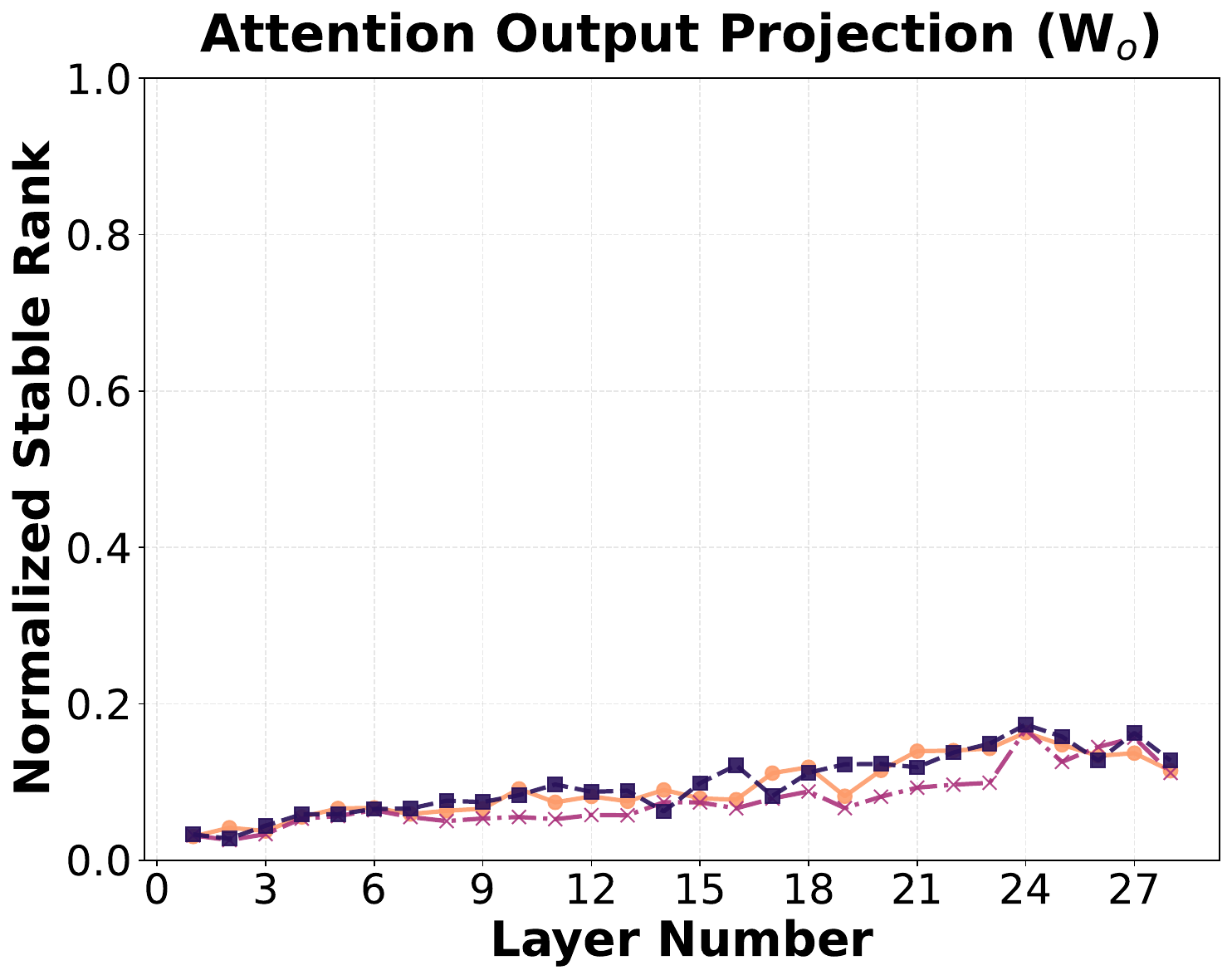}
        \caption{Attn Output Projection}
        \label{fig:stable_rank_vs_layer_comparison_qwen3_scheduler_ablation:stable_rank_wo}
    \end{subfigure}
    \begin{subfigure}[b]{0.32\textwidth}
    \end{subfigure}
    \caption{Normalized stable rank for weight matrices of Qwen3-0.6B models trained with NuMuon under different rank schedulers. Each subplot shows the stable rank (normalized by the maximum rank) of the converged model across all layers.}
    \label{fig:stable_rank_vs_layer_comparison_qwen3_scheduler_ablation}
\end{figure*}
}

\newcommand{\figRankFraction}{
\begin{figure*}[t!]
    \centering
    \includegraphics[width=0.5\textwidth]{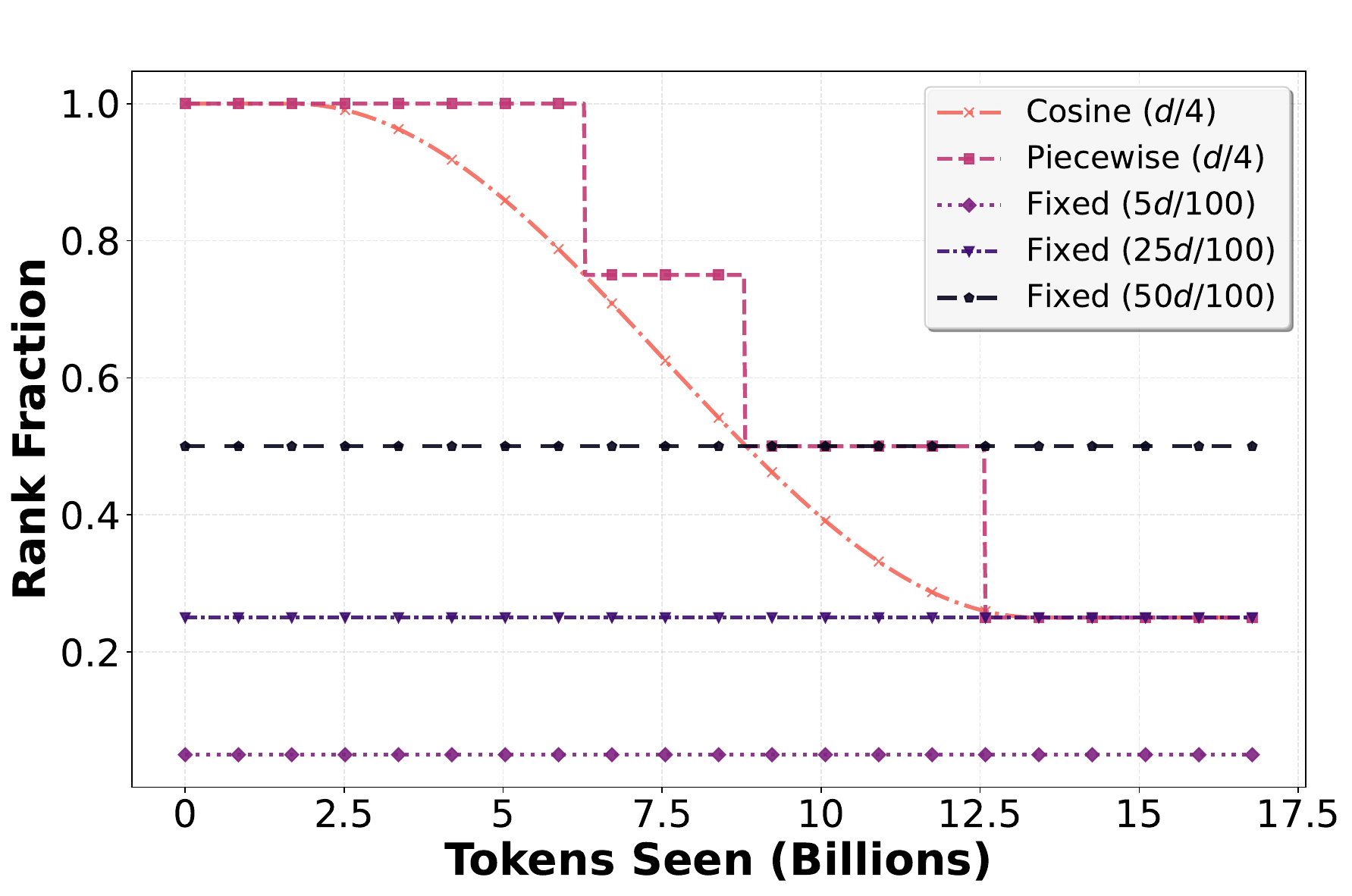}
    \caption{Different rank schedulers used for training Qwen3-0.6B models with NuMuon. Please see~\Cref{sec:method:practical} for more details.}
    \label{fig:rank_schedulers}
\end{figure*}
}

\newcommand{\figScatterAllMethods}{
\begin{figure}[t!]
    \centering
    \includegraphics[width=0.55\textwidth]{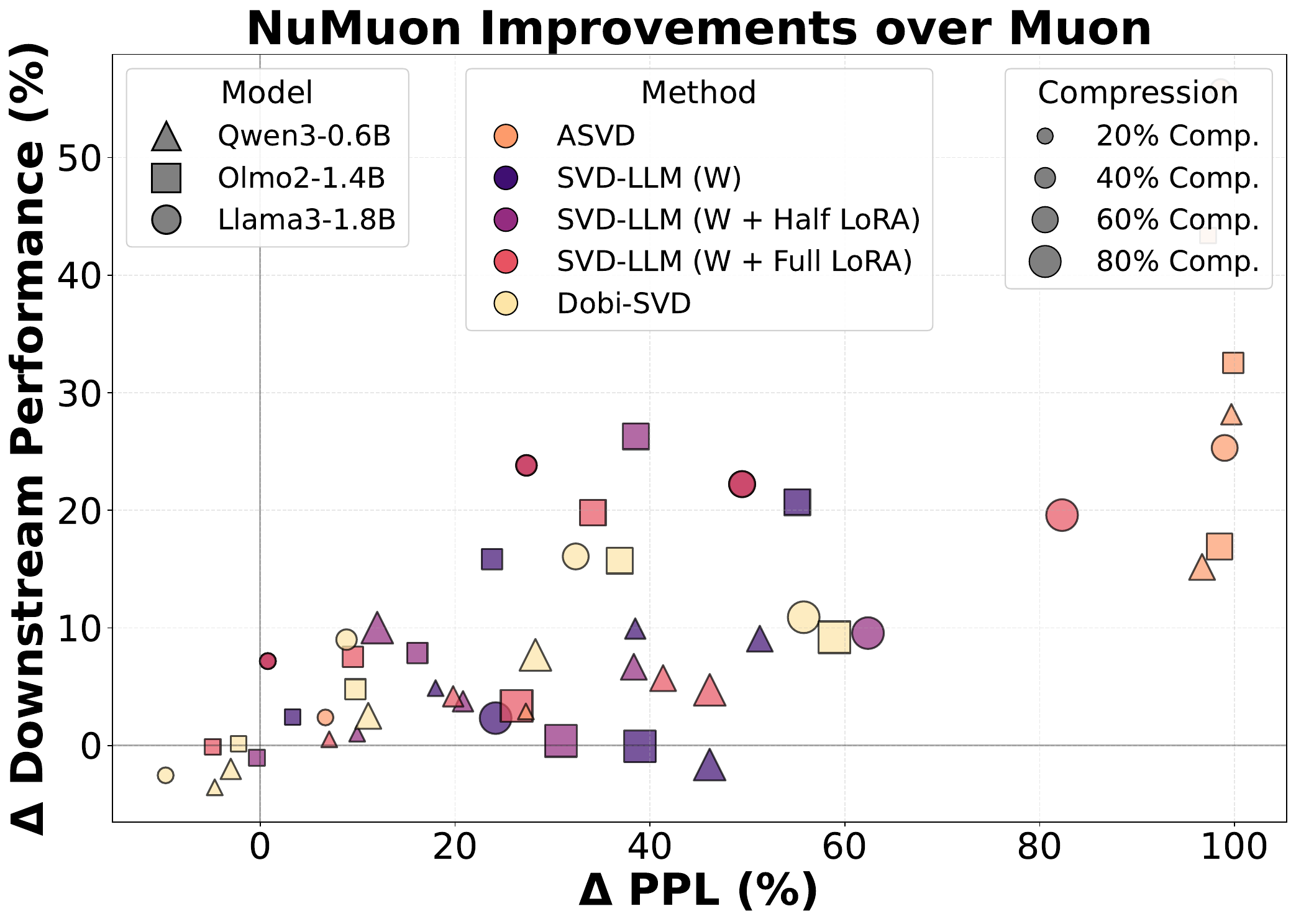}
    \caption{NuMuon's relative performance against Muon under SoTA LLM compression methods. The scatter plot displays the performance improvement on downstream tasks as well as validation perplexity improvements for all compression methods at all rates. The positive quadrant shows performance improvement.}
    \label{fig:compression_all}
\end{figure}
}

\newcommand{\figLossCurvesLarge}{
\begin{figure*}[t!]
    \centering
    \begin{subfigure}[b]{0.4\textwidth}
        \centering
        \includegraphics[width=\textwidth]{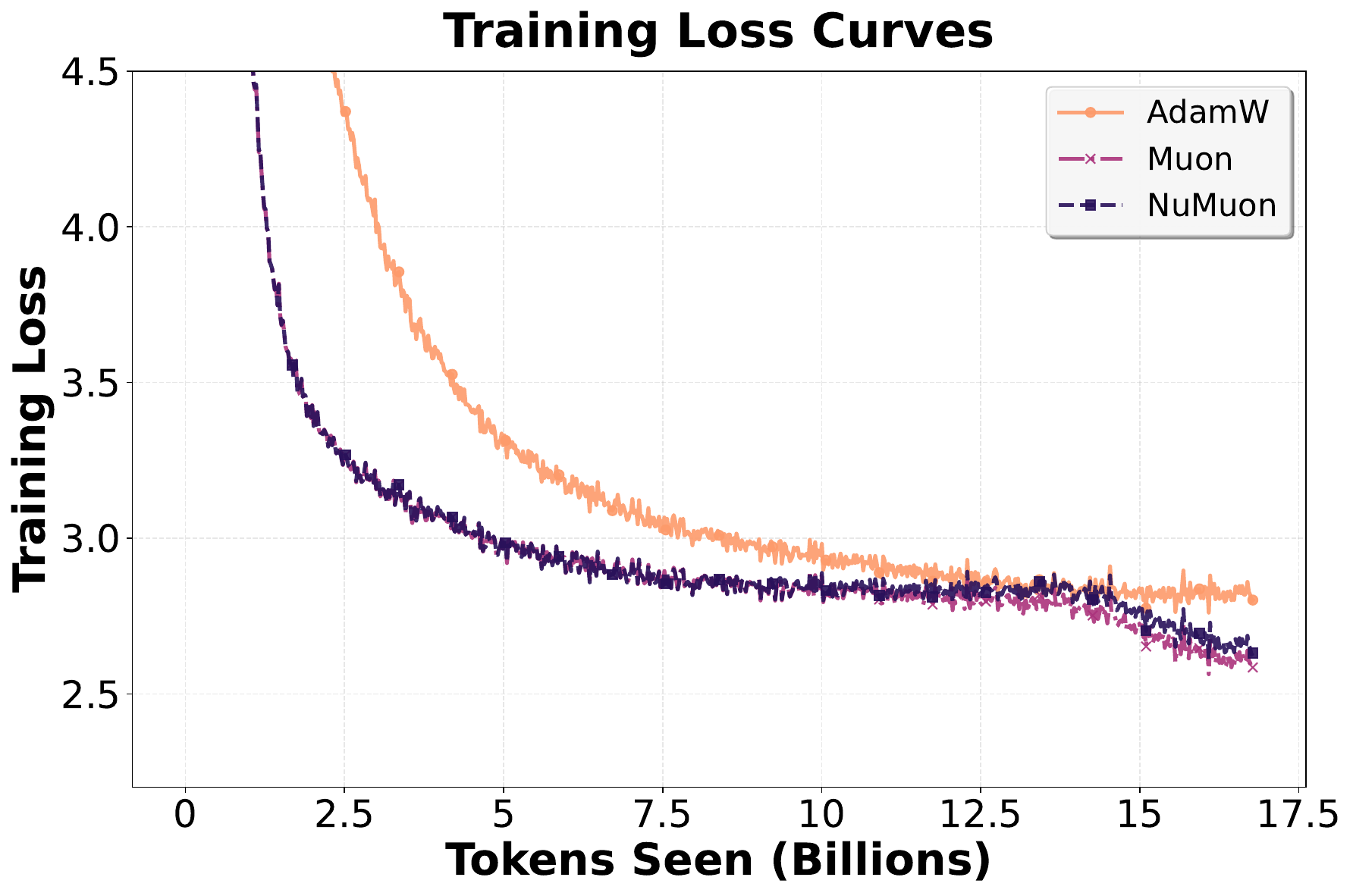}
        \caption{Qwen3-0.6B}
        \label{fig:loss_curves:qwen3}
    \end{subfigure}
    \hspace{-0.5em}
    \begin{subfigure}[b]{0.4\textwidth}
        \centering
        \includegraphics[width=\textwidth]{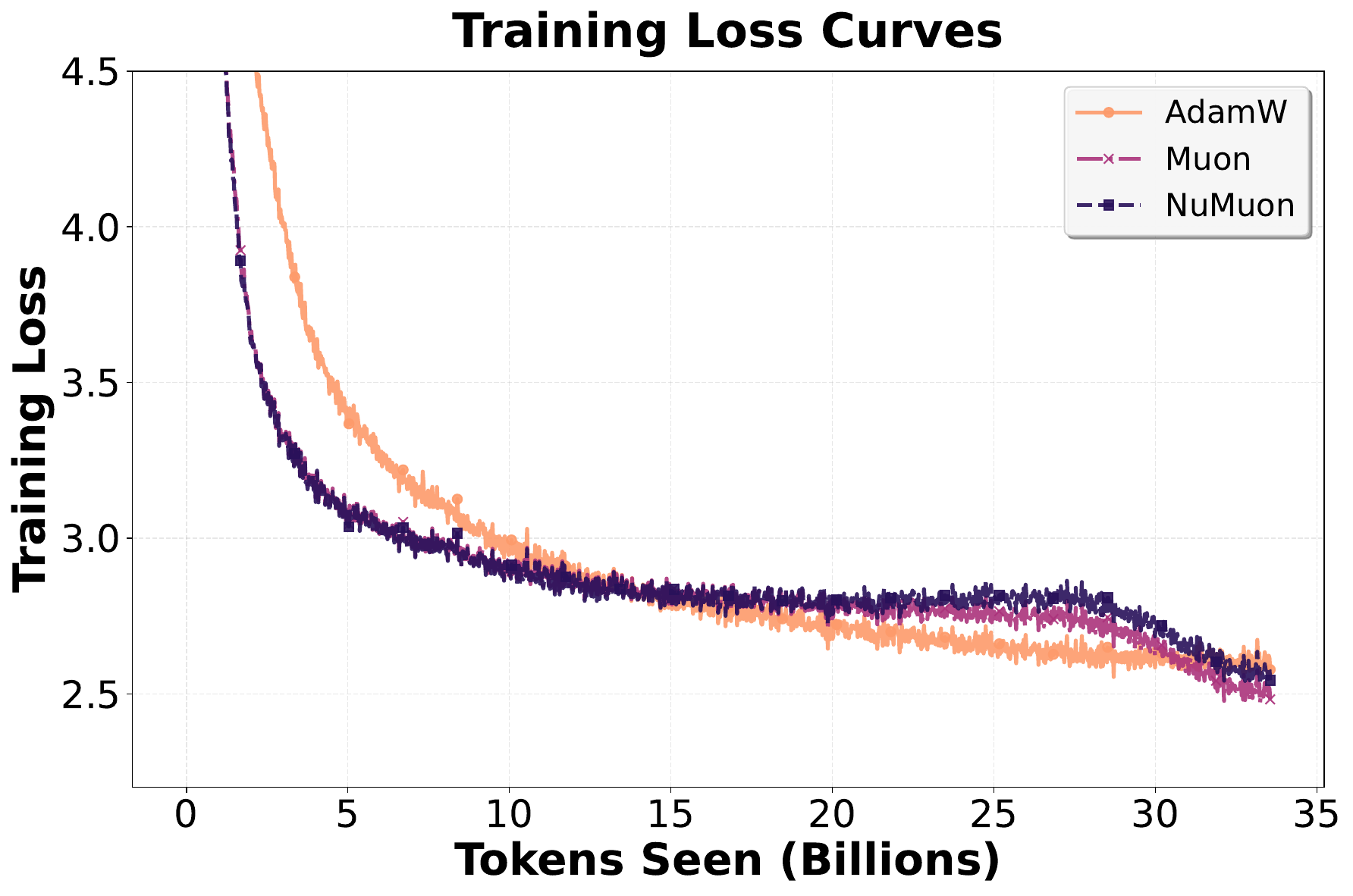}
        \caption{Olmo2-1.4B}
        \label{fig:loss_curves:olmo2}
    \end{subfigure}
    \hspace{-0.5em}
    \begin{center}
        \begin{subfigure}[b]{0.4\textwidth}
            \centering
            \includegraphics[width=\textwidth]{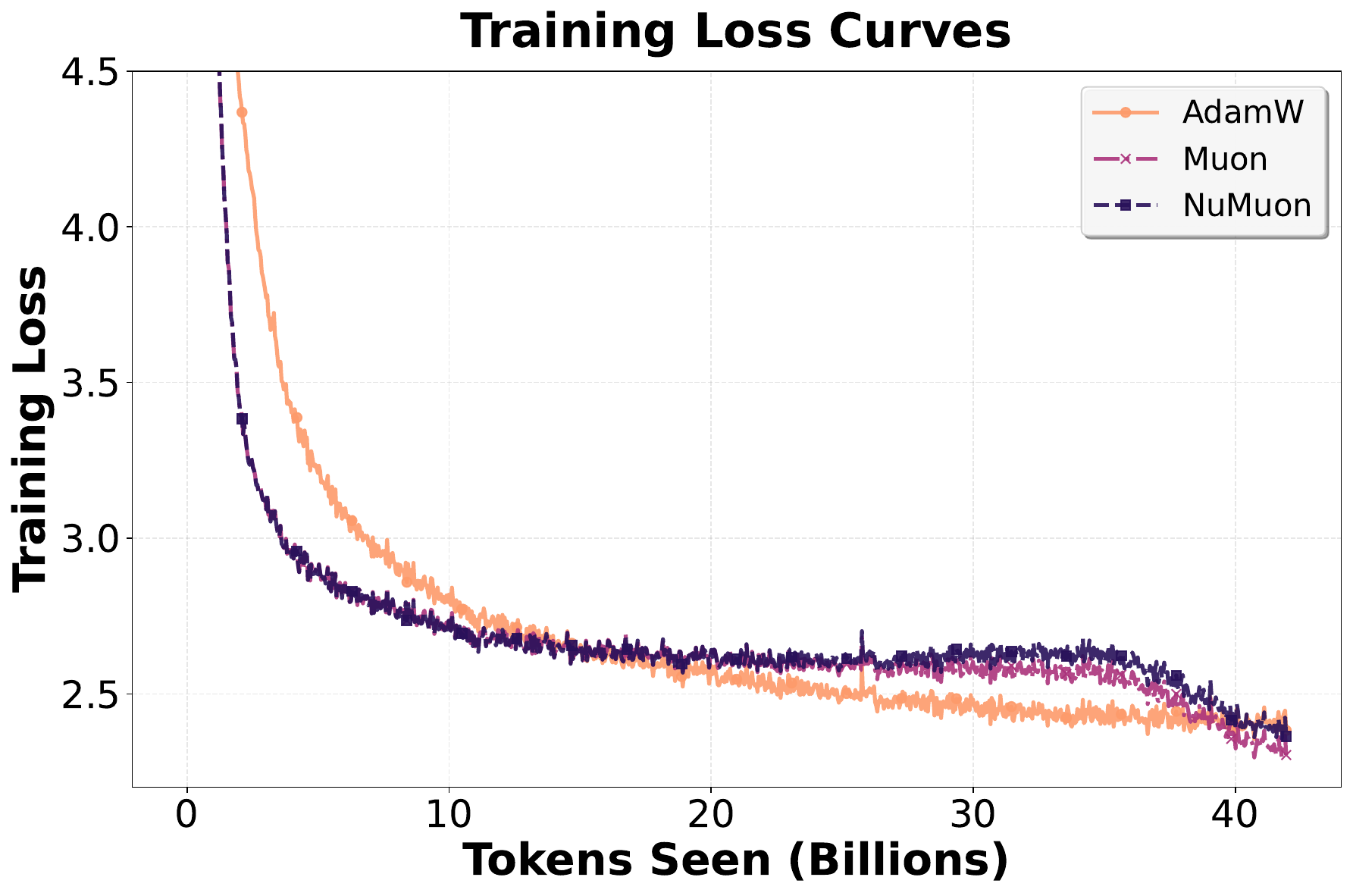}
            \caption{Llama3-1.8B}
            \label{fig:loss_curves:llama3}
        \end{subfigure}
    \end{center}
    \caption{Training loss convergence for language models of size 0.6B-1.8B parameter count. For each model family, we use AdamW, Muon, and NuMuon to train the model. For more details, please see~\Cref{app:extended_results:details}.}
    \label{fig:loss_curves}
\end{figure*}
}

\newcommand{\figLRSchedulesLarge}{
\begin{figure*}[t!]
    \centering
    \begin{subfigure}[b]{0.4\textwidth}
        \centering
        \includegraphics[width=\textwidth]{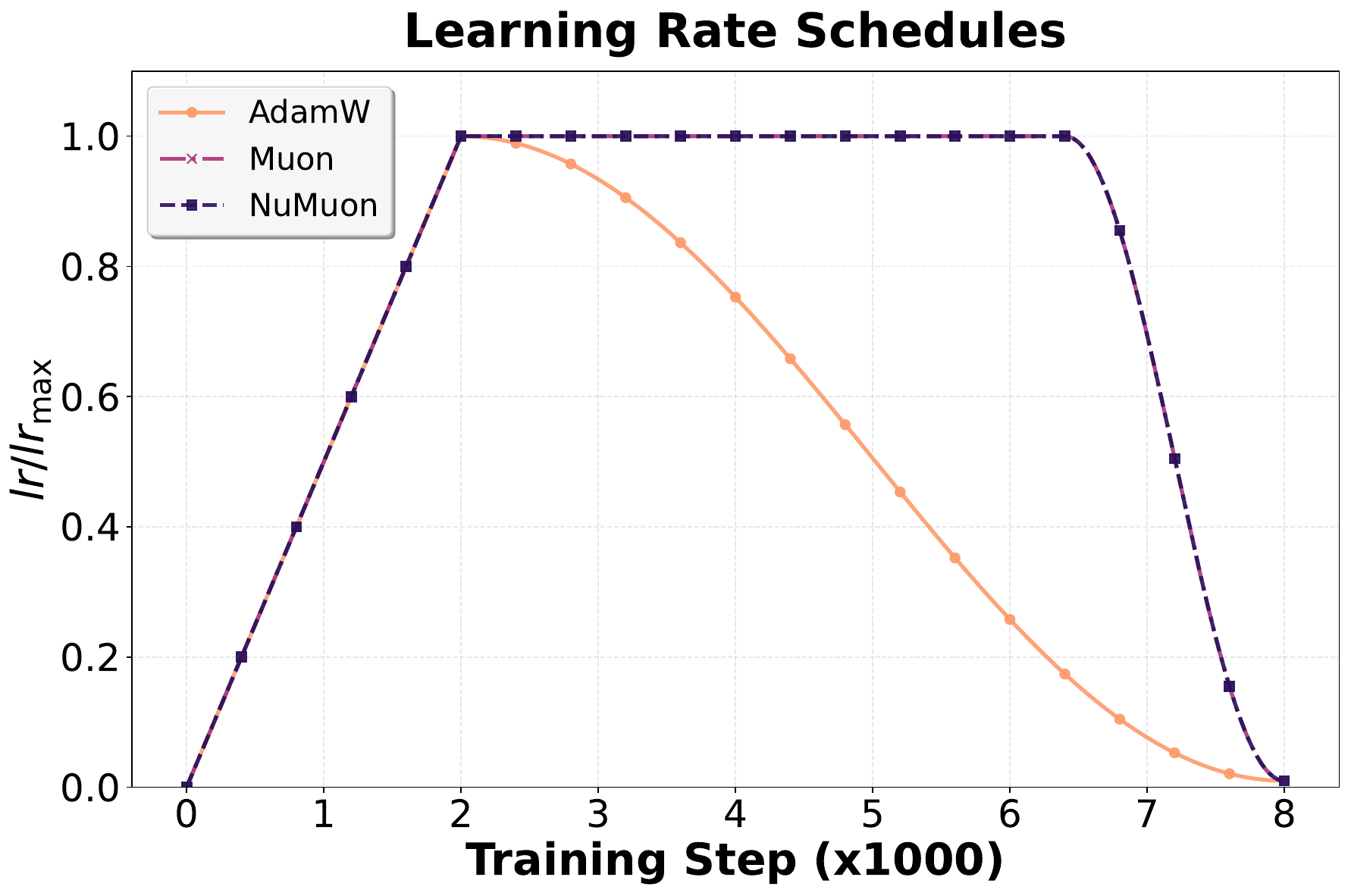}
        \caption{Qwen3-0.6B}
        \label{fig:lr_scheduler:qwen3}
    \end{subfigure}
    \hspace{-0.5em}
    \begin{subfigure}[b]{0.4\textwidth}
        \centering
        \includegraphics[width=\textwidth]{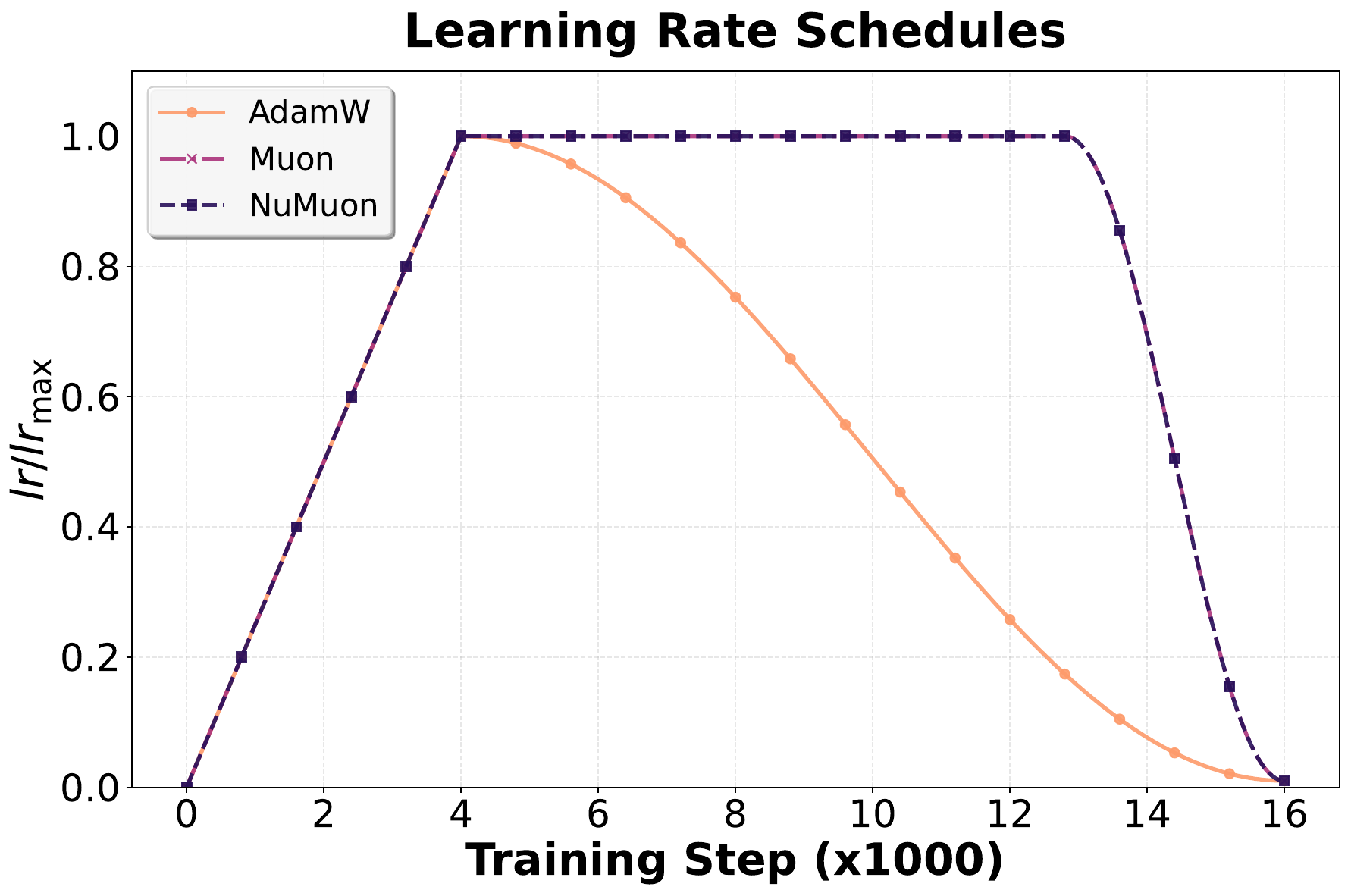}
        \caption{Olmo2-1.4B}
        \label{fig:lr_scheduler:olmo2}
    \end{subfigure}
    \\\vspace{1em}
    \begin{center}
        \begin{subfigure}[b]{0.4\textwidth}
            \centering
            \includegraphics[width=\textwidth]{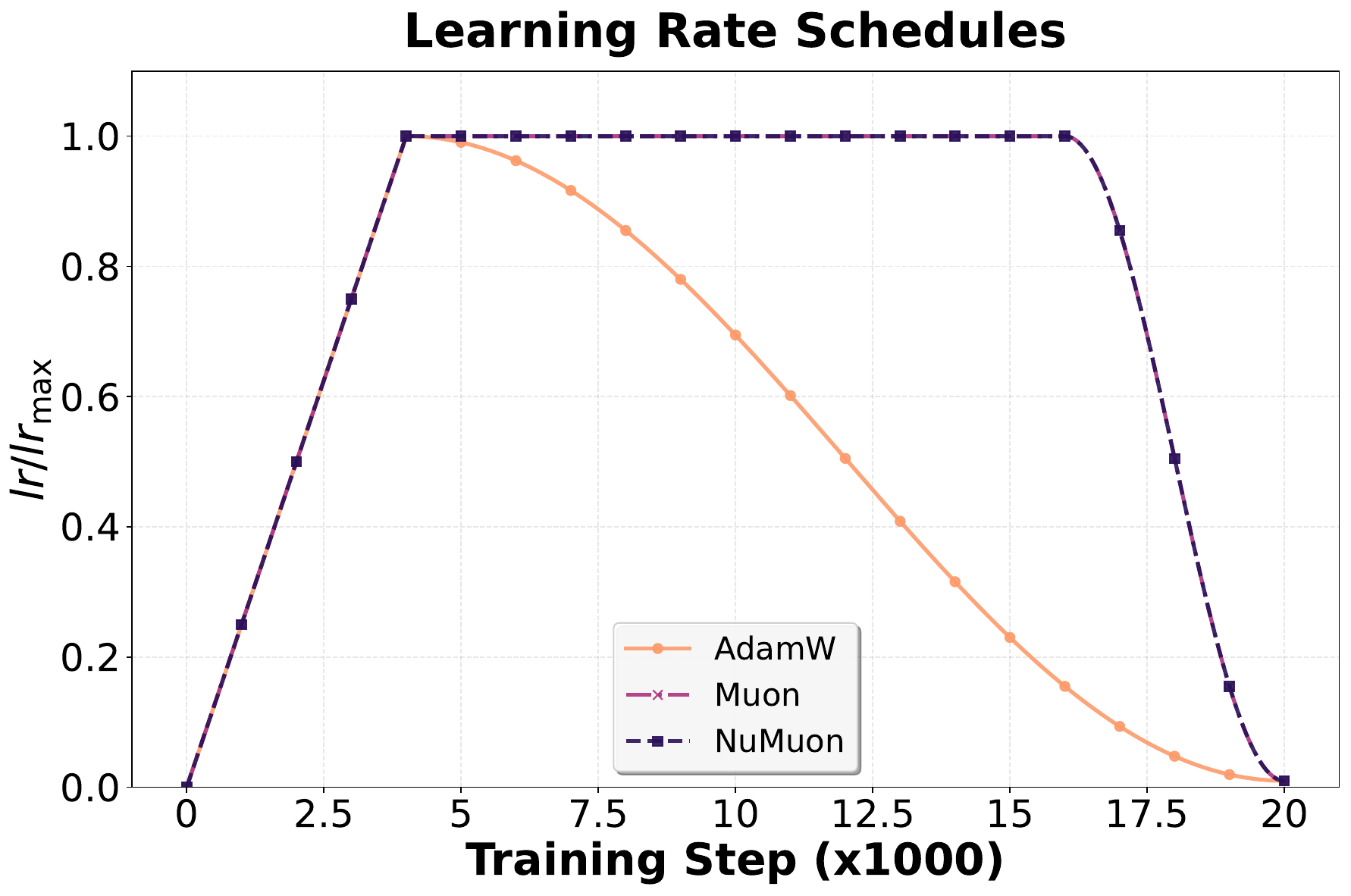}
            \caption{Llama3-1.8B}
            \label{fig:lr_scheduler:llama3}
        \end{subfigure}
    \end{center}
    \caption{Normalized learning rate scheduler for training language models of size 0.6B-1.8B parameter count. For each model family, we use AdamW, Muon, and NuMuon to train the model. Following~\citet{semenov2025benchmarking} extensive benchmarks, we use a WSD~\citep{hu2024minicpm} schedule for training Muon and NuMuon, while we use a cosine decay for AdamW. For more details, please see~\Cref{app:extended_results:details}.}
    \label{fig:lr_scheduler}
\end{figure*}
}

\newcommand{\figLossVsRankAblation}{
    \centering
    \includegraphics[width=\linewidth]{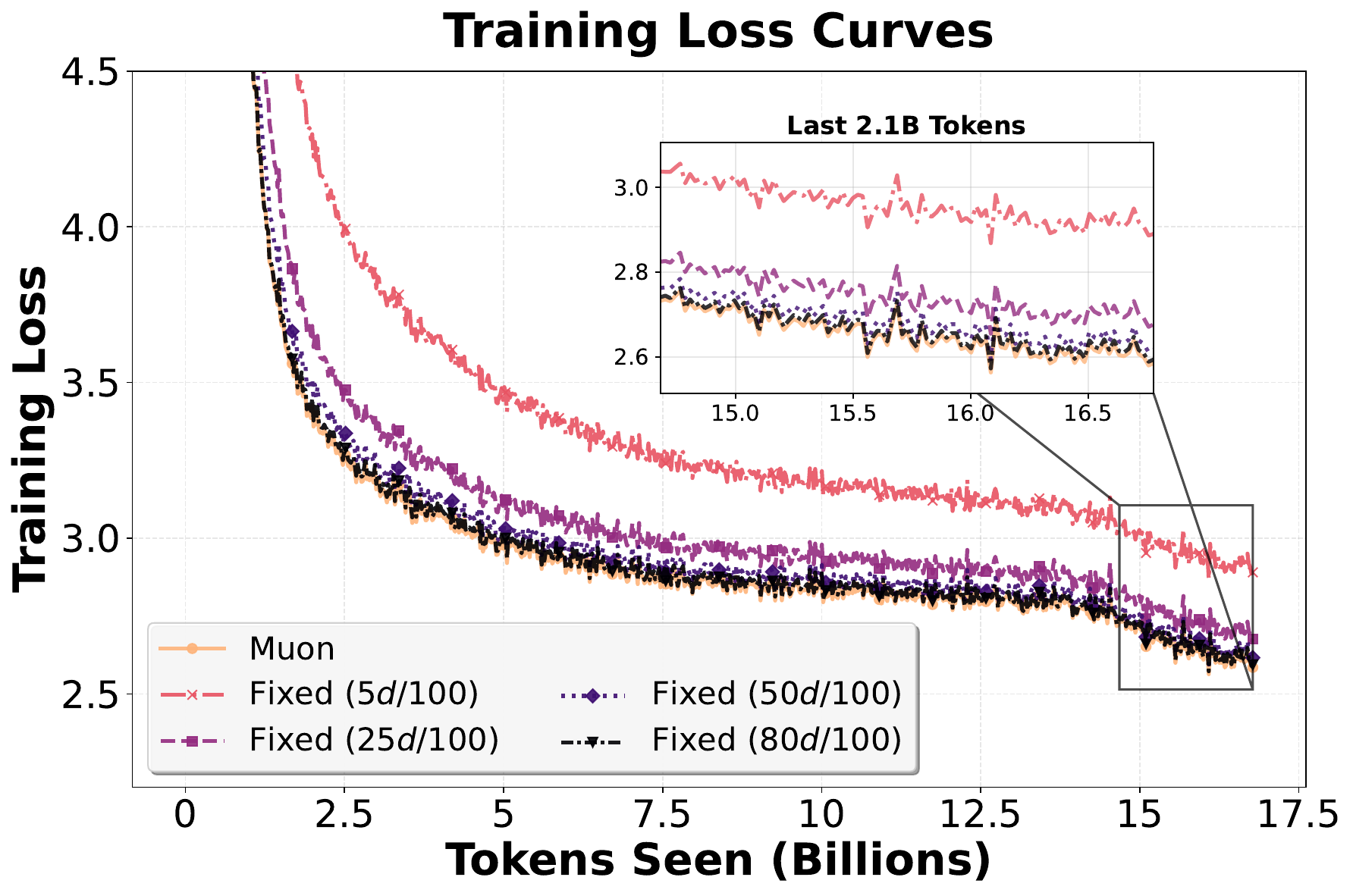}
    \captionof{figure}{Convergence behavior of various rank budgets ($\tau$ in our nuclear norm formulation) used for training Qwen3-0.6B models with NuMuon. Muon training loss is also shown for comparison. As seen, NuMuon gets closer to Muon behavior as we increase the rank budget. However, this rank increase provides diminishing returns while hurting the final model's compressibility.}
    \label{fig:loss_vs_rank_budget}
}

\newcommand{\figLossVsSchedAblation}{
    \centering
    \includegraphics[width=\textwidth]{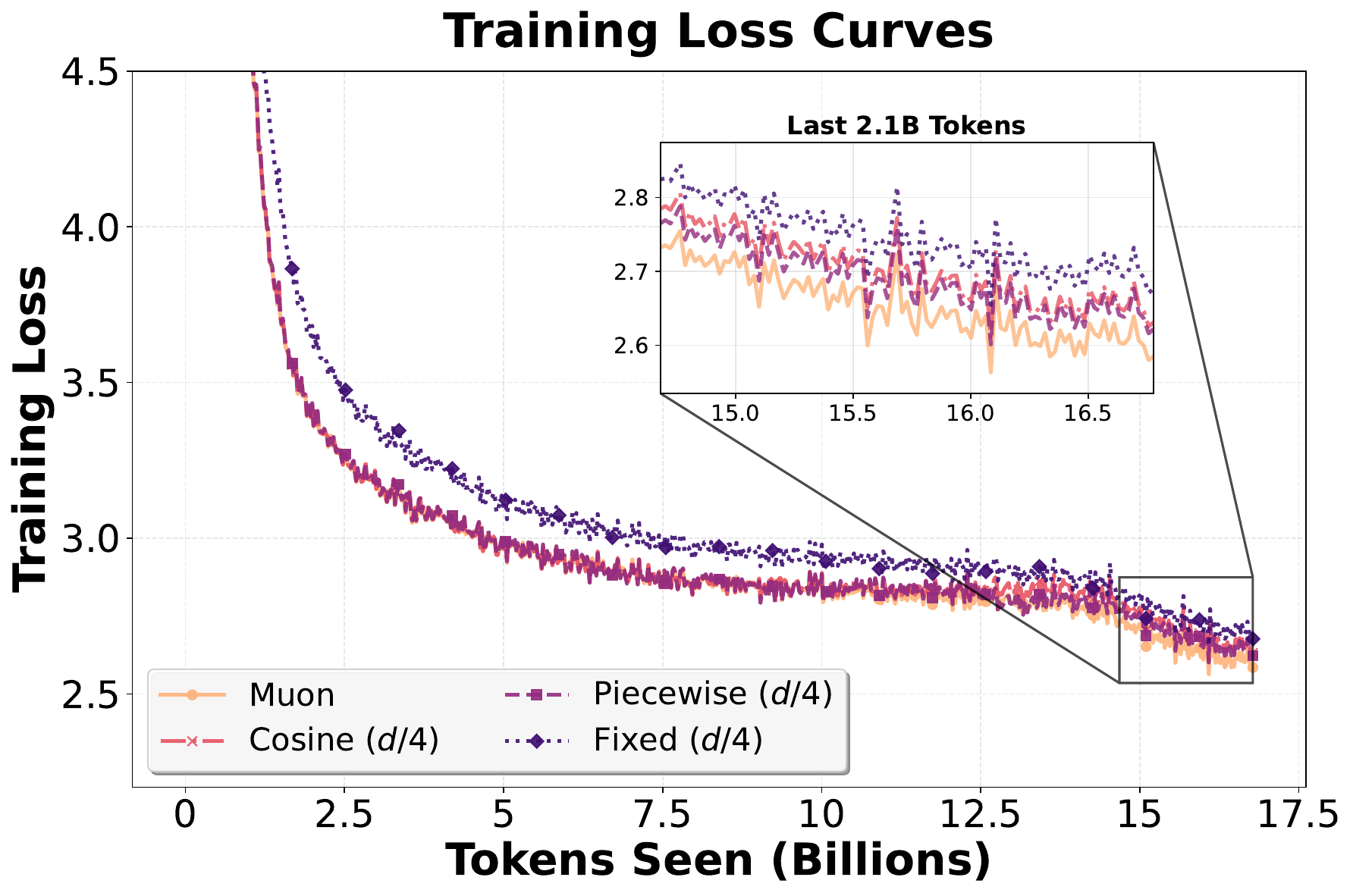}
    \captionof{figure}{Convergence behavior of rank schedulers used for training Qwen3-0.6B.}
    \label{fig:loss_vs_rank_sched}
}

\newcommand{\figGrassmannQwenSmall}{
    \centering
    \includegraphics[width=\textwidth]{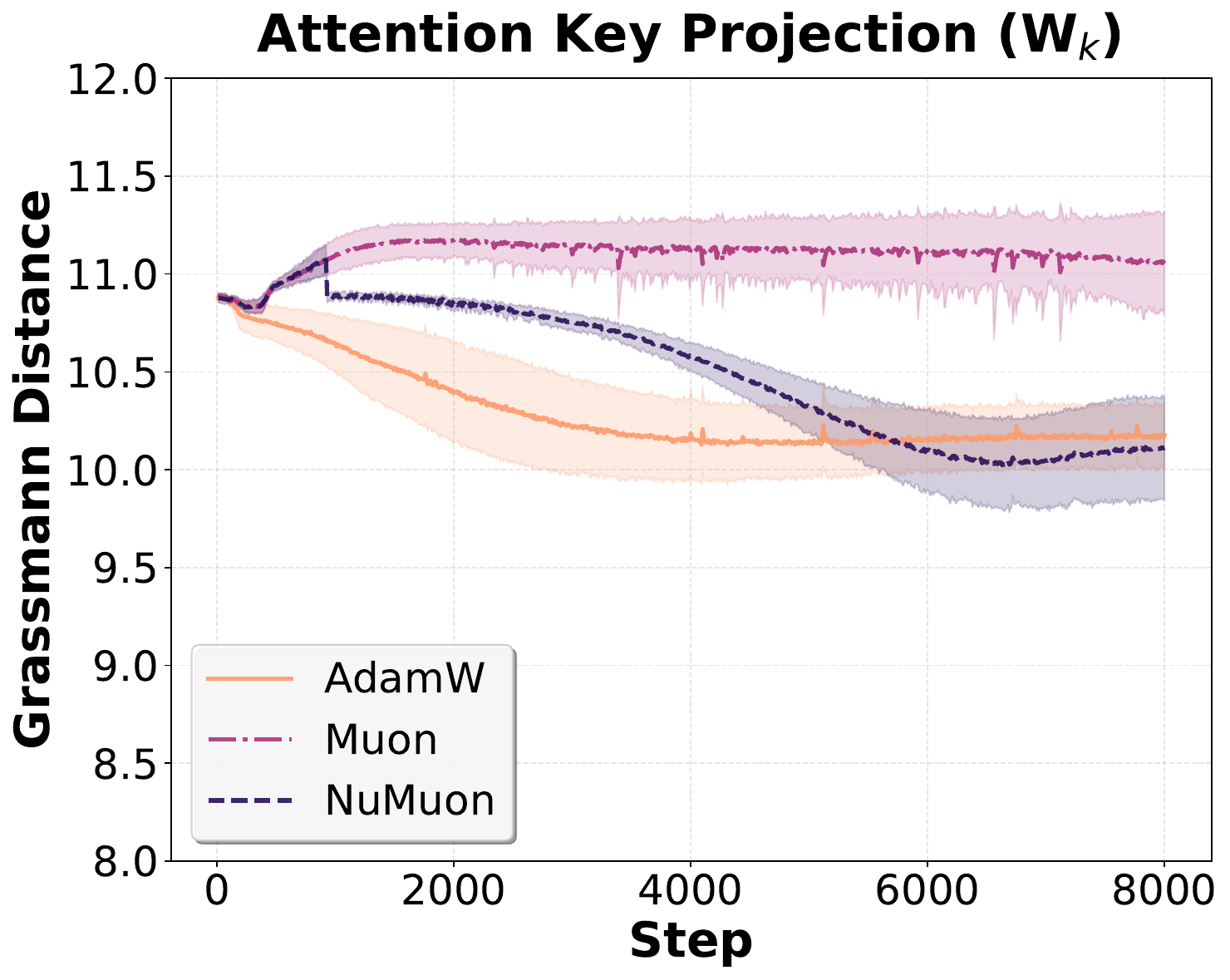}
    \captionof{figure}{$d_G(\boldsymbol{U_W},\boldsymbol{U_{\Delta W}})$. For other weight matrices, please see~\Cref{fig:grassman_distance}.}
    \label{fig:grassman_distance_wk}
}

\newcommand{\figGrassmannQwenLarge}{
\begin{figure*}[t!]
    \centering
    \begin{subfigure}[b]{0.32\textwidth}
        \centering
        \includegraphics[width=\textwidth]{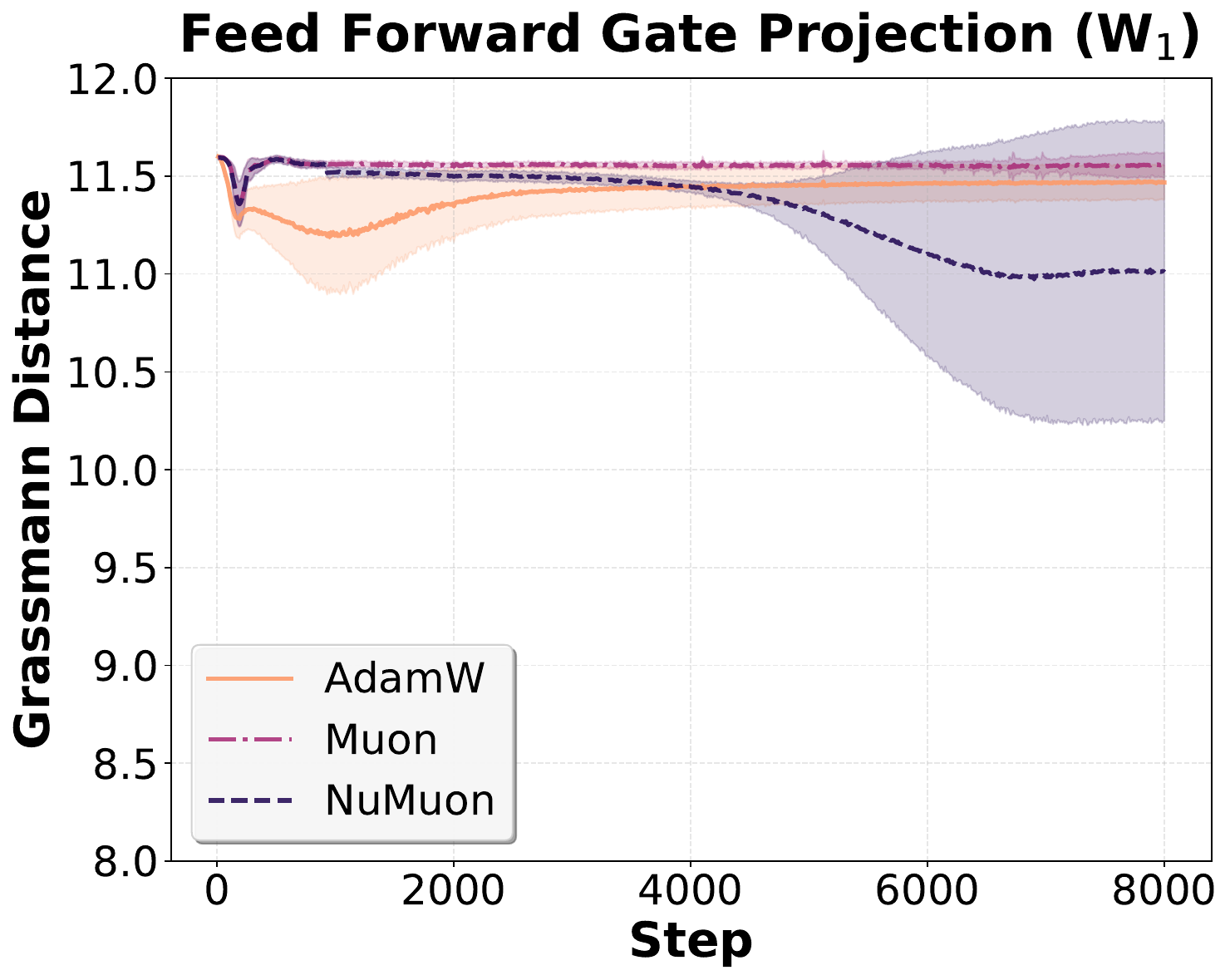}
        \caption{FFN Gate Projection}
        \label{fig:grassman_distance:w1}
    \end{subfigure}
    \hspace{-0.5em}
    \begin{subfigure}[b]{0.32\textwidth}
        \centering
        \includegraphics[width=\textwidth]{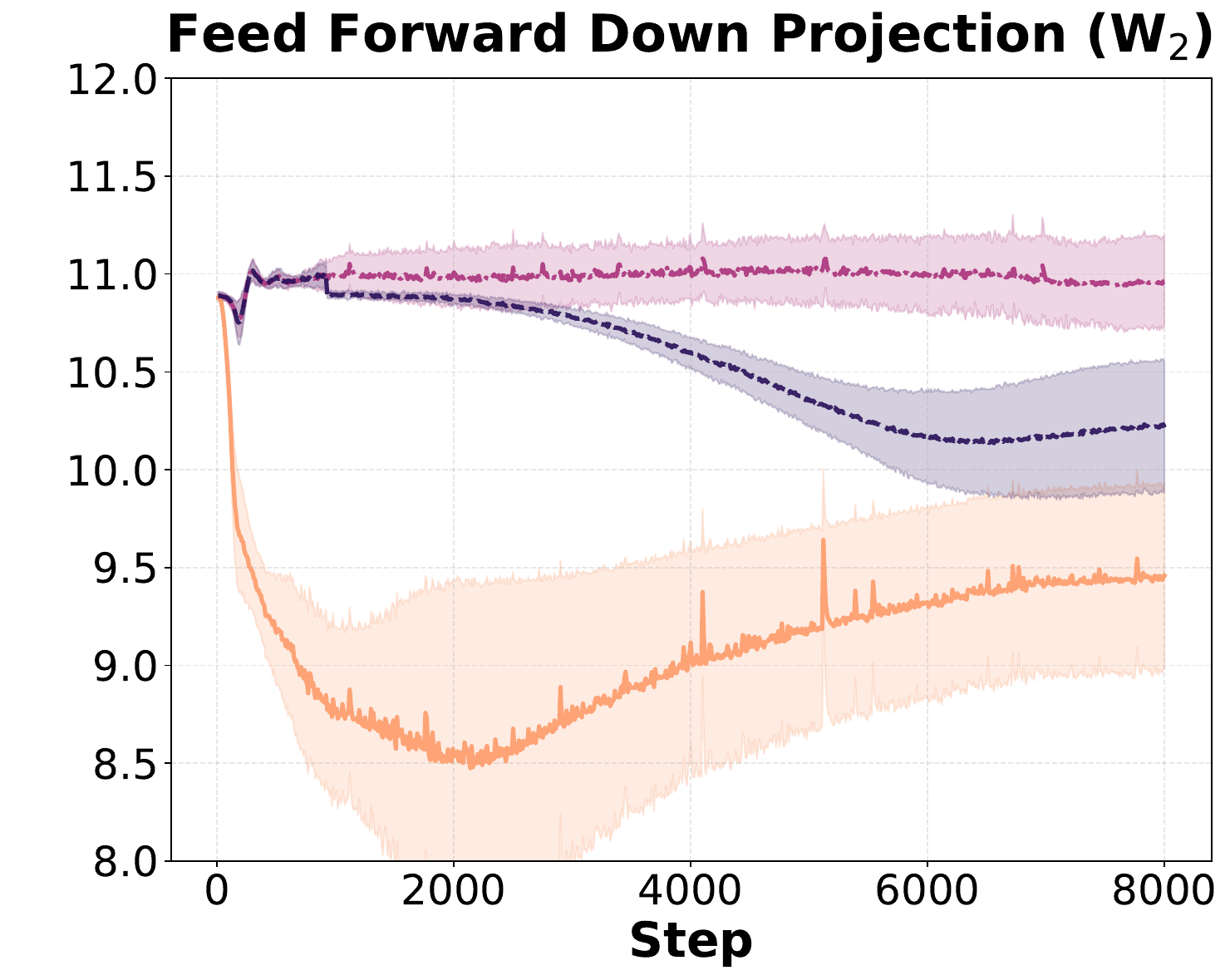}
        \caption{FFN Down Projection}
        \label{fig:grassman_distance:w2}
    \end{subfigure}
    \hspace{-0.5em}
    \begin{subfigure}[b]{0.32\textwidth}
        \centering
        \includegraphics[width=\textwidth]{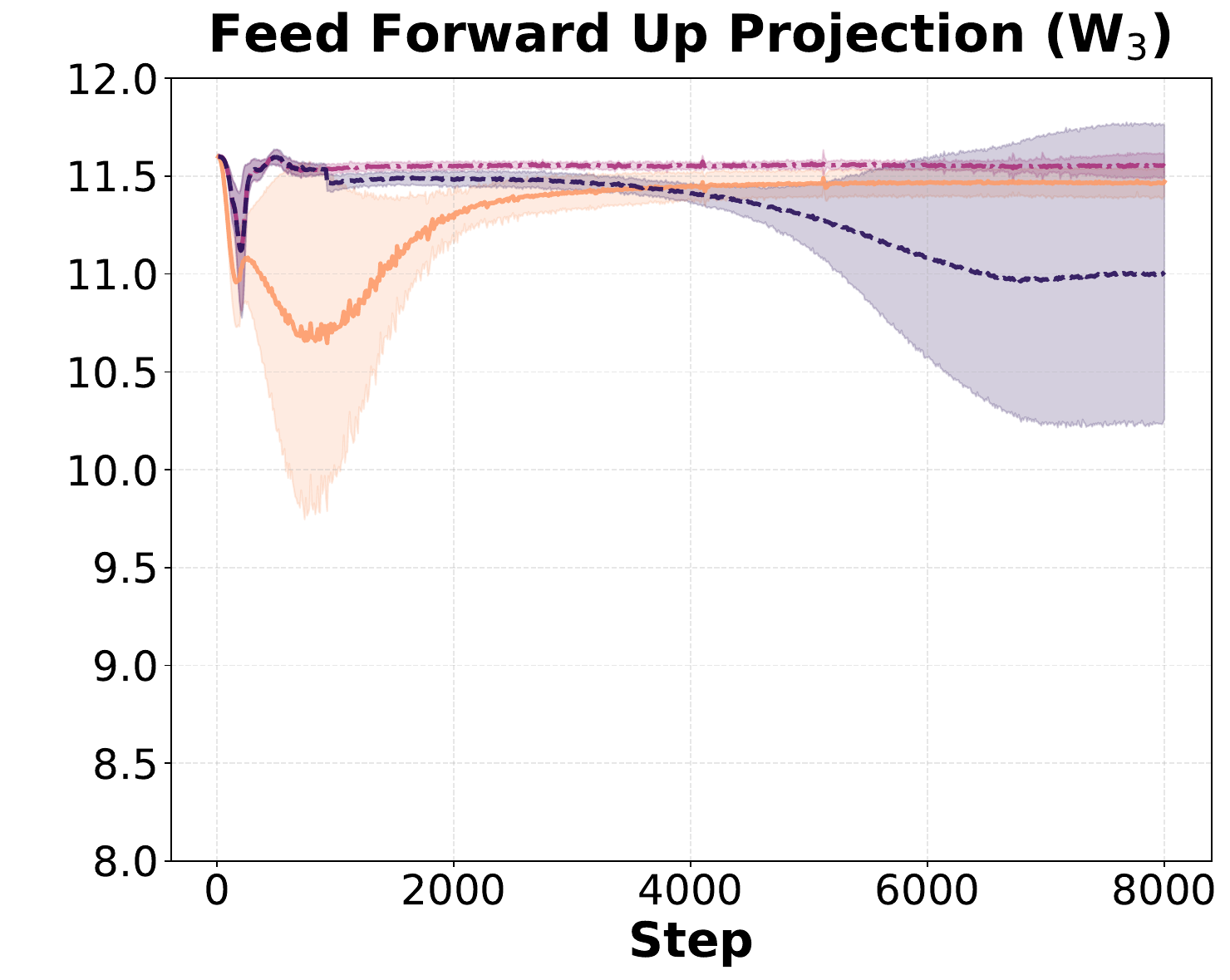}
        \caption{FFN Up Projection}
        \label{fig:grassman_distance:w3}
    \end{subfigure} 
    \\\vspace{1em}
    \begin{subfigure}[b]{0.32\textwidth}
        \centering
        \includegraphics[width=\textwidth]{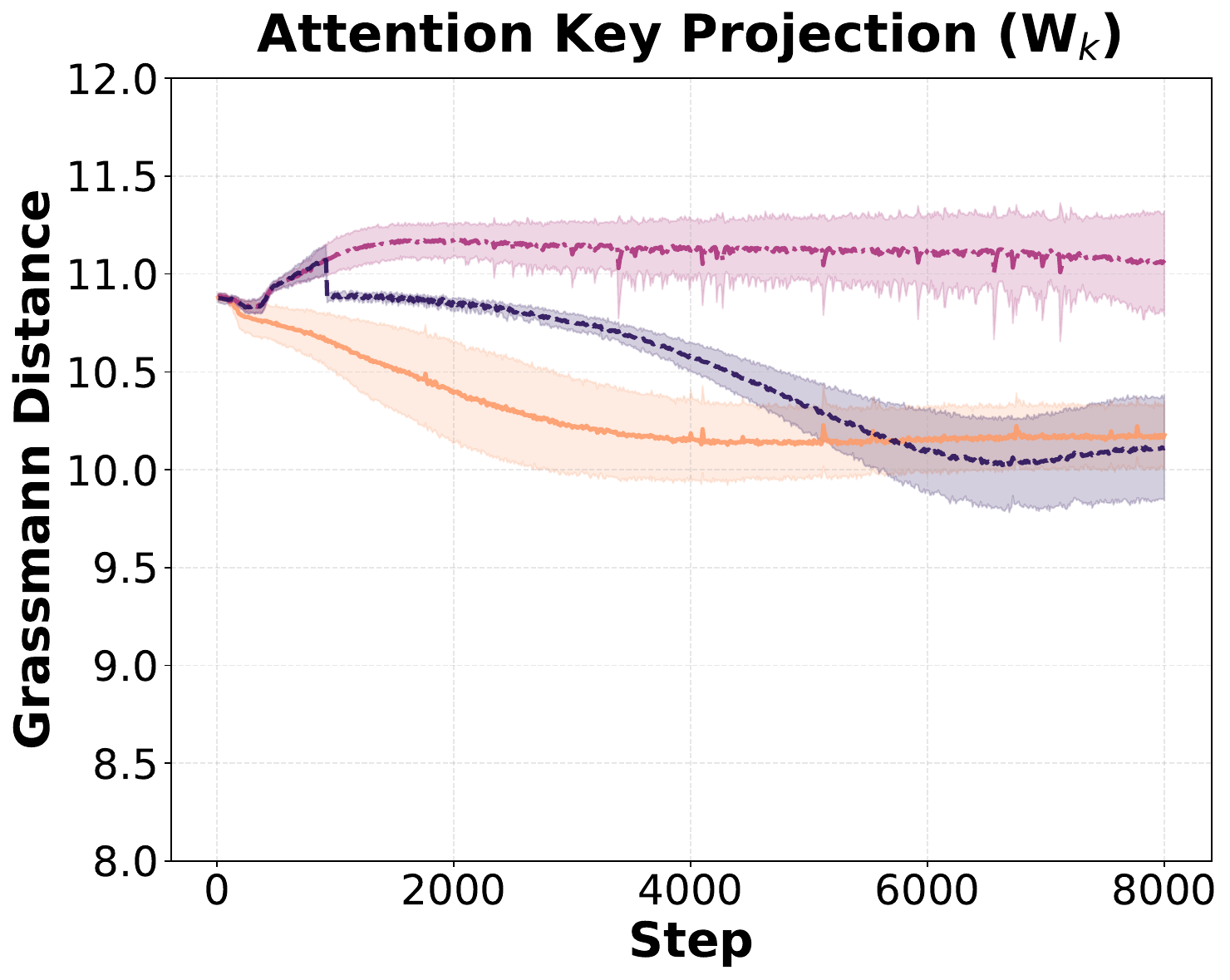}
        \caption{Attn Key Projection}
        \label{fig:grassman_distance:wk}
    \end{subfigure}
    \hspace{-0.5em}
    \begin{subfigure}[b]{0.32\textwidth}
        \centering
        \includegraphics[width=\textwidth]{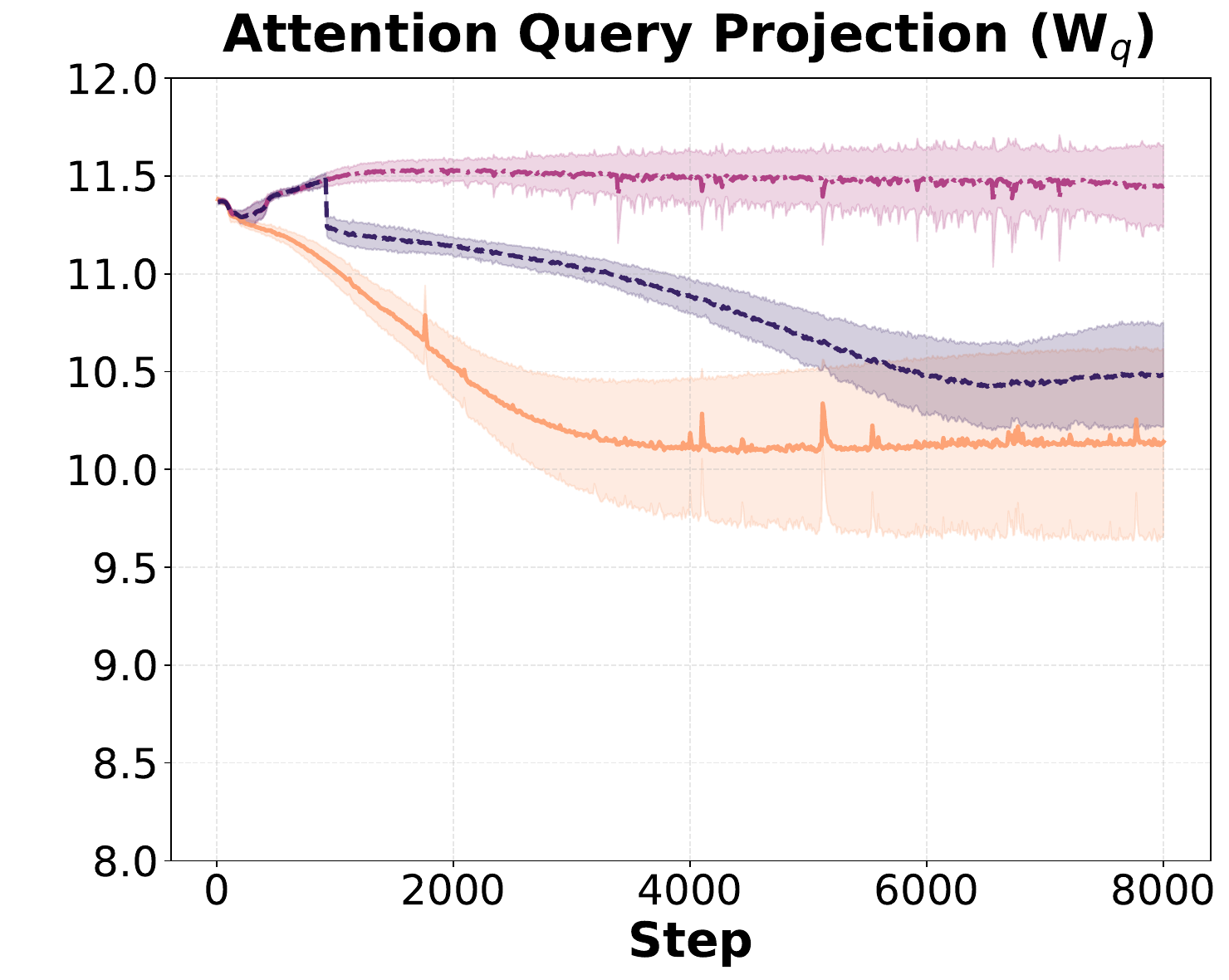}
        \caption{Attn Query Projection}
        \label{fig:grassman_distance:wq}
    \end{subfigure}
    \hspace{-0.5em}
    \begin{subfigure}[b]{0.32\textwidth}
        \centering
        \includegraphics[width=\textwidth]{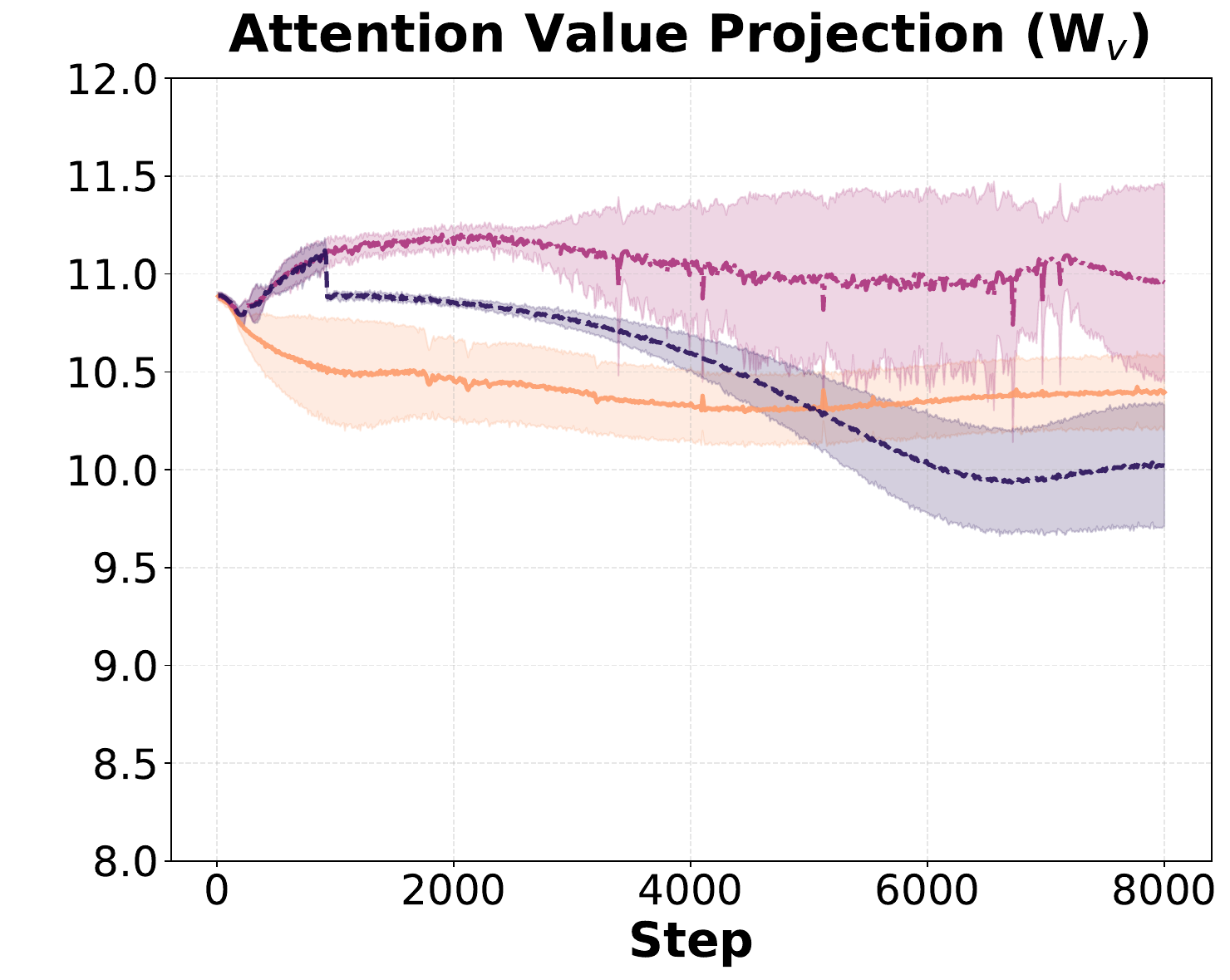}
        \caption{Attn Value Projection}
        \label{fig:grassman_distance:wv}
    \end{subfigure}
    \\\vspace{1em}
    \begin{subfigure}[b]{0.32\textwidth}
    \end{subfigure}
    \begin{subfigure}[b]{0.32\textwidth}
        \centering
        \includegraphics[width=\textwidth]{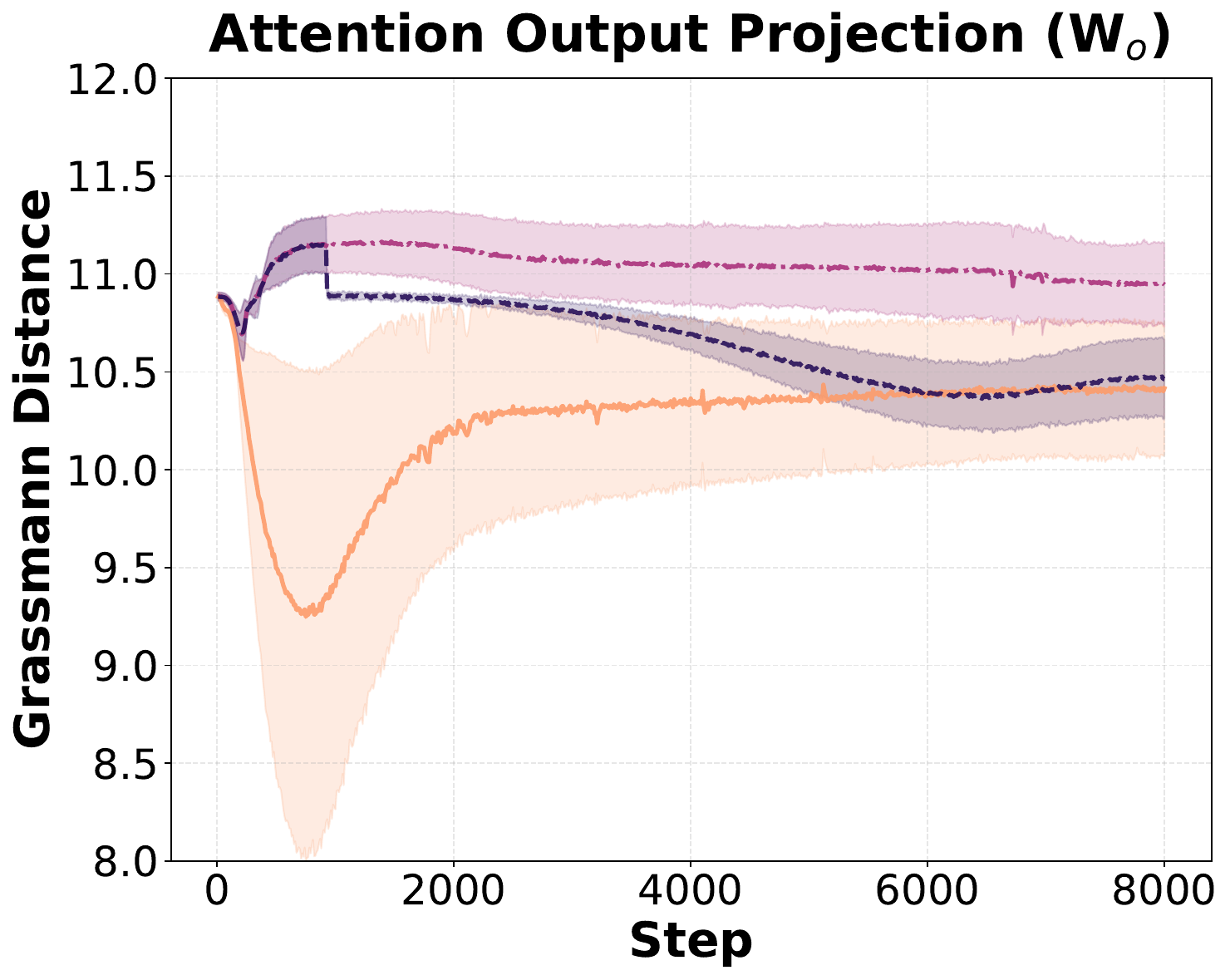}
        \caption{Attn Output Projection}
        \label{fig:grassman_distance:wo}
    \end{subfigure}
    \begin{subfigure}[b]{0.32\textwidth}
    \end{subfigure}
    \caption{Grassmann distance between the top-$64$ left singular vectors of the weight space $\boldsymbol{W}$ and the optimizer update $\boldsymbol{\Delta W}$. For NuMuon, we use the cosine rank scheduler. As demonstrated, while Muon ignores the structure of the weight space in its updates (high, fixed Grassmann distance), NuMuon gradually concentrates the update's energy around the top subspaces of the weight space (lower Grassmann distance).}
    \label{fig:grassman_distance}
\end{figure*}
}

\newcommand{\figEffLlama}{
\begin{figure}[t!]
    \centering
    \includegraphics[width=0.5\textwidth]{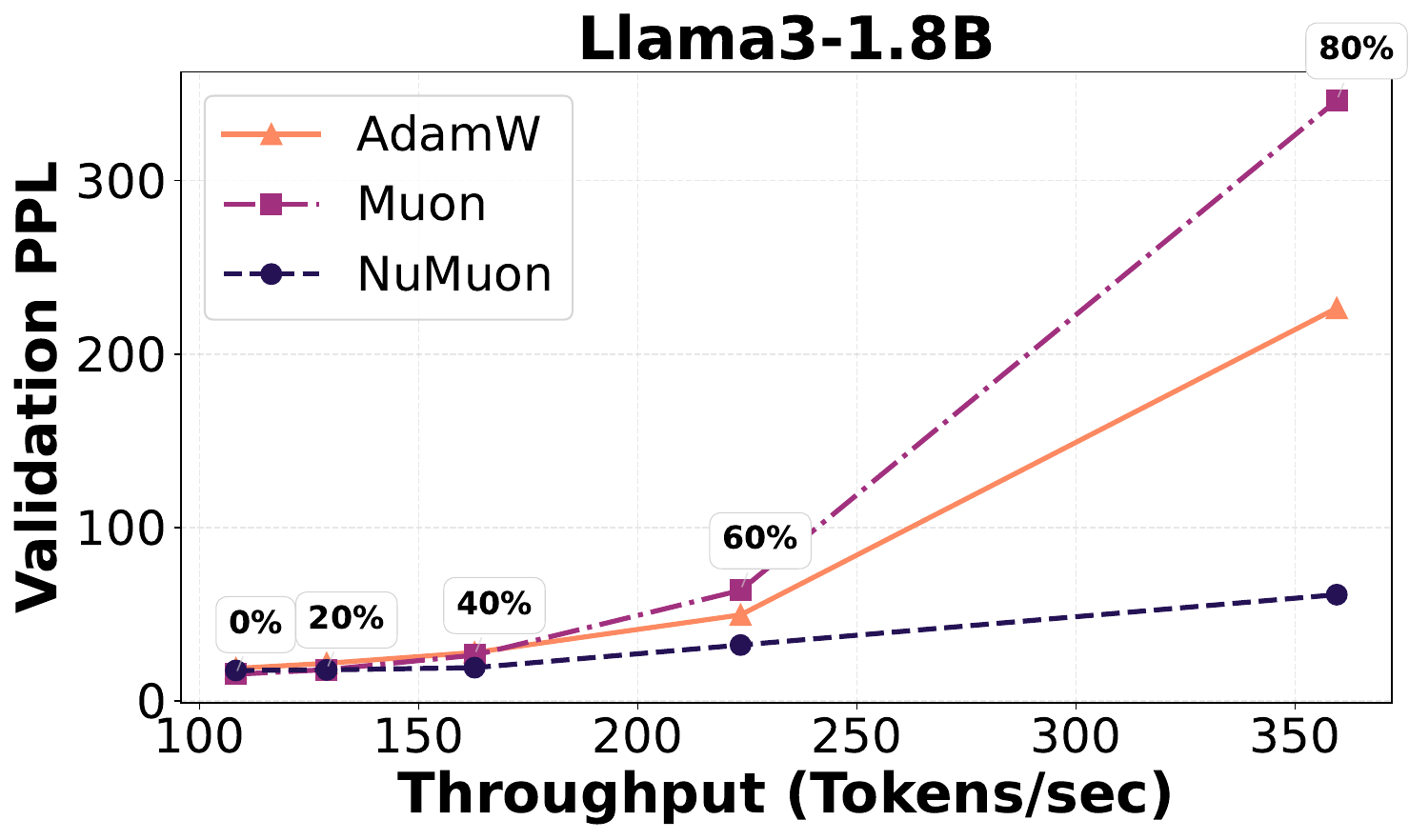}
    \caption{Validation perplexity on WikiText2 against generation inference throughput for Llama3-1.8B models compressed via SVD-LLM~\citep{wang2025svdllm}. As seen, for a given perplexity, NuMuon-trained models provide the fastest inference for moderate to extreme compression rates (40-80\%). Our results for other models can be found in~\Cref{fig:svdllm_efficiency}.}
    \label{fig:efficiency}
    
\end{figure}
}

\newcommand{\figDeltaQwen}{
\begin{figure}[t!]
    \centering
    \begin{subfigure}[b]{0.33\textwidth}
        \centering
        \includegraphics[width=\textwidth]{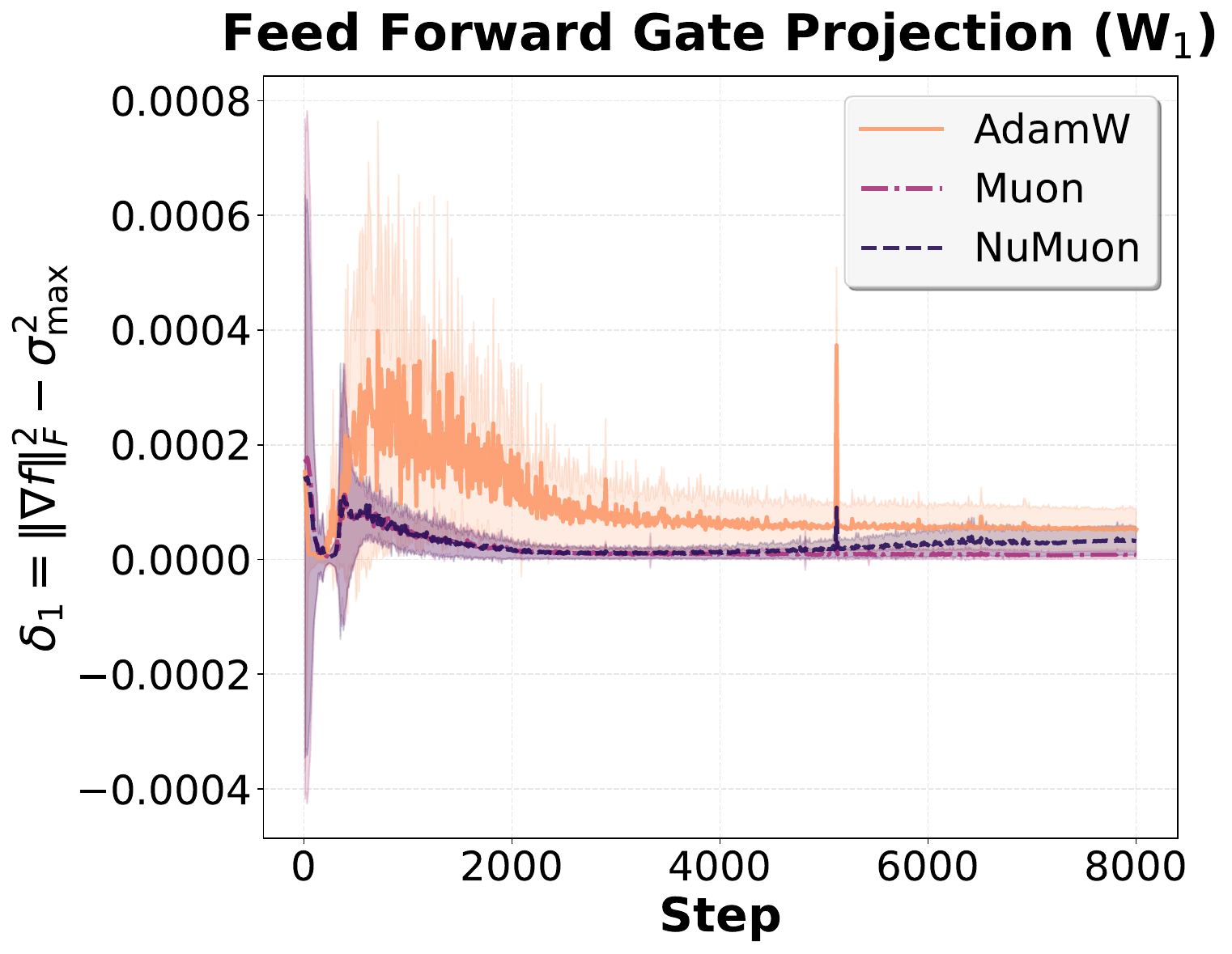}
        \caption{FFN Gate Projection}
        \label{fig:delta_1_ffn:w1}
    \end{subfigure}
    \hspace{-0.5em}
    \begin{subfigure}[b]{0.33\textwidth}
        \centering
        \includegraphics[width=\textwidth]{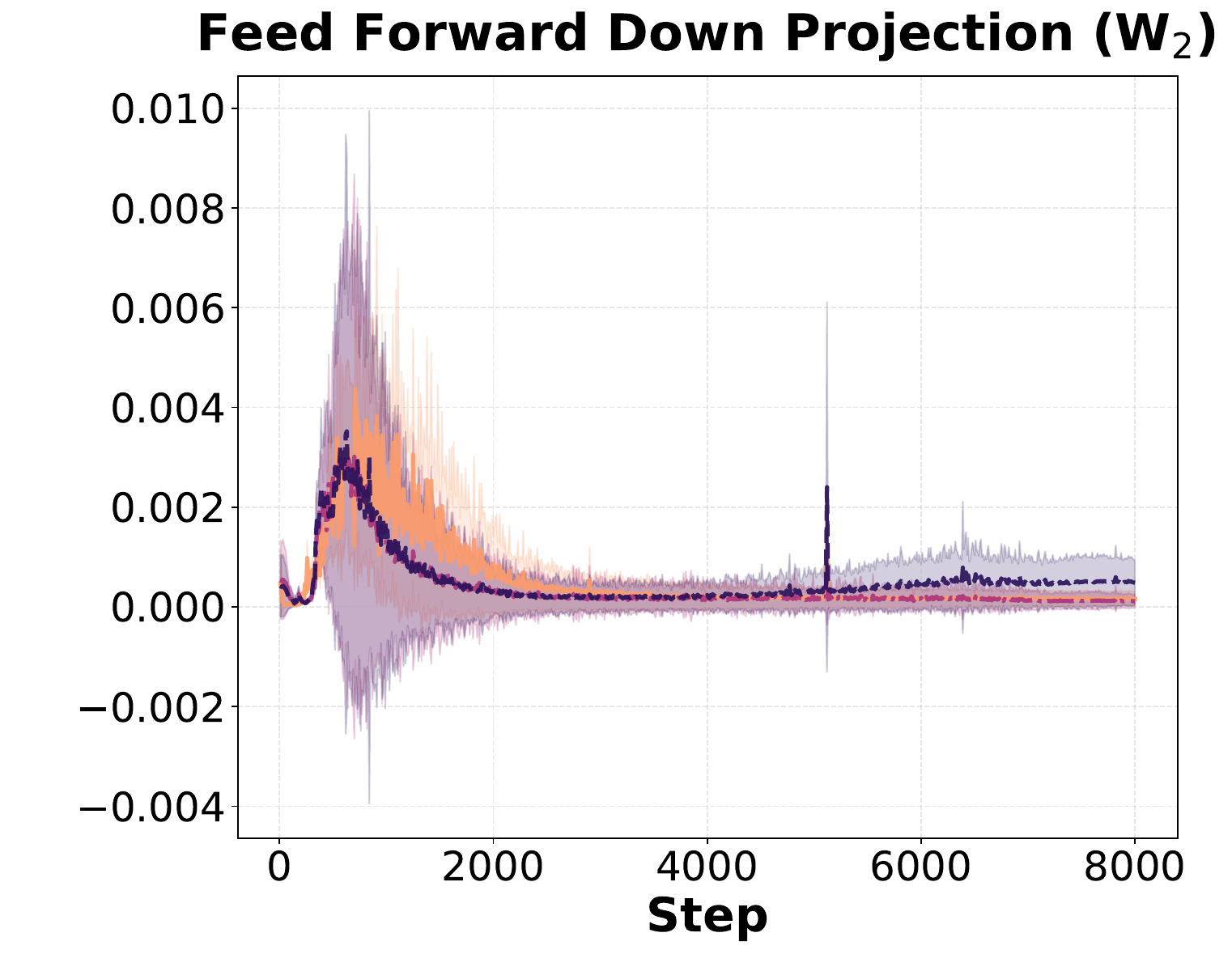}
        \caption{FFN Down Projection}
        \label{fig:delta_1_ffn:w2}
    \end{subfigure}
    \hspace{-0.5em}
    \begin{subfigure}[b]{0.33\textwidth}
        \centering
        \includegraphics[width=\textwidth]{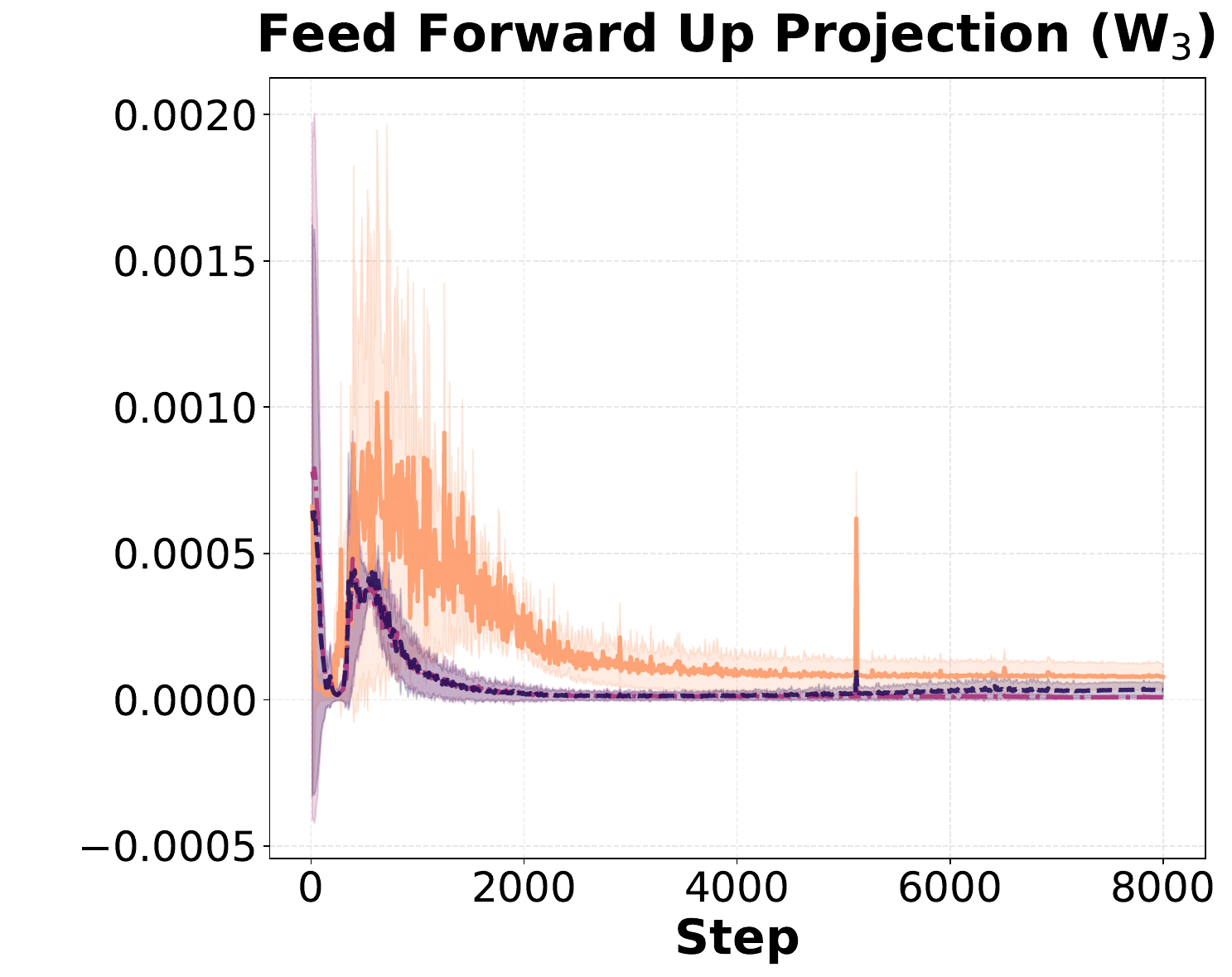}
        \caption{FFN Up Projection}
        \label{fig:delta_1_ffn:w3}
    \end{subfigure}
    \caption{$\delta_1^{(\F)}$ as a proxy for the tail bound in~\Cref{ass:tail_control} for the feedforward projection matrices. As we see, this quantity is bounded and close to zero for NuMuon, supporting this assumption. For other parameters, please see~\Cref{fig:delta_1}.}
    \label{fig:delta_1_ffn}
    \vspace{-1em}
\end{figure}
}

\newcommand{\figDeltaQwenLarge}{
\begin{figure*}[t!]
    \centering
    \begin{subfigure}[b]{0.32\textwidth}
        \centering
        \includegraphics[width=\textwidth]{delta_1/feed_forward_w1.pdf}
        \caption{FFN Gate Projection}
        \label{fig:delta_1:w1}
    \end{subfigure}
    \hspace{-0.5em}
    \begin{subfigure}[b]{0.32\textwidth}
        \centering
        \includegraphics[width=\textwidth]{delta_1/feed_forward_w2.pdf}
        \caption{FFN Down Projection}
        \label{fig:delta_1:w2}
    \end{subfigure}
    \hspace{-0.5em}
    \begin{subfigure}[b]{0.32\textwidth}
        \centering
        \includegraphics[width=\textwidth]{delta_1/feed_forward_w3.pdf}
        \caption{FFN Up Projection}
        \label{fig:delta_1:w3}
    \end{subfigure} 
    \\\vspace{1em}
    \begin{subfigure}[b]{0.32\textwidth}
        \centering
        \includegraphics[width=\textwidth]{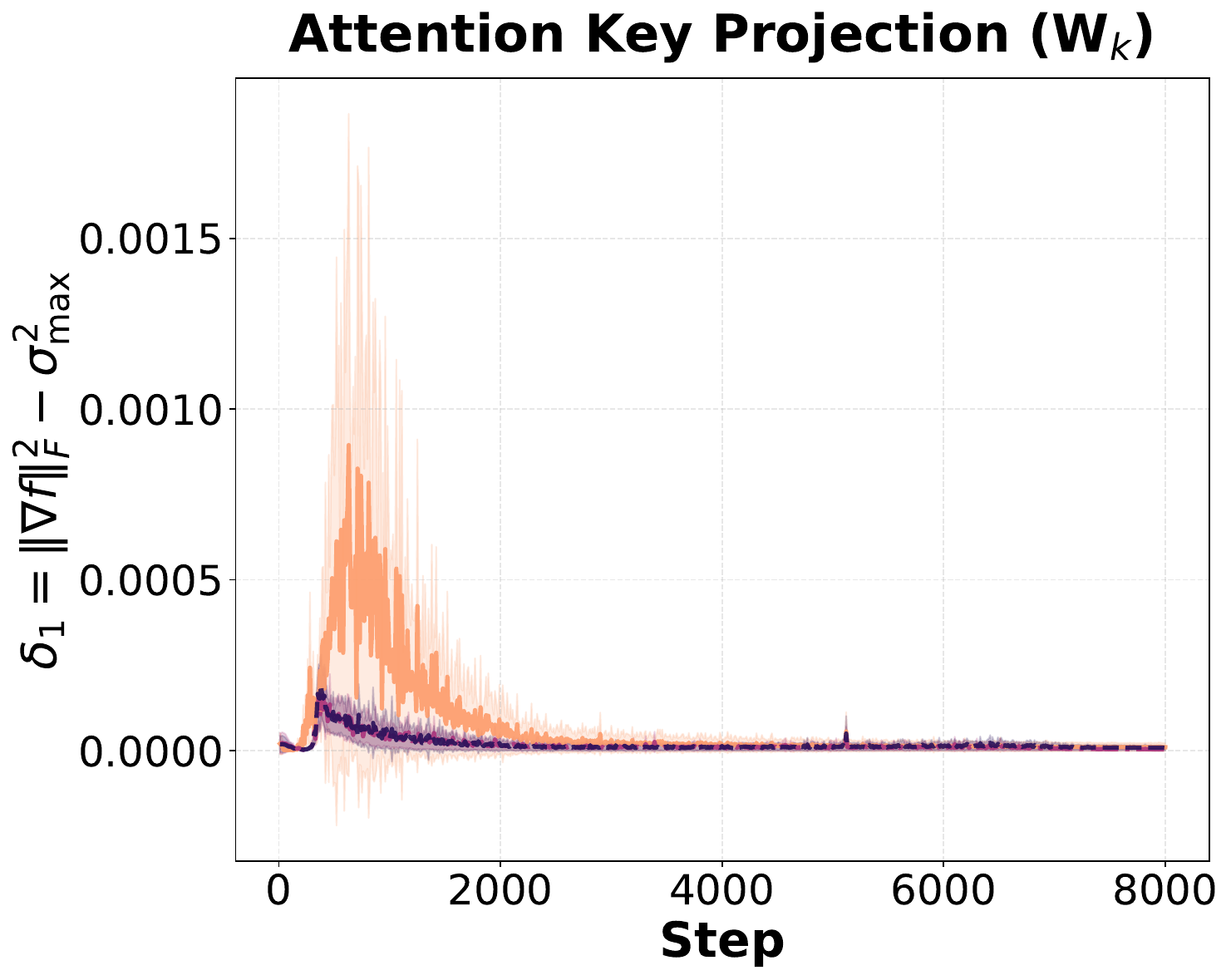}
        \caption{Attn Key Projection}
        \label{fig:delta_1:wk}
    \end{subfigure}
    \hspace{-0.5em}
    \begin{subfigure}[b]{0.32\textwidth}
        \centering
        \includegraphics[width=\textwidth]{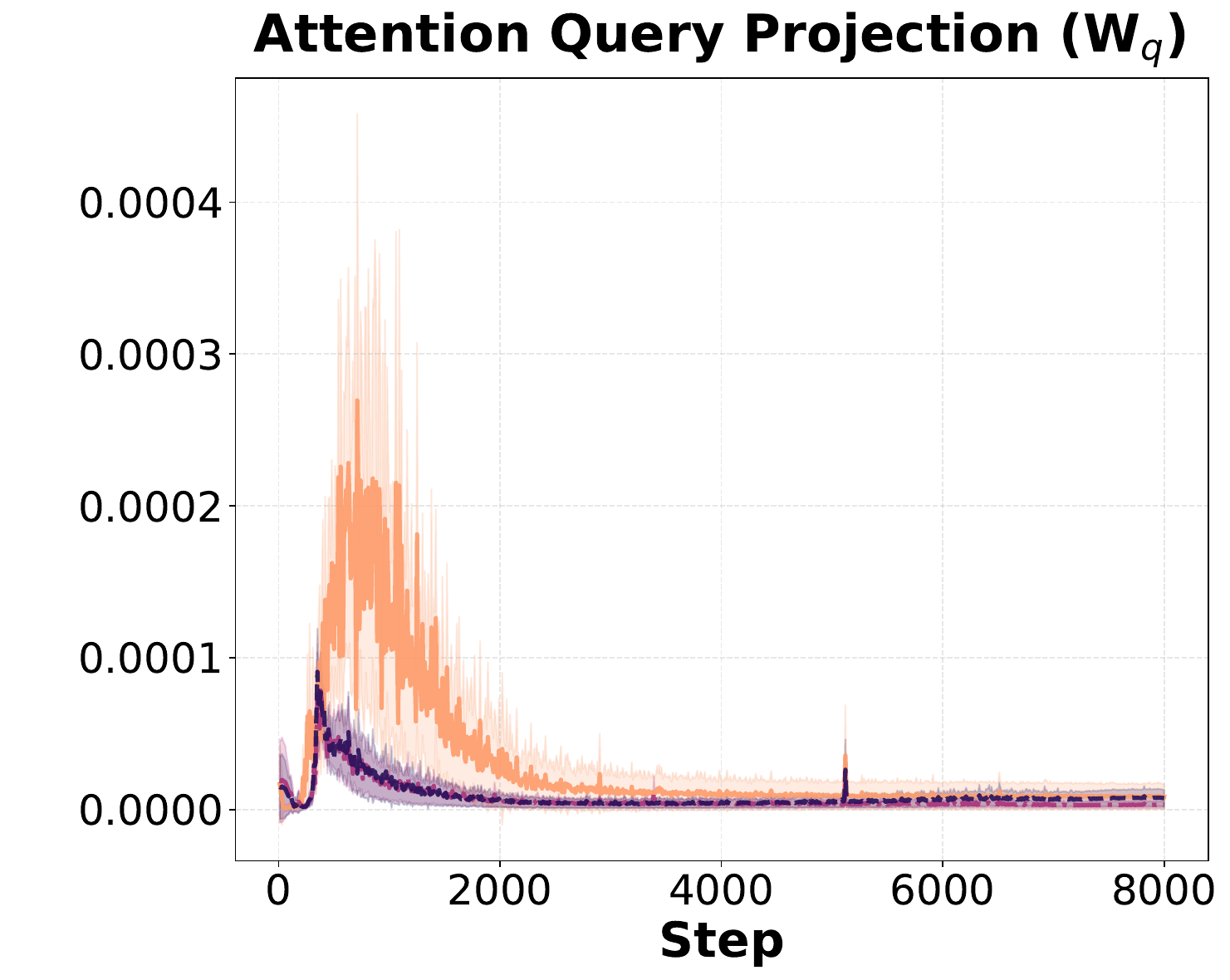}
        \caption{Attn Query Projection}
        \label{fig:delta_1:wq}
    \end{subfigure}
    \hspace{-0.5em}
    \begin{subfigure}[b]{0.32\textwidth}
        \centering
        \includegraphics[width=\textwidth]{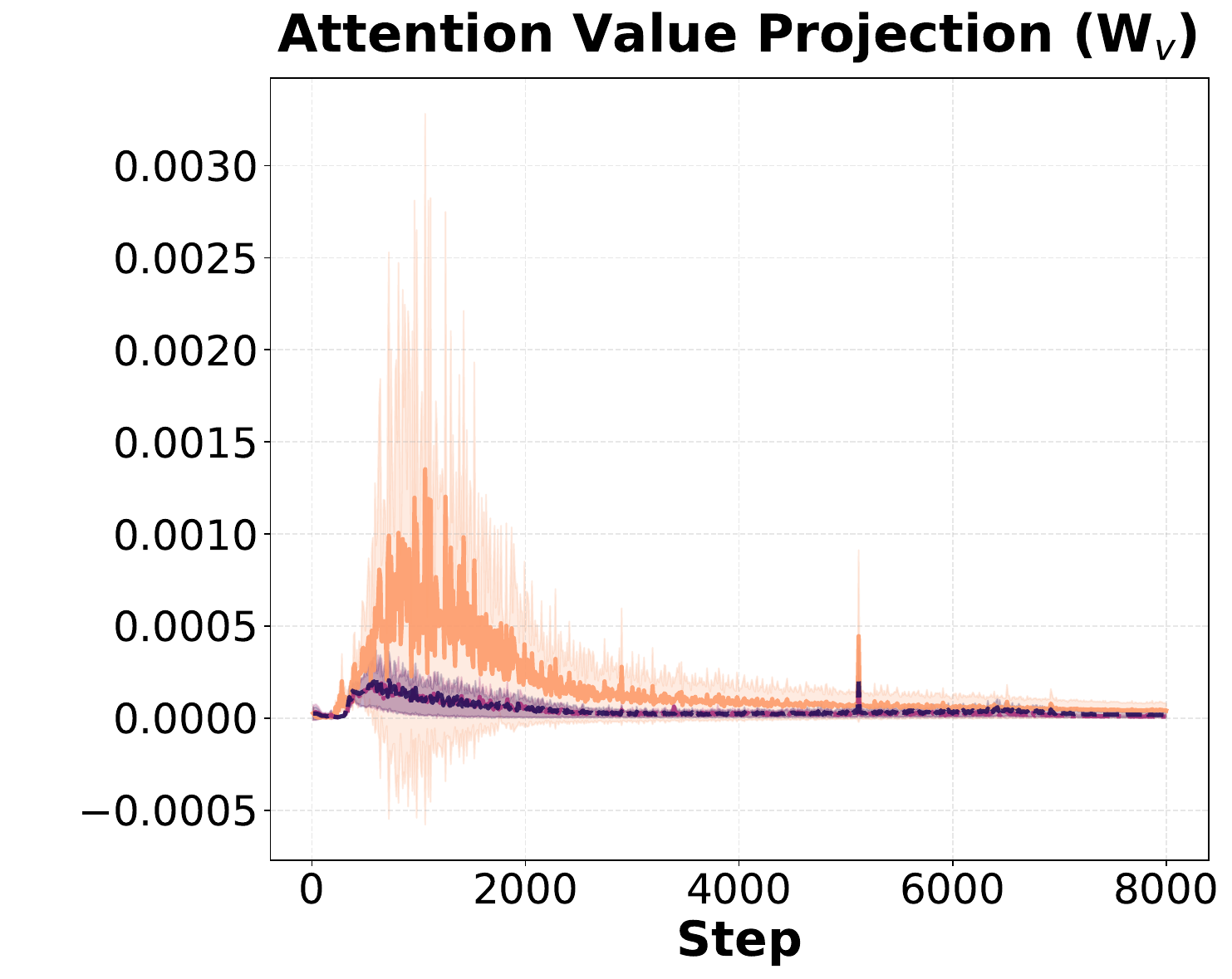}
        \caption{Attn Value Projection}
        \label{fig:delta_1:wv}
    \end{subfigure}
    \\\vspace{1em}
    \begin{subfigure}[b]{0.32\textwidth}
    \end{subfigure}
    \begin{subfigure}[b]{0.32\textwidth}
        \centering
        \includegraphics[width=\textwidth]{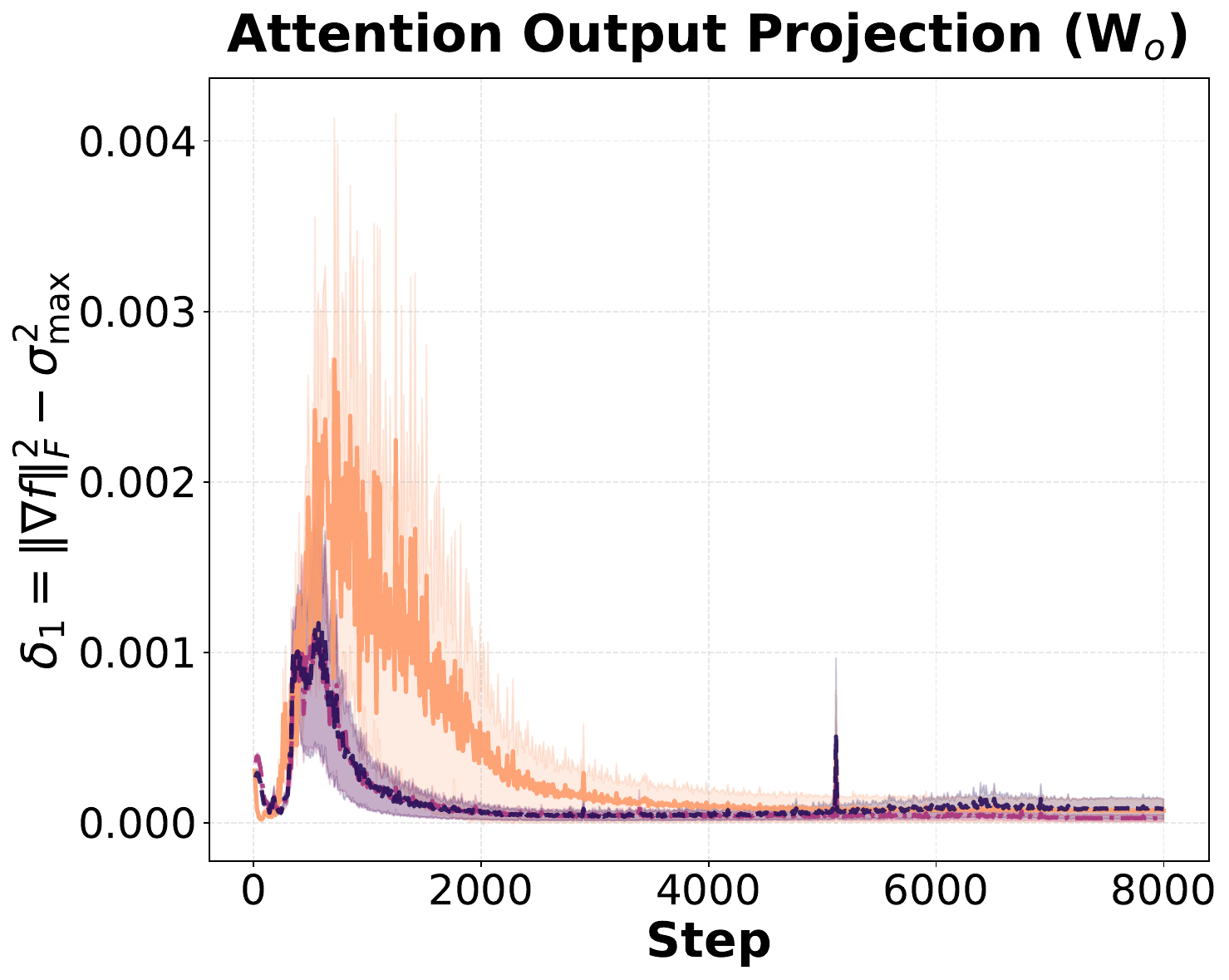}
        \caption{Attn Output Projection}
        \label{fig:delta_1:wo}
    \end{subfigure}
    \begin{subfigure}[b]{0.32\textwidth}
    \end{subfigure}
    \caption{$\delta_1^{(\F)}$ as a proxy for the tail bound in~\Cref{ass:tail_control}. As seen, with NuMuon this quantity gets very close to $0$ as training progresses, indicating that our Ky Fan $k$-norm stationarity convergence gets very close to the standard nuclear-norm stationarity convergence guarantee of \citet{shen2025convergence} for Muon. Among all weight matrices, the feedforward output projection $\boldsymbol{W}_2$ usually exhibits a larger $\delta_1^{(\F)}$ which could potentially require a higher nuclear norm budget. This is interesting research direction which we leave for future work.}
    \label{fig:delta_1}
\end{figure*}
}

\newcommand{\figPPLCompAllMethods}{
\begin{figure*}[t!]
    \centering
    \begin{subfigure}[b]{0.32\textwidth}
        \centering
        \includegraphics[width=\textwidth]{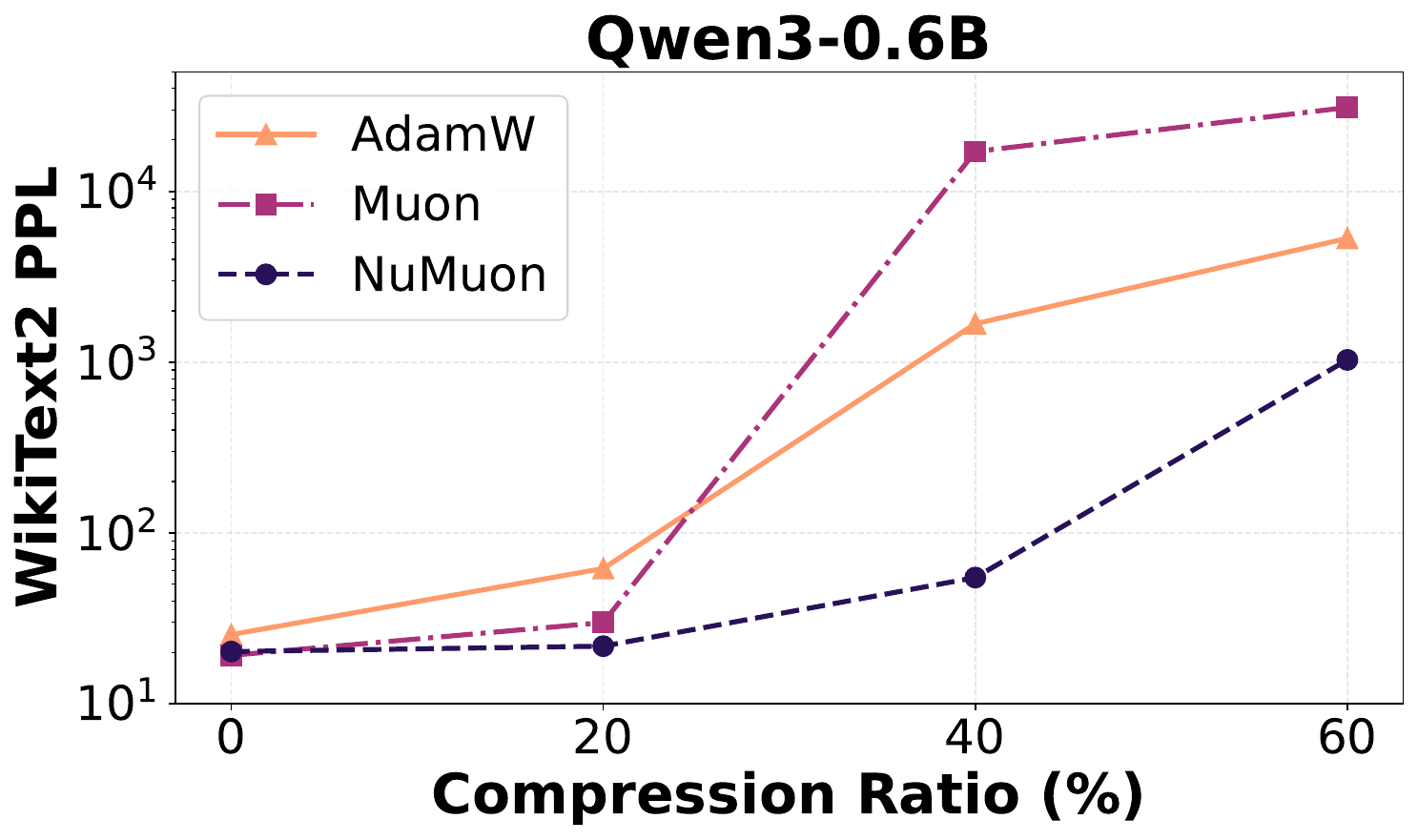}
        \caption*{}
        \label{fig:compression_ppl_comparison:asvd:qwen3}
    \end{subfigure}
    \hspace{-0.5em}
    \begin{subfigure}[b]{0.32\textwidth}
        \centering
        \includegraphics[width=\textwidth]{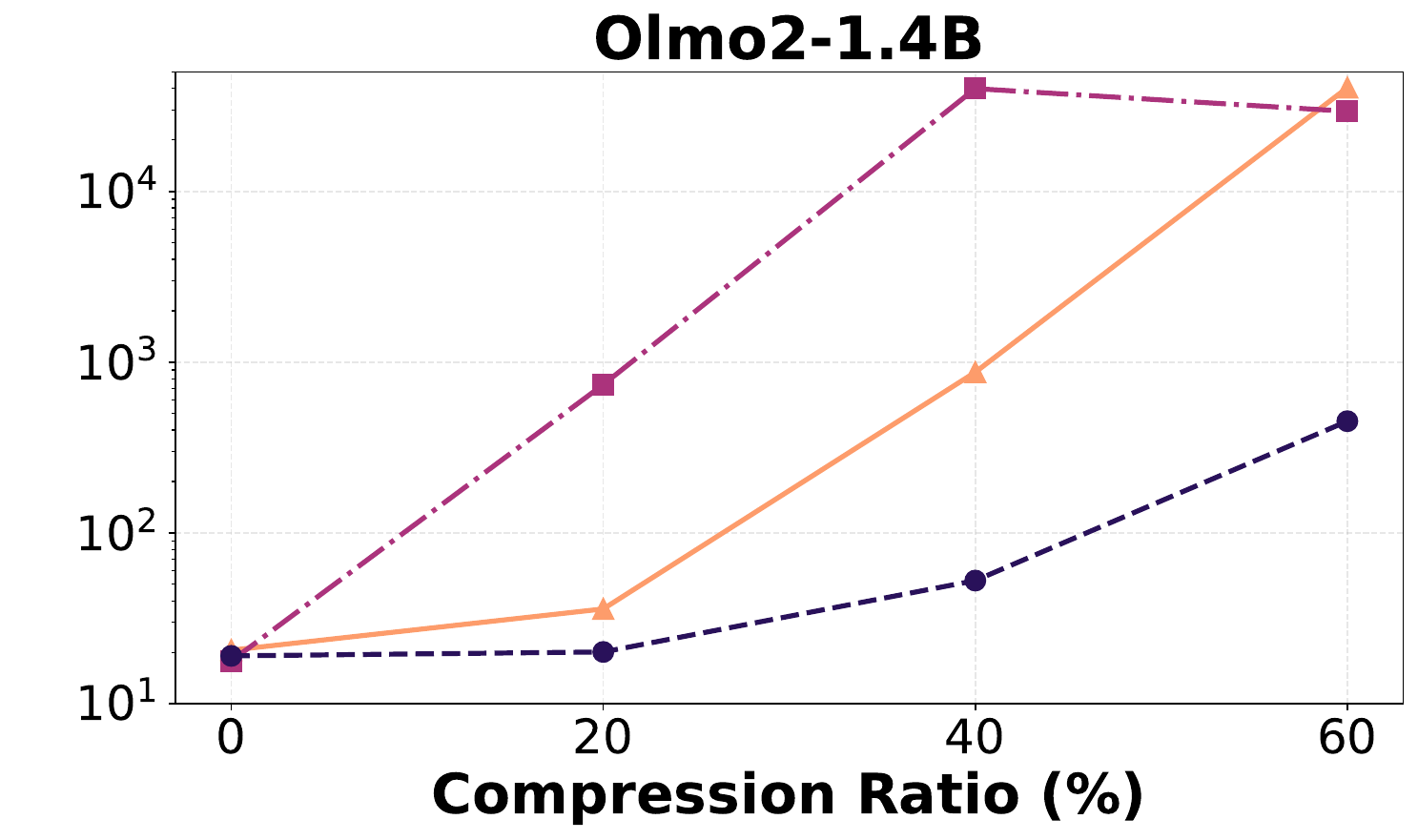}
        \caption{ASVD}
        \label{fig:compression_ppl_comparison:asvd:olmo2}
    \end{subfigure}
    \hspace{-0.5em}
    \begin{subfigure}[b]{0.32\textwidth}
        \centering
        \includegraphics[width=\textwidth]{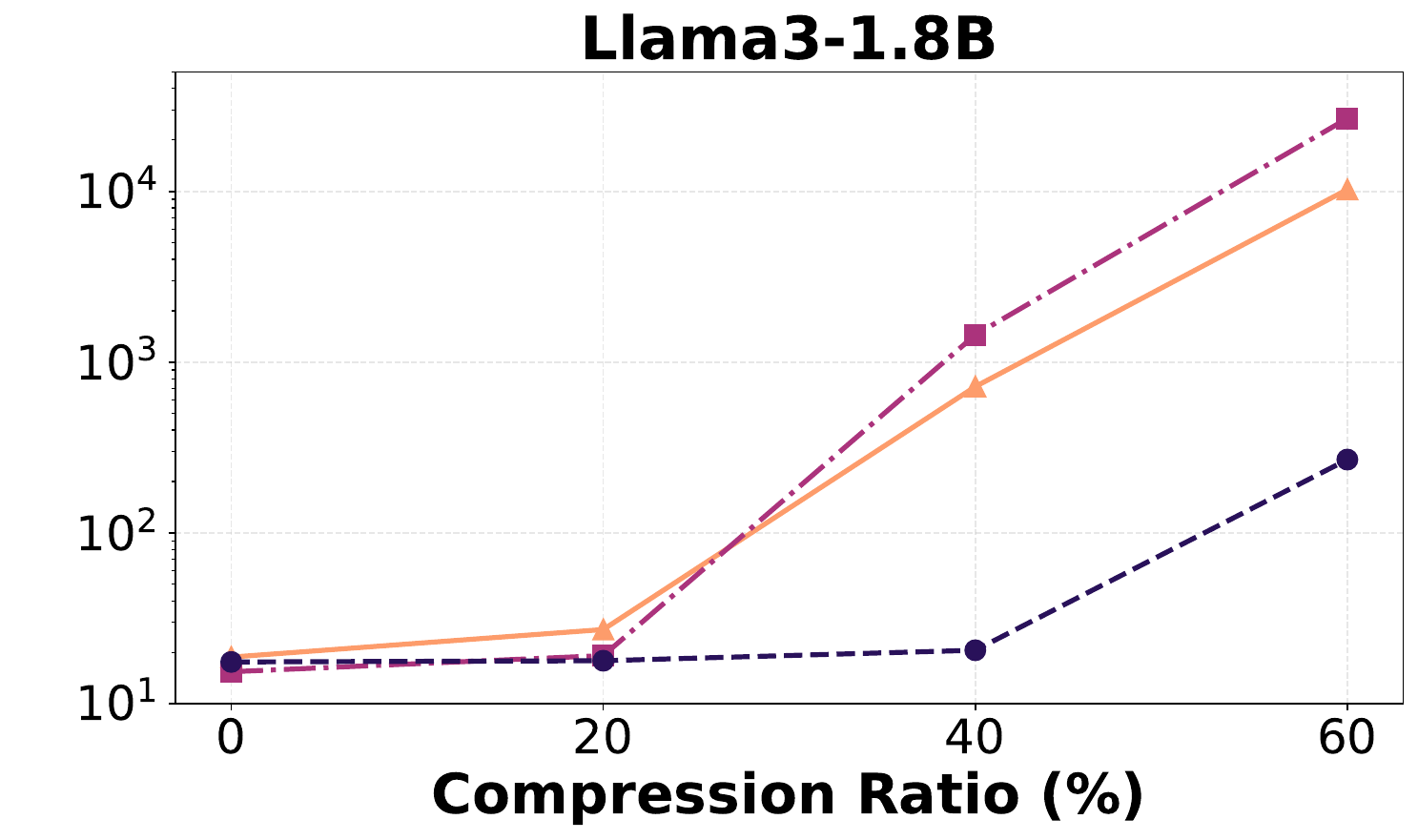}
        \caption*{}
        \label{fig:compression_ppl_comparison:asvd:llama3}
    \end{subfigure} 
    \\\vspace{1em}
    \begin{subfigure}[b]{0.32\textwidth}
        \centering
        \includegraphics[width=\textwidth]{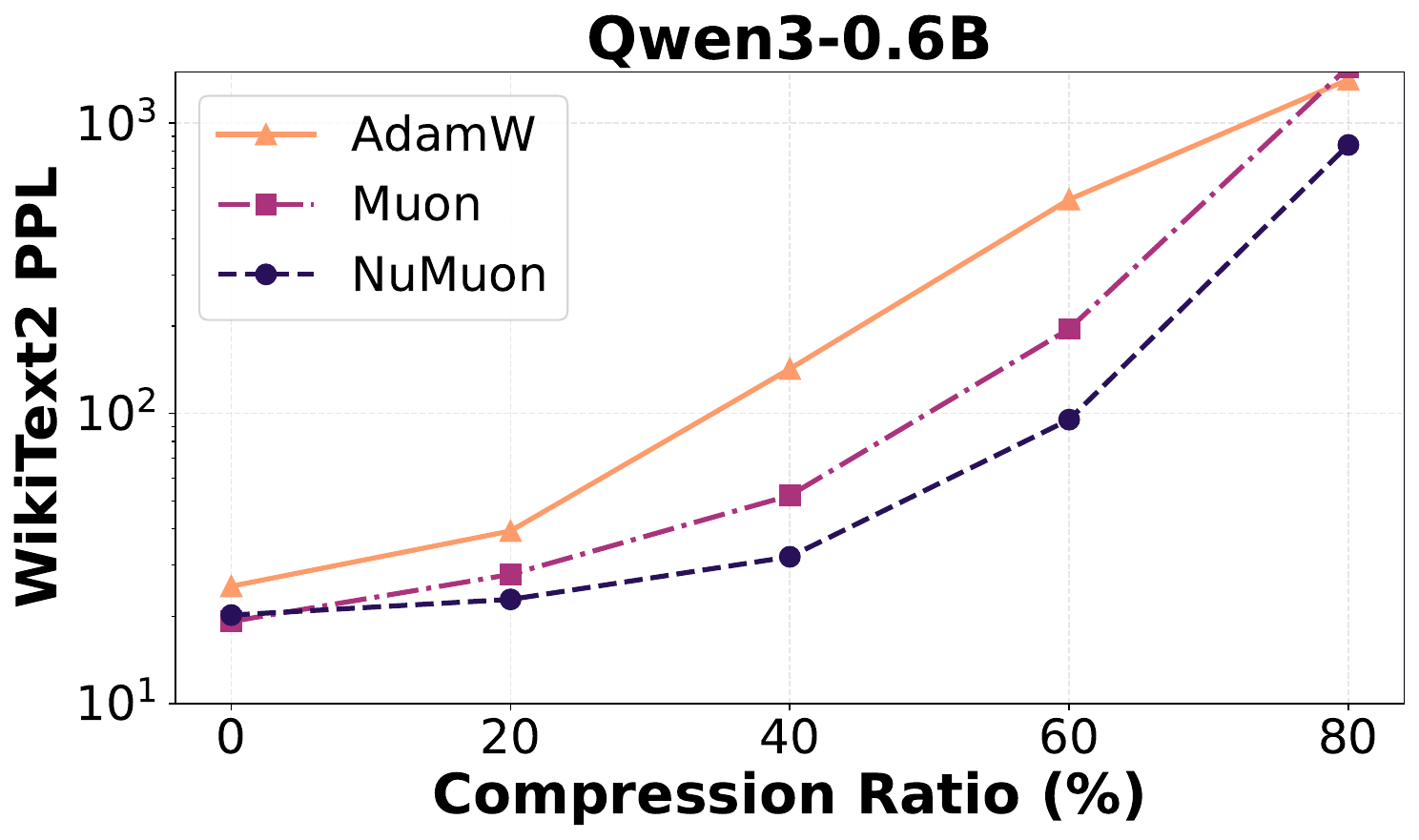}
        \caption*{}
        \label{fig:compression_ppl_comparison:svdllm_whiten:qwen3}
    \end{subfigure}
    \hspace{-0.5em}
    \begin{subfigure}[b]{0.32\textwidth}
        \centering
        \includegraphics[width=\textwidth]{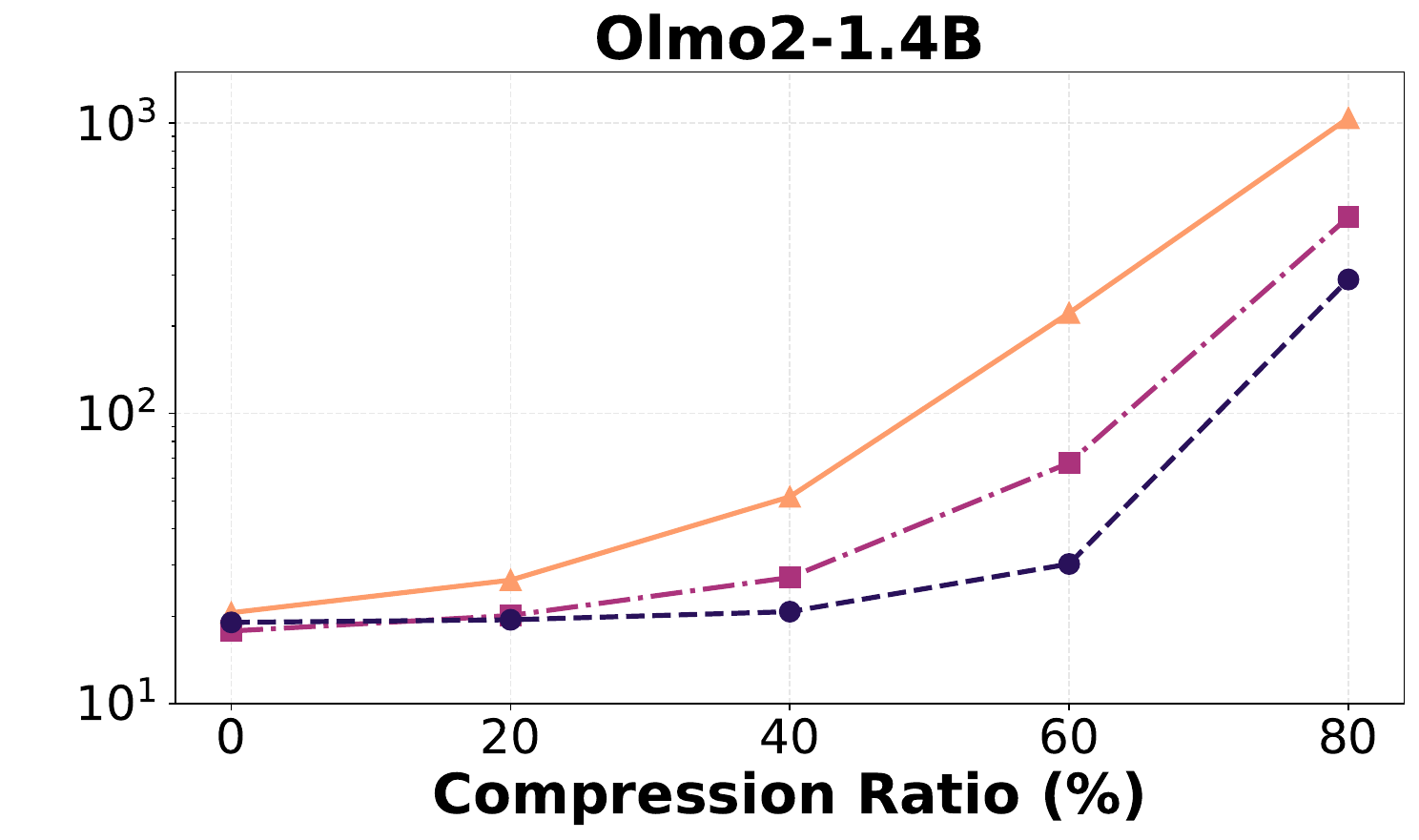}
        \caption{SVD-LLM (Whitening)}
        \label{fig:compression_ppl_comparison:svdllm_whiten:olmo2}
    \end{subfigure}
    \hspace{-0.5em}
    \begin{subfigure}[b]{0.32\textwidth}
        \centering
        \includegraphics[width=\textwidth]{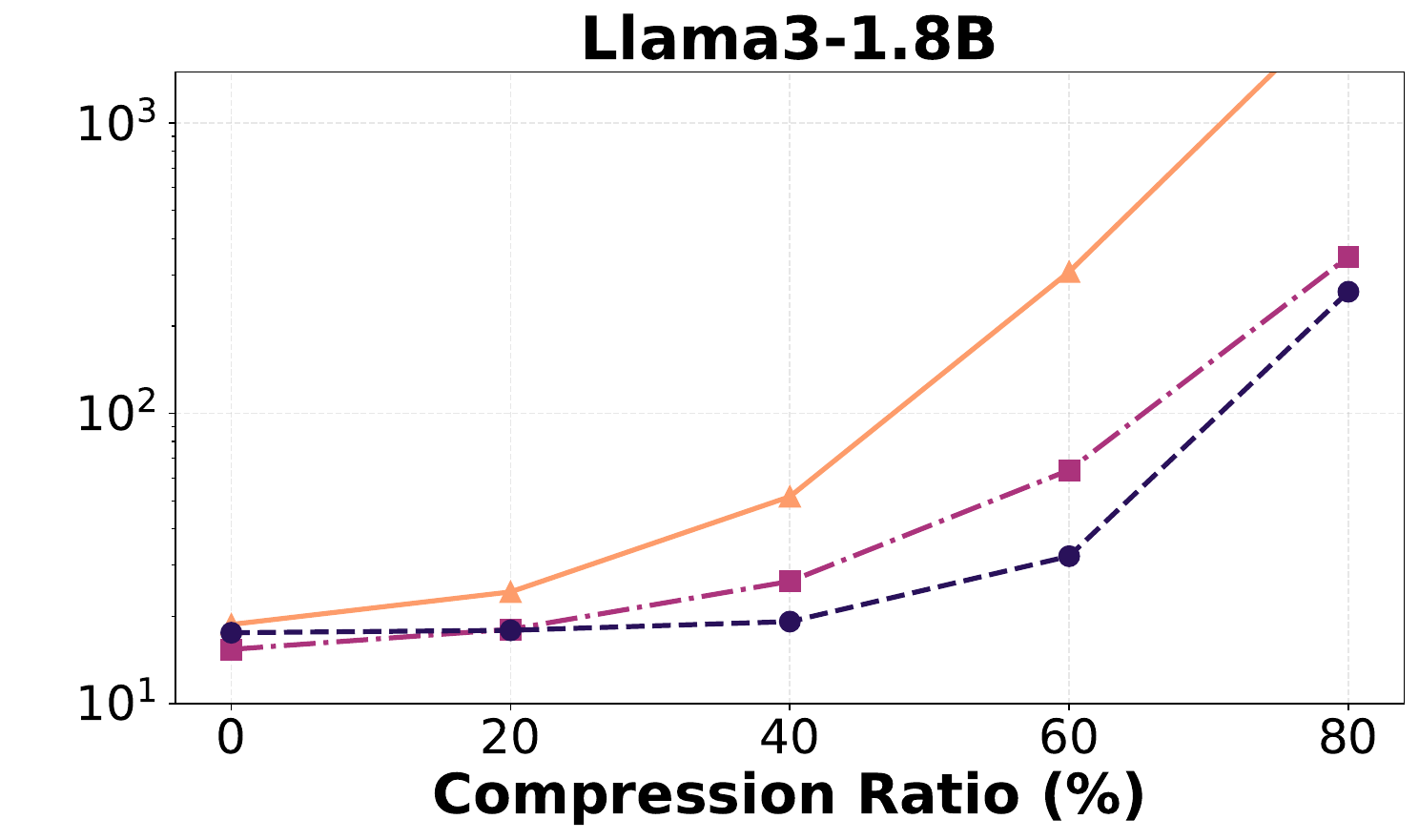}
        \caption*{}
        \label{fig:compression_ppl_comparison:svdllm_whiten:llama3}
    \end{subfigure} 
    \\\vspace{1em}
    \begin{subfigure}[b]{0.32\textwidth}
        \centering
        \includegraphics[width=\textwidth]{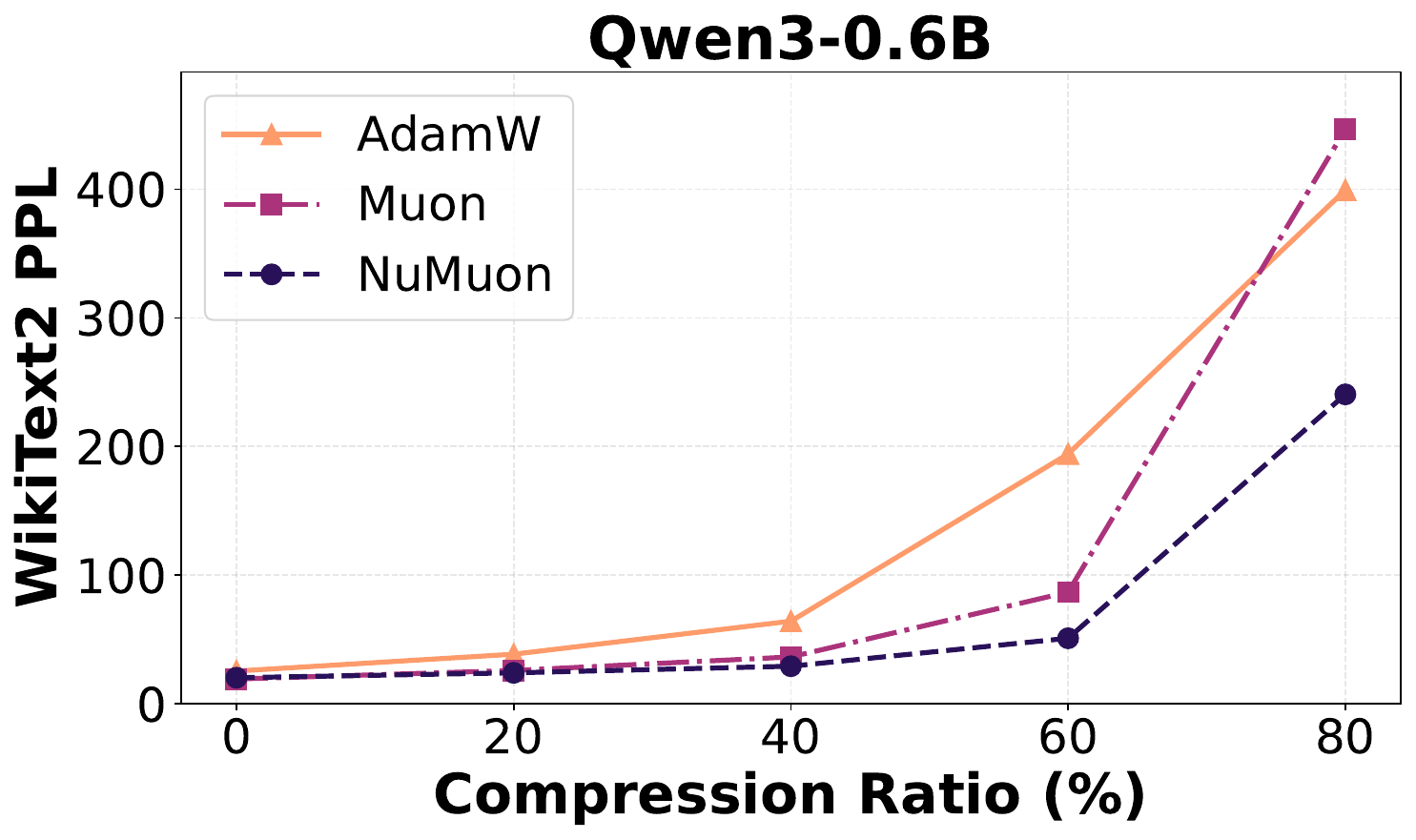}
        \caption*{}
        \label{fig:compression_ppl_comparison:svdllm_full:qwen3}
    \end{subfigure}
    \hspace{-0.5em}
    \begin{subfigure}[b]{0.32\textwidth}
        \centering
        \includegraphics[width=\textwidth]{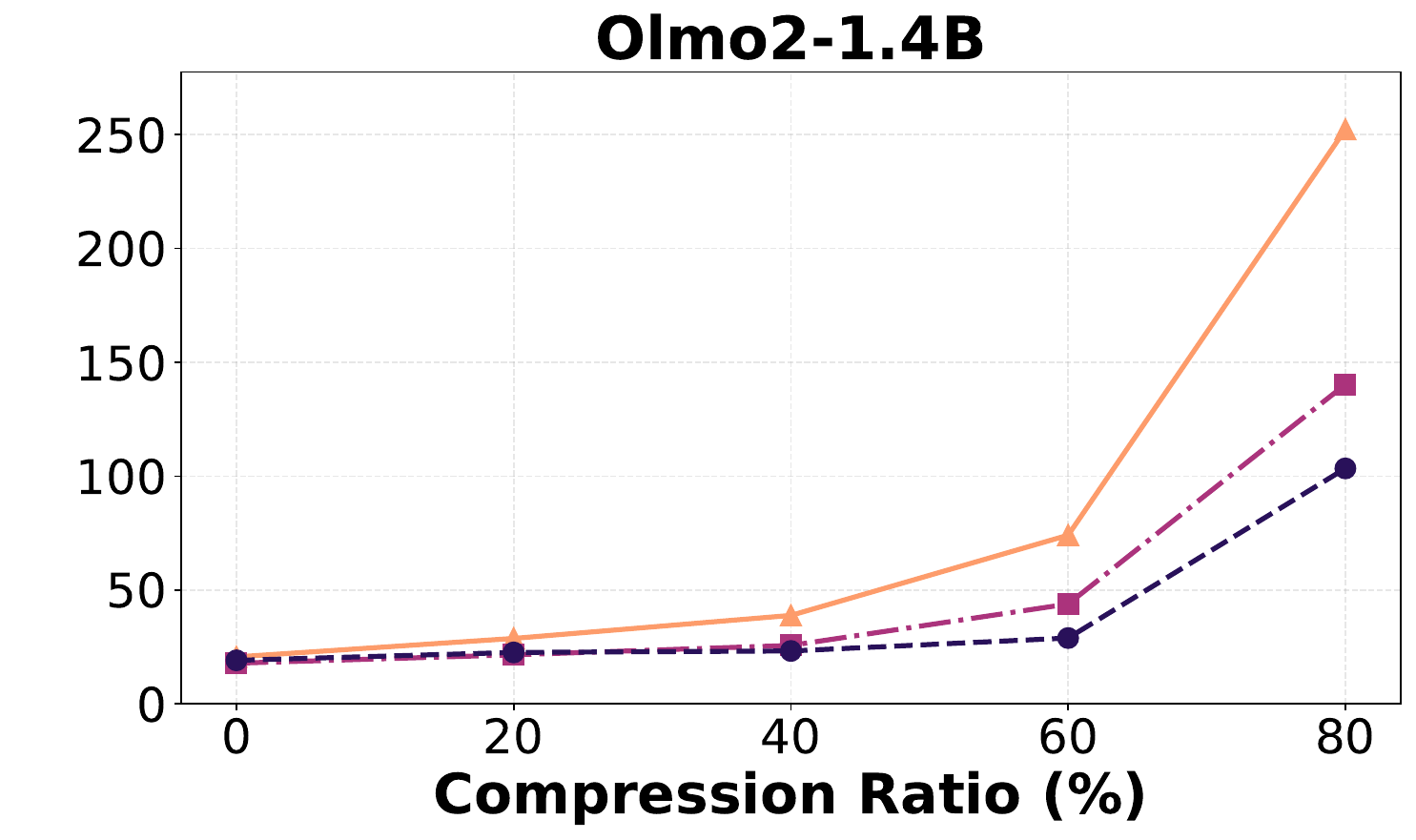}
        \caption{SVD-LLM (Whitening + LoRA)}
        \label{fig:compression_ppl_comparison:svdllm_full:olmo2}
    \end{subfigure}
    \hspace{-0.5em}
    \begin{subfigure}[b]{0.32\textwidth}
        \centering
        \includegraphics[width=\textwidth]{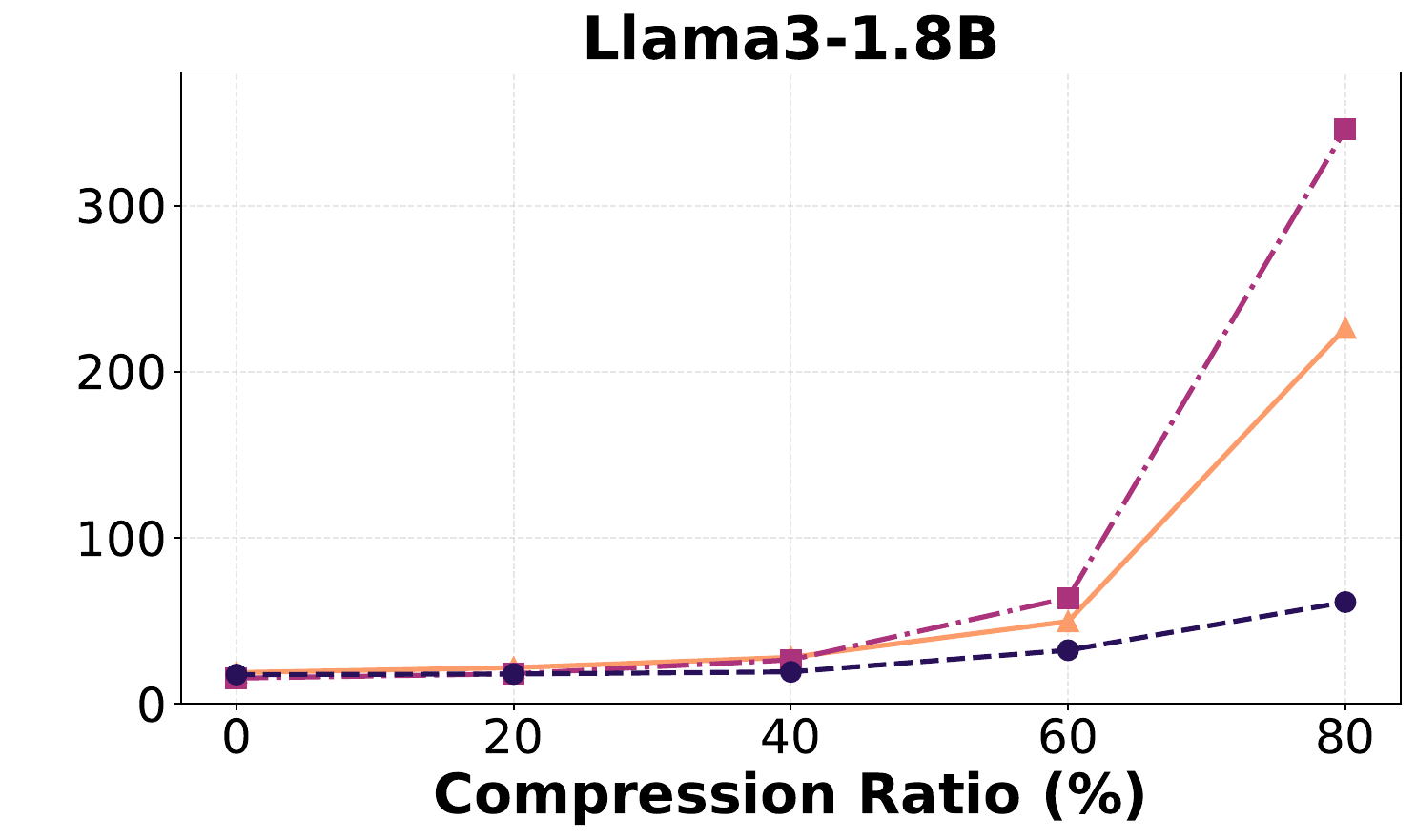}
        \caption*{}
        \label{fig:compression_ppl_comparison:svdllm_full:llama3}
    \end{subfigure} 
    \\\vspace{1em}
    \begin{subfigure}[b]{0.32\textwidth}
        \centering
        \includegraphics[width=\textwidth]{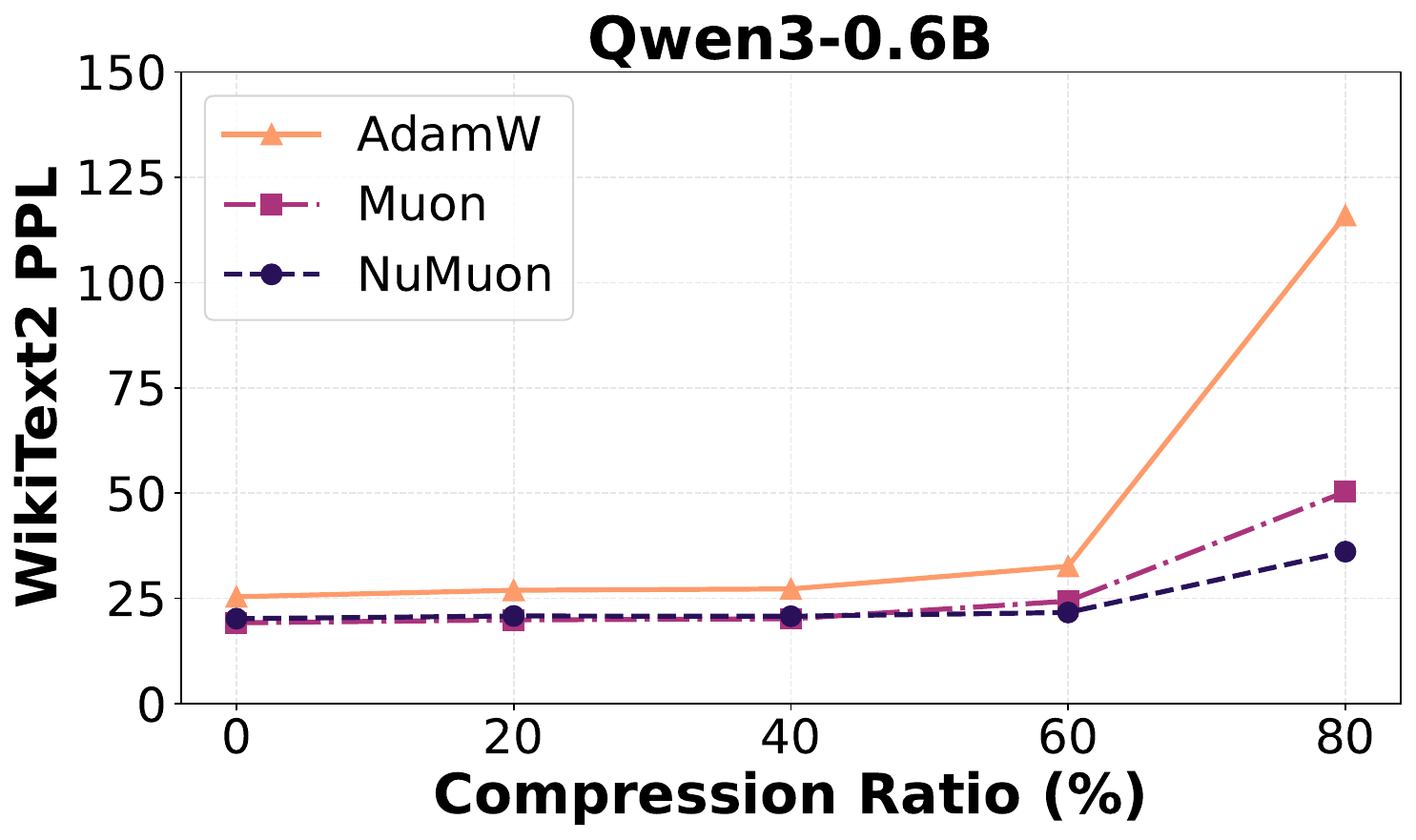}
        \caption*{}
        \label{fig:compression_ppl_comparison:dobisvd:qwen3}
    \end{subfigure}
    \hspace{-0.5em}
    \begin{subfigure}[b]{0.32\textwidth}
        \centering
        \includegraphics[width=\textwidth]{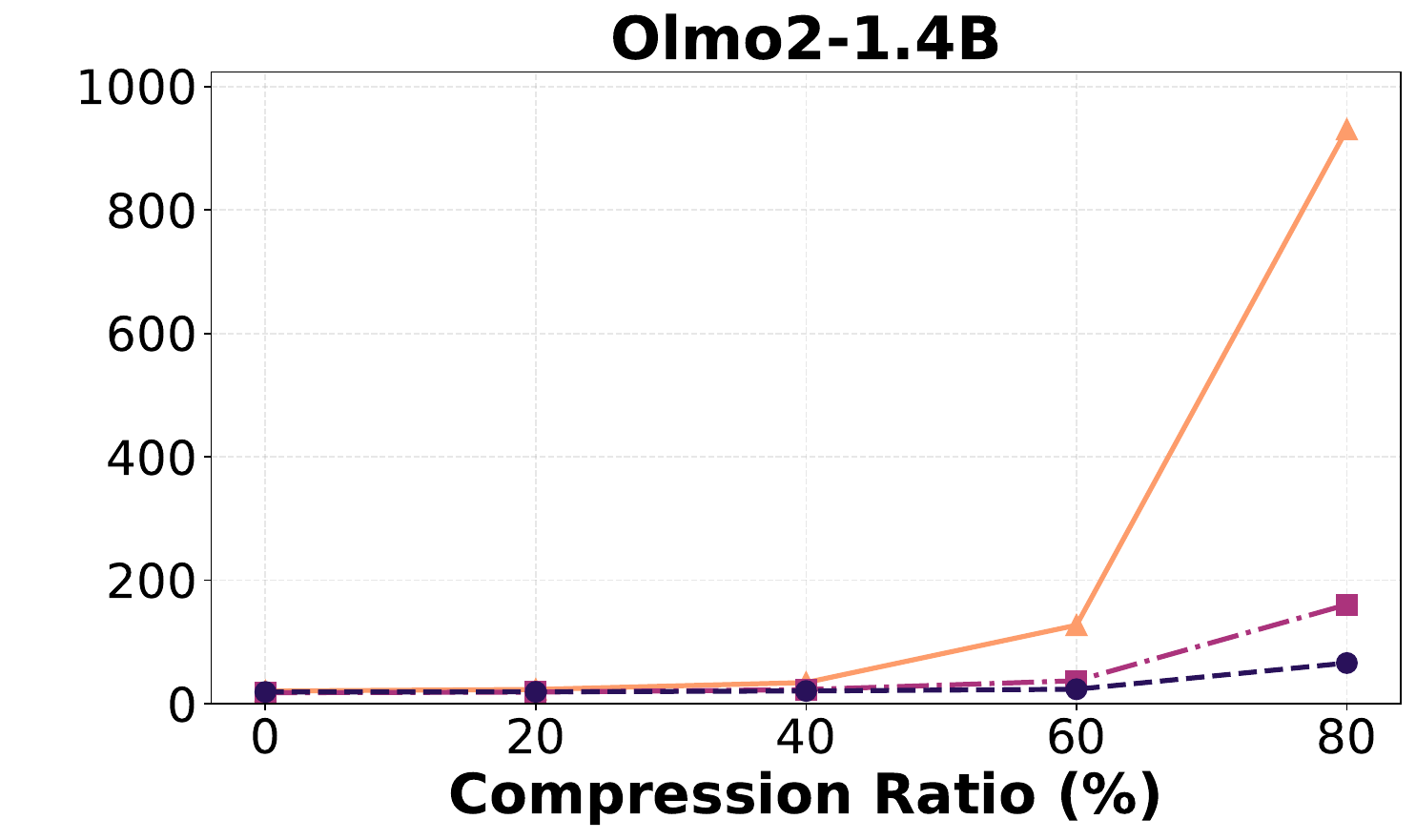}
        \caption{Dobi-SVD}
        \label{fig:compression_ppl_comparison:dobisvd:olmo2}
    \end{subfigure}
    \hspace{-0.5em}
    \begin{subfigure}[b]{0.32\textwidth}
        \centering
        \includegraphics[width=\textwidth]{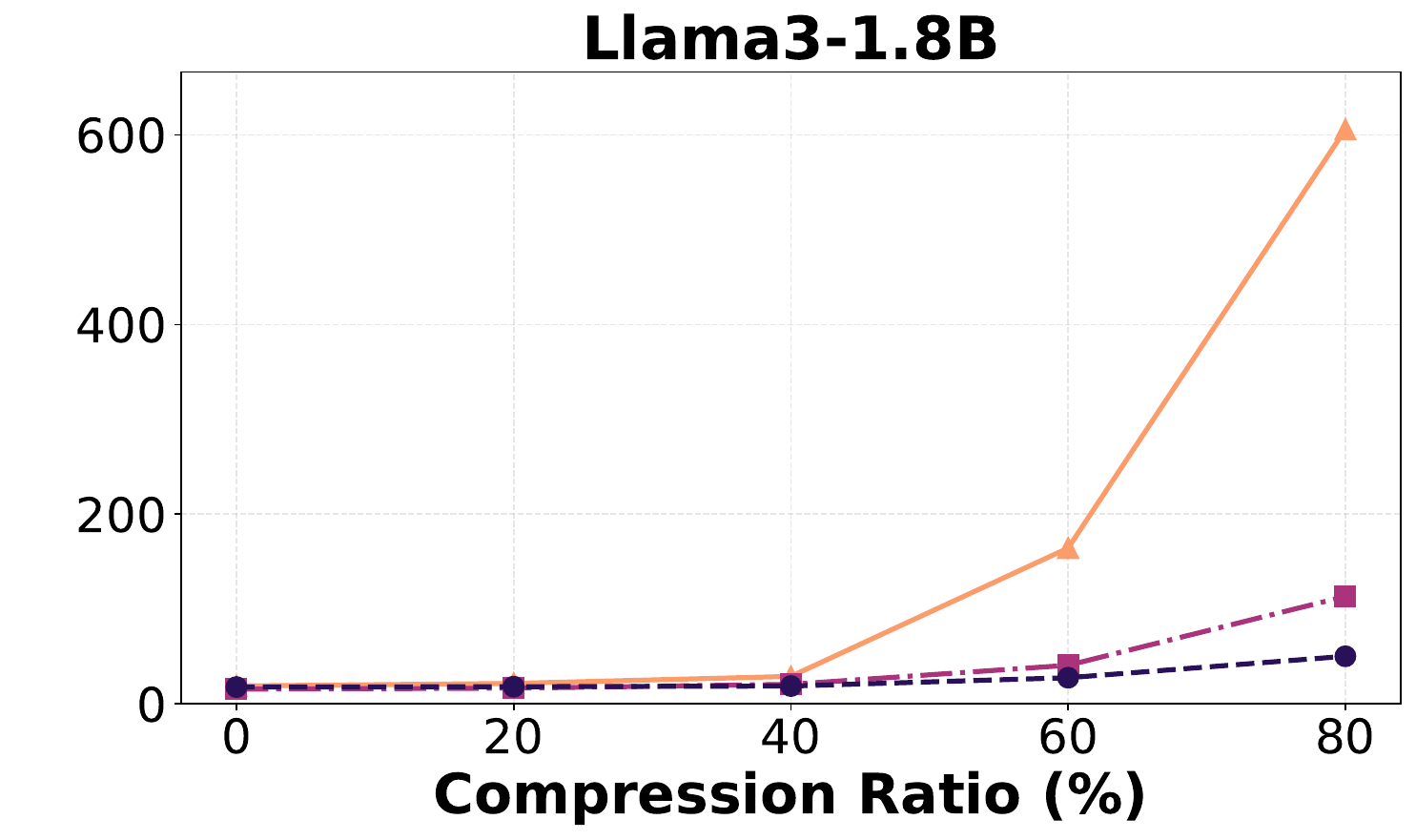}
        \caption*{}
        \label{fig:compression_ppl_comparison:dobisvd:llama3}
    \end{subfigure} 
    \\\vspace{1em}
    \caption{Validation perplexity on WikiText2 dataset for compressed LLMs~(Qwen3-0.6B, Olmo2-1.4B, and Llama3-1.8B) using different compression methods. As seen, models pretrained via NuMuon consistently achieve the lowest validation perplexity, yielding better performance at extreme compression ratios.}
    \label{fig:compression_ppl_comparison}
\end{figure*}
}

\newcommand{\figStableRankWellKnown}{
\begin{figure*}[t!]
    \centering
    \begin{subfigure}[b]{0.45\textwidth}
        \centering
        \includegraphics[width=\textwidth]{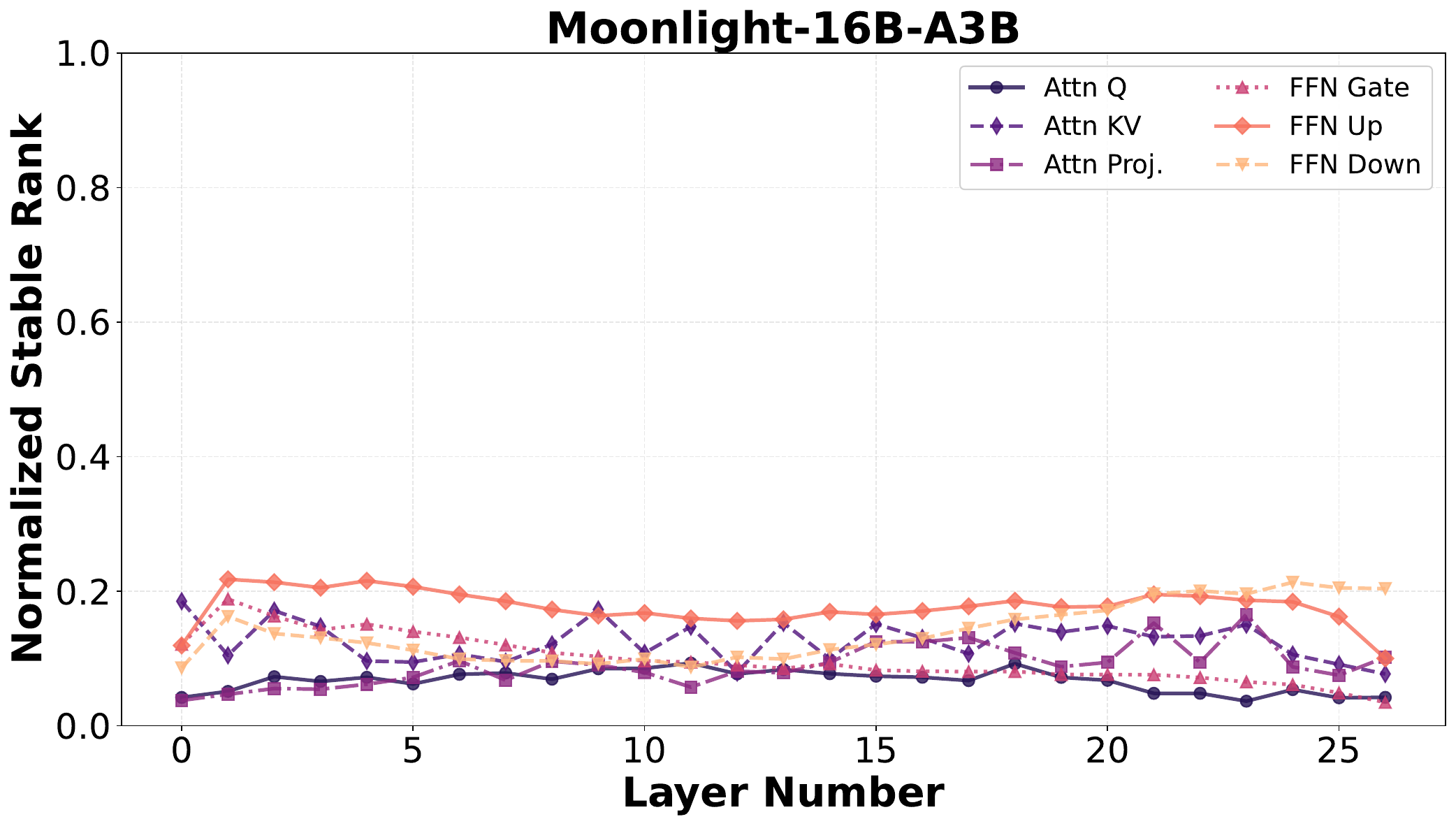}
        \caption{Moonlight-16B-A3B (stable rank per layer)}
        \label{fig:stable_rank_large_scale:moonlight:dist}
    \end{subfigure}
    \hspace{0.5em}
    \begin{subfigure}[b]{0.45\textwidth}
        \centering
        \includegraphics[width=\textwidth]{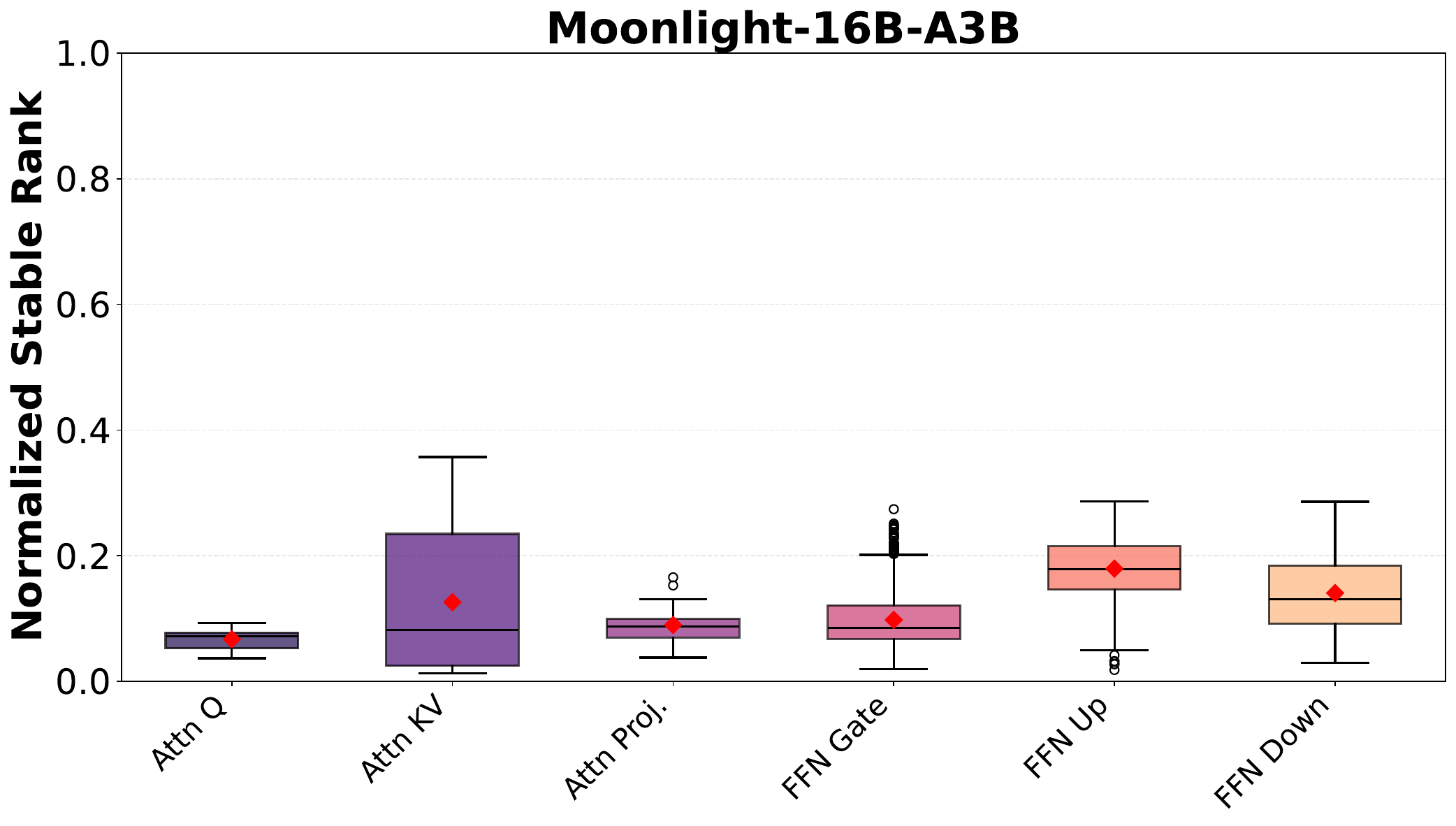}
        \caption{Moonlight-16B-A3B (stable rank distribution)}
        \label{fig:stable_rank_large_scale:moonlight:layer}
    \end{subfigure}
    \\\vspace{1em}
    \begin{subfigure}[b]{0.45\textwidth}
        \centering
        \includegraphics[width=\textwidth]{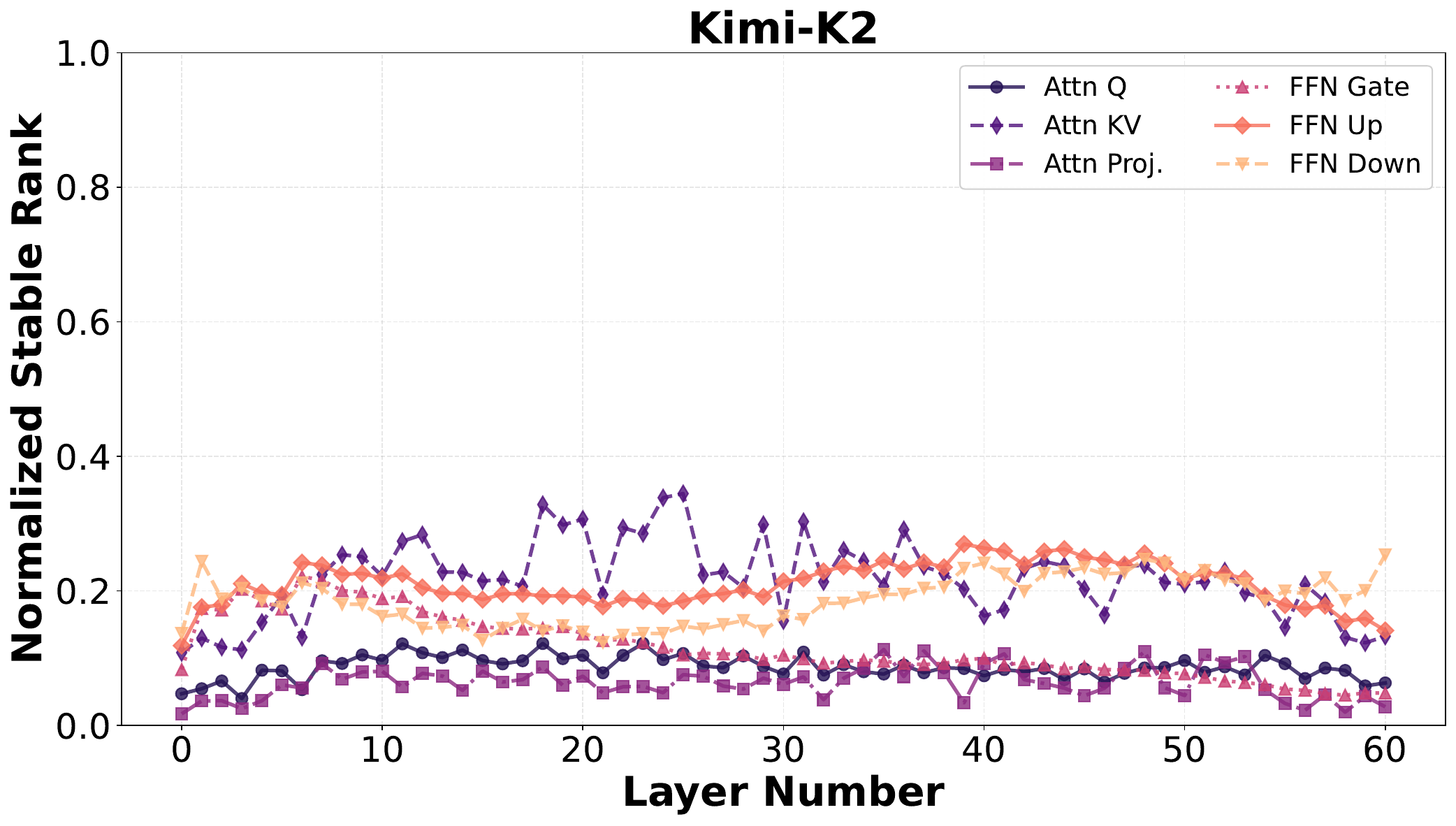}
        \caption{Kimi-K2 (stable rank per layer)}
        \label{fig:stable_rank_large_scale:kimik2:dist}
    \end{subfigure}
    \hspace{0.5em}
    \begin{subfigure}[b]{0.45\textwidth}
        \centering
        \includegraphics[width=\textwidth]{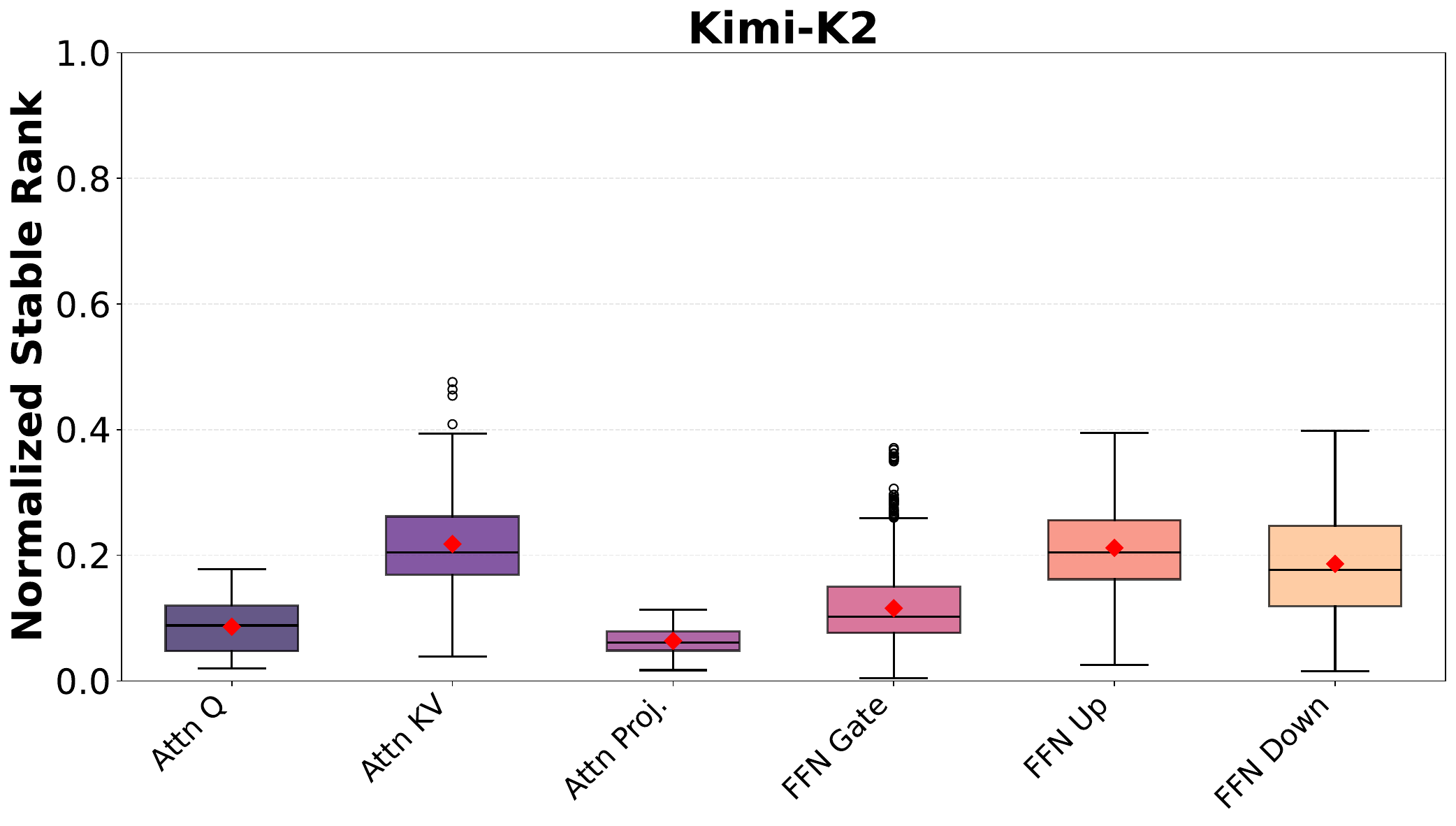}
        \caption{Kimi-K2 (stable rank distribution)}
        \label{fig:stable_rank_large_scale:kimik2:layer}
    \end{subfigure}
    \caption{Normalized stable rank of well-known, large-scale models trained via Muon. The plot shows Moonlight-16B-A3B~\citep{moonlight2025} (top) as well as Kimi-K2-1T~\citep{kimik2025} (bottom). For each model, we show the distribution of the normalized stable rank (left) as well as its evolution by the layer index (right). }
    \label{fig:stable_rank_large_scale}
\end{figure*}
}

\newcommand{\figSVDLLMEfficiencyGains}{
\begin{figure*}[t!]
    \centering
    \begin{subfigure}[b]{0.4\textwidth}
        \centering
        \includegraphics[width=\textwidth]{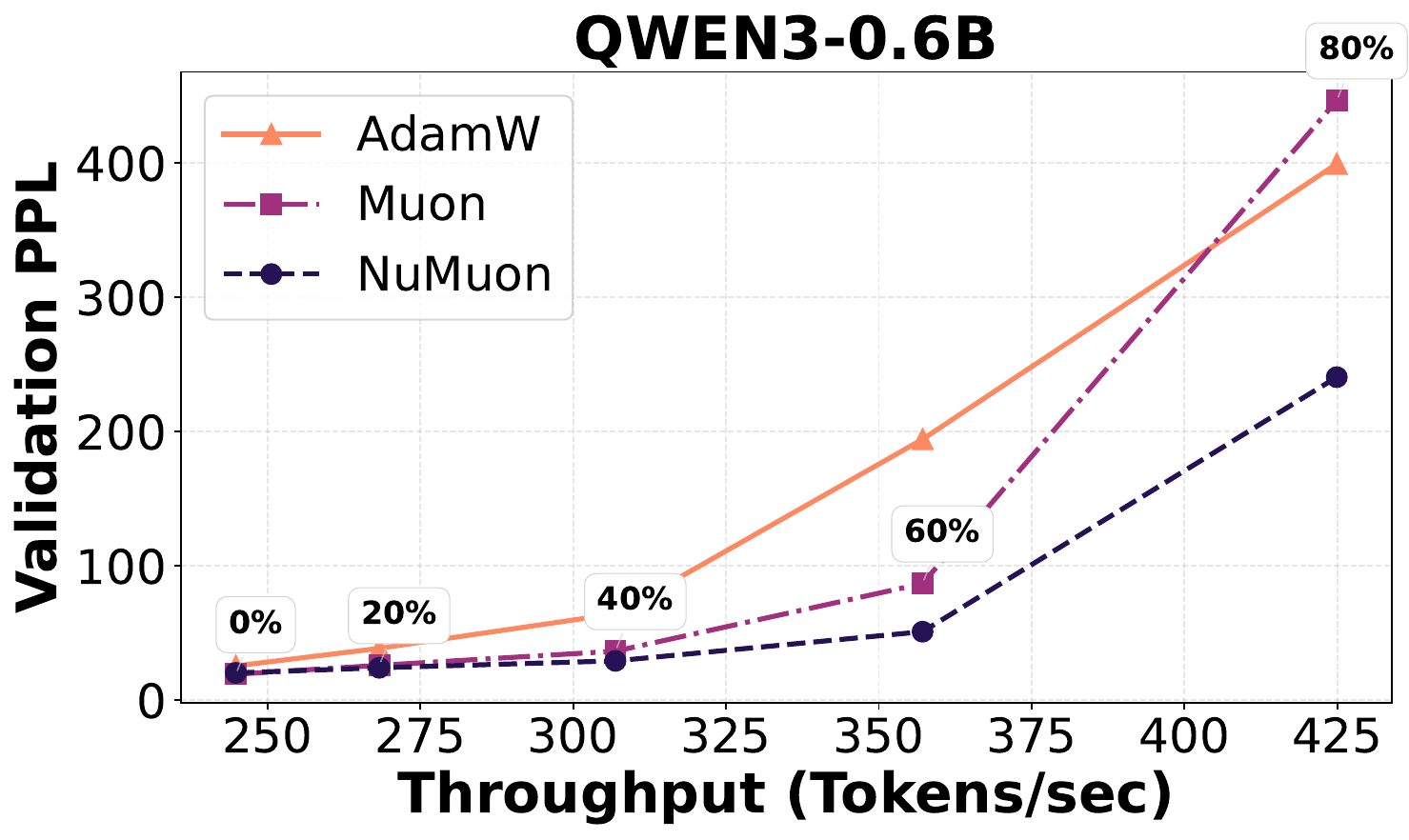}
        \caption{Qwen3-0.6B}
        \label{fig:svdllm_efficiency:qwen3}
    \end{subfigure}
    \hspace{0.5em}
    \begin{subfigure}[b]{0.4\textwidth}
        \centering
        \includegraphics[width=\textwidth]{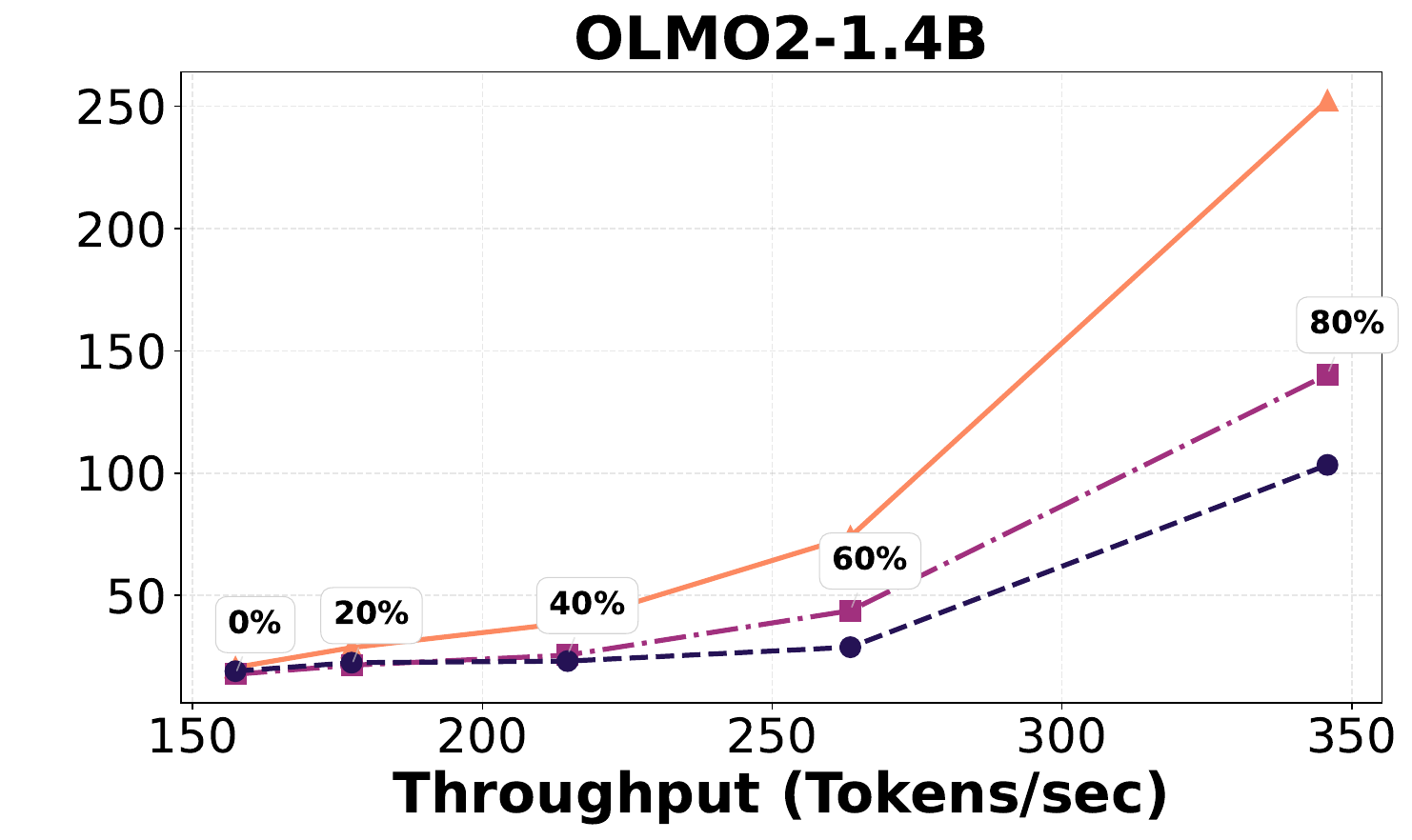}
        \caption{Olmo2-1.4B}
        \label{fig:svdllm_efficiency:olmo2}
    \end{subfigure}
    \begin{subfigure}[b]{0.4\textwidth}
        \centering
        \includegraphics[width=\textwidth]{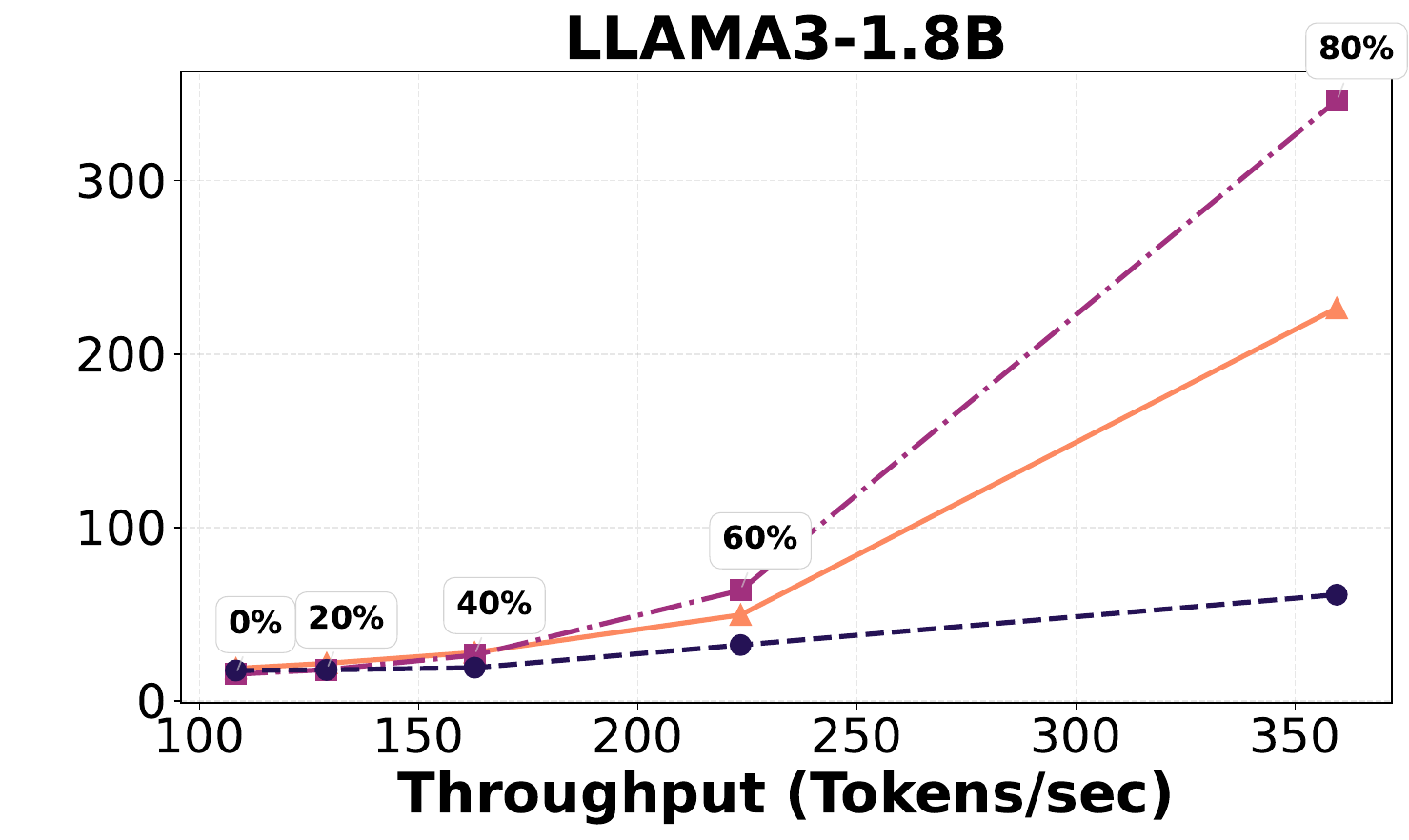}
        \caption{Llama3-1.8B}
        \label{fig:svdllm_efficiency:llama3}
    \end{subfigure}
    \caption{Validation perplexity on WikiText2 against generation inference throughput for Qwen3-0.6B, Olmo2-1.4B, and Llama3-1.8B models compressed via SVD-LLM. NuMuon trained models deliver the lowest perplexity given a throughput, and as such, are far more efficienct that AdamW- or Muon-trained models at moderate to extreme compression rates. Also, we observe that as the model size grows, NuMuon-trained models become more amenable to compression.}
    \label{fig:svdllm_efficiency}
\end{figure*}
}

%% file: tables.tex
\newcommand{\tabLlamaConvergence}{
\begin{table}[t!]
\setlength{\tabcolsep}{1pt}%
\renewcommand{\arraystretch}{0.85}%
\centering
\scriptsize
\caption{Training/validation perplexity (PPL) and efficiency metrics across optimizers for training Llama3-1.8B models. Validation PPL is computed on WikiText2 dataset.}
\vspace{-0.5em}
\label{tab:llama_convergence}
\begin{tabular}{@{}lcccc@{}}
\toprule
\textbf{\textsc{Optim.}} & \textbf{\textsc{Train PPL}} \textdownarrow & \textbf{\textsc{Val. PPL}} \textdownarrow & \textbf{\textsc{Time/Step~(s)}} \textdownarrow & \textbf{\textsc{GPU Mem.~(GB)}} \textdownarrow \\
\midrule
\midrule
\textsc{AdamW}  & 10.83 & 18.75 & 25.12 & 37.62 \\
\textsc{Muon}   & 10.01 & 15.39 & 25.39 & 32.51 \\
\textsc{NuMuon} & 10.59 & 17.56 & 30.03 & 33.88 \\
\bottomrule
\end{tabular}
\end{table}
}

\newcommand{\tabModelConfig}{%
\begin{table}[htbp]
\setlength{\tabcolsep}{5pt}%
\renewcommand{\arraystretch}{0.95}%
\centering
\footnotesize
\caption{Model architectures and training configuration.}
\vspace{-0.5em}
\label{tab:model_config}
\begin{tabular}{@{}lccc@{}}
    \toprule
    \textsc{\textbf{Model Configuration}} & \textbf{\textsc{Qwen3-0.6B}} & \textbf{\textsc{Olmo2-1.4B}} & \textbf{\textsc{Llama3-1.8B}} \\
    \midrule
    \midrule
    \multicolumn{4}{l}{\textit{\textsc{Model Architecture}}} \\
    \midrule
    \midrule
    Hidden Dimensions & 1024 & 2048 & 2048 \\
    Number of Layers & 28 & 16 & 30 \\
    Attention Heads & 16 & 32 & 32 \\
    Key-Value Heads & 8 & 8 & 8 \\
    Head Dimension & 128 & 64 & 64 \\
    FFN Dimension Multiplier & -- & 1.3 & 1.3 \\
    Multiple Of & -- & 1024 & 1024 \\
    RoPE Theta & 1000000 & 500000 & 500000 \\
    Vocabulary Size & 151936 & 128256 & 50257 \\
    QK Norm & Yes & Yes & No \\
    \midrule
    \midrule
    \multicolumn{4}{l}{\textit{\textsc{Distributed Setup}}} \\
    \midrule
    \midrule
    Data Parallel Replicate Degree & 8 & 8 & 8 \\
    Pipeline Parallel Degree & 1 & 1 & 1 \\
    \midrule
    \midrule
    \multicolumn{4}{l}{\textit{\textsc{Training}}} \\
    \midrule
    \midrule
    Local Batch Size & 16 & 8 & 2 \\
    Global Batch Size & 1024 & 1024 & 1024 \\
    Sequence Length & 2048 & 2048 & 2048 \\
    Total Steps & 8000 & 16000 & 20000 \\
    Dataset & Fineweb-EDU & Fineweb-EDU & Fineweb-EDU \\
    Gradient Clipping (Max Norm) & 1.0 & 1.0 & 1.0 \\
    \bottomrule
\end{tabular}
\end{table}
}

\newcommand{\tabOptimizerConfig}{%
\begin{table}[htbp]
\setlength{\tabcolsep}{5pt}%
\renewcommand{\arraystretch}{0.95}%
\centering
\footnotesize
\caption{Optimizer configuration and learning rate scheduler settings.}
\vspace{-0.5em}
\label{tab:optimizer_config}
\begin{tabular}{@{}lccc@{}}
    \toprule
    \textsc{\textbf{Configuration}} & \textbf{\textsc{AdamW}} & \textbf{\textsc{Muon}} & \textbf{\textsc{NuMuon}} \\
    \midrule
    \midrule
    \multicolumn{4}{l}{\textit{\textsc{Optimizer}}} \\
    \midrule
    \midrule
    Learning Rate & 3e-4 & 1e-2 & 1e-2 \\
    Weight Decay & -- & 0.1 & 0.1 \\
    Beta1 & -- & 0.9 & 0.9 \\
    Beta2 & -- & 0.95 & 0.95 \\
    Epsilon & 1e-8 & 1e-8 & 1e-8 \\
    \midrule
    \midrule
    \multicolumn{4}{l}{\textit{\textsc{Muon/NuMuon-Specific Parameters}}} \\
    \midrule
    \midrule
    Scalar Optimizer & -- & Lion & Lion \\
    Embedding Optimizer & -- & AdamW & AdamW \\
    Head Optimizer & -- & AdamW & AdamW \\
    Scalar LR Factor & -- & 1.0 & 1.0 \\
    Embedding LR Factor & -- & 0.5 & 0.5 \\
    Head LR Factor & -- & 1.0 & 1.0 \\
    \midrule
    \midrule
    \multicolumn{4}{l}{\textit{\textsc{NuMuon-Specific Parameters}}} \\
    \midrule
    \midrule
    Krylov Iterations & -- & -- & 2 \\
    Oversample & -- & -- & 8 \\
    Rank Scheduler Type & -- & -- & Cosine \\
    \midrule
    \midrule
    \multicolumn{4}{l}{\textit{\textsc{Learning Rate Scheduler}}} \\
    \midrule
    \midrule
    Scheduler Type & Cosine & WSD & WSD \\
    Warmup Steps & 25\% & 25\% & 25\% \\
    Minimum LR Factor & 0.01 & 0.01 & 0.01 \\
    Decay Type & Cosine & Cosine & Cosine \\
    Warmdown Ratio & 75\% & 20\% & 20\% \\
    \bottomrule
\end{tabular}
\end{table}
}

\newcommand{\tabQwenASVDFull}{
\begin{table}[htbp]
\setlength{\tabcolsep}{2pt}%
\renewcommand{\arraystretch}{0.8}%
\centering
\scriptsize
\caption{WikiText2 validation perplexity and downstream task performance across compressed Qwen3-0.6B models using ASVD.}
\vspace{-0.5em}
\label{tab:qwen3_asvd}
\begin{tabular}{@{}lccccccccccc@{}}
\toprule
\multirow{2}{*}{\rotatebox{90}{\tiny{\textbf{\textsc{Comp.}}}}} & \multirow{2}{*}{\textsc{\textbf{Optimizer}}} & \multirow{2}{*}{\textbf{\textsc{Val. PPL}}} & \multicolumn{7}{c}{\textbf{\textsc{Downstream Tasks}}} & \multirow{2}{*}{\textbf{\textsc{Avg}}} \\
\cmidrule(lr){4-10}
 & & & \textbf{\textsc{ARC-C}} & \textbf{\textsc{ARC-E}} & \textbf{\textsc{H'SWag}} & \textbf{\textsc{LAM'DA}} & \textbf{\textsc{OpenBkQA}} & \textbf{\textsc{PIQA}} & \textbf{\textsc{W'Grande}} & \\
\midrule
\midrule
\multirow{3}{*}{\rotatebox{90}{0\%}} & \textsc{AdamW} & 25.39 & 27.30 & 49.66 & 36.04 & 27.77 & 31.20 & 65.02 & 49.64 & 40.95 \\
 & \textsc{Muon} & 19.20 & \textbf{30.72} & \textbf{57.95} & \textbf{45.56} & \textbf{37.69} & 33.60 & \textbf{67.74} & \textbf{54.38} & \textbf{46.81} \\
 & \textsc{NuMuon} & 20.20 \textcolor{red}{(+5.2\%)} & 29.95 & 55.68 & 43.51 & 35.40 & \textbf{34.40} & 67.41 & 53.51 & 45.69 \textcolor{red}{(-2.4\%)} \\
\midrule
\multirow{3}{*}{\rotatebox{90}{20\%}} & \textsc{AdamW} & 61.91 & 22.53 & 37.21 & 30.83 & 12.13 & 28.20 & 54.68 & 51.70 & 33.90 \\
 & \textsc{Muon} & 29.87 & \textbf{28.67} & 52.06 & 40.80 & 30.49 & 33.80 & 65.94 & \textbf{52.96} & 43.53 \\
 & \textsc{NuMuon} & 21.73 \textcolor{green!40!black}{(-27.3\%)} & 28.33 & \textbf{52.27} & \textbf{42.47} & \textbf{37.55} & \textbf{34.20} & \textbf{66.32} & 52.41 & \textbf{44.79} \textcolor{green!40!black}{(+2.9\%)} \\
\midrule
\multirow{3}{*}{\rotatebox{90}{40\%}} & \textsc{AdamW} & 1677.81 & 24.49 & 28.83 & 26.68 & 0.10 & 28.40 & 51.85 & 47.59 & 29.71 \\
 & \textsc{Muon} & 17097.41 & \textbf{24.57} & 27.31 & 27.57 & 1.07 & 26.20 & 52.12 & 50.12 & 29.85 \\
 & \textsc{NuMuon} & 54.79 \textcolor{green!40!black}{(-99.7\%)} & 23.89 & \textbf{40.49} & \textbf{35.79} & \textbf{24.92} & \textbf{30.60} & \textbf{61.48} & \textbf{50.67} & \textbf{38.26} \textcolor{green!40!black}{(+28.2\%)} \\
\midrule
\multirow{3}{*}{\rotatebox{90}{60\%}} & \textsc{AdamW} & 5306.46 & 25.85 & 27.78 & 26.51 & 0.06 & \textbf{30.60} & 52.45 & 50.12 & 30.48 \\
 & \textsc{Muon} & 30980.97 & 25.17 & 27.06 & 27.08 & 0.25 & 29.20 & 51.58 & 47.99 & 29.76 \\
 & \textsc{NuMuon} & 1030.72 \textcolor{green!40!black}{(-96.7\%)} & \textbf{27.13} & \textbf{33.08} & \textbf{30.36} & \textbf{12.98} & 28.20 & \textbf{57.07} & \textbf{51.14} & \textbf{34.28} \textcolor{green!40!black}{(+15.2\%)} \\
\bottomrule
\end{tabular}
\end{table}
}

\newcommand{\tabOlmoASVDFull}{
\begin{table}[htbp]
\setlength{\tabcolsep}{2pt}%
\renewcommand{\arraystretch}{0.8}%
\centering
\scriptsize
\caption{WikiText2 validation perplexity and downstream task performance across compressed Olmo2-1.4B models using ASVD.}
\vspace{-0.5em}
\label{tab:olmo2_asvd}
\begin{tabular}{@{}lccccccccccc@{}}
\toprule
\multirow{2}{*}{\rotatebox{90}{\tiny{\textbf{\textsc{Comp.}}}}} & \multirow{2}{*}{\textsc{\textbf{Optimizer}}} & \multirow{2}{*}{\textbf{\textsc{Val. PPL}}} & \multicolumn{7}{c}{\textbf{\textsc{Downstream Tasks}}} & \multirow{2}{*}{\textbf{\textsc{Avg}}} \\
\cmidrule(lr){4-10}
 & & & \textbf{\textsc{ARC-C}} & \textbf{\textsc{ARC-E}} & \textbf{\textsc{H'SWag}} & \textbf{\textsc{LAM'DA}} & \textbf{\textsc{OpenBkQA}} & \textbf{\textsc{PIQA}} & \textbf{\textsc{W'Grande}} & \\
\midrule
\midrule
\multirow{3}{*}{\rotatebox{90}{0\%}} & \textsc{AdamW} & 20.58 & 30.20 & 55.47 & 44.96 & 34.76 & 35.80 & 68.93 & 50.67 & 45.83 \\
 & \textsc{Muon} & 17.80 & \textbf{34.04} & \textbf{60.10} & \textbf{50.94} & \textbf{39.28} & \textbf{37.80} & \textbf{70.46} & 54.85 & \textbf{49.64} \\
 & \textsc{NuMuon} & 19.06 \textcolor{red}{(+7.1\%)} & 31.06 & 59.55 & 48.40 & 37.80 & 37.00 & 68.99 & \textbf{55.41} & 48.32 \textcolor{red}{(-2.7\%)} \\
\midrule
\multirow{3}{*}{\rotatebox{90}{20\%}} & \textsc{AdamW} & 35.89 & 28.07 & 49.12 & 40.12 & 23.09 & 30.20 & 64.47 & 50.67 & 40.82 \\
 & \textsc{Muon} & 738.86 & 22.78 & 40.19 & 28.98 & 2.87 & 27.40 & 58.43 & 50.51 & 33.02 \\
 & \textsc{NuMuon} & 20.09 \textcolor{green!40!black}{(-97.3\%)} & \textbf{30.38} & \textbf{57.24} & \textbf{47.22} & \textbf{36.44} & \textbf{35.20} & \textbf{69.10} & \textbf{55.80} & \textbf{47.34} \textcolor{green!40!black}{(+43.4\%)} \\
\midrule
\multirow{3}{*}{\rotatebox{90}{40\%}} & \textsc{AdamW} & 873.00 & 22.18 & 32.37 & 28.62 & 2.76 & 27.40 & 51.80 & 49.64 & 30.68 \\
 & \textsc{Muon} & 39946.67 & \textbf{24.32} & 25.67 & 26.06 & 0.00 & 27.40 & 51.36 & \textbf{52.64} & 29.64 \\
 & \textsc{NuMuon} & 52.61 \textcolor{green!40!black}{(-99.9\%)} & 24.23 & \textbf{50.59} & \textbf{37.02} & \textbf{17.33} & \textbf{30.80} & \textbf{63.60} & 51.46 & \textbf{39.29} \textcolor{green!40!black}{(+32.6\%)} \\
\midrule
\multirow{3}{*}{\rotatebox{90}{60\%}} & \textsc{AdamW} & 40436.60 & 25.51 & 26.22 & 26.26 & 0.00 & 26.60 & 52.29 & 49.33 & 29.46 \\
 & \textsc{Muon} & 29425.49 & 25.09 & 25.25 & 26.74 & 0.00 & \textbf{27.80} & 52.61 & 50.59 & 29.73 \\
 & \textsc{NuMuon} & 451.06 \textcolor{green!40!black}{(-98.5\%)} & \textbf{25.94} & \textbf{39.23} & \textbf{30.71} & \textbf{9.06} & \textbf{27.80} & \textbf{58.76} & \textbf{51.85} & \textbf{34.76} \textcolor{green!40!black}{(+16.9\%)} \\
\bottomrule
\end{tabular}
\end{table}
}

\newcommand{\tabLlamaASVDFull}{
\begin{table}[htbp]
\setlength{\tabcolsep}{2pt}%
\renewcommand{\arraystretch}{0.8}%
\centering
\scriptsize
\caption{WikiText2 validation perplexity and downstream task performance across compressed Llama3-1.8B models using ASVD.}
\vspace{-0.5em}
\label{tab:llama3_asvd}
\begin{tabular}{@{}lccccccccccc@{}}
\toprule
\multirow{2}{*}{\rotatebox{90}{\tiny{\textbf{\textsc{Comp.}}}}} & \multirow{2}{*}{\textsc{\textbf{Optimizer}}} & \multirow{2}{*}{\textbf{\textsc{Val. PPL}}} & \multicolumn{7}{c}{\textbf{\textsc{Downstream Tasks}}} & \multirow{2}{*}{\textbf{\textsc{Avg}}} \\
\cmidrule(lr){4-10}
 & & & \textbf{\textsc{ARC-C}} & \textbf{\textsc{ARC-E}} & \textbf{\textsc{H'SWag}} & \textbf{\textsc{LAM'DA}} & \textbf{\textsc{OpenBkQA}} & \textbf{\textsc{PIQA}} & \textbf{\textsc{W'Grande}} & \\
\midrule
\midrule
\multirow{3}{*}{\rotatebox{90}{0\%}} & \textsc{AdamW} & 18.75 & 33.79 & 62.67 & 51.28 & 40.46 & 37.00 & 71.82 & 55.33 & 50.34 \\
 & \textsc{Muon} & 15.39 & \textbf{38.82} & \textbf{67.80} & \textbf{58.35} & \textbf{47.47} & \textbf{41.80} & \textbf{74.37} & \textbf{59.27} & \textbf{55.41} \\
 & \textsc{NuMuon} & 17.56 \textcolor{red}{(+14.1\%)} & 35.49 & 65.19 & 54.96 & 44.38 & 38.20 & 72.91 & 58.56 & 52.81 \textcolor{red}{(-4.7\%)} \\
\midrule
\multirow{3}{*}{\rotatebox{90}{20\%}} & \textsc{AdamW} & 27.19 & 31.57 & 57.53 & 46.29 & 32.25 & 36.20 & 68.06 & 52.01 & 46.27 \\
 & \textsc{Muon} & 19.13 & 34.47 & 62.71 & 52.62 & 42.60 & 37.60 & 71.38 & 56.27 & 51.09 \\
 & \textsc{NuMuon} & 17.85 \textcolor{green!40!black}{(-6.7\%)} & \textbf{35.41} & \textbf{64.18} & \textbf{53.98} & \textbf{44.89} & \textbf{38.40} & \textbf{71.98} & \textbf{57.30} & \textbf{52.31} \textcolor{green!40!black}{(+2.4\%)} \\
\midrule
\multirow{3}{*}{\rotatebox{90}{40\%}} & \textsc{AdamW} & 718.71 & 22.61 & 37.25 & 30.15 & 6.00 & 26.40 & 54.52 & 49.88 & 32.40 \\
 & \textsc{Muon} & 1443.55 & 24.66 & 31.90 & 29.28 & 7.08 & 27.80 & 55.93 & 50.12 & 32.40 \\
 & \textsc{NuMuon} & 20.56 \textcolor{green!40!black}{(-98.6\%)} & \textbf{32.76} & \textbf{62.88} & \textbf{50.27} & \textbf{42.64} & \textbf{37.20} & \textbf{71.38} & \textbf{56.20} & \textbf{50.48} \textcolor{green!40!black}{(+55.8\%)} \\
\midrule
\multirow{3}{*}{\rotatebox{90}{60\%}} & \textsc{AdamW} & 10251.84 & 24.49 & 29.38 & 26.50 & 0.19 & 30.40 & 51.85 & 49.64 & 30.35 \\
 & \textsc{Muon} & 26731.09 & 25.94 & 26.26 & 25.86 & 0.02 & 28.20 & 48.69 & 50.28 & 29.32 \\
 & \textsc{NuMuon} & 268.86 \textcolor{green!40!black}{(-99.0\%)} & \textbf{26.11} & \textbf{39.73} & \textbf{29.66} & \textbf{19.81} & \textbf{33.00} & \textbf{58.43} & \textbf{50.43} & \textbf{36.74} \textcolor{green!40!black}{(+25.3\%)} \\
\bottomrule
\end{tabular}
\end{table}
}

\newcommand{\tabQwenWhitenFull}{
\begin{table}[htbp]
\setlength{\tabcolsep}{2pt}%
\renewcommand{\arraystretch}{0.8}%
\centering
\scriptsize
\caption{WikiText2 validation perplexity and downstream task performance across compressed Qwen3-0.6B models using SVD-LLM (Whitening).}
\vspace{-0.5em}
\label{tab:qwen3_whiten}
\begin{tabular}{@{}lccccccccccc@{}}
\toprule
\multirow{2}{*}{\rotatebox{90}{\tiny{\textbf{\textsc{Comp.}}}}} & \multirow{2}{*}{\textsc{\textbf{Optimizer}}} & \multirow{2}{*}{\textbf{\textsc{Val. PPL}}} & \multicolumn{7}{c}{\textbf{\textsc{Downstream Tasks}}} & \multirow{2}{*}{\textbf{\textsc{Avg}}} \\
\cmidrule(lr){4-10}
 & & & \textbf{\textsc{ARC-C}} & \textbf{\textsc{ARC-E}} & \textbf{\textsc{H'SWag}} & \textbf{\textsc{LAM'DA}} & \textbf{\textsc{OpenBkQA}} & \textbf{\textsc{PIQA}} & \textbf{\textsc{W'Grande}} & \\
\midrule
\midrule
\multirow{3}{*}{\rotatebox{90}{0\%}} & \textsc{AdamW} & 25.39 & 27.30 & 49.66 & 36.04 & 27.77 & 31.20 & 65.02 & 49.64 & 40.95 \\
 & \textsc{Muon} & 19.20 & \textbf{30.72} & \textbf{57.95} & \textbf{45.56} & \textbf{37.69} & 33.60 & \textbf{67.74} & \textbf{54.38} & \textbf{46.81} \\
 & \textsc{NuMuon} & 20.20 \textcolor{red}{(+5.2\%)} & 29.95 & 55.68 & 43.51 & 35.40 & \textbf{34.40} & 67.41 & 53.51 & 45.69 \textcolor{red}{(-2.4\%)} \\
\midrule
\multirow{3}{*}{\rotatebox{90}{20\%}} & \textsc{AdamW} & 39.36 & 25.17 & 42.30 & 31.10 & 14.30 & 27.40 & 56.86 & 49.96 & 35.30 \\
 & \textsc{Muon} & 27.89 & 26.62 & 45.33 & 37.70 & 28.04 & 31.40 & 62.46 & \textbf{53.83} & 40.77 \\
 & \textsc{NuMuon} & 22.87 \textcolor{green!40!black}{(-18.0\%)} & \textbf{27.47} & \textbf{48.65} & \textbf{40.80} & \textbf{31.17} & \textbf{34.60} & \textbf{65.02} & 51.62 & \textbf{42.76} \textcolor{green!40!black}{(+4.9\%)} \\
\midrule
\multirow{3}{*}{\rotatebox{90}{40\%}} & \textsc{AdamW} & 142.39 & 22.61 & 31.90 & 26.65 & 2.17 & 27.40 & 51.41 & \textbf{50.75} & 30.41 \\
 & \textsc{Muon} & 52.14 & 23.12 & 34.68 & 32.34 & 18.16 & 28.00 & 56.75 & 49.72 & 34.68 \\
 & \textsc{NuMuon} & 32.07 \textcolor{green!40!black}{(-38.5\%)} & \textbf{24.32} & \textbf{41.29} & \textbf{35.88} & \textbf{23.52} & \textbf{29.60} & \textbf{61.53} & \textbf{50.75} & \textbf{38.13} \textcolor{green!40!black}{(+9.9\%)} \\
\midrule
\multirow{3}{*}{\rotatebox{90}{60\%}} & \textsc{AdamW} & 545.52 & 23.89 & 29.92 & 26.00 & 0.21 & 27.80 & 51.47 & 50.43 & 29.96 \\
 & \textsc{Muon} & 195.35 & 22.95 & 28.11 & 28.06 & 3.10 & \textbf{28.20} & 51.74 & 50.43 & 30.37 \\
 & \textsc{NuMuon} & 95.18 \textcolor{green!40!black}{(-51.3\%)} & \textbf{25.09} & \textbf{30.77} & \textbf{29.88} & \textbf{12.03} & 26.80 & \textbf{55.77} & \textbf{51.54} & \textbf{33.13} \textcolor{green!40!black}{(+9.1\%)} \\
\midrule
\multirow{3}{*}{\rotatebox{90}{80\%}} & \textsc{AdamW} & 1412.22 & 23.63 & 27.23 & 26.59 & 0.00 & 28.20 & 51.03 & 49.41 & 29.44 \\
 & \textsc{Muon} & 1560.59 & \textbf{25.00} & \textbf{27.61} & 26.31 & 0.00 & \textbf{29.20} & 51.58 & \textbf{52.41} & \textbf{30.30} \\
 & \textsc{NuMuon} & 840.76 \textcolor{green!40!black}{(-46.1\%)} & \textbf{25.00} & 26.64 & \textbf{26.95} & \textbf{0.10} & 27.40 & \textbf{51.96} & 50.59 & 29.81 \textcolor{red}{(-1.6\%)} \\
\bottomrule
\end{tabular}
\end{table}
}

\newcommand{\tabOlmoWhitenFull}{
\begin{table}[htbp]
\setlength{\tabcolsep}{2pt}%
\renewcommand{\arraystretch}{0.8}%
\centering
\scriptsize
\caption{WikiText2 validation perplexity and downstream task performance across compressed Olmo2-1.4B models using SVD-LLM (Whitening).}
\vspace{-0.5em}
\label{tab:olmo2_whiten}
\begin{tabular}{@{}lccccccccccc@{}}
\toprule
\multirow{2}{*}{\rotatebox{90}{\tiny{\textbf{\textsc{Comp.}}}}} & \multirow{2}{*}{\textsc{\textbf{Optimizer}}} & \multirow{2}{*}{\textbf{\textsc{Val. PPL}}} & \multicolumn{7}{c}{\textbf{\textsc{Downstream Tasks}}} & \multirow{2}{*}{\textbf{\textsc{Avg}}} \\
\cmidrule(lr){4-10}
 & & & \textbf{\textsc{ARC-C}} & \textbf{\textsc{ARC-E}} & \textbf{\textsc{H'SWag}} & \textbf{\textsc{LAM'DA}} & \textbf{\textsc{OpenBkQA}} & \textbf{\textsc{PIQA}} & \textbf{\textsc{W'Grande}} & \\
\midrule
\midrule
\multirow{3}{*}{\rotatebox{90}{0\%}} & \textsc{AdamW} & 20.58 & 30.20 & 55.47 & 44.96 & 34.76 & 35.80 & 68.93 & 50.67 & 45.83 \\
 & \textsc{Muon} & 17.80 & \textbf{34.04} & \textbf{60.10} & \textbf{50.94} & \textbf{39.28} & \textbf{37.80} & \textbf{70.46} & 54.85 & \textbf{49.64} \\
 & \textsc{NuMuon} & 19.06 \textcolor{red}{(+7.1\%)} & 31.06 & 59.55 & 48.40 & 37.80 & 37.00 & 68.99 & \textbf{55.41} & 48.32 \textcolor{red}{(-2.7\%)} \\
\midrule
\multirow{3}{*}{\rotatebox{90}{20\%}} & \textsc{AdamW} & 26.65 & 26.11 & 50.38 & 39.95 & 27.60 & 32.00 & 65.23 & 51.85 & 41.87 \\
 & \textsc{Muon} & 20.13 & \textbf{30.72} & 56.31 & 43.98 & \textbf{37.94} & 33.80 & 65.72 & \textbf{56.67} & 46.45 \\
 & \textsc{NuMuon} & 19.46 \textcolor{green!40!black}{(-3.3\%)} & 29.95 & \textbf{58.84} & \textbf{47.38} & 36.44 & \textbf{35.20} & \textbf{68.82} & 56.35 & \textbf{47.57} \textcolor{green!40!black}{(+2.4\%)} \\
\midrule
\multirow{3}{*}{\rotatebox{90}{40\%}} & \textsc{AdamW} & 51.62 & 24.57 & 39.44 & 32.81 & 14.07 & 27.60 & 56.58 & 51.46 & 35.22 \\
 & \textsc{Muon} & 27.22 & 25.26 & 43.98 & 36.77 & 30.93 & 29.20 & 60.01 & 53.51 & 39.95 \\
 & \textsc{NuMuon} & 20.75 \textcolor{green!40!black}{(-23.8\%)} & \textbf{28.84} & \textbf{56.94} & \textbf{44.95} & \textbf{35.38} & \textbf{35.40} & \textbf{67.03} & \textbf{55.41} & \textbf{46.28} \textcolor{green!40!black}{(+15.8\%)} \\
\midrule
\multirow{3}{*}{\rotatebox{90}{60\%}} & \textsc{AdamW} & 221.80 & 23.63 & 30.64 & 27.12 & 2.31 & 26.20 & 52.56 & 50.75 & 30.46 \\
 & \textsc{Muon} & 67.48 & 23.21 & 28.87 & 29.30 & 15.43 & 24.80 & 52.61 & 52.09 & 32.33 \\
 & \textsc{NuMuon} & 30.29 \textcolor{green!40!black}{(-55.1\%)} & \textbf{24.83} & \textbf{42.47} & \textbf{37.10} & \textbf{26.53} & \textbf{30.40} & \textbf{59.58} & \textbf{52.25} & \textbf{39.02} \textcolor{green!40!black}{(+20.7\%)} \\
\midrule
\multirow{3}{*}{\rotatebox{90}{80\%}} & \textsc{AdamW} & 1040.87 & 24.66 & 24.79 & 25.62 & 0.08 & 25.20 & \textbf{51.63} & 48.54 & 28.65 \\
 & \textsc{Muon} & 474.07 & 24.83 & \textbf{28.11} & 26.46 & 0.70 & \textbf{27.00} & 50.49 & \textbf{50.36} & \textbf{29.71} \\
 & \textsc{NuMuon} & 289.28 \textcolor{green!40!black}{(-39.0\%)} & \textbf{26.02} & 27.02 & \textbf{26.74} & \textbf{1.11} & 26.00 & 50.60 & \textbf{50.36} & 29.69 \\
\bottomrule
\end{tabular}
\end{table}
}

\newcommand{\tabLlamaWhitenFull}{
\begin{table}[htbp]
\setlength{\tabcolsep}{2pt}%
\renewcommand{\arraystretch}{0.8}%
\centering
\scriptsize
\caption{WikiText2 validation perplexity and downstream task performance across compressed Llama3-1.8B models using SVD-LLM (Whitening).}
\vspace{-0.5em}
\label{tab:llama3_whiten}
\begin{tabular}{@{}lccccccccccc@{}}
\toprule
\multirow{2}{*}{\rotatebox{90}{\tiny{\textbf{\textsc{Comp.}}}}} & \multirow{2}{*}{\textsc{\textbf{Optimizer}}} & \multirow{2}{*}{\textbf{\textsc{Val. PPL}}} & \multicolumn{7}{c}{\textbf{\textsc{Downstream Tasks}}} & \multirow{2}{*}{\textbf{\textsc{Avg}}} \\
\cmidrule(lr){4-10}
 & & & \textbf{\textsc{ARC-C}} & \textbf{\textsc{ARC-E}} & \textbf{\textsc{H'SWag}} & \textbf{\textsc{LAM'DA}} & \textbf{\textsc{OpenBkQA}} & \textbf{\textsc{PIQA}} & \textbf{\textsc{W'Grande}} & \\
\midrule
\midrule
\multirow{3}{*}{\rotatebox{90}{0\%}} & \textsc{AdamW} & 18.75 & 33.79 & 62.67 & 51.28 & 40.46 & 37.00 & 71.82 & 55.33 & 50.34 \\
 & \textsc{Muon} & 15.39 & \textbf{38.82} & \textbf{67.80} & \textbf{58.35} & \textbf{47.47} & \textbf{41.80} & \textbf{74.37} & \textbf{59.27} & \textbf{55.41} \\
 & \textsc{NuMuon} & 17.56 \textcolor{red}{(+14.1\%)} & 35.49 & 65.19 & 54.96 & 44.38 & 38.20 & 72.91 & 58.56 & 52.81 \textcolor{red}{(-4.7\%)} \\
\midrule
\multirow{3}{*}{\rotatebox{90}{20\%}} & \textsc{AdamW} & 24.27 & 29.78 & 56.19 & 43.82 & 35.77 & 33.40 & 67.03 & 53.67 & 45.67 \\
 & \textsc{Muon} & 18.04 & 30.20 & 58.38 & 47.86 & 40.69 & 34.40 & 66.54 & \textbf{58.33} & 48.06 \\
 & \textsc{NuMuon} & 17.90 \textcolor{green!40!black}{(-0.8\%)} & \textbf{34.39} & \textbf{63.80} & \textbf{53.32} & \textbf{42.29} & \textbf{37.00} & \textbf{71.71} & 58.09 & \textbf{51.51} \textcolor{green!40!black}{(+7.2\%)} \\
\midrule
\multirow{3}{*}{\rotatebox{90}{40\%}} & \textsc{AdamW} & 51.68 & 23.38 & 39.02 & 33.24 & 14.85 & 28.00 & 57.40 & 52.25 & 35.45 \\
 & \textsc{Muon} & 26.39 & 26.19 & 43.14 & 36.74 & 28.86 & 28.20 & 59.52 & 53.83 & 39.50 \\
 & \textsc{NuMuon} & 19.18 \textcolor{green!40!black}{(-27.3\%)} & \textbf{31.66} & \textbf{60.77} & \textbf{49.49} & \textbf{38.87} & \textbf{34.80} & \textbf{69.37} & \textbf{57.38} & \textbf{48.91} \textcolor{green!40!black}{(+23.8\%)} \\
\midrule
\multirow{3}{*}{\rotatebox{90}{60\%}} & \textsc{AdamW} & 307.36 & 23.72 & 29.21 & 26.87 & 0.74 & 26.60 & 51.20 & 49.88 & 29.75 \\
 & \textsc{Muon} & 63.75 & 22.44 & 31.48 & 29.69 & 10.63 & 24.80 & 52.99 & 52.17 & 32.03 \\
 & \textsc{NuMuon} & 32.22 \textcolor{green!40!black}{(-49.5\%)} & \textbf{23.81} & \textbf{45.41} & \textbf{37.54} & \textbf{23.48} & \textbf{29.00} & \textbf{61.15} & \textbf{53.67} & \textbf{39.15} \textcolor{green!40!black}{(+22.2\%)} \\
\midrule
\multirow{3}{*}{\rotatebox{90}{80\%}} & \textsc{AdamW} & 2672.33 & \textbf{25.00} & 25.67 & 26.22 & 0.00 & 27.00 & 50.11 & 50.43 & 29.20 \\
 & \textsc{Muon} & 346.10 & 24.32 & 27.53 & 26.80 & 0.21 & \textbf{27.20} & 50.76 & 50.36 & 29.60 \\
 & \textsc{NuMuon} & 262.48 \textcolor{green!40!black}{(-24.2\%)} & 24.15 & \textbf{30.09} & \textbf{27.52} & \textbf{1.47} & 25.60 & \textbf{51.58} & \textbf{51.62} & \textbf{30.29} \textcolor{green!40!black}{(+2.3\%)} \\
\bottomrule
\end{tabular}
\end{table}
}

\newcommand{\tabQwenSVDLLMFull}{
\begin{table}[htbp]
\setlength{\tabcolsep}{2pt}%
\renewcommand{\arraystretch}{0.8}%
\centering
\scriptsize
\caption{WikiText2 validation perplexity and downstream task performance across compressed Qwen3-0.6B models using SVD-LLM (Whitening + LoRA).}
\vspace{-0.5em}
\label{tab:qwen3_svdllm}
\begin{tabular}{@{}lccccccccccc@{}}
\toprule
\multirow{2}{*}{\rotatebox{90}{\tiny{\textbf{\textsc{Comp.}}}}} & \multirow{2}{*}{\textsc{\textbf{Optimizer}}} & \multirow{2}{*}{\textbf{\textsc{Val. PPL}}} & \multicolumn{7}{c}{\textbf{\textsc{Downstream Tasks}}} & \multirow{2}{*}{\textbf{\textsc{Avg}}} \\
\cmidrule(lr){4-10}
 & & & \textbf{\textsc{ARC-C}} & \textbf{\textsc{ARC-E}} & \textbf{\textsc{H'SWag}} & \textbf{\textsc{LAM'DA}} & \textbf{\textsc{OpenBkQA}} & \textbf{\textsc{PIQA}} & \textbf{\textsc{W'Grande}} & \\
\midrule
\midrule
\multirow{3}{*}{\rotatebox{90}{0\%}} & \textsc{AdamW} & 25.39 & 27.30 & 49.66 & 36.04 & 27.77 & 31.20 & 65.02 & 49.64 & 40.95 \\
 & \textsc{Muon} & 19.20 & \textbf{30.72} & \textbf{57.95} & \textbf{45.56} & \textbf{37.69} & 33.60 & \textbf{67.74} & \textbf{54.38} & \textbf{46.81} \\
 & \textsc{NuMuon} & 20.20 \textcolor{red}{(+5.2\%)} & 29.95 & 55.68 & 43.51 & 35.40 & \textbf{34.40} & 67.41 & 53.51 & 45.69 \textcolor{red}{(-2.4\%)} \\
\midrule
\multirow{3}{*}{\rotatebox{90}{20\%}} & \textsc{AdamW} & 38.49 & 25.51 & 46.21 & 32.10 & 19.66 & 30.60 & 61.21 & 50.59 & 37.98 \\
 & \textsc{Muon} & 25.84 & \textbf{29.52} & \textbf{51.09} & 40.22 & 30.43 & 32.40 & \textbf{66.87} & \textbf{53.04} & 43.37 \\
 & \textsc{NuMuon} & 24.01 \textcolor{green!40!black}{(-7.1\%)} & 28.41 & 49.83 & \textbf{41.40} & \textbf{32.60} & \textbf{35.40} & 65.18 & 52.41 & \textbf{43.60} \textcolor{green!40!black}{(+0.5\%)} \\
\midrule
\multirow{3}{*}{\rotatebox{90}{40\%}} & \textsc{AdamW} & 64.17 & 22.53 & 39.18 & 28.39 & 9.86 & 27.60 & 57.18 & 51.70 & 33.78 \\
 & \textsc{Muon} & 36.36 & \textbf{27.65} & 45.54 & 36.33 & 25.73 & 31.80 & 62.02 & 50.20 & 39.90 \\
 & \textsc{NuMuon} & 29.16 \textcolor{green!40!black}{(-19.8\%)} & 26.62 & \textbf{47.47} & \textbf{38.10} & \textbf{28.60} & \textbf{32.60} & \textbf{63.71} & \textbf{53.91} & \textbf{41.57} \textcolor{green!40!black}{(+4.2\%)} \\
\midrule
\multirow{3}{*}{\rotatebox{90}{60\%}} & \textsc{AdamW} & 194.18 & 22.10 & 33.50 & 26.33 & 2.74 & 26.00 & 53.54 & 49.57 & 30.54 \\
 & \textsc{Muon} & 86.75 & 23.63 & 37.12 & 31.25 & 18.61 & 27.60 & 56.80 & 50.36 & 35.05 \\
 & \textsc{NuMuon} & 50.88 \textcolor{green!40!black}{(-41.3\%)} & \textbf{24.57} & \textbf{41.04} & \textbf{32.60} & \textbf{20.09} & \textbf{30.00} & \textbf{60.55} & \textbf{50.59} & \textbf{37.06} \textcolor{green!40!black}{(+5.7\%)} \\
\midrule
\multirow{3}{*}{\rotatebox{90}{80\%}} & \textsc{AdamW} & 399.33 & 21.67 & 30.09 & 25.86 & 0.56 & \textbf{27.60} & 52.29 & 49.49 & 29.65 \\
 & \textsc{Muon} & 446.48 & \textbf{23.89} & 31.78 & 27.33 & 5.59 & 26.40 & 52.77 & 50.91 & 31.24 \\
 & \textsc{NuMuon} & 240.48 \textcolor{green!40!black}{(-46.1\%)} & 23.21 & \textbf{33.67} & \textbf{27.45} & \textbf{9.82} & 27.40 & \textbf{55.71} & \textbf{51.78} & \textbf{32.72} \textcolor{green!40!black}{(+4.7\%)} \\
\bottomrule
\end{tabular}
\end{table}
}

\newcommand{\tabOlmoSVDLLMFull}{
\begin{table}[htbp]
\setlength{\tabcolsep}{2pt}%
\renewcommand{\arraystretch}{0.8}%
\centering
\scriptsize
\caption{WikiText2 validation perplexity and downstream task performance across compressed Olmo2-1.4B models using SVD-LLM (Whitening + LoRA).}
\vspace{-0.5em}
\label{tab:olmo2_svdllm}
\begin{tabular}{@{}lccccccccccc@{}}
\toprule
\multirow{2}{*}{\rotatebox{90}{\tiny{\textbf{\textsc{Comp.}}}}} & \multirow{2}{*}{\textsc{\textbf{Optimizer}}} & \multirow{2}{*}{\textbf{\textsc{Val. PPL}}} & \multicolumn{7}{c}{\textbf{\textsc{Downstream Tasks}}} & \multirow{2}{*}{\textbf{\textsc{Avg}}} \\
\cmidrule(lr){4-10}
 & & & \textbf{\textsc{ARC-C}} & \textbf{\textsc{ARC-E}} & \textbf{\textsc{H'SWag}} & \textbf{\textsc{LAM'DA}} & \textbf{\textsc{OpenBkQA}} & \textbf{\textsc{PIQA}} & \textbf{\textsc{W'Grande}} & \\
\midrule
\midrule
\multirow{3}{*}{\rotatebox{90}{0\%}} & \textsc{AdamW} & 20.58 & 30.20 & 55.47 & 44.96 & 34.76 & 35.80 & 68.93 & 50.67 & 45.83 \\
 & \textsc{Muon} & 17.80 & \textbf{34.04} & \textbf{60.10} & \textbf{50.94} & \textbf{39.28} & \textbf{37.80} & \textbf{70.46} & 54.85 & \textbf{49.64} \\
 & \textsc{NuMuon} & 19.06 \textcolor{red}{(+7.1\%)} & 31.06 & 59.55 & 48.40 & 37.80 & 37.00 & 68.99 & \textbf{55.41} & 48.32 \textcolor{red}{(-2.7\%)} \\
\midrule
\multirow{3}{*}{\rotatebox{90}{20\%}} & \textsc{AdamW} & 28.69 & 26.96 & 52.10 & 41.15 & 30.70 & 32.80 & 66.21 & 52.09 & 43.14 \\
 & \textsc{Muon} & 21.48 & \textbf{32.76} & 57.74 & 46.80 & \textbf{38.19} & 35.60 & \textbf{68.93} & \textbf{54.70} & \textbf{47.82} \\
 & \textsc{NuMuon} & 22.53 \textcolor{red}{(+4.9\%)} & 31.66 & \textbf{58.54} & \textbf{47.63} & 36.72 & \textbf{36.20} & \textbf{68.93} & 54.62 & 47.76 \textcolor{red}{(-0.1\%)} \\
\midrule
\multirow{3}{*}{\rotatebox{90}{40\%}} & \textsc{AdamW} & 38.80 & 26.11 & 45.79 & 37.09 & 20.47 & 29.00 & 62.57 & 52.01 & 39.01 \\
 & \textsc{Muon} & 25.61 & 29.35 & 51.94 & 42.51 & 32.23 & 32.40 & 64.96 & 53.91 & 43.90 \\
 & \textsc{NuMuon} & 23.17 \textcolor{green!40!black}{(-9.5\%)} & \textbf{30.55} & \textbf{55.72} & \textbf{46.51} & \textbf{38.23} & \textbf{35.80} & \textbf{68.23} & \textbf{55.41} & \textbf{47.21} \textcolor{green!40!black}{(+7.5\%)} \\
\midrule
\multirow{3}{*}{\rotatebox{90}{60\%}} & \textsc{AdamW} & 74.04 & 23.55 & 39.18 & 31.17 & 9.86 & 28.00 & 58.32 & 47.51 & 33.94 \\
 & \textsc{Muon} & 43.80 & 26.19 & 36.95 & 34.47 & 16.81 & 28.00 & 58.11 & 52.49 & 36.15 \\
 & \textsc{NuMuon} & 28.85 \textcolor{green!40!black}{(-34.1\%)} & \textbf{28.50} & \textbf{50.17} & \textbf{41.67} & \textbf{31.54} & \textbf{31.20} & \textbf{66.32} & \textbf{53.75} & \textbf{43.31} \textcolor{green!40!black}{(+19.8\%)} \\
\midrule
\multirow{3}{*}{\rotatebox{90}{80\%}} & \textsc{AdamW} & 252.34 & 22.95 & \textbf{33.00} & 26.04 & 2.19 & \textbf{28.80} & 54.35 & 50.43 & 31.11 \\
 & \textsc{Muon} & 140.29 & 23.21 & 29.88 & 27.48 & 6.04 & 26.00 & 54.24 & 49.88 & 30.96 \\
 & \textsc{NuMuon} & 103.38 \textcolor{green!40!black}{(-26.3\%)} & \textbf{23.89} & 32.07 & \textbf{28.16} & \textbf{6.73} & 27.00 & \textbf{54.95} & \textbf{51.22} & \textbf{32.00} \textcolor{green!40!black}{(+3.4\%)} \\
\bottomrule
\end{tabular}
\end{table}
}

\newcommand{\tabLlamaSVDLLMFull}{
\begin{table}[htbp]
\setlength{\tabcolsep}{2pt}%
\renewcommand{\arraystretch}{0.8}%
\centering
\scriptsize
\caption{WikiText2 validation perplexity and downstream task performance across compressed Llama3-1.8B models using SVD-LLM (Whitening + LoRA).}
\vspace{-0.5em}
\label{tab:llama3_svdllm}
\begin{tabular}{@{}lccccccccccc@{}}
\toprule
\multirow{2}{*}{\rotatebox{90}{\tiny{\textbf{\textsc{Comp.}}}}} & \multirow{2}{*}{\textsc{\textbf{Optimizer}}} & \multirow{2}{*}{\textbf{\textsc{Val. PPL}}} & \multicolumn{7}{c}{\textbf{\textsc{Downstream Tasks}}} & \multirow{2}{*}{\textbf{\textsc{Avg}}} \\
\cmidrule(lr){4-10}
 & & & \textbf{\textsc{ARC-C}} & \textbf{\textsc{ARC-E}} & \textbf{\textsc{H'SWag}} & \textbf{\textsc{LAM'DA}} & \textbf{\textsc{OpenBkQA}} & \textbf{\textsc{PIQA}} & \textbf{\textsc{W'Grande}} & \\
\midrule
\midrule
\multirow{3}{*}{\rotatebox{90}{0\%}} & \textsc{AdamW} & 18.75 & 33.79 & 62.67 & 51.28 & 40.46 & 37.00 & 71.82 & 55.33 & 50.34 \\
 & \textsc{Muon} & 15.39 & \textbf{38.82} & \textbf{67.80} & \textbf{58.35} & \textbf{47.47} & \textbf{41.80} & \textbf{74.37} & \textbf{59.27} & \textbf{55.41} \\
 & \textsc{NuMuon} & 17.56 \textcolor{red}{(+14.1\%)} & 35.49 & 65.19 & 54.96 & 44.38 & 38.20 & 72.91 & 58.56 & 52.81 \textcolor{red}{(-4.7\%)} \\
\midrule
\multirow{3}{*}{\rotatebox{90}{20\%}} & \textsc{AdamW} & 21.69 & 30.03 & 55.01 & 46.22 & 37.69 & 34.80 & 68.55 & 53.83 & 46.59 \\
 & \textsc{Muon} & 18.04 & 30.20 & 58.38 & 47.86 & 40.69 & 34.40 & 66.54 & \textbf{58.33} & 48.06 \\
 & \textsc{NuMuon} & 17.90 \textcolor{green!40!black}{(-0.8\%)} & \textbf{34.39} & \textbf{63.80} & \textbf{53.32} & \textbf{42.29} & \textbf{37.00} & \textbf{71.71} & 58.09 & \textbf{51.51} \textcolor{green!40!black}{(+7.2\%)} \\
\midrule
\multirow{3}{*}{\rotatebox{90}{40\%}} & \textsc{AdamW} & 27.94 & 28.41 & 49.37 & 40.69 & 29.52 & 30.00 & 65.13 & 51.22 & 42.05 \\
 & \textsc{Muon} & 26.39 & 26.19 & 43.14 & 36.74 & 28.86 & 28.20 & 59.52 & 53.83 & 39.50 \\
 & \textsc{NuMuon} & 19.18 \textcolor{green!40!black}{(-27.3\%)} & \textbf{31.66} & \textbf{60.77} & \textbf{49.49} & \textbf{38.87} & \textbf{34.80} & \textbf{69.37} & \textbf{57.38} & \textbf{48.91} \textcolor{green!40!black}{(+23.8\%)} \\
\midrule
\multirow{3}{*}{\rotatebox{90}{60\%}} & \textsc{AdamW} & 49.59 & \textbf{25.09} & 40.24 & 33.08 & 15.04 & 27.00 & 58.27 & 49.72 & 35.49 \\
 & \textsc{Muon} & 63.75 & 22.44 & 31.48 & 29.69 & 10.63 & 24.80 & 52.99 & 52.17 & 32.03 \\
 & \textsc{NuMuon} & 32.22 \textcolor{green!40!black}{(-49.5\%)} & 23.81 & \textbf{45.41} & \textbf{37.54} & \textbf{23.48} & \textbf{29.00} & \textbf{61.15} & \textbf{53.67} & \textbf{39.15} \textcolor{green!40!black}{(+22.2\%)} \\
\midrule
\multirow{3}{*}{\rotatebox{90}{80\%}} & \textsc{AdamW} & 226.44 & 22.61 & 30.85 & 26.52 & 2.23 & 26.20 & 53.75 & \textbf{50.83} & 30.43 \\
 & \textsc{Muon} & 346.10 & \textbf{24.32} & 27.53 & 26.80 & 0.21 & 27.20 & 50.76 & 50.36 & 29.60 \\
 & \textsc{NuMuon} & 61.25 \textcolor{green!40!black}{(-82.3\%)} & 23.89 & \textbf{39.35} & \textbf{31.10} & \textbf{15.37} & \textbf{29.60} & \textbf{58.43} & 50.04 & \textbf{35.40} \textcolor{green!40!black}{(+19.6\%)} \\
\bottomrule
\end{tabular}
\end{table}
}

\newcommand{\tabQwenDobiSVDFull}{
\begin{table}[htbp]
\setlength{\tabcolsep}{2pt}%
\renewcommand{\arraystretch}{0.8}%
\centering
\scriptsize
\caption{WikiText2 validation perplexity and downstream task performance across compressed Qwen3-0.6B models using Dobi-SVD.}
\vspace{-0.5em}
\label{tab:qwen3_dobisvd}
\begin{tabular}{@{}lccccccccccc@{}}
\toprule
\multirow{2}{*}{\rotatebox{90}{\tiny{\textbf{\textsc{Comp.}}}}} & \multirow{2}{*}{\textsc{\textbf{Optimizer}}} & \multirow{2}{*}{\textbf{\textsc{Val. PPL}}} & \multicolumn{7}{c}{\textbf{\textsc{Downstream Tasks}}} & \multirow{2}{*}{\textbf{\textsc{Avg}}} \\
\cmidrule(lr){4-10}
 & & & \textbf{\textsc{ARC-C}} & \textbf{\textsc{ARC-E}} & \textbf{\textsc{H'SWag}} & \textbf{\textsc{LAM'DA}} & \textbf{\textsc{OpenBkQA}} & \textbf{\textsc{PIQA}} & \textbf{\textsc{W'Grande}} & \\
\midrule
\midrule
\multirow{3}{*}{\rotatebox{90}{0\%}} & \textsc{AdamW} & 25.39 & 27.30 & 49.66 & 36.04 & 27.77 & 31.20 & 65.02 & 49.64 & 40.95 \\
 & \textsc{Muon} & 19.20 & \textbf{30.72} & \textbf{57.95} & \textbf{45.56} & \textbf{37.69} & 33.60 & \textbf{67.74} & \textbf{54.38} & \textbf{46.81} \\
 & \textsc{NuMuon} & 20.20 \textcolor{red}{(+5.2\%)} & 29.95 & 55.68 & 43.51 & 35.40 & \textbf{34.40} & 67.41 & 53.51 & 45.69 \textcolor{red}{(-2.4\%)} \\
\midrule
\multirow{3}{*}{\rotatebox{90}{20\%}} & \textsc{AdamW} & 26.94 & 26.19 & 48.95 & 35.48 & 28.47 & 31.00 & 65.02 & 49.01 & 40.59 \\
 & \textsc{Muon} & 19.90 & 30.03 & \textbf{57.15} & \textbf{45.58} & \textbf{37.51} & 33.60 & \textbf{68.93} & \textbf{54.22} & \textbf{46.72} \\
 & \textsc{NuMuon} & 20.83 \textcolor{red}{(+4.7\%)} & \textbf{30.46} & 54.42 & 42.95 & 33.26 & \textbf{34.00} & 66.87 & 53.43 & 45.06 \textcolor{red}{(-3.6\%)} \\
\midrule
\multirow{3}{*}{\rotatebox{90}{40\%}} & \textsc{AdamW} & 27.25 & 27.05 & 49.28 & 35.57 & 31.05 & 30.60 & 64.09 & 49.72 & 41.05 \\
 & \textsc{Muon} & 20.16 & \textbf{30.72} & \textbf{56.06} & \textbf{45.54} & \textbf{33.73} & 33.20 & \textbf{68.01} & 53.12 & \textbf{45.77} \\
 & \textsc{NuMuon} & 20.77 \textcolor{red}{(+3.0\%)} & 29.18 & 53.58 & 43.31 & 32.52 & \textbf{33.80} & 67.36 & \textbf{54.30} & 44.86 \textcolor{red}{(-2.0\%)} \\
\midrule
\multirow{3}{*}{\rotatebox{90}{60\%}} & \textsc{AdamW} & 32.64 & 25.68 & 42.63 & 34.45 & 16.77 & 32.40 & 60.83 & 50.28 & 37.58 \\
 & \textsc{Muon} & 24.36 & \textbf{29.01} & 49.28 & 41.71 & 24.94 & 31.40 & 64.47 & \textbf{53.12} & 41.99 \\
 & \textsc{NuMuon} & 21.66 \textcolor{green!40!black}{(-11.1\%)} & 27.65 & \textbf{51.73} & \textbf{42.38} & \textbf{28.95} & \textbf{33.20} & \textbf{65.67} & 51.78 & \textbf{43.05} \textcolor{green!40!black}{(+2.5\%)} \\
\midrule
\multirow{3}{*}{\rotatebox{90}{80\%}} & \textsc{AdamW} & 115.94 & 23.21 & 31.02 & 27.38 & 2.31 & 26.40 & 51.20 & 50.36 & 30.27 \\
 & \textsc{Muon} & 50.31 & 22.95 & 33.08 & 31.47 & 11.00 & 26.40 & 56.15 & \textbf{52.17} & 33.32 \\
 & \textsc{NuMuon} & 36.10 \textcolor{green!40!black}{(-28.2\%)} & \textbf{26.19} & \textbf{36.41} & \textbf{35.40} & \textbf{13.62} & \textbf{27.40} & \textbf{60.45} & 51.78 & \textbf{35.89} \textcolor{green!40!black}{(+7.7\%)} \\
\bottomrule
\end{tabular}
\end{table}
}

\newcommand{\tabOlmoDobiSVDFull}{
\begin{table}[htbp]
\setlength{\tabcolsep}{2pt}%
\renewcommand{\arraystretch}{0.8}%
\centering
\scriptsize
\caption{WikiText2 validation perplexity and downstream task performance across compressed Olmo2-1.4B models using Dobi-SVD.}
\vspace{-0.5em}
\label{tab:olmo2_dobisvd}
\begin{tabular}{@{}lccccccccccc@{}}
\toprule
\multirow{2}{*}{\rotatebox{90}{\tiny{\textbf{\textsc{Comp.}}}}} & \multirow{2}{*}{\textsc{\textbf{Optimizer}}} & \multirow{2}{*}{\textbf{\textsc{Val. PPL}}} & \multicolumn{7}{c}{\textbf{\textsc{Downstream Tasks}}} & \multirow{2}{*}{\textbf{\textsc{Avg}}} \\
\cmidrule(lr){4-10}
 & & & \textbf{\textsc{ARC-C}} & \textbf{\textsc{ARC-E}} & \textbf{\textsc{H'SWag}} & \textbf{\textsc{LAM'DA}} & \textbf{\textsc{OpenBkQA}} & \textbf{\textsc{PIQA}} & \textbf{\textsc{W'Grande}} & \\
\midrule
\midrule
\multirow{3}{*}{\rotatebox{90}{0\%}} & \textsc{AdamW} & 20.58 & 30.20 & 55.47 & 44.96 & 34.76 & 35.80 & 68.93 & 50.67 & 45.83 \\
 & \textsc{Muon} & 17.80 & \textbf{34.04} & \textbf{60.10} & \textbf{50.94} & \textbf{39.28} & \textbf{37.80} & \textbf{70.46} & 54.85 & \textbf{49.64} \\
 & \textsc{NuMuon} & 19.06 \textcolor{red}{(+7.1\%)} & 31.06 & 59.55 & 48.40 & 37.80 & 37.00 & 68.99 & \textbf{55.41} & 48.32 \textcolor{red}{(-2.7\%)} \\
\midrule
\multirow{3}{*}{\rotatebox{90}{20\%}} & \textsc{AdamW} & 23.12 & 29.18 & 53.45 & 43.09 & 28.45 & 33.80 & 67.36 & 51.14 & 43.78 \\
 & \textsc{Muon} & 18.81 & \textbf{31.91} & 58.59 & \textbf{48.39} & \textbf{35.57} & 35.20 & 68.88 & \textbf{56.67} & 47.89 \\
 & \textsc{NuMuon} & 19.23 \textcolor{red}{(+2.2\%)} & 31.23 & \textbf{59.97} & 48.19 & 35.38 & \textbf{36.20} & \textbf{69.21} & 55.56 & \textbf{47.96} \textcolor{green!40!black}{(+0.1\%)} \\
\midrule
\multirow{3}{*}{\rotatebox{90}{40\%}} & \textsc{AdamW} & 34.30 & 26.28 & 45.66 & 37.74 & 17.48 & 28.80 & 62.13 & 50.51 & 38.37 \\
 & \textsc{Muon} & 22.90 & 27.13 & 49.37 & 42.45 & \textbf{32.19} & 30.80 & 63.33 & \textbf{55.41} & 42.95 \\
 & \textsc{NuMuon} & 20.66 \textcolor{green!40!black}{(-9.8\%)} & \textbf{27.99} & \textbf{55.01} & \textbf{46.21} & 29.56 & \textbf{33.60} & \textbf{67.68} & 54.93 & \textbf{45.00} \textcolor{green!40!black}{(+4.8\%)} \\
\midrule
\multirow{3}{*}{\rotatebox{90}{60\%}} & \textsc{AdamW} & 127.21 & 24.49 & 32.91 & 29.37 & 4.17 & 26.60 & 54.46 & 50.04 & 31.72 \\
 & \textsc{Muon} & 37.23 & 24.06 & 36.95 & 34.08 & 22.98 & 27.60 & 58.00 & \textbf{54.46} & 36.88 \\
 & \textsc{NuMuon} & 23.49 \textcolor{green!40!black}{(-36.9\%)} & \textbf{27.65} & \textbf{52.27} & \textbf{43.40} & \textbf{23.33} & \textbf{32.40} & \textbf{65.40} & 54.30 & \textbf{42.68} \textcolor{green!40!black}{(+15.7\%)} \\
\midrule
\multirow{3}{*}{\rotatebox{90}{80\%}} & \textsc{AdamW} & 930.61 & \textbf{24.40} & 27.74 & 25.85 & 0.19 & 25.00 & 52.88 & 50.20 & 29.47 \\
 & \textsc{Muon} & 160.56 & 24.32 & 27.40 & 28.29 & 3.07 & 28.20 & 51.20 & 51.14 & 30.52 \\
 & \textsc{NuMuon} & 65.93 \textcolor{green!40!black}{(-58.9\%)} & 23.81 & \textbf{32.32} & \textbf{31.25} & \textbf{8.69} & \textbf{29.20} & \textbf{55.33} & \textbf{52.72} & \textbf{33.33} \textcolor{green!40!black}{(+9.2\%)} \\
\bottomrule
\end{tabular}
\end{table}
}

\newcommand{\tabLlamaDobiSVDFull}{
\begin{table}[htbp]
\setlength{\tabcolsep}{2pt}%
\renewcommand{\arraystretch}{0.8}%
\centering
\scriptsize
\caption{WikiText2 validation perplexity and downstream task performance across compressed Llama3-1.8B models using Dobi-SVD.}
\vspace{-0.5em}
\label{tab:llama3_dobisvd}
\begin{tabular}{@{}lccccccccccc@{}}
\toprule
\multirow{2}{*}{\rotatebox{90}{\tiny{\textbf{\textsc{Comp.}}}}} & \multirow{2}{*}{\textsc{\textbf{Optimizer}}} & \multirow{2}{*}{\textbf{\textsc{Val. PPL}}} & \multicolumn{7}{c}{\textbf{\textsc{Downstream Tasks}}} & \multirow{2}{*}{\textbf{\textsc{Avg}}} \\
\cmidrule(lr){4-10}
 & & & \textbf{\textsc{ARC-C}} & \textbf{\textsc{ARC-E}} & \textbf{\textsc{H'SWag}} & \textbf{\textsc{LAM'DA}} & \textbf{\textsc{OpenBkQA}} & \textbf{\textsc{PIQA}} & \textbf{\textsc{W'Grande}} & \\
\midrule
\midrule
\multirow{3}{*}{\rotatebox{90}{0\%}} & \textsc{AdamW} & 18.75 & 33.79 & 62.67 & 51.28 & 40.46 & 37.00 & 71.82 & 55.33 & 50.34 \\
 & \textsc{Muon} & 15.39 & \textbf{38.82} & \textbf{67.80} & \textbf{58.35} & \textbf{47.47} & \textbf{41.80} & \textbf{74.37} & \textbf{59.27} & \textbf{55.41} \\
 & \textsc{NuMuon} & 17.56 \textcolor{red}{(+14.1\%)} & 35.49 & 65.19 & 54.96 & 44.38 & 38.20 & 72.91 & 58.56 & 52.81 \textcolor{red}{(-4.7\%)} \\
\midrule
\multirow{3}{*}{\rotatebox{90}{20\%}} & \textsc{AdamW} & 21.42 & 34.04 & 56.52 & 49.41 & 31.85 & 35.00 & 70.13 & 55.09 & 47.43 \\
 & \textsc{Muon} & 16.37 & \textbf{34.56} & \textbf{63.59} & \textbf{54.99} & \textbf{45.97} & \textbf{38.80} & 71.33 & 58.01 & \textbf{52.46} \\
 & \textsc{NuMuon} & 17.96 \textcolor{red}{(+9.7\%)} & 33.96 & 62.29 & 53.64 & 43.10 & 35.40 & \textbf{71.38} & \textbf{58.17} & 51.13 \textcolor{red}{(-2.5\%)} \\
\midrule
\multirow{3}{*}{\rotatebox{90}{40\%}} & \textsc{AdamW} & 28.83 & 30.72 & 49.79 & 43.15 & 31.30 & 33.20 & 64.69 & 53.67 & 43.79 \\
 & \textsc{Muon} & 20.44 & 28.16 & 52.27 & 46.22 & \textbf{40.29} & 29.80 & 63.98 & 56.35 & 45.30 \\
 & \textsc{NuMuon} & 18.63 \textcolor{green!40!black}{(-8.9\%)} & \textbf{31.31} & \textbf{60.23} & \textbf{51.98} & 39.59 & \textbf{34.80} & \textbf{70.18} & \textbf{57.54} & \textbf{49.38} \textcolor{green!40!black}{(+9.0\%)} \\
\midrule
\multirow{3}{*}{\rotatebox{90}{60\%}} & \textsc{AdamW} & 164.19 & 23.72 & 32.53 & 29.14 & 2.95 & 25.80 & 54.57 & 50.43 & 31.31 \\
 & \textsc{Muon} & 40.55 & 23.81 & 34.09 & 34.13 & 15.56 & 25.80 & 55.77 & 53.35 & 34.64 \\
 & \textsc{NuMuon} & 27.42 \textcolor{green!40!black}{(-32.4\%)} & \textbf{26.71} & \textbf{45.16} & \textbf{43.86} & \textbf{17.16} & \textbf{31.80} & \textbf{62.89} & \textbf{53.91} & \textbf{40.21} \textcolor{green!40!black}{(+16.1\%)} \\
\midrule
\multirow{3}{*}{\rotatebox{90}{80\%}} & \textsc{AdamW} & 605.85 & \textbf{24.83} & 28.03 & 26.59 & 0.12 & 26.80 & 50.54 & 48.54 & 29.35 \\
 & \textsc{Muon} & 113.19 & 24.57 & 28.96 & 28.83 & 2.50 & 26.20 & 51.63 & \textbf{52.88} & 30.80 \\
 & \textsc{NuMuon} & 50.05 \textcolor{green!40!black}{(-55.8\%)} & 23.63 & \textbf{35.98} & \textbf{33.45} & \textbf{8.99} & \textbf{27.20} & \textbf{57.24} & 52.64 & \textbf{34.16} \textcolor{green!40!black}{(+10.9\%)} \\
\bottomrule
\end{tabular}
\end{table}
}

\newcommand{\tabLlamaLLMComp}{
\begin{table*}[t!]
\setlength{\tabcolsep}{2pt}%
\renewcommand{\arraystretch}{0.8}%
\centering
\scriptsize
\caption{WikiText2 validation perplexity and downstream task performance across compressed Llama3-1.8B models using various LLM compression methods \@ 40\% compression. Full results for all models and methods can be found in~\Cref{app:extended_results:results}.}
\vspace{-1em}
\label{tab:llama_comp_sum}
\begin{tabular}{@{}lccccccccccc@{}}
    \toprule
    \multirow{2}{*}{\rotatebox{0}{\textbf{\textsc{Method}}}} & \multirow{2}{*}{\textsc{\textbf{Optimizer}}} & \multirow{2}{*}{\textbf{\textsc{Val. PPL}} \textdownarrow} & \multicolumn{7}{c}{\textbf{\textsc{Downstream Tasks}}} & \multirow{2}{*}{\textbf{\textsc{Avg}} \textuparrow} \\
    \cmidrule(lr){4-10}
     & & & \textbf{\textsc{ARC-C}} & \textbf{\textsc{ARC-E}} & \textbf{\textsc{H'SWag}} & \textbf{\textsc{LAM'DA}} & \textbf{\textsc{OpenBkQA}} & \textbf{\textsc{PIQA}} & \textbf{\textsc{W'Grande}} & \\
    \midrule
    \midrule
    \multirow{3}{*}{\rotatebox{0}{ASVD}} & \textsc{AdamW} & 718.71 & 22.61 & 37.25 & 30.15 & 6.00 & 26.40 & 54.52 & 49.88 & 32.40 \\
     & \textsc{Muon} & 1443.55 & 24.66 & 31.90 & 29.28 & 7.08 & 27.80 & 55.93 & 50.12 & 32.40 \\
     & \textsc{NuMuon} & 20.56 \textcolor{green!40!black}{(-98.6\%)} & \textbf{32.76} & \textbf{62.88} & \textbf{50.27} & \textbf{42.64} & \textbf{37.20} & \textbf{71.38} & \textbf{56.20} & \textbf{50.48} \textcolor{green!40!black}{(+55.8\%)} \\
    \midrule
    \multirow{3}{*}{\rotatebox{0}{SVD-LLM}} & \textsc{AdamW} & 27.94 & 28.41 & 49.37 & 40.69 & 29.52 & 30.00 & 65.13 & 51.22 & 42.05 \\
     & \textsc{Muon} & 26.39 & 26.19 & 43.14 & 36.74 & 28.86 & 28.20 & 59.52 & 53.83 & 39.50 \\
     & \textsc{NuMuon} & 19.18 \textcolor{green!40!black}{(-27.3\%)} & \textbf{31.66} & \textbf{60.77} & \textbf{49.49} & \textbf{38.87} & \textbf{34.80} & \textbf{69.37} & \textbf{57.38} & \textbf{48.91} \textcolor{green!40!black}{(+23.8\%)} \\
    \midrule
    \multirow{3}{*}{\rotatebox{0}{Dobi-SVD}} & \textsc{AdamW} & 28.83 & 30.72 & 49.79 & 43.15 & 31.30 & 33.20 & 64.69 & 53.67 & 43.79 \\
     & \textsc{Muon} & 20.44 & 28.16 & 52.27 & 46.22 & \textbf{40.29} & 29.80 & 63.98 & 56.35 & 45.30 \\
     & \textsc{NuMuon} & 18.63 \textcolor{green!40!black}{(-8.9\%)} & \textbf{31.31} & \textbf{60.23} & \textbf{51.98} & 39.59 & \textbf{34.80} & \textbf{70.18} & \textbf{57.54} & \textbf{49.38} \textcolor{green!40!black}{(+9.0\%)} \\
    \bottomrule
\end{tabular}
\vspace{-0.135in}
\end{table*}
}

\newcommand{\tabAblationRankSched}{
\begin{table}[htbp]
\setlength{\tabcolsep}{2pt}%
\renewcommand{\arraystretch}{0.8}%
\centering
\scriptsize
\caption{WikiText2 validation perplexity and downstream task performance for various rank schedules.}
\vspace{-0.5em}
\label{tab:ablation}
\begin{tabular}{@{}lccccccccccc@{}}
\toprule
\multirow{2}{*}{\rotatebox{90}{\tiny{\textbf{\textsc{Comp.}}}}} & \multirow{2}{*}{\textsc{\textbf{Optim.}}} & \multirow{2}{*}{\textsc{\textbf{Rank Schedule}}} & \multirow{2}{*}{\textbf{\textsc{Val. PPL}} \textdownarrow} & \multicolumn{7}{c}{\textbf{\textsc{Downstream Tasks}}} & \multirow{2}{*}{\textbf{\textsc{Avg}} \textuparrow} \\
\cmidrule(lr){5-11}
 & & & & \textbf{\textsc{ARC-C}} & \textbf{\textsc{ARC-E}} & \textbf{\textsc{H'SWag}} & \textbf{\textsc{LAM'DA}} & \textbf{\textsc{OpenBkQA}} & \textbf{\textsc{PIQA}} & \textbf{\textsc{W'Grande}} & \\
\midrule
\midrule
\multirow{8}{*}{\rotatebox{90}{0\%}} & \textsc{AdamW} & - & 25.39 & 27.30 & 49.66 & 36.04 & 27.77 & 31.20 & 65.02 & 49.64 & 40.95 \\
 & \textsc{Muon} & - & 19.20 & 30.72 & 57.95 & 45.56 & 37.69 & 33.60 & 67.74 & 54.38 & 46.81 \\
 \cmidrule{2-12} & \multirow{6}{*}{\rotatebox{0}{\textsc{NuMuon}}}
 & \textsc{Piecewise ($25d/100$)} & 20.01 & 29.69 & 56.44 & 43.78 & 35.51 & 35.00 & 67.79 & 52.88 & 45.87 \\
 && \textsc{Cosine ($25d/100$)} & 20.20 & 29.95 & 55.68 & 43.51 & 35.40 & 34.40 & 67.41 & 53.51 & 45.69 \\
 && \textsc{Fixed ($5d/100$)}  & 29.13 & 25.60 & 48.06 & 33.68 & 25.13 & 31.20 & 63.98 & 49.80 & 39.64 \\
 && \textsc{Fixed ($25d/100$)} & 21.85 & 29.01 & 53.03 & 41.37 & 33.30 & 33.40 & 66.70 & 53.28 & 44.30 \\
 && \textsc{Fixed ($50d/100$)} & 20.27 & 29.52 & 54.63 & 43.79 & 36.64 & 33.80 & 67.52 & 50.28 & 45.17 \\
 && \textsc{Fixed ($80d/100$)} & 19.31 & 30.80 & 55.68 & 44.54 & 36.95 & 36.00 & 66.97 & 53.35 & 46.33 \\
\midrule
\multirow{8}{*}{\rotatebox{90}{80\%}} & \textsc{AdamW} & - & 115.94 & 23.21 & 31.02 & 27.38 & 2.31 & 26.40 & 51.20 & 50.36 & 30.27 \\
 & \textsc{Muon} & - & 50.31 & 22.95 & 33.08 & 31.47 & 11.00 & 26.40 & 56.15 & 52.17 & 33.32 \\
 \cmidrule{2-12} & \multirow{6}{*}{\rotatebox{0}{\textsc{NuMuon}}}
 & \textsc{Piecewise ($25d/100$)}  & 33.39 & 24.91 & 43.90 & 32.95 & 20.20 & 31.00 & 62.35 & 50.12 & 37.92 \\
 && \textsc{Cosine ($25d/100$)}    & 36.10 & 26.19 & 36.41 & 35.40 & 13.62 & 27.40 & 60.45 & 51.78 & 35.89 \\
 && \textsc{Fixed ($5d/100$)}      & 31.93 & 26.96 & 42.13 & 36.36 & 22.39 & 29.00 & 62.08 & 52.25 & 38.74 \\
 && \textsc{Fixed ($25d/100$)}     & 44.69 & 25.85 & 34.72 & 33.84 & 9.57  & 27.00 & 57.40 & 52.17 & 34.36 \\
 && \textsc{Fixed ($50d/100$)}     & 46.66 & 24.91 & 32.28 & 32.36 & 13.55 & 29.80 & 56.53 & 51.22 & 34.38 \\
 && \textsc{Fixed ($80d/100$)}     & 43.70 & 23.98 & 35.82 & 33.33 & 14.03 & 26.20 & 56.91 & 50.67 & 34.42 \\
\bottomrule
\end{tabular}
\end{table}
}

\newcommand{\tabAblationRankSchedCompact}{
\begin{table}
\setlength{\tabcolsep}{8pt}%
\renewcommand{\arraystretch}{0.95}%
\centering
\scriptsize
\captionof{table}{80\% Dobi-SVD LLM compression performance for Qwen3-0.6B models compared to non-compressed base models under various rank schedulers for NuMuon.}
\vspace{-0.5em}
\label{tab:ablation_compact}
\begin{tabular}{@{}lccccc@{}}
\toprule
\multirow{2}{*}{\textsc{\textbf{Scheduler}}} & \multirow{2}{*}{\textbf{\textsc{Time/Step (s)}}} & \multicolumn{2}{c}{\textbf{\textsc{Val. PPL}} \textdownarrow} & \multicolumn{2}{c}{\textbf{\textsc{Task Acc}} \textuparrow} \\
\cmidrule(lr){3-4} \cmidrule(lr){5-6}
& & \textbf{\textsc{0\%}} & \textbf{\textsc{80\%}} & \textbf{\textsc{0\%}} & \textbf{\textsc{80\%}} \\
\midrule
\midrule
\textsc{Piecewise ($0.25$)} & 9.94 & 20.01 & 33.39 & 45.87 & 37.92 \\
\textsc{Cosine ($0.25$)} & 10.49 & 20.20 & 36.10 & 45.69 & 35.89 \\
\textsc{Fixed ($0.05$)} & 8.96 & 29.13 & 31.93 & 39.64 & 38.74 \\
\textsc{Fixed ($0.25$)} & 9.81 & 21.85 & 44.69 & 44.30 & 34.36 \\
\textsc{Fixed ($0.50$)} & 11.33 & 20.27 & 46.66 & 45.17 & 34.38 \\
\textsc{Fixed ($0.80$)} & 11.49 & 19.31 & 43.70 & 46.33 & 34.42 \\
\midrule
\textsc{Muon} & 8.86 & 19.20 & 50.31 & 46.81 & 33.32 \\
\bottomrule
\end{tabular}
\end{table}
}